\documentclass[twoside,11pt]{article}

\usepackage[preprint]{jmlr2e}
\usepackage{amsmath}
\usepackage{amssymb}
\usepackage{bm}
\usepackage{mathtools}
\usepackage{booktabs}
\usepackage{multirow}
\usepackage[table]{xcolor}
\usepackage{graphicx}
\usepackage{hyperref}
\usepackage{tikz}
\usetikzlibrary{arrows.meta,positioning}
\hypersetup{colorlinks=true, citecolor=blue, linkcolor=blue, urlcolor=blue}

\newcommand{\R}{\mathbb{R}}
\newcommand{\E}{\mathbb{E}}
\newcommand{\dist}{\operatorname{dist}}
\newcommand{\Conv}{\operatorname{Conv}}
\newcommand{\Proj}{\operatorname{Proj}}
\newcommand{\supp}{\operatorname{supp}}
\newcommand{\diam}{\operatorname{diam}}
\newcommand{\Prob}{\mathcal{P}}
\newcommand{\dd}{\mathrm{d}}
\newcommand{\bx}{\boldsymbol{x}}
\newcommand{\bz}{\boldsymbol{z}}
\newcommand{\by}{\boldsymbol{y}}
\newcommand{\bq}{\boldsymbol{q}}
\newcommand{\ba}{\boldsymbol{a}}
\newcommand{\bxi}{\boldsymbol{\xi}}
\newcommand{\bzero}{\boldsymbol{0}}
\newcommand{\rX}{\mathsf{X}}
\newcommand{\rY}{\mathsf{Y}}
\newcommand{\vv}{\boldsymbol{v}}
\newcommand{\bu}{\boldsymbol{u}}
\newcommand{\vu}{\boldsymbol{v}_{\emptyset}}
\newcommand{\vc}{\boldsymbol{v}_{c}}
\newcommand{\bgc}{\boldsymbol{g}_{c}}
\newcommand{\vtheta}{\boldsymbol{v}_{\theta}}
\newcommand{\vcfg}{\boldsymbol{v}_{\mathrm{cfg}}}
\newcommand{\vcfgtheta}{\boldsymbol{v}^{\mathrm{cfg}}_\theta}
\newcommand{\Id}{\mathbf{I}}
\newcommand{\Cov}{\mathbf{Cov}}
\newcommand{\bigO}{\mathcal{O}}
\newcommand{\littleo}{\mathop{\mathrm{o}}}

\newtheorem{assumption}{Assumption}[section]

\jmlrheading{0}{2026}{1-53}{4/26}{4/26}{0000}{Jianfeng Cai, Zhengyi Su, Chao Wang}
\ShortHeadings{Particle Dynamics of Flow Matching and CFG}{Cai, Su, and Wang}
\firstpageno{1}

\begin{document}

\title{
Particle Dynamics of Flow Matching and Classifier-Free Guidance from a Stagewise Geometry Perspective}

\author{\name Jian-Feng Cai \email jfcai@ust.hk \\
       \addr Department of Mathematics\\
       The Hong Kong University of Science and Technology\\
       Hong Kong SAR, China\\
       \name Zhengyi Su \email 12311113@mail.sustech.edu.cn \\
       \addr Department of Statistics and Data Science\\
       Southern University of Science and Technology\\
       Shenzhen, Guangdong, China\\
       \name Chao Wang\thanks{Correspondence to Chao Wang} \email wangc6@sustech.edu.cn \\
       \addr Department of Statistics and Data Science\\
    Southern University of Science and Technology\\
     Shenzhen, Guangdong, China}

\editor{Under review}

\maketitle

\begin{abstract}

Flow matching, together with classifier-free guidance (CFG), is widely used in generative modeling, yet much of the theoretical understanding remains distribution-wise. Since practical sampling follows individual trajectories, distribution-level guarantees alone do not fully capture how trajectories interact with the data geometry or how guidance reshapes it. To overcome this limitation, we establish a unified stagewise geometric theory of attraction and absorption for both continuous dynamics and explicit Euler discretization.  Specifically, 
with $t\in[0,1]$ running from noise to data, we show that unconditional flow trajectories are successively attracted toward a neighborhood of the global mean, the data convex hull, and a neighborhood of a possibly nonconvex local cluster. Across these stages, the corresponding distance satisfies a common contraction estimate, yielding an $\mathcal{O}(1-t)$ decay of the distance in the final stage. For CFG, the same structure persists with an extrapolated mean, an inflated conditional convex hull, and, near the target cluster, the restored local geometry of conditional flow matching. We further show that a general time schedule $a(t)$ replaces the $\mathcal{O}(1-t)$ decay by $\mathcal{O}(1-a(t))$. Together, these results
provide a unified particle-level geometric account of flow matching and CFG across continuous and discrete sampling.

\end{abstract}


\section{Introduction}
Flow matching has become a central continuous-time framework for modern generative modeling
\citep{lipman2023flow,liu2023flow,stochastic_interpolants}. Its practical impact is now visible in large-scale image and video generation systems
\citep{peebles2023scalable,DBLP:conf/icml/EsserKBEMSLLSBP24,
DBLP:journals/corr/abs-2506-15742,jin2024pyramidal}.
Meanwhile, classifier-free guidance (CFG) has become a standard mechanism for improving conditional generation quality and
alignment \citep{ho2021classifierfree,peebles2023scalable,
DBLP:conf/icml/EsserKBEMSLLSBP24,DBLP:journals/corr/abs-2506-15742}.
Alongside these developments, the theoretical understanding of flow matching and CFG has also advanced rapidly and can be divided into two parts: distributional-level and particle-level.

Most existing theories of flow matching characterize distributional convergence under metrics such as KL divergence, total variation, and Wasserstein distance \citep{guan2026totalvariation,DBLP:journals/tmlr/BentonDD24,
DBLP:conf/iclr/FukumizuSIOK25,gentiloni2024theoretical,zhou2025erroranalysis,
kumar2026manifoldflow,su2025discreteflow}. Related convergence theories have also been developed for stochastic and deterministic diffusion samplers \citep{debortoli2022manifold,lee2022polynomial,
lee2023general,oko2023minimax,conforti2025kl,silveri2025beyond,
DBLP:journals/corr/abs-2408-02320,JMLR:v26:25-0272}. Meanwhile, recent work has begun to study particle-level trajectory geometry, attraction, robustness, and terminal behavior
\citep{cao2026robustness,wan2025elucidating, chen2025geometric}.
This distinction is essential because practical generation follows a single numerical trajectory from noise to samples, and many recent methods modify the trajectory directly
\citep{bai2025zigzag,wang2025towards,saini2025rectifiedcfgpp,
cai2026improvingclassifierfreeguidanceflow}. 

Recent studies of classifier-free guidance further indicate that the effect of guidance varies across the sampling trajectory. Practical methods restrict guidance to selected noise levels or vary its strength over time \citep{DBLP:conf/nips/KynkaanniemiAKL24,DBLP:conf/ecai/MalarzKZTS25,sadat2024cads,saharia2022photorealistic,
chung2025cfgpp,papalampidi2025dynamic}. Theoretical studies explain CFG from complementary perspectives, including overshoot and shrinkage, statistical properties of classifier-free conditioning, mean-shift and contrastive-principal-component effects, and averaged distributional behavior \citep{pavasovic2026overshoot,fu2024unveil,
li2025understandingmechanismsclassifierfreeguidance,jiao2025unified,DBLP:conf/aaai/ZhaoS26,
chidambaram2024guidance,chandramoorthy2025inexact,cai2026improvingclassifierfreeguidanceflow}. Yet it remains unclear how CFG reshapes the geometry of individual flow-matching trajectories across different stages, and how the guided dynamics relate to ordinary conditional flow matching near the target data.

Building on the emerging particle-level perspective of flow matching, this paper establishes a unified metric-geometric theory for the stagewise dynamics of unconditional flow matching and CFG. A common attraction--absorption mechanism governs these dynamics, and persists exactly under explicit Euler discretization. This framework reveals how data geometry progressively constrains individual trajectories and how CFG reshapes the relevant geometry throughout sampling.

The contributions of this paper are summarized as follows:

\begin{itemize}
    \item \emph{Stagewise geometry of original unconditional flow matching.}
    We characterize its progression through three time-scaled data geometries: a ball centered at the global mean,
     the data convex hull, and, under suitable local conditions, a neighborhood of a possibly nonconvex local cluster. 
     We establish quantitative attraction toward these geometries and absorption in the latter two regimes.

    \item \emph{Stagewise geometry of classifier-free guidance.}
    We characterize how guidance reshapes the relevant geometry across stages: from an extrapolated mean in the early stage,
     to an inflated conditional convex hull in the intermediate stage, and finally to the restored local geometry of conditional 
     flow matching.

    \item \emph{Euler discrete counterparts.}
    For both unconditional flow matching and CFG, we establish exact grid-level counterparts of the continuous attraction estimates
     together with the corresponding absorption mechanisms, carrying the same stagewise geometric picture directly to explicit Euler sampling.

    \item \emph{General time schedules.}
    We show that a general time schedule reparameterizes the ideal continuous trajectory without changing its spatial path
     and changes the final-stage decay rate from  $\bigO(1-t)$ to $\bigO(1-a(t))$.
\end{itemize}

Several recent works are particularly close to our particle-level perspective. For unconditional flow matching, \citet{wan2025elucidating} provide the closest comparison. 
They analyzed flow-matching trajectories through a scaled ODE, showing stagewise attraction from the data mean to the convex hull of a local cluster. In addition, they proved the convergence to data atoms for empirical distributions, which links to memorization. 
Under the linear schedule and for $t>0$, their scaled convex-hull contraction is equivalent to our continuous statement (Theorem~\ref{thm:convex-hull}). On the other hand, some closely related results are obtained by \citet{liu2025pdeperspectivegenerativediffusion} for the deterministic heat-score ODE. After some transformation, their convex-hull estimate is likewise equivalent to ours, and they also obtained atom-level convergence in the empirical setting. 
Our results differ from these works in three substantive ways. First, we analyze the original, unscaled flow-matching trajectory, which directly matches practical sampling and avoids the rescaling-induced singularity at \(t=0\) in \citet{wan2025elucidating}. 
Second, whereas \citet{wan2025elucidating} convexify the target cluster, our final-stage theory preserves its possibly nonconvex geometry and establishes attraction and absorption directly at the cluster level. Third, we develop exact counterparts for explicit Euler sampling, extending the attraction--absorption picture from continuous-time dynamics \citep{wan2025elucidating,liu2025pdeperspectivegenerativediffusion} to the finite-step trajectories produced in practice. 
Together, these advances bring the geometric theory closer to the actual sampling process.


For classifier-free guidance, \citet{cao2026robustness} study terminal support robustness for guided DDIM and DDPM dynamics, while \citet{liu2026geometricasymptoticsscoremixing} analyze the small-noise asymptotics of mixed heat-flow scores and shows that the limiting dynamics are governed by a geometric potential determined by squared distances to the component supports. 
Our results differ from these works in two main ways.
First, existing works have primarily focused on the final stage, with \citet{cao2026robustness} studying terminal support robustness and \citet{liu2026geometricasymptoticsscoremixing} analyzing small-noise asymptotics. In contrast, our analysis provides a unified geometric characterization of CFG across the entire sampling trajectory.
Second, whereas \citet{cao2026robustness} assume convex target supports, our CFG final-stage theory continues to accommodate possibly nonconvex local clusters.
Overall, these results provide a more faithful geometric description of practical CFG dynamics.

Table~\ref{tab:stagewise-summary} summarizes the main theoretical results.
All ODEs start from Gaussian noise at $t=0$ and run toward $t=1$; thus $1-t$ denotes the remaining time.
For discrete results, $t_i$ denotes the Euler grid time, with $t_i=i\Delta t$ on complete grids, and $\Delta t$ is the stepsize.
In the last column, $0\le s\le t<1$ are two times in the relevant continuous interval, while $j\le i$ are two Euler-grid indices with $t_j<1$, so $0\le t_j\le t_i\le1$. Hence, the displayed ratios record the exact contraction estimates; for the general schedule, $a(s)\le a(t)$ gives the analogous conclusion.
 An appended $\Rightarrow\bigO(\cdot)$ records the corresponding terminal rate. Thus the early/intermediate entries give contraction bounds, while the final-stage entries also report terminal rates.
Here $\mathcal{D}$ and $\mathcal{D}_c$ denote the unconditional and conditional data supports, respectively; convex-hull and local-cluster results also include absorption after first entrance.

\begin{table}[h]
\centering
\footnotesize
\caption{Stagewise summary of the main theoretical results.}
\label{tab:stagewise-summary}
\setlength{\tabcolsep}{3.5pt}
\renewcommand{\arraystretch}{1.15}
\resizebox{\textwidth}{!}{%
\begin{tabular}{@{} l l l l l @{}}
\toprule
\textbf{Stage} & \textbf{Ref} & \textbf{Dynamics} & \textbf{Geometry and result} & \textbf{Contraction / terminal rate} \\ \midrule

\multirow{4}{1.55cm}{\centering Early}
& \cellcolor[gray]{0.96}Theorem~\ref{thm:early-mean}
& \cellcolor[gray]{0.96}Uncond. flow (cont.)
& \cellcolor[gray]{0.96}mean-ball attraction
& \cellcolor[gray]{0.96}$(1-t)/(1-s)$ \\
& \cellcolor[gray]{0.96}Theorem~\ref{cor:early-mean-euler}
& \cellcolor[gray]{0.96}Uncond. flow (disc.)
& \cellcolor[gray]{0.96}mean-ball attraction
& \cellcolor[gray]{0.96}$(1-t_i)/(1-t_j)$ \\
& \cellcolor[gray]{0.88}Theorem~\ref{thm:cfg-early}
& \cellcolor[gray]{0.88}CFG (cont.)
& \cellcolor[gray]{0.88}extrapolated mean-ball attraction
& \cellcolor[gray]{0.88}$(1-t)/(1-s)$ \\
& \cellcolor[gray]{0.88}Theorem~\ref{thm:cfg-early-euler}
& \cellcolor[gray]{0.88}CFG (disc.)
& \cellcolor[gray]{0.88}extrapolated mean-ball attraction
& \cellcolor[gray]{0.88}$(1-t_i)/(1-t_j)$ \\ \cmidrule(l){1-5}

\multirow{6}{1.55cm}{\centering Early and\\ inter.}
& \cellcolor[gray]{0.96}Theorem~\ref{thm:convex-hull}
& \cellcolor[gray]{0.96}Uncond. flow (cont.)
& \cellcolor[gray]{0.96}$\Conv(\mathcal{D})$ attraction
& \cellcolor[gray]{0.96}$(1-t)/(1-s)$ \\
& \cellcolor[gray]{0.96}Theorem~\ref{cor:convex-hull-euler}
& \cellcolor[gray]{0.96}Uncond. flow (disc.)
& \cellcolor[gray]{0.96}$\Conv(\mathcal{D})$ attraction
& \cellcolor[gray]{0.96}$(1-t_i)/(1-t_j)$ \\
& \cellcolor[gray]{0.88}Theorem~\ref{thm:cfg-convex}
& \cellcolor[gray]{0.88}CFG (cont.)
& \cellcolor[gray]{0.88}tail-dependent inflated $\Conv(\mathcal{D}_c)$
& \cellcolor[gray]{0.88}$(1-t)/(1-s)$ \\
& \cellcolor[gray]{0.88}Corollary~\ref{cor:cfg-convex-small}
& \cellcolor[gray]{0.88}CFG (cont.)
& \cellcolor[gray]{0.88}exponentially refined inflated $\Conv(\mathcal{D}_c)$ 
& \cellcolor[gray]{0.88}$(1-t)/(1-s)$ \\
& \cellcolor[gray]{0.88}Theorem~\ref{thm:cfg-convex-euler}
& \cellcolor[gray]{0.88}CFG (disc.)
& \cellcolor[gray]{0.88}tail-dependent inflated $\Conv(\mathcal{D}_c)$
& \cellcolor[gray]{0.88}$(1-t_i)/(1-t_j)$ \\
& \cellcolor[gray]{0.88}Corollary~\ref{cor:cfg-convex-euler-small}
& \cellcolor[gray]{0.88}CFG (disc.)
& \cellcolor[gray]{0.88}exponentially refined inflated $\Conv(\mathcal{D}_c)$
& \cellcolor[gray]{0.88}$(1-t_i)/(1-t_j)$ \\ \cmidrule(l){1-5}

\multirow{6}{1.55cm}{\centering Final}
& \cellcolor[gray]{0.96}Theorem~\ref{thm:final-local}
& \cellcolor[gray]{0.96}Uncond. flow (cont.)
& \cellcolor[gray]{0.96}local-cluster neighborhood attraction
& \cellcolor[gray]{0.96}$(1-t)/(1-s)\Rightarrow\bigO(1-t)$ \\
& \cellcolor[gray]{0.96}Theorem~\ref{cor:final-local-euler}
& \cellcolor[gray]{0.96}Uncond. flow (disc.)
& \cellcolor[gray]{0.96}local-cluster neighborhood attraction
& \cellcolor[gray]{0.96}$(1-t_i)/(1-t_j)\Rightarrow\bigO(1-t_i)$ \\
& \cellcolor[gray]{0.88}Theorem~\ref{thm:cfg-final}
& \cellcolor[gray]{0.88}CFG (cont.)
& \cellcolor[gray]{0.88}local-cluster neighborhood attraction
& \cellcolor[gray]{0.88}$(1-t)/(1-s)\Rightarrow\bigO(1-t)$ \\
& \cellcolor[gray]{0.88}Theorem~\ref{thm:cfg-final-euler}
& \cellcolor[gray]{0.88}CFG (disc.)
& \cellcolor[gray]{0.88}local-cluster neighborhood attraction
& \cellcolor[gray]{0.88}$(1-t_i)/(1-t_j)\Rightarrow\bigO(1-t_i)$ \\ \cmidrule(l){2-5}
& \cellcolor[gray]{0.88}Theorem~\ref{thm:prediction-gap}
& \cellcolor[gray]{0.88}CFG (cont.)
& \cellcolor[gray]{0.88}prediction-gap decay near cluster
& \cellcolor[gray]{0.88}$\bigO\!\left(\exp{(-C/(1-t)^2)}/(1-t)\right)$ \\
& \cellcolor[gray]{0.94}Proposition~\ref{prop:general-schedule}
& \cellcolor[gray]{0.94}Flow / CFG (cont.)
& \cellcolor[gray]{0.94}general time schedule local-cluster attraction
& \cellcolor[gray]{0.94}$(1-a(t))/(1-a(s))\Rightarrow\bigO(1-a(t))$ \\ \bottomrule

\end{tabular}
}
\end{table}

\textbf{Paper Organization.} Section~\ref{sec:background} introduces flow matching, CFG, and the notation used throughout. Section~\ref{sec:cfm} develops the posterior-mean and smoothed-distance representations of the ideal vector fields. Sections~\ref{sec:uncond} and~\ref{sec:cfg} establish the stagewise theory for unconditional flow matching and CFG, respectively. Section~\ref{sec:gen-schedule} studies flow matching with a general time schedule, and Section~\ref{sec:experiments} presents the experiments. Section~\ref{sec:conclusion} concludes the paper.
Detailed proofs are deferred to the appendix. 
\section{Problem Formulation and Notation}\label{sec:background}

\subsection{Flow Matching}

Flow matching (FM) \citep{lipman2023flow, liu2023flow} uses a neural network to learn
 a time-dependent vector field
$\vtheta : [0,1]\times \R^d \times (\mathcal{Y}\cup\{\emptyset\}) \to \R^d,$
where $\theta$ denotes the trainable network parameters and $\mathcal{Y}$ denotes the space of conditioning signals. Let
$
\mathcal{N}(\bzero,\Id) 
$
be the standard Gaussian distribution, where $\mathbf{I}$ is the identity matrix. 
We denote the unconditional clean-data distribution by $p_{\emptyset}$, using
the same symbol for the corresponding probability measure and its density when
no ambiguity arises. 
For unconditional flow matching (UFM), the learned vector field $\vtheta(\cdot,\cdot,\emptyset)$ 
generates samples by the \emph{parametrized unconditional flow ODE}
\begin{equation}
\frac{\dd \bx_t}{\dd t} = \vtheta(t,\bx_t,\emptyset),
\qquad
\bx_0 \sim  \mathcal{N}(\bzero,\Id) ,
\qquad
t\in[0,1].
\label{eq:ufm-ode}
\end{equation}
In training, we sample
$
\rX_{t,\emptyset} = (1-t)\bxi + t\rY_\emptyset,$
where 
$\bxi \sim  \mathcal{N}(\bzero,\Id) ,
\rY_\emptyset \sim p_\emptyset, $
 and $t\sim \mathrm{Unif}[0,1]$, with $\mathrm{Unif}[0,1]$ being the uniform distribution over $[0,1]$. We aim to minimize
\begin{equation}
\mathcal{L}_{\emptyset}(\theta)
=
\E_{t,\bxi,\rY_\emptyset}
\left[
\left\|
\vtheta(t,\rX_{t,\emptyset},\emptyset)-(\rY_\emptyset-\bxi)
\right\|^2
\right],
\label{eq:ufm-loss}
\end{equation}
where $\mathbb{E}$ is the standard expectation.
Similarly, for a fixed condition $c\in\mathcal{Y}$, let $p_c$ denote the conditional clean-data distribution. 
Then, conditional flow matching (CFM) learns the vector field $\vtheta(\cdot,\cdot,c)$ through the \emph{parametrized conditional flow ODE}
\begin{equation}
\frac{\dd \bx_t}{\dd t} = \vtheta(t,\bx_t,c),
\qquad
\bx_0 \sim  \mathcal{N}(\bzero,\Id) ,
\qquad
t\in[0,1],
\label{eq:cfm-ode}
\end{equation}
and trains it using the interpolation
$
\rX_{t,c} = (1-t)\bxi + t\rY_c,$ 
where 
$\bxi \sim  \mathcal{N}(\bzero,\Id), $
and $
\rY_c \sim p_c,$
with loss
\begin{equation}
\mathcal{L}_{c}(\theta)
=
\E_{t,\bxi,\rY_c}
\left[
\left\|
\vtheta(t,\rX_{t,c},c)-(\rY_c-\bxi)
\right\|^2
\right].
\label{eq:cfm-loss}
\end{equation}
In the rest of the paper, we develop the theory of the unconditional case first; the conditional results will follow as corollaries obtained by restricting the data distribution from $p_\emptyset$ to $p_c$.

\subsection{Classifier-Free Guidance}

CFG \citep{ho2021classifierfree, DBLP:conf/icml/EsserKBEMSLLSBP24, DBLP:journals/corr/abs-2506-15742}
uses a single network $\vtheta(t,\bx,c)$ to learn both conditional and unconditional vector fields. In the standard CFG training scheme, with probability $q\in(0,1)$, we draw a conditional clean sample $\rY_c\sim p_c$ and train the conditional field; with probability $1-q$, we draw an unconditional clean sample $\rY_\emptyset\sim p_\emptyset$ and train the unconditional field.
In both cases, we sample $\bxi \sim \mathcal{N}(\bzero,\Id)$ and form the corresponding interpolant as in $\rX_{t,\emptyset}$ or $\rX_{t,c}. $
The joint training objective is
\begin{align}
\mathcal{L}_{\mathrm{cfg}}(\theta)
:={}&
q\,\E_{t,\rY_c\sim p_c,\bxi}
\left[
\left\|
\vtheta(t,\rX_{t,c},c) - (\rY_c-\bxi)
\right\|^2
\right]
\nonumber\\
&\quad
+ (1-q)\,\E_{t,\rY_\emptyset\sim p_\emptyset,\bxi}
\left[
\left\|
\vtheta(t,\rX_{t,\emptyset},\emptyset) - (\rY_\emptyset-\bxi)
\right\|^2
\right],
\label{eq:cfg-loss}
\end{align}
where $t\sim \mathrm{Unif}[0,1]$.
The fields $\vtheta(\cdot,\cdot,\emptyset)$ and $\vtheta(\cdot,\cdot,c)$ correspond to the unconditional and conditional vector fields, respectively. 
At sampling time, CFG extrapolates between the learned conditional and unconditional vector fields. Specifically, for a \emph{guidance scale} $w>1$, we define
\begin{align}\label{eq:cfg field}
\vcfgtheta(t,\bx,c)
&:= \vtheta(t,\bx,\emptyset) + w\bigl(\vtheta(t,\bx,c)-\vtheta(t,\bx,\emptyset)\bigr),
\qquad
w>1.
\end{align}
The corresponding CFG sampling dynamics are given by the \emph{parametrized CFG ODE}
\begin{equation}
\frac{\dd \bx_t}{\dd t} = \vcfgtheta(t,\bx_t,c),
\qquad
\bx_0 \sim \mathcal{N}(\bzero,\Id),
\qquad
t\in[0,1].
\label{eq:cfg-ode}
\end{equation}

\subsection{Basic notations}

Here, we collect the geometric and probabilistic notations used throughout the paper. We work in Euclidean space $\R^d$ equipped with its Borel $\sigma$-algebra $\mathcal{B}(\R^d)$.
For an interval $I \subset [0,1]$ and $k\in\{1,2\}$, $C^k(I;\R^d)$ denotes the space of maps from $I$ to $\R^d$ with continuous derivatives up to order $k$, and $AC(I;\R^d)$ denotes the space of absolutely continuous maps from $I$ to $\R^d$. For brevity, we write $C^k(I)$ and $AC(I)$ for these vector-valued spaces. In particular, $C^k[0,t] := C^k([0,t];\R^d)$ and $AC[0,t] := AC([0,t];\R^d)$.

For a nonempty closed set $A \subset \R^d$, write $\dist(\bx,A) := \inf_{\by \in A}\|\bx-\by\|$.
When $A$ is bounded, define $\diam(A) := \sup_{\by,\bz \in A}\|\by-\bz\|$.
Whenever the minimizer is unique, $\Proj_A(\bx)$ denotes the projection of $\bx$ onto $A$.
Write $\Conv(A)$ for the convex hull of $A$ and
$B_r(A) := \{\bx \in \R^d : \dist(\bx,A)\le r\}$ for the closed
$r$-neighborhood. For sets $A,B\subset\R^d$ and a scalar
$\lambda\in\R$, define their Minkowski sum and scalar multiple by
\[
A+B:=\{\bx+\by:\bx\in A,\ \by\in B\},
\qquad
\lambda A:=\{\lambda\bx:\bx\in A\}.
\]
We use the standard notion of $\rho$-prox-regularity \citep{Federer1959CurvatureMeasures,clarke1995proximal,poliquin1996prox}. For any nonempty closed set $A\subset \mathbb R^d$, we say that $A$ is
$\rho$-prox-regular, if
the metric projection $\operatorname{Proj}_A(\bx)$
is single-valued for every $\bx$ satisfying
$0<\operatorname{dist}(\bx,A)<\rho$.

For a vector $\bx \in \R^d$, we write $\|\bx\|$ for its Euclidean norm. For a matrix $\mathbf{M} \in \mathbb{R}^{d \times d}$, we denote its operator norm as $\|\mathbf{M}\|_{\text{op}}$, which is the largest singular value of $\mathbf{M}$. For a probability measure $\mu$ on $\R^d$, we write $\supp(\mu)$ for its support.
The support of the unconditional clean-data distribution is denoted by $\mathcal{D}:= \supp(p_\emptyset)$.
Throughout this paper, we assume that the clean-data support $\mathcal{D}$ is
bounded. Since supports are closed, this makes $\mathcal D$ compact.
For every condition $c$ considered below, let
$\mathcal D_c\subsetneq\mathcal D$ be a nonempty closed measurable set
satisfying
\(
0<p_\emptyset(\mathcal D_c)<1.
\)
Define the restricted conditional distribution by
\[
p_c(A)
:=
\frac{p_\emptyset(A\cap\mathcal D_c)}
{p_\emptyset(\mathcal D_c)},
\qquad
A\in\mathcal B(\R^d),
\]
and impose the standing convention
\(
\mathcal D_c=\supp(p_c).
\)
Thus both $\mathcal D_c$ and $\mathcal D\setminus\mathcal D_c$ have positive
$p_\emptyset$-mass. Since $\mathcal D$ is compact and $\mathcal D_c$ is
closed, $\mathcal D_c$ is compact. Every conditional local cluster
$\Omega_c\subset\mathcal D_c$ considered later is bounded as well.
  We denote the covariance matrix of a probability measure $\mu$ by $\Cov_\mu(\by) := \int_{\R^d} (\by-\mathbb{E}_\mu(\by))(\by-\mathbb{E}_\mu(\by))^\top \mu(\dd \by)$.
We also write $\bar{\by} := \int \by\,p_\emptyset(\dd \by)$ and $\bar{\by}_c := \int \by\,p_c(\dd \by)$ for the clean unconditional and conditional means.

\section{The Ideal Vector Fields and Corresponding ODEs}\label{sec:cfm}

\subsection{The Ideal Flow ODEs}
Our previous work \cite{cai2026improvingclassifierfreeguidanceflow} interprets the FM training process as an optimization problem over an infinite-dimensional function space.
 Specifically, let $\mathcal{V}:=C((0,1)\times \R^d,\R^d)$, and then the unconditional and conditional training objectives are defined as the functionals on $\mathcal{V}$ given by
\begin{align}
\mathcal{L}_{\emptyset}(\bu)
:={}&
\E_{t,\bxi,\rY_\emptyset}
\left[
\left\|
\bu(t,\rX_{t,\emptyset})-(\rY_\emptyset-\bxi)
\right\|^2
\right],
\qquad \bu\in\mathcal{V},
\label{eq:func-ufm-loss}\\
\mathcal{L}_{c}(\bu)
:={}&
\E_{t,\bxi,\rY_c}
\left[
\left\|
\bu(t,\rX_{t,c})-(\rY_c-\bxi)
\right\|^2
\right],
\qquad \bu\in\mathcal{V}.
\label{eq:func-cfm-loss}
\end{align}
Accordingly, the CFG training objective is the functional on $\mathcal{V}^2$ defined by
\begin{equation}
\mathcal{L}_{\mathrm{cfg}}(\bu_1,\bu_2)
:=
q\,\mathcal{L}_{c}(\bu_1)+(1-q)\,\mathcal{L}_{\emptyset}(\bu_2),
\qquad
(\bu_1,\bu_2)\in\mathcal{V}^2.
\label{eq:func-cfg-loss}
\end{equation}
Therefore, the parametrized objectives in (\ref{eq:ufm-loss}), (\ref{eq:cfm-loss}), and (\ref{eq:cfg-loss}) are obtained by restricting this problem to the finite-dimensional parametrization.

For each fixed $t \in [0,1)$ and $\bx \in \R^d$,  we define the \emph{unconditional and conditional posterior measures}\citep{cai2026improvingclassifierfreeguidanceflow,wan2025elucidating} by
\begin{equation*}
\begin{aligned}
\eta_t^{\bx}(A)
&:=
\frac{
\int_A \exp\left(-\tfrac{\|\bx-t\by\|^2}{2(1-t)^2}\right)p_\emptyset(\dd \by)
}
{
\int_{\R^d} \exp\left(-\tfrac{\|\bx-t\by\|^2}{2(1-t)^2}\right){p_\emptyset(\dd \by)}
},
&
\eta_{t,c}^{\bx}(A)
&:=
\frac{
\int_A \exp\left(-\tfrac{\|\bx-t\by\|^2}{2(1-t)^2}\right)p_c(\dd \by)
}
{
\int_{\R^d} \exp\left(-\tfrac{\|\bx-t\by\|^2}{2(1-t)^2}\right)p_c(\dd \by)
},
\end{aligned}
\end{equation*}
for every $A \in \mathcal{B}(\R^d)$.
We define the \emph{posterior means} by $\hat{\by}_t(\bx) := \int \by\,\eta_t^{\bx}(\dd \by)$ and $\hat{\by}_{t,c}(\bx) := \int \by\,\eta_{t,c}^{\bx}(\dd \by)$.
The following result from \citet{cai2026improvingclassifierfreeguidanceflow} gives the corresponding  minimizers in posterior-mean form.

\begin{proposition}[\citep{cai2026improvingclassifierfreeguidanceflow} Ideal vector fields as minimizers]\label{thm:posterior-mean}
For $t \in [0,1)$, define
\begin{equation}\label{eq:uncond ideal field}
\vu(t,\bx)
=
\frac{\int_{\R^d}\by\,\eta_t^{\bx}(\dd\by)-\bx}{1-t}
=
\frac{\hat{\by}_t(\bx)-\bx}{1-t},
\end{equation}
\begin{equation}\label{eq:cond ideal field}
\vc(t,\bx)
=
\frac{\int_{\R^d}\by\,\eta_{t,c}^{\bx}(\dd\by)-\bx}{1-t}
=
\frac{\hat{\by}_{t,c}(\bx)-\bx}{1-t}.
\end{equation}
Assume that $\mathcal{L}_{\emptyset}(\vu),\ \mathcal{L}_{c}(\vc)<\infty.$
Then $\vu$ and $\vc$ are the global minimizers of (\ref{eq:func-ufm-loss}) and (\ref{eq:func-cfm-loss}) over $\mathcal{V}$, respectively.
Moreover, the pair $(\vu,\vc)$ is the global minimizer of (\ref{eq:func-cfg-loss}) over $\mathcal{V}^2$.
\end{proposition}
This proposition identifies the targets of UFM, CFM, and CFG training. 
In the remainder of the paper, we focus on the ideal ODEs generated by these minimizers. 
In all three cases, sampling starts at time $0$ from Gaussian noise: we draw $\bx_0\sim\mathcal N(\bzero,\Id)$, so $\|\bx_0\|<\infty$ holds almost surely. The ODEs below are analyzed on $[0,T]$ with $0<T<1$. Specifically, we study
\begin{enumerate}
	    \item \emph{Unconditional flow ODE:}
	    \begin{equation}\label{eq:uncond flow}
\frac{\dd \bx_t}{\dd t} = \vu(t,\bx_t).
	    \end{equation}
	    \item \emph{Conditional flow ODE:}
	    \begin{equation}\label{eq:cond flow}
\frac{\dd \bx_t}{\dd t} = \vc(t,\bx_t).
	    \end{equation}
	    \item \emph{CFG ODE} with guidance scale $w>1$:
	    \begin{equation}\label{eq:cfg flow}
\frac{\dd \bx_t}{\dd t} = \vcfg(t,\bx_t) = (1-w)\vu(t,\bx_t)+w\vc(t,\bx_t).
	    \end{equation}
	\end{enumerate}
The parametrized ODEs (\ref{eq:ufm-ode}), (\ref{eq:cfm-ode}), and
(\ref{eq:cfg-ode}) are obtained by replacing these ideal
vector fields by neural-network approximations. We therefore study the ideal dynamics directly. 

For the corresponding explicit Euler schemes on a grid $t_i=i\Delta t$, we use
\begin{equation}\label{eq:ideal-euler}
\bz_{i+1}
=
\bz_i+\Delta t\,
\vv_\star(t_i,\bz_i),
\qquad
\vv_\star\in\{\vu,\vc,\vcfg\}.
\end{equation}
Next, we establish regularity of the ideal vector fields and existence and
uniqueness of the corresponding ODE solutions on every compact interval
$[0,T]$ with $0<T<1$.

\begin{theorem}[Existence and uniqueness of the ODEs]\label{thm:wellposed-flow}
Fix any $T \in (0,1)$. Then the following hold.
\begin{enumerate}
\item[(i)] The ideal vector fields $\vu$ and $\vc$ in
(\ref{eq:uncond ideal field}) and (\ref{eq:cond ideal field}) are
continuously differentiable in $(t,\bx)$ on $[0,T]\times\R^d$, and their
spatial Jacobians satisfy
\begin{equation*}
\|\mathsf{D}_{\bx}\vu(t,\bx)\|_{\rm op},
\ \|\mathsf{D}_{\bx}\vc(t,\bx)\|_{\rm op}
\le
\tfrac{1}{1-T}
+ \tfrac{T}{(1-T)^3}\diam(\mathcal{D})^2,
\qquad
(t,\bx)\in[0,T]\times\R^d.
\end{equation*}

\item[(ii)] For every initial condition  $\bx_0\in\R^d$, the unconditional and conditional
flow ODEs (\ref{eq:uncond flow}) and (\ref{eq:cond flow}) each admit a
unique classical solution $\bx_\cdot\in C^1[0,T]\subset AC[0,T]$.

\item[(iii)] For every fixed guidance scale $w>1$ and every initial condition
$\bx_0\in\R^d$, the CFG ODE (\ref{eq:cfg flow}) admits a unique classical
solution $\bx_\cdot\in C^1[0,T]\subset AC[0,T]$.
\end{enumerate}
\end{theorem}
For the proof, see Appendix~\ref{proof:thm:wellposed-flow}.
Also note that the regularity of the ODEs at the terminal time $t=1$ requires additional
assumptions and a separate analysis; for further discussion, see
\citet{wan2025elucidating}.

 Next, we review the smoothed-distance perspective behind the flow ODEs,
which gives a geometric intuition for the trajectory analysis below.

\subsection{The Smoothed-Distance Objective behind the Flow ODE}\label{sec:homotopy}

We recall the smoothed-distance viewpoint of flow matching from
\citet{cai2026improvingclassifierfreeguidanceflow}, which interpreted the
ideal vector fields in (\ref{eq:uncond ideal field}, \ref{eq:cond ideal field}) through a
time-varying objective built from a log-sum-exp smoothing of squared
distance. 
Specifically in the UFM setting, for $t\in[0,1)$, define the smoothed squared distance to
 $t\mathcal D$ by
\begin{equation}\label{soft scaled distance}
\operatorname{dist}_{t}^2(\bx,t\mathcal D)
:=
-2(1-t)^2
\log\left(
\int_{\mathcal{D}}
\exp\left(-\tfrac{\|\bx-t\by\|^2}{2(1-t)^2}\right)
p_\emptyset(\dd \by)
\right)
\end{equation}
and define the corresponding smoothed objective 
$f_t(\bx):= \frac12\operatorname{dist}_{t}^2(\bx,t\mathcal D) - \frac{1-t}{2}\|\bx\|^2.$

The next proposition summarizes the main roles of the above construction in UFM.
First, $f_t$ determines the ideal flow direction through a gradient
identity. Second, the smoothed-distance term captures the moving geometry
$t\mathcal D$ and converges to the terminal hard squared distance as
$t\uparrow1$. Consequently, the ideal flow ODE can be viewed, in the broad
continuation sense, as a homotopy-optimization dynamics \citep{watson1989modern,mobahi2015link} associated with the
terminal problem
\begin{equation*}
    \min_{\bx \in \mathbb{R}^d}~\tfrac{1}{2}\dist^2(\bx,\mathcal D).
\end{equation*}

\begin{proposition}[Smoothed-distance geometry]\label{prop:soft-distance-representation}
For every $t\in[0,1)$ and $\bx\in\R^d$, 
$\operatorname{dist}_{t}^2(\bx,t\mathcal D)$ and $f_t$ defined above have the following properties.
\begin{enumerate}
\item[(i)] For every \(t\in(0,1)\), the unconditional ideal vector field satisfies
\begin{equation*}
\vu(t,\bx)
=
\frac{\hat{\by}_t(\bx)-\bx}{1-t}
=
-\frac{\nabla_{\bx} f_t(\bx)}{t(1-t)}, 
\qquad \bx\in\R^d.
\end{equation*}
\item[(ii)] The smoothed squared distance is controlled by the hard squared distance to the moving data set:
\begin{equation*}
\dist^2(\bx,t\mathcal{D})
\le
\operatorname{dist}_{t}^2(\bx,t\mathcal D)
\le
\max_{\by\in \mathcal{D}}\|\bx-t\by\|^2.
\end{equation*}
\item[(iii)] There exists a point $\bz_t(\bx)\in t\Conv(\mathcal{D})$ such that $\operatorname{dist}_{t}^2(\bx,t\mathcal D) = \|\bx-\bz_t(\bx)\|^2.$

\item[(iv)] For every fixed $\bx\in\R^d$,
\begin{equation*}
\operatorname{dist}_{t}^2(\bx,t\mathcal D)
\to
\dist^2(\bx,\mathcal D)
\quad\text{and}\quad
f_t(\bx)
\to
\tfrac{1}{2}\dist^2(\bx,\mathcal D),
\qquad t\uparrow1.
\end{equation*}
\end{enumerate}
\end{proposition}
Part (i) is directly from
\citet{cai2026improvingclassifierfreeguidanceflow}, while parts (ii)--(iv)
are proved in Appendix~\ref{proof:prop:soft-distance-representation}.
 The same
identity holds for the conditional flow ODE (\ref{eq:cond flow}) after
replacing $p_\emptyset$, $\eta_t^{\bx}$, and $f_t$ by their conditional
counterparts.

Under this interpretation, the ODE (\ref{eq:uncond flow}) is a time-dependent gradient
dynamics whose objective $f_t(\bx)$ at time $t$ consists of $\tfrac{1}{2}\operatorname{dist}_{t}^2(\bx,t\mathcal D)$ 
 and a spatial correction $- \tfrac{1-t}{2}\|\bx\|^2$. Specifically,
$\operatorname{dist}_{t}^2(\bx,t\mathcal D)$ carries the data-dependent moving geometry, while $- \tfrac{1-t}{2}\|\bx\|^2$ corrects its spatial scaling. This is the geometric picture illustrated in
Figure~\ref{fig:illustration-of-geometry}. 
This viewpoint also suggests the natural metric to study along the
ODE (\ref{eq:uncond flow}) is $\dist(\bx,t\mathcal S)$, where \(\mathcal S\) is a relevant set in data space. The stagewise
theory below develops this idea by choosing \(\mathcal S\) according to the time regime. In the early stage, \(\mathcal S\) is a global mean ball; in the
early and intermediate stages, \(\mathcal S\) becomes the data convex hull; and in the
final stage, \(\mathcal S\) is refined to a local cluster neighborhood. Moreover, the resulting
moving-distance estimates give the attraction and absorption mechanisms that constitute the stagewise geometry of ODE (\ref{eq:uncond flow}).


\begin{figure}[h]
\centering
\includegraphics[width=0.70\textwidth]{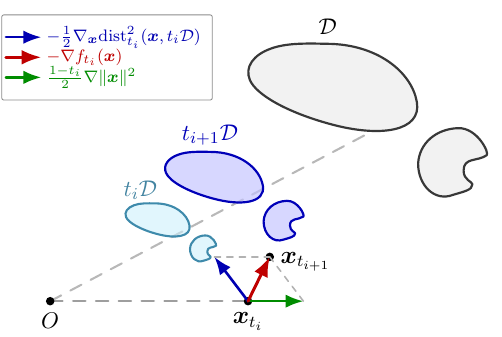}
\caption{
Schematic illustration of the geometry behind
UFM. The smoothed squared-distance term
\(\operatorname{dist}_{t}^2(\bx,t\mathcal D)\)
captures the moving data geometry associated with \(t\mathcal D\), while the
full objective
\(f_t(\bx)=\frac12\operatorname{dist}_{t}^2(\bx,t\mathcal D)
-\frac{1-t}{2}\|\bx\|^2\) determines the ideal flow direction.
}
\label{fig:illustration-of-geometry}
\end{figure}


\section{Stagewise Geometry of Unconditional Flow Matching}\label{sec:uncond}

We now analyze how the geometry controlling the unconditional flow ODE
(\ref{eq:uncond flow}) evolves along a single trajectory. The resulting picture has both a temporal and a hierarchical structure.
 Temporally, the trajectory first moves toward the global mean, then is still controlled by the scaled convex hull of the data, and finally localizes near a local cluster as $t\uparrow1$. 
 Hierarchically, the convex hull provides a coarse global localization backbone, while the mean ball and the local-cluster neighborhood give sharper descriptions in the early and final regimes, respectively.

Accordingly, we present the unconditional theory in temporal order. Section~\ref{sec:uncond flow early stage} identifies the early-stage mean geometry. Section~\ref{sec:inter stage for convex hull} proves the convex-hull attraction and absorption mechanism,
 which acts as the global coarse geometry connecting the early and intermediate stages. Section~\ref{sec:final stage of cluster} then refines the terminal geometry by replacing the global convex hull with a local neighborhood of a target cluster.

\begin{figure*}[h]
\centering
\begin{minipage}[t]{0.32\textwidth}
\centering
\includegraphics[page=1,width=\linewidth,trim=2pt 2pt 2pt 2pt,clip]{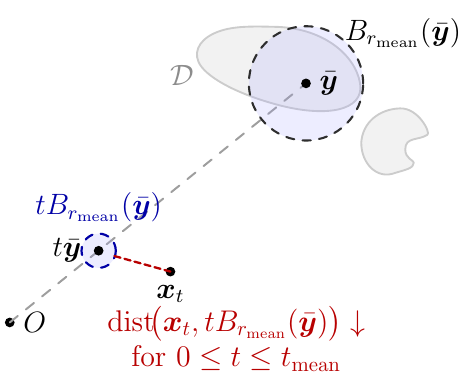}\\[-0.3em]
{\small (a) Early mean ball.\par}
\end{minipage}
\hfill
\begin{minipage}[t]{0.32\textwidth}
\centering
\includegraphics[page=2,width=\linewidth,trim=2pt 2pt 2pt 2pt,clip]{figure/sec3/early_and_final.pdf}\\[-0.3em]
{\small (b) Data convex hull.\par}
\end{minipage}
\hfill
\begin{minipage}[t]{0.32\textwidth}
\centering
\includegraphics[page=3,width=\linewidth,trim=2pt 2pt 2pt 2pt,clip]{figure/sec3/early_and_final.pdf}\\[-0.3em]
{\small (c) Final local cluster.\par}
\end{minipage}
\caption{
Stagewise moving geometries for the unconditional trajectory. In the early
stage, $\bx_t$ moves toward the moving mean ball
\(tB_{r_{\mathrm{mean}}}(\bar{\by})\). The scaled data convex hull
\(t\Conv(\mathcal D)\) provides a global localization geometry. In the final
stage near a target cluster \(\Omega\), $\bx_t$ moves toward the cluster
neighborhood \(tB_{r_\Omega}(\Omega)\).
}
\label{fig:early-final-geometry}
\end{figure*}

\subsection{Early-stage Attraction by the Global Mean}\label{sec:uncond flow early stage}

The early stage is governed by the data mean. Define
$D_\emptyset = \diam(\mathcal{D})$
and 
$M_{\mathcal{D}} := \max_{\by \in \mathcal{D}}\|\by\|$.
Fix $\bx_0 \in \R^d\setminus\{\bzero\}$ and, for
$t\in [0,1)$, define
\[
r(t)
:=
M_{\mathcal D}
\left[
\exp\left(
\tfrac{
tD_{\emptyset}
}{
(1-t)^2
}
\bigl(
(1-t)\|\bx_0\|
+
2tM_{\mathcal D}
\bigr)
\right)-1
\right],
\]
then we choose $t_{\mathrm{mean}}$ such that
\[
(1-t_{\mathrm{mean}})\|\bx_0\|
>
t_{\mathrm{mean}}
\bigl(D_\emptyset+r(t_{\mathrm{mean}})\bigr).
\]
Such a choice exists because
$r(t)=\bigO(t)$ as
$t\downarrow0$, while the left-hand side of the preceding
inequality converges to $\|\bx_0\|>0$. Define $r_{\mathrm{mean}} := r(t_{\mathrm{mean}})$, then, the following theorem describes the
resulting early-stage behaviour for \((\bx_t)_{t\in [0,t_{\mathrm{mean}}]}\); see also
Figure~\ref{fig:early-final-geometry}(a).

\begin{theorem}[Early-stage attraction by the global mean]\label{thm:early-mean}
Let $\bx_t$ solve the unconditional flow ODE~(\ref{eq:uncond flow}) with
initial condition $\bx_0$, and let $t_{\mathrm{mean}}$ and
$r_{\mathrm{mean}}$ be fixed as above. Then the following hold.
\begin{enumerate}
\item[(i)]
\emph{(Non-entry)}
For every $t\in[0,t_{\mathrm{mean}}]$,
\(
\bx_t
\notin
tB_{r_{\mathrm{mean}}}(\bar{\by}).
\)

\item[(ii)]
\emph{(Attraction)}
For every $0\le s\le t\le t_{\mathrm{mean}}$,
\[
\dist\!\left(
\bx_t,
tB_{r_{\mathrm{mean}}}(\bar{\by})
\right)
\le
\tfrac{1-t}{1-s}
\dist\!\left(
\bx_s,
sB_{r_{\mathrm{mean}}}(\bar{\by})
\right).
\]
Hence the function
\(
t\longmapsto
\dist\!\left(
\bx_t,
tB_{r_{\mathrm{mean}}}(\bar{\by})
\right)
\)
is strictly decreasing on $[0,t_{\mathrm{mean}}]$.

\item[(iii)]
Taking $s=0$ in part~(ii) gives
\(
\dist\!\left(
\bx_t,
tB_{r_{\mathrm{mean}}}(\bar{\by})
\right)
\le
(1-t)\|\bx_0\|,
\
t\in[0,t_{\mathrm{mean}}].
\)
\end{enumerate}
\end{theorem}
For the proof, see Appendix~\ref{proof:thm:early-mean}.
The radius $r_{\mathrm{mean}}$ is tied to the length of the early
interval: a smaller $t_{\mathrm{mean}}$ gives a smaller mean ball,
whereas a larger $t_{\mathrm{mean}}$ gives a longer interval but a
less localized early-stage description.

The same early-stage result holds for the Euler
iterates on a sufficiently fine initial grid.
\begin{theorem}[Discrete version of Theorem~\ref{thm:early-mean}]
\label{cor:early-mean-euler}
Under the early-stage setting above, consider the explicit Euler
scheme~(\ref{eq:ideal-euler}) with
$\vv_\star=\vu$ on the complete sampling grid
\(
t_i=i\Delta t,
\
i=0,\ldots,K,
\
\Delta t=1/K,
\)
initialized at $\bz_0=\bx_0$.
If $\Delta t\le t_{\mathrm{mean}}/2$, then the following hold.
\begin{enumerate}
\item[(i)]
\emph{(Non-entry)} we have
\(
\bz_i
\notin
t_iB_{r_{\mathrm{mean}}}(\bar{\by})
\) for every \(i\) with $t_i\le t_{\mathrm{mean}}$.

\item[(ii)]
\emph{(Attraction)}
For every $0\le j\le i$ with $t_i\le t_{\mathrm{mean}}$,
\[
\dist\!\left(
\bz_i,
t_iB_{r_{\mathrm{mean}}}(\bar{\by})
\right)
\le
\tfrac{1-t_i}{1-t_j}
\dist\!\left(
\bz_j,
t_jB_{r_{\mathrm{mean}}}(\bar{\by})
\right).
\]
Hence the sequence
\(
i
\longmapsto
\dist\!\left(
\bz_i,
t_iB_{r_{\mathrm{mean}}}(\bar{\by})
\right)
\)
is strictly decreasing.

\item[(iii)]
Taking $j=0$ in part~(ii) gives
\(
\dist\!\left(
\bz_i,
t_iB_{r_{\mathrm{mean}}}(\bar{\by})
\right)
\le
(1-t_i)\|\bx_0\|
\)
with $t_i\le t_{\mathrm{mean}}$.
\end{enumerate}
\end{theorem}
For the proof, see Appendix~\ref{proof:cor:early-mean-euler}.

\subsection{Early and Intermediate Attraction and Absorption by the Convex Hull}\label{sec:inter stage for convex hull}

The early mean ball gives a description of the initial stage, but it does not describe the global data geometry that controls the whole \((\bx_t)_{t\in [0,1)}\). The natural coarse geometry is the scaled convex hull $t\Conv(\mathcal D)$. 
Indeed, the posterior mean $\hat{\by}_t(\bx)$ is a convex combination of data points and hence belongs to $\Conv(\mathcal D)$ (see Lemma~\ref{lem:mean-in-closed-conv}).
 Consequently, \(\vu(t,\bx)\) always points from the current position toward a point in $\Conv(\mathcal D)$, and \(\bx_t\) keeps moving toward $t\Conv(\mathcal D)$  before the possible entrance.
Further, once \(\bx_t\) enters $t\Conv(\mathcal D)$, it remains inside for all later times;  see illustration in Figure~\ref{fig:early-final-geometry}(b).
Thus \(t\Conv(\mathcal D)\) acts as the coarse global backbone of the whole trajectory, and the following result makes this backbone precise: 
\begin{theorem}[Convex-hull attraction and absorption]\label{thm:convex-hull}
Let $\bx_t$ solve the unconditional flow ODE (\ref{eq:uncond flow}) on $[0,1)$ with initial condition $\bx_0$.
Then the following hold.
\begin{enumerate}
\item[(i)] \emph{(Attraction)}
For every \(0\le s\le t<1\),
\[
\dist\!\left(
\bx_t,
t\Conv(\mathcal D)
\right)
\le
\tfrac{1-t}{1-s}
\dist\!\left(
\bx_s,
s\Conv(\mathcal D)
\right).
\]
Hence the function
\(
t\mapsto\dist(\bx_t,t\Conv(\mathcal D))
\)
is non-increasing on \([0,1)\).
\item[(ii)] Taking $s=0$ in part~(i) gives
\(
\dist\!\left(
\bx_t,
t\Conv(\mathcal D)
\right)
\le
(1-t)\|\bx_0\|,
\
t\in[0,1).
\)
\item[(iii)] \emph{(Absorption)}
If there exists \(t_{\mathrm{conv}}\in[0,1)\) such that
\(
\bx_{t_{\mathrm{conv}}}
\in
t_{\mathrm{conv}}\Conv(\mathcal D),
\)
then
\(
\bx_t\in t\Conv(\mathcal D)
\)
for every \(t\in[t_{\mathrm{conv}},1)\).
\end{enumerate}
\end{theorem}
For the proof, see Appendix~\ref{proof:thm:convex-hull}. The theorem says that the trajectory moves toward  $t\Conv(\mathcal D)$ at the rate $1-t$, and that after first entrance it remains inside  $t\Conv(\mathcal D)$. Thus $t\Conv(\mathcal D)$ gives a coarse but end-to-end localization region for the unconditional flow ODE.

For $t>0$, this unscaled formulation is equivalent to the scaled-coordinate statement in \citet{wan2025elucidating}, which controls $\dist(\bx_t/t,\Conv(\mathcal D))$. Indeed, $\dist(\bx_t/t,\Conv(\mathcal D))=\frac{1}{t}\dist(\bx_t,t\Conv(\mathcal D))$.
We use the unscaled form because it follows the original flow trajectory directly and remains meaningful at the initial time $t=0$, 
where the scaled coordinate $\bx_t/t$ is singular.
In addition, the same formulation extends exactly to the discrete scheme: the Euler update preserves the same attraction and absorption mechanism on the grid, as the following discrete counterpart shows:

\begin{theorem}[Discrete version of Theorem~\ref{thm:convex-hull}]\label{cor:convex-hull-euler}
Consider the explicit Euler scheme (\ref{eq:ideal-euler}) with $\vv_\star=\vu$ on a uniform grid $t_i=i\Delta t$, $i=0,\ldots,N$, where \(N\ge1\), \(\Delta t>0\), and \(t_N\le1\), initialized at \(\boldsymbol z_0=\boldsymbol x_0\).
Then the following hold.
\begin{enumerate}
\item[(i)] \emph{(Attraction)} For every \(0\le j\le i\le N\) with \(t_j<1\),
\begin{equation*}
\dist(\bz_i,t_i\Conv(\mathcal D))
\le
\tfrac{1-t_i}{1-t_j}
\dist(\bz_j,t_j\Conv(\mathcal D)).
\end{equation*}
Hence the sequence $i\mapsto\dist(\bz_i,t_i\Conv(\mathcal D))$ is nonincreasing on $i=0,\ldots,N$.
\item[(ii)] Taking $j=0$ in part~(i) gives
\(\dist(\bz_i,t_i\Conv(\mathcal{D}))
\le
(1-t_i)\|\bx_0\|,
\
i=0,\ldots,N.\)
\item[(iii)] \emph{(Absorption)} If there exists an entrance index $k\in\{0,\ldots,N\}$ such that $\bz_k\in t_k\Conv(\mathcal{D})$, then
$\bz_i\in t_i\Conv(\mathcal{D})$ for every $i=k,k+1,\ldots,N$.
\end{enumerate}
\end{theorem}
For the proof, see Appendix~\ref{proof:cor:convex-hull-euler}.

\subsection{Final-stage Attraction and Absorption by the Local Cluster}\label{sec:final stage of cluster}

The convex-hull result gives a coarse end-to-end localization region, but it is not sharp near the terminal time. As $t\uparrow1$, the posterior weights concentrate around data points close to the current trajectory, so the relevant geometry becomes local. 
It turns out that, once the trajectory is located near a target cluster, it moves toward a scaled neighborhood of that cluster; see illustration in Figure~\ref{fig:early-final-geometry}(c).
Since the local data geometry is typically non-convex, we formulate the final-stage geometry directly for a local cluster $\Omega\subsetneq\mathcal D$. Fix $R>0$, which is the radius of the working neighborhood, then define $$D_\Omega:=\diam(\Omega), \delta_\Omega:=\dist(B_R(\Omega),\mathcal D\setminus\Omega), \text{ and }\gamma_\Omega:=\delta_\Omega^2-(R+D_\Omega)^2.$$ Under the prox-regularity condition below, whenever $\dist(\bx,t\Omega)<tR$, the projection is single-valued and we write 
\begin{equation}
\label{eq:y_t^omega}
    \by_t^\Omega(\bx):=\Proj_{t\Omega}(\bx)/t = \Proj_{\Omega}(\bx/t). 
\end{equation}
We first list several geometric conditions on the local cluster.

\begin{assumption}[Local-cluster regularity]\label{assump:local-cluster}
Let \(\Omega\subsetneq\mathcal D\) be a nonempty closed local cluster
satisfying \(0<p_\emptyset(\Omega)<1\), and let \(R>0\). Assume the
following conditions:
\begin{enumerate}
\item[(i)] \emph{Local prox-regular geometry:} the target cluster $\Omega$ is $\rho_\Omega$-prox-regular for some $\rho_\Omega>R$.
\item[(ii)] \emph{Uniform local polynomial mass:} there exist constants $c_\Omega>0$, $\alpha_\Omega>0$, and $s_\Omega>0$ such that
\[
p_\emptyset\bigl(\Omega\cap B_s(\by)\bigr)
\ge c_\Omega s^{\alpha_\Omega},
\qquad
\by\in\Omega,\quad 0<s\le s_\Omega.
\]
\item[(iii)] \emph{Cluster isolation:} $\gamma_\Omega=\delta_\Omega^2-(R+D_\Omega)^2>0$.
\end{enumerate}
\end{assumption}

\begin{center}
\begin{minipage}{\textwidth}
\centering
\begin{minipage}[t]{0.32\textwidth}
\centering
\includegraphics[page=1,width=\linewidth]{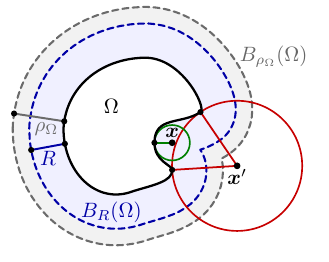}\\[-0.2em]
\parbox[t][2.6em][t]{\linewidth}{\centering\small
(i) Local prox-regularity}
\end{minipage}
\hfill
\begin{minipage}[t]{0.32\textwidth}
\centering
\includegraphics[page=2,width=\linewidth]{figure/sec3/prox_regular_local_mass.pdf}\\[-0.2em]
\parbox[t][2.6em][t]{\linewidth}{\centering\small
(ii) Uniform local mass}
\end{minipage}
\hfill
\begin{minipage}[t]{0.32\textwidth}
\centering
\includegraphics[page=3,width=\linewidth]{figure/sec3/prox_regular_local_mass.pdf}\\[-0.2em]
\parbox[t][2.6em][t]{\linewidth}{\centering\small
(iii) Cluster isolation}
\end{minipage}
\makeatletter
\def\@captype{figure}
\makeatother
\caption{
Conditions behind the final-stage theorem.
(i) The working neighborhood \(B_R(\Omega)\) lies inside the
\(\rho_\Omega\)-neighborhood on which the metric projection is single-valued.
(ii) The polynomial lower mass bound holds uniformly at every
\(\by\in\Omega\).
(iii) The separation from \(B_R(\Omega)\) to the rest of the data is larger
than \(R+D_\Omega\), equivalently
\(\gamma_\Omega=\delta_\Omega^2-(R+D_\Omega)^2>0\).
}
\label{fig:prox-regular-local-mass}
\end{minipage}
\end{center}

\begin{remark}[Rationale of the assumptions]\label{rem:local-cluster-roles}
 Assumption~\ref{assump:local-cluster}(i) and
Lemma~\ref{lem:prox-regular-projection-estimate} give the projection estimate on the final-stage
neighborhood. The strict inequality \(R<\rho_\Omega\) ensures that
\(1-R/\rho_\Omega>0\); the factor diverges as
\(R\uparrow\rho_\Omega\), as illustrated in
Figure~\ref{fig:prox-regular-local-mass}(i).
Assumption~\ref{assump:local-cluster}(ii) is a uniform local polynomial-mass condition at every point of \(\Omega\). It is applied at \(\by_t^\Omega(\bx)\) in Lemma~\ref{lem:polynomial-local-mass-radius}, which leads to the posterior mean localization in Lemma~\ref{lem:final-local-posterior}, as illustrated in Figure~\ref{fig:prox-regular-local-mass}(ii).
Assumption~\ref{assump:local-cluster}(iii) is a sufficient separation condition preventing capture by the rest of the data, as illustrated in Figure~\ref{fig:prox-regular-local-mass}(iii).
\end{remark}
Under Assumption~\ref{assump:local-cluster}, choose a final-stage starting
time $t_{\mathrm{loc}}\in(0,1)$ satisfying
\begin{equation}\label{eq:final-local-time-choice}
(1-t_{\mathrm{loc}})^2
\log\tfrac{1}{1-t_{\mathrm{loc}}}
\le
\min\left\{
\tfrac{1}{4(2\alpha_\Omega+1)},
\tfrac{(2R+1)s_\Omega}{2\alpha_\Omega+1}
\right\}.
\end{equation}
Such a choice exists because the left-hand side converges to zero as
$t_{\mathrm{loc}}\uparrow1$. For any such $t_{\mathrm{loc}}$, define the
inflation radius
\begin{equation}\label{eq:final-local-inflation-radius}
\begin{aligned}
r_\Omega
:={}
\tfrac{\varepsilon_\Omega}
{t_{\mathrm{loc}}\sqrt{1-R/\rho_\Omega}}
+
\tfrac{D_\Omega}{c_{\Omega,\mathrm{loc}}}
\exp\!\left(-\tfrac{\varepsilon_\Omega^2}{4(1-t_{\mathrm{loc}})^2}\right)
+
2M_{\mathcal{D}}
\tfrac{1-p_\emptyset(\Omega)}{p_\emptyset(\Omega)}
\exp\!\left(-\tfrac{t_{\mathrm{loc}}^2\gamma_\Omega}{2(1-t_{\mathrm{loc}})^2}\right),
\end{aligned}
\end{equation}
where
\begin{equation}\label{eq:final-local-quantities}
\begin{aligned}
\varepsilon_\Omega
&:=
2(1-t_{\mathrm{loc}})
\sqrt{
(2\alpha_\Omega+1)
\log\tfrac{1}{1-t_{\mathrm{loc}}}
}, \quad
c_{\Omega,\mathrm{loc}}
:=
c_\Omega(8R+4)^{-\alpha_\Omega}
\varepsilon_\Omega^{2\alpha_\Omega}.
\end{aligned}
\end{equation}
With $\Omega,\mathcal D,R,\rho_\Omega,c_\Omega,\alpha_\Omega,s_\Omega$,
and $p_\emptyset$ fixed, the inflation radius satisfies
\[
r_\Omega
=
\bigO\!\left(
(1-t_{\mathrm{loc}})
\sqrt{\log\tfrac{1}{1-t_{\mathrm{loc}}}}
\right)
\qquad\text{as }t_{\mathrm{loc}}\uparrow1.
\]
A proof of this order is given in
Lemma~\ref{lem:polynomial-local-mass-radius}.
The theorem below shows that \((\bx_t)_{t\in[t_{\mathrm{loc}},1)}\) moves toward \(tB_{r_\Omega}(\Omega)\) at the rate $1-t$ 
under certain initial conditions at $t_{\mathrm{loc}}$; once it enters this moving
neighborhood, it remains inside for all later times.

\begin{theorem}[Final-stage local-cluster attraction and absorption]\label{thm:final-local}
Under Assumption~\ref{assump:local-cluster}, let $t_{\mathrm{loc}}$ and
$r_\Omega$ be as above, and  assume further that 
\begin{equation}\label{eq:final-local-capture-condition}
\dist\bigl(
\bx_{t_{\mathrm{loc}}},
t_{\mathrm{loc}}B_{r_\Omega}(\Omega)
\bigr)+r_\Omega
\le t_{\mathrm{loc}}R. 
\end{equation} Let $\bx_t$ solve the unconditional flow ODE
(\ref{eq:uncond flow}) on $[t_{\mathrm{loc}},1)$, 
then the following hold.
\begin{enumerate}
\item[(i)] \emph{(Attraction)} For every
$t_{\mathrm{loc}}\le s\le t<1$,
\[
\dist\bigl(\bx_t,tB_{r_\Omega}(\Omega)\bigr)
\le
\tfrac{1-t}{1-s}
\dist\bigl(\bx_s,sB_{r_\Omega}(\Omega)\bigr).
\]
Hence the function
$t\mapsto\dist\bigl(\bx_t,tB_{r_\Omega}(\Omega)\bigr)$ is
non-increasing on $[t_{\mathrm{loc}},1)$.
\item[(ii)] Taking $s=t_{\mathrm{loc}}$ in part~(i) gives
\begin{equation*}
\dist\bigl(\bx_t,tB_{r_\Omega}(\Omega)\bigr)
\le
\tfrac{1-t}{1-t_{\mathrm{loc}}}
\dist\bigl(\bx_{t_{\mathrm{loc}}},t_{\mathrm{loc}}B_{r_\Omega}(\Omega)\bigr)=\bigO(1-t),
\qquad
t\in[t_{\mathrm{loc}},1).
\end{equation*}
\item[(iii)] \emph{(Absorption)} If there exists $t_{\Omega}\in[t_{\mathrm{loc}},1)$ such that $\bx_{t_{\Omega}}\in t_{\Omega}B_{r_\Omega}(\Omega)$, then $\bx_t\in tB_{r_\Omega}(\Omega)$ for every $t\in[t_{\Omega},1)$.
\end{enumerate}
\end{theorem}

For the proof, see Appendix~\ref{proof:thm:final-local}.

\begin{remark}[Rationale of the final-stage conditions]
Here, \eqref{eq:final-local-time-choice} ensures the validity of the
local-mass construction in Lemma~\ref{lem:polynomial-local-mass-radius},
which is used in the posterior-localization estimate of
Lemma~\ref{lem:final-local-posterior}. In contrast,
\eqref{eq:final-local-capture-condition} is trajectory-dependent and places
the starting point of the final-stage tail in the local regime around
$t_{\mathrm{loc}}\Omega$.
\end{remark}
The final-stage result also has a discrete counterpart on a shifted Euler grid.

\begin{theorem}[Discrete version of Theorem~\ref{thm:final-local}]\label{cor:final-local-euler}
Under Assumption~\ref{assump:local-cluster}, let $t_{\mathrm{loc}}$ and
$r_\Omega$ be as above. Let $(\bz_i)_{i=0}^N$ be generated by the
shifted Euler update
$\bz_{i+1}=\bz_i+\Delta t\,\vu(t_i,\bz_i)$ on the uniform grid
$t_i=t_{\mathrm{loc}}+i\Delta t$, $i=0,\ldots,N$, with $t_N\le1$,
initialized at the exact final-stage starting point
$\bz_0=\bx_{t_{\mathrm{loc}}}$. Assume 
\[
\dist\bigl(
\bz_0,t_{\mathrm{loc}}B_{r_\Omega}(\Omega)
\bigr)+r_\Omega\le t_{\mathrm{loc}}R,
\]
then the following hold.
\begin{enumerate}
\item[(i)] \emph{(Attraction)} For every $0\le j\le i\le N$ with
$t_j<1$,
\[
\dist\bigl(\bz_i,t_iB_{r_\Omega}(\Omega)\bigr)
\le
\tfrac{1-t_i}{1-t_j}
\dist\bigl(\bz_j,t_jB_{r_\Omega}(\Omega)\bigr).
\]
Hence the sequence
$i\mapsto\dist\bigl(\bz_i,t_iB_{r_\Omega}(\Omega)\bigr)$ is
non-increasing on $i=0,\ldots,N$.
\item[(ii)] Taking $j=0$ in part~(i) gives
\[
\dist\bigl(\bz_i,t_iB_{r_\Omega}(\Omega)\bigr)
\le
\tfrac{1-t_i}{1-t_{\mathrm{loc}}}
\dist\bigl(\bz_0,t_{\mathrm{loc}}B_{r_\Omega}(\Omega)\bigr)
=\bigO(1-t_i),
\qquad
i=0,\ldots,N.
\]
\item[(iii)] \emph{(Absorption)} If there exists
$k\in\{0,\ldots,N\}$ such that
$\bz_k\in t_kB_{r_\Omega}(\Omega)$, then
$\bz_i\in t_iB_{r_\Omega}(\Omega)$ for all $i=k,k+1,\ldots,N$.
\end{enumerate}
\end{theorem}
For the proof, see Appendix~\ref{proof:cor:final-local-euler}.

\subsection{Conditional Flow Matching by Restriction}\label{sec:cond}

The unconditional theory comes first because it contains the main geometric mechanism. The conditional setting is obtained by restricting the data distribution from $p_\emptyset$ to $p_c$ and the ambient geometry from $\mathcal{D}$ to $\mathcal{D}_c$. The corresponding conditional flow ODE is (\ref{eq:cond flow}).
Apart from these substitutions, the assumptions and proofs are identical to those in Section~\ref{sec:uncond}. In particular, Theorem~\ref{thm:early-mean} gives attraction to the conditional-mean ball after replacing $\bar{\by}$ by $\bar{\by}_c$; Theorem~\ref{thm:convex-hull} gives attraction to $t\Conv(\mathcal{D}_c)$ and absorption after any entrance time; and Theorem~\ref{thm:final-local} applies to local clusters $\Omega_c\subset \mathcal{D}_c$ with probabilities and constants computed under $p_c$. In the final-stage specialization, the uniform mass condition is
\(
p_c\bigl(\Omega_c\cap B_s(\by)\bigr)
\ge c_{\Omega_c}s^{\alpha_{\Omega_c}},\
\by\in\Omega_c,\ 0<s\le s_{\Omega_c}.
\)
The same restriction \eqref{eq:final-local-time-choice} and definitions
\eqref{eq:final-local-inflation-radius}--\eqref{eq:final-local-quantities}
apply with all constants computed under $p_c$. The corresponding continuous rates remain
$\bigO(1-t)$, and the Euler counterparts also inherit corresponding estimates. Since these statements require no new
argument, we do not repeat them as separate formal corollaries.

\section{Analysis of Flow Matching with Classifier-Free Guidance Extrapolation}\label{sec:cfg}

We now extend the stagewise geometric theory to the CFG ODE. In practice, the CFG ODE (\ref{eq:cfg flow}) is more commonly used than the conditional flow ODE (\ref{eq:cond flow}) 
\citep{DBLP:journals/corr/abs-2506-15742, DBLP:conf/icml/EsserKBEMSLLSBP24, peebles2023scalable}, 
but its geometry remains less well understood
\citep{pavasovic2026overshoot, 
DBLP:conf/aaai/ZhaoS26,
li2025understandingmechanismsclassifierfreeguidance}.
Following \citet{cai2026improvingclassifierfreeguidanceflow,wang2025towards}, we call 
\[\bgc(t,\bx) := \vc(t,\bx)-\vu(t,\bx)\]
the \emph{prediction gap}. 
Compared with ordinary conditional flow, CFG changes the geometry through an additional $w$-weighted prediction gap, and this affects the stages in different ways. 
In the early stage, the conditional mean ball is replaced by a ball around the extrapolated mean. 
In the convex-hull regime, the extrapolated CFG posterior mean need not lie in $\Conv(\mathcal D_c)$.
 Instead, each tail of a fixed CFG trajectory is controlled by an inflated conditional convex hull. This inflation radius is nonincreasing as the tail starting time moves later; after capture and under a quantitative separation condition, it admits an explicit exponential bound. In the final stage, once the trajectory is localized near a well-isolated conditional cluster, 
 the prediction gap $\bgc(t,\bx)$ vanishes. Hence the terminal dynamics return to conditional local-cluster geometry, up to an exponentially small enlargement of the inflation radius.

\begin{figure*}[h]
\centering
\begin{minipage}[t]{0.47\textwidth}
\centering
\includegraphics[page=1,width=\linewidth,trim=2pt 2pt 2pt 2pt,clip]{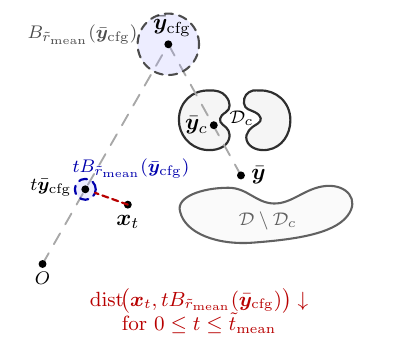}\\[-0.3em]
{\small (a) Early extrapolated mean ball.\par}
\end{minipage}
\hfill
\begin{minipage}[t]{0.47\textwidth}
\centering
\includegraphics[page=2,width=\linewidth,trim=2pt 2pt 2pt 2pt,clip]{figure/sec3/cfg_stagewise_geometry.pdf}\\[-0.3em]
{\small (b) Inflated conditional convex hull.\par}
\end{minipage}
\par\vspace{0.6em}
\begin{minipage}[t]{0.47\textwidth}
\centering
\includegraphics[page=3,width=\linewidth,trim=2pt 2pt 2pt 2pt,clip]{figure/sec3/cfg_stagewise_geometry.pdf}\\[-0.3em]
{\small (c) Refined inflated conditional convex hull.\par}
\end{minipage}
\hfill
\begin{minipage}[t]{0.47\textwidth}
\centering
\includegraphics[page=4,width=\linewidth,trim=2pt 2pt 2pt 2pt,clip]{figure/sec3/cfg_stagewise_geometry.pdf}\\[-0.3em]
{\small (d) Final conditional local cluster.\par}
\end{minipage}
\caption{
Stagewise moving geometries under classifier-free guidance. In the early
stage, guidance changes the attractor to the extrapolated mean
\(\bar{\by}_{\mathrm{cfg}}\). The neighborhood of the scaled, inflated conditional convex hull provides a coarse localization geometry; after capture and
separation, this neighborhood is refined to a smaller one. In the final stage
near a conditional cluster \(\Omega_c\), \(\bx_t\) moves toward
\(tB_{\tilde r_{\Omega_c}}(\Omega_c)\).
}
\label{fig:cfg-stagewise-geometry}
\end{figure*}
\subsection{Early-stage Attraction by the Extrapolated Mean}

The first effect of classifier-free guidance appears in the early stages. For the  conditional flow ODE (\ref{eq:cond flow}), the early trajectory is controlled by the conditional clean-data mean; for CFG ODE (\ref{eq:cfg flow}), the additional term $\bgc(t,\bx)$ leads to
\[
\frac{\dd \bx_t}{\dd t} = \frac{\hat{\by}_{t,\mathrm{cfg}}(\bx_t) - \bx_t}{1-t}, \quad 
\text{with } \
\hat{\by}_{t,\mathrm{cfg}}(\bx)
=
w\hat{\by}_{t,c}(\bx)
+
(1-w)\hat{\by}_t(\bx).
\]
Thus, the early-stage moving geometry of CFG becomes a ball centered at the extrapolated mean
$\bar{\by}_{\mathrm{cfg}}:=\bar{\by}+w(\bar{\by}_c-\bar{\by}),$ as illustrated in Figure~\ref{fig:cfg-stagewise-geometry}(a).
Now we define
$M_{\mathcal D}:=\max_{\by\in\mathcal D}\|\by\|$ and
$M_{\mathcal D_c}:=\max_{\by\in\mathcal D_c}\|\by\|$, and define
$D_\emptyset:=\diam(\mathcal D)$ and $D_c:=\diam(\mathcal D_c)$ as their
corresponding diameters.
Fix $\bx_0\in\R^d\setminus\{\bzero\}$ and $w>1$, for $t\in [0,1)$, define
\begin{equation*}
    \begin{split}
\tilde r(t)
&:=
(w-1)M_{\mathcal D}
\bigl(\exp (g_{\emptyset}(t))-1\bigr)
+
wM_{\mathcal D_c}
\bigl(\exp (g_c(t))-1\bigr),
        \\
        g_{\emptyset}(t)
&:=
\tfrac{tD_{\emptyset}}{(1-t)^2}
\left[
(1-t)\|\boldsymbol x_0\|
+
tw\bigl(M_{\mathcal D}+M_{\mathcal D_c}\bigr)
\right],\\
g_c(t)
& :=
\tfrac{tD_c}{(1-t)^2}
\left[
(1-t)\|\boldsymbol x_0\|
+
t\Bigl(
(w-1)M_{\mathcal D}
+
(w+1)M_{\mathcal D_c}
\Bigr)
\right].
    \end{split}
\end{equation*}
and choose $\tilde t_{\mathrm{mean}}$ such that
\(
(1-\tilde t_{\mathrm{mean}})\|\bx_0\|
>
\tilde t_{\mathrm{mean}}
\left[
(w-1)D_\emptyset+wD_c+\tilde r(\tilde t_{\mathrm{mean}})
\right].
\)
Such a choice exists because
$\tilde r(t)=\bigO(t)$ as
$t\downarrow0$, while the left-hand side of the
preceding inequality converges to $\|\bx_0\|>0$. 
Define $\tilde r_{\mathrm{mean}} := \tilde r(\tilde t_{\mathrm{mean}})$, then we have:

\begin{theorem}[CFG early-stage attraction by the extrapolated mean]
\label{thm:cfg-early}
Let $\bx_t$ solve the CFG ODE~(\ref{eq:cfg flow}) with initial condition
$\bx_0$, and let $\tilde t_{\mathrm{mean}}$ and
$\tilde r_{\mathrm{mean}}$ be fixed as above. Then the following hold.
\begin{enumerate}
\item[(i)]
\emph{(Non-entry)}
For every $t\in[0,\tilde t_{\mathrm{mean}}]$,
\(
\bx_t
\notin
tB_{\tilde r_{\mathrm{mean}}}
(\bar{\by}_{\mathrm{cfg}}).
\)

\item[(ii)]
\emph{(Attraction)}
For every $0\le s\le t\le\tilde t_{\mathrm{mean}}$,
\[
\dist\!\left(
\bx_t,
tB_{\tilde r_{\mathrm{mean}}}
(\bar{\by}_{\mathrm{cfg}})
\right)
\le
\tfrac{1-t}{1-s}
\dist\!\left(
\bx_s,
sB_{\tilde r_{\mathrm{mean}}}
(\bar{\by}_{\mathrm{cfg}})
\right).
\]
Hence the function
\(
t\longmapsto
\dist\!\left(
\bx_t,
tB_{\tilde r_{\mathrm{mean}}}
(\bar{\by}_{\mathrm{cfg}})
\right)
\)
is strictly decreasing on $[0,\tilde t_{\mathrm{mean}}]$.

\item[(iii)]
Taking $s=0$ in part~(ii) gives
\(
\dist\!\left(
\bx_t,
tB_{\tilde r_{\mathrm{mean}}}
(\bar{\by}_{\mathrm{cfg}})
\right)
\le
(1-t)\|\bx_0\|,
\
t\in[0,\tilde t_{\mathrm{mean}}].
\)
\end{enumerate}
\end{theorem}
For the proof, see Appendix~\ref{proof:thm:cfg-early}. 
The theorem identifies the early-stage mechanism of guidance: the guidance
scale \(w\) acts by relocating the extrapolated mean
\(\bar{\by}_{\mathrm{cfg}}\). Increasing \(w\)
moves this attractor from the unconditional mean toward and beyond the
conditional mean, thereby changing the moving target that controls \((\bx_t)_{t \in [0,\tilde t_{\mathrm{mean}})}\). In this sense, the early condition strength is fundamentally determined by the
position of \(\bar{\by}_{\mathrm{cfg}}\).
Empirically, it can be observed that
 moderate
extrapolation can adequately strengthen condition-specific features in \(\bar{\by}_{\mathrm{cfg}}\) and preserve sample quality. In contrast, overly small or overly large $w$ may under-amplify or over-amplify the features in \(\bar{\by}_{\mathrm{cfg}}\), leading to  weakly conditioned or oversaturated samples; see 
  Figure~\ref{fig:cfg-imagenet-mean-exp} for an illustration.

\subsection{Early and Intermediate Attraction and Absorption by Inflated Convex Hulls}

For ordinary conditional flow, the posterior mean satisfies
\(
\hat{\by}_{t,c}(\bx)\in \Conv(\mathcal D_c).
\)
Consequently,
\(t\Conv(\mathcal D_c)\) directly governs the corresponding dynamics: the conditional trajectory is attracted to this
moving convex hull and, after entrance, is absorbed by it.

Under CFG, the effective posterior mean is instead extrapolated according to
\[
\hat{\by}_{t,\mathrm{cfg}}(\bx)
=
\hat{\by}_{t,c}(\bx)
+
(w-1)
\bigl(
\hat{\by}_{t,c}(\bx)-\hat{\by}_t(\bx)
\bigr).
\]
Although \(\hat{\by}_{t,c}(\bx)\in\Conv(\mathcal D_c)\), the CFG posterior
mean need not lie in the conditional convex hull. Instead,
\[
\dist\!\left(
\hat{\by}_{t,\mathrm{cfg}}(\bx),
\Conv(\mathcal D_c)
\right)
\le
(w-1)
\left\|
\hat{\by}_{t,c}(\bx)-\hat{\by}_t(\bx)
\right\|.
\]
Thus, under CFG, the exact conditional convex hull is replaced by an
inflated conditional convex hull whose radius is controlled by this
posterior-mean discrepancy. Once a uniform bound on the discrepancy is
available along a trajectory tail, the same  mechanism
applies to the corresponding inflated geometry.

For a fixed CFG trajectory, Theorem~\ref{thm:cfg-convex} controls each tail
using the supremum of the posterior-mean discrepancy over that tail.
By construction, delaying the tail starting time can only decrease this
supremum. Corollary~\ref{cor:cfg-convex-small} further shows that, after
the trajectory has entered the initially inflated conditional convex hull
and under a quantitative separation condition, the required tail inflation
admits an explicit exponential bound.

Fix the data measures \(p_\emptyset\) and \(p_c\), the condition \(c\), the
guidance scale \(w>1\), and the initial point \(\bx_0\). These quantities
determine a unique CFG ODE \((\bx_t)_{t\in[0,1)}\). Next define
\begin{equation}\label{eq:cfg-tail-inflation-function}
\epsilon_{\mathrm{infl}}(t)
:=
(w-1)
\sup_{s\in[t,1)}
\left\|
\hat{\by}_{s,c}(\bx_s)-\hat{\by}_s(\bx_s)
\right\|,
\qquad t\in[0,1).
\end{equation}
Thus \(\epsilon_{\mathrm{infl}}(t)\) depends on \(p_\emptyset\), \(p_c\), \(c\), \(w\), and \(\bx_0\). Since \([t_2,1)\subset[t_1,1)\) whenever
\(0\le t_1\le t_2<1\),
\begin{equation}\label{eq:cfg-tail-inflation-monotonicity}
\epsilon_{\mathrm{infl}}(t_2)
\le
\epsilon_{\mathrm{infl}}(t_1).
\end{equation}
Hence, a later tail starting time requires no larger inflation radius. Now we can state the result quantitatively: 

\begin{theorem}[CFG inflated convex-hull attraction and absorption]\label{thm:cfg-convex}
Let $\bx_t$ be the CFG trajectory fixed above. Fix
\(t_{\mathrm{infl}}\in[0,1)\), and let
\(\epsilon_{\mathrm{infl}}(t_{\mathrm{infl}})\) be given by
\eqref{eq:cfg-tail-inflation-function}. Then the following hold.
\begin{enumerate}
\item[(i)]
\emph{(Attraction)}
For every
\(t_{\mathrm{infl}}\le s\le t<1\),
\[
\dist\!\left(
\bx_t,
tB_{\epsilon_{\mathrm{infl}}(t_{\mathrm{infl}})}(\Conv(\mathcal D_c))
\right)
\le
\tfrac{1-t}{1-s}
\dist\!\left(
\bx_s,
sB_{\epsilon_{\mathrm{infl}}(t_{\mathrm{infl}})}(\Conv(\mathcal D_c))
\right).
\]
Hence the function
\(
t\longmapsto
\dist\!\left(
\bx_t,
tB_{\epsilon_{\mathrm{infl}}(t_{\mathrm{infl}})}(\Conv(\mathcal D_c))
\right)
\)
is non-increasing on \([t_{\mathrm{infl}},1)\).

\item[(ii)] Taking $s=t_{\mathrm{infl}}$ in part~(i) gives, for every
\(t\in[t_{\mathrm{infl}},1)\),
\[
\begin{aligned}
\dist\!\left(
\bx_t,
tB_{\epsilon_{\mathrm{infl}}(t_{\mathrm{infl}})}(\Conv(\mathcal D_c))
\right)
\le
\tfrac{1-t}{1-t_{\mathrm{infl}}}
\dist\!\left(
\bx_{t_{\mathrm{infl}}},
t_{\mathrm{infl}}
B_{\epsilon_{\mathrm{infl}}(t_{\mathrm{infl}})}(\Conv(\mathcal D_c))
\right).
\end{aligned}
\]

\item[(iii)]
\emph{(Absorption)}
If there exists
\(\tilde{t}_{\mathrm{conv}}\in[t_{\mathrm{infl}},1)\)
such that
\(
\bx_{\tilde{t}_{\mathrm{conv}}}
\in
\tilde{t}_{\mathrm{conv}}
B_{\epsilon_{\mathrm{infl}}(t_{\mathrm{infl}})}(\Conv(\mathcal D_c)),
\)
then
\(
\bx_t\in tB_{\epsilon_{\mathrm{infl}}(t_{\mathrm{infl}})}(\Conv(\mathcal D_c))
\)
for every \(t\in[\tilde{t}_{\mathrm{conv}},1)\).
\end{enumerate}
\end{theorem}
For the proof, see Appendix~\ref{proof:thm:cfg-convex}. 
For each fixed \(t_{\mathrm{infl}}\), Theorem~\ref{thm:cfg-convex}
controls the tail \((\bx_t)_{t \in [t_{\mathrm{infl}},1)}\) by the inflated conditional
convex hull
\(
tB_{\epsilon_{\mathrm{infl}}(t_{\mathrm{infl}})}
\bigl(\Conv(\mathcal D_c)\bigr).
\)
To quantify when this geometry approaches the exact conditional convex
hull, define
\begin{equation}\label{eq:cfg-convex-separation-condition}
\begin{aligned}
\rho_{\mathrm{infl}}
&:=
\dist\!\left(
B_{\epsilon_{\mathrm{infl}}(t_{\mathrm{infl}})}
(\Conv(\mathcal D_c)),
\mathcal D\setminus\mathcal D_c
\right), \quad 
\Delta_{\mathrm{infl}}
:=
\rho_{\mathrm{infl}}^2
-
\bigl(\epsilon_{\mathrm{infl}}(t_{\mathrm{infl}})+D_c\bigr)^2.
\end{aligned}
\end{equation}
Now we assume 
\begin{equation}\label{eq:cfg-convex-positive-separation}
\epsilon_{\mathrm{infl}}(t_{\mathrm{infl}}) > 0, \quad
\Delta_{\mathrm{infl}}>0.
\end{equation}
The condition \(\epsilon_{\mathrm{infl}}(t_{\mathrm{infl}})>0\) excludes the degenerate case in which the conditional and unconditional posterior means already coincide along the entire tail, so that no inflation is needed.
The condition \(\Delta_{\mathrm{infl}}>0\) is a separation condition requiring the remaining data outside \(\mathcal D_c\) to lie sufficiently far from \(\Conv(\mathcal D_c)\), relative to both the intrinsic scale \(D_c\) of the conditional data and the tail inflation \(\epsilon_{\mathrm{infl}}(t_{\mathrm{infl}})\).

If \((\bx_t)_{t \in [t_{\mathrm{infl}},1)}\) indeed enters this inflated convex hull before the terminal time $t = 1$, then it remains there along the remaining tail.
Together with \eqref{eq:cfg-convex-positive-separation}, these properties
yield the following exponential refinement.


\begin{corollary}[Exponential refinement of the CFG inflation radius]\label{cor:cfg-convex-small}
 Assume \eqref{eq:cfg-convex-positive-separation} holds at some \(t_{\mathrm{infl}}\in[0,1)\), and suppose 
\(
\bx_{\tilde{t}_{\mathrm{conv}}}
\in
\tilde{t}_{\mathrm{conv}}B_{\epsilon_{\mathrm{infl}}(t_{\mathrm{infl}})}(\Conv(\mathcal D_c)),
\) for some
\(\tilde{t}_{\mathrm{conv}}\in[t_{\mathrm{infl}},1)\).  
Then
\begin{equation}\label{eq:cfg-convex-exponential-inflation}
\begin{aligned}
\epsilon_{\mathrm{infl}}(t)
\le
(w-1)
D_\emptyset
\tfrac{1-p_\emptyset(\mathcal D_c)}
{p_\emptyset(\mathcal D_c)}
\exp\left(
-\tfrac{t^2\Delta_{\mathrm{infl}}}{2(1-t)^2}
\right), \text{ for every  } t\in[\tilde t_{\mathrm{conv}},1)
\end{aligned}
\end{equation}
Consequently, there exists
\(t_{\mathrm{ref}}\in[\tilde{t}_{\mathrm{conv}},1)\)
such that, with
\[
0\le
\epsilon_{\mathrm{ref}}
:=
\epsilon_{\mathrm{infl}}(t_{\mathrm{ref}})
<\epsilon_{\mathrm{infl}}(t_{\mathrm{infl}}).
\]
Moreover, the following hold.
\begin{enumerate}
\item[(i)]
\emph{(Refined attraction)}
For every
\(t_{\mathrm{ref}}\le s\le t<1\),
\[
\dist\!\left(
\bx_t,
tB_{\epsilon_{\mathrm{ref}}}(\Conv(\mathcal D_c))
\right)
\le
\tfrac{1-t}{1-s}
\dist\!\left(
\bx_s,
sB_{\epsilon_{\mathrm{ref}}}(\Conv(\mathcal D_c))
\right).
\]
Hence the function
\(
t\longmapsto
\dist\!\left(
\bx_t,
tB_{\epsilon_{\mathrm{ref}}}(\Conv(\mathcal D_c))
\right)
\)
is non-increasing on \([t_{\mathrm{ref}},1)\).

\item[(ii)] Taking $s=t_{\mathrm{ref}}$ in part~(i) gives, for every
\(t\in[t_{\mathrm{ref}},1)\),
\[
\dist\!\left(
\bx_t,
tB_{\epsilon_{\mathrm{ref}}}(\Conv(\mathcal D_c))
\right)
\le
\tfrac{1-t}{1-t_{\mathrm{ref}}}
\dist\!\left(
\bx_{t_{\mathrm{ref}}},
t_{\mathrm{ref}}
B_{\epsilon_{\mathrm{ref}}}(\Conv(\mathcal D_c))
\right).
\]

\item[(iii)]
\emph{(Refined absorption)}
If there exists
\(\tilde{t}'_{\mathrm{conv}}\in[t_{\mathrm{ref}},1)\)
such that
\(
\bx_{\tilde{t}'_{\mathrm{conv}}}
\in
\tilde{t}'_{\mathrm{conv}}
B_{\epsilon_{\mathrm{ref}}}(\Conv(\mathcal D_c)),
\)
then
\(
\bx_t\in tB_{\epsilon_{\mathrm{ref}}}(\Conv(\mathcal D_c))
\)
for every \(t\in[\tilde{t}'_{\mathrm{conv}},1)\).
\end{enumerate}
\end{corollary}

For the proof, see
Appendix~\ref{proof:cor:cfg-convex-small}. Consequently,
when such \(t_{\mathrm{infl}}\) and \(\tilde t_{\mathrm{conv}}\) exist,
for sufficiently late \(t_{\mathrm{ref}}\), the tail of the trajectory \((\bx_t)_{t \in [t_{\mathrm{ref}},1)}\) is attracted to
and, after entrance, absorbed by a conditional convex-hull neighborhood
whose inflation radius is bounded at the explicit scale
\(\exp(-t_{\mathrm{ref}}^2\Delta_{\mathrm{infl}}/[2(1-t_{\mathrm{ref}})^2])\).
This exponential refinement is illustrated in
Figure~\ref{fig:cfg-stagewise-geometry}(b) and (c).

\subsection{Final-stage Attraction and Absorption by the Conditional Local Cluster}

The convex-hull neighborhood result explains how CFG remains controlled by conditional
geometry in the early and intermediate regimes. We next pass to the
terminal local-cluster regime. \(\hat{\by}_{t,\mathrm{cfg}}(\bx)\) differs from 
\(\hat{\by}_{t,{c}}(\bx)\) by
\((w-1)(\hat{\by}_{t,c}-\hat{\by}_t)\). The next theorem bounds this
quantity and the corresponding difference between \(\vc(t,\bx)\) and \(\vu(t,\bx)\)
at every \(\bx\in B_{tR}(t\Omega_c)\) under certain conditions 
stated below.

\begin{theorem}[Prediction-gap bound]\label{thm:prediction-gap}
Fix \(t_{\mathrm{gap}}\in(0,1)\) and \(R>0\). Let \(\Omega_c\) be a
closed subset of \(\mathcal D_c\) satisfying
\(p_\emptyset(\Omega_c)>0\), and write 
\(M_{\mathcal D}:=\max_{\by\in\mathcal D}\|\by\|,
\
D_{\Omega_c}:=\diam(\Omega_c) 
\) and 
\[
\gamma_{\Omega_c}
:=
\dist\bigl(B_R(\Omega_c),\mathcal D\setminus\Omega_c\bigr)^2
-(R+D_{\Omega_c})^2.
\]
If \(\gamma_{\Omega_c}>0\),
then for every \(t\in[t_{\mathrm{gap}},1)\) and
\(\bx\in B_{tR}(t\Omega_c)\) we have
\[
\left\|\hat{\by}_{t,c}(\bx) - \hat{\by}_{t}(\bx)\right\|
\le
2M_{\mathcal D}
\tfrac{1-p_\emptyset(\Omega_c)}{p_\emptyset(\Omega_c)}
\exp\left(
-\tfrac{t_{\mathrm{gap}}^2\gamma_{\Omega_c}}
{2(1-t)^2}
\right).
\]
Moreover, the prediction gap \(\bgc(t,\bx)\) satisfies the estimate
\[
\|\bgc(t,\bx)\| = \|\vc(t,\bx)-\vu(t,\bx)\|
\le
\tfrac{2M_{\mathcal D}}{1-t}
\tfrac{1-p_\emptyset(\Omega_c)}{p_\emptyset(\Omega_c)}
\exp\left(
-\tfrac{t_{\mathrm{gap}}^2\gamma_{\Omega_c}}
{2(1-t)^2}
\right).
\]
Consequently, both displayed upper bounds converge to zero as \(t\uparrow1\).
\end{theorem}
For the proof, see Appendix~\ref{proof:thm:prediction-gap}.

We now state the final-stage local-cluster result for the CFG ODE
\((\ref{eq:cfg flow})\). The CFG local-cluster assumptions are stated in
Appendix~\ref{app:cfg-proofs}; see
Assumption~\ref{assump:cfg-local-cluster}. They are mostly the conditional
counterparts of Assumption~\ref{assump:local-cluster}. In particular,
$\Omega_c$, $R$, $\rho_{\Omega_c}$, $D_{\Omega_c}$,
$\gamma_{\Omega_c}$, $\tilde c_{\Omega_c}$,
$\tilde\alpha_{\Omega_c}$, and $\tilde s_{\Omega_c}$ used below are
defined in Assumption~\ref{assump:cfg-local-cluster}.
Under Assumption~\ref{assump:cfg-local-cluster}, fix $w>1$.
Choose a final-stage starting time
$\tilde t_{\mathrm{loc}}\in(0,1)$ satisfying
\begin{equation}\label{eq:cfg-final-time-choice}
(1-\tilde t_{\mathrm{loc}})^2
\log\tfrac{1}{1-\tilde t_{\mathrm{loc}}}
\le
\min\left\{
\tfrac{1}{4(2\tilde\alpha_{\Omega_c}+1)},
\tfrac{(2R+1)\tilde s_{\Omega_c}}{2\tilde\alpha_{\Omega_c}+1}
\right\}.
\end{equation}
Such a choice exists because the left-hand side converges to zero as
$\tilde t_{\mathrm{loc}}\uparrow1$. For any such
$\tilde t_{\mathrm{loc}}$, define the CFG inflation radius
\begin{equation}\label{eq:cfg-final-inflation-radius}
\begin{aligned}
\tilde r_{\Omega_c}
:={}
\tfrac{\tilde\varepsilon_{\Omega_c}}
{\tilde{t}_{\mathrm{loc}}\sqrt{1-R/\rho_{\Omega_c}}}
+
\tfrac{D_{\Omega_c}}{\tilde c_{\Omega_c,\mathrm{loc}}}
\exp\!\left(
-\tfrac{\tilde\varepsilon_{\Omega_c}^2}
{4(1-\tilde{t}_{\mathrm{loc}})^2}
\right)
+
2wM_{\mathcal D}
\tfrac{1-p_\emptyset(\Omega_c)}{p_\emptyset(\Omega_c)}
\exp\!\left(
-\tfrac{\tilde t_{\mathrm{loc}}^2\gamma_{\Omega_c}}
{2(1-\tilde{t}_{\mathrm{loc}})^2}
\right),
\end{aligned}
\end{equation}
where
\begin{equation}\label{eq:cfg-final-quantities}
\begin{aligned}
\tilde\varepsilon_{\Omega_c}
&:=
2(1-\tilde{t}_{\mathrm{loc}})
\sqrt{
(2\tilde\alpha_{\Omega_c}+1)
\log\tfrac{1}{1-\tilde{t}_{\mathrm{loc}}}
}, \quad
\tilde c_{\Omega_c,\mathrm{loc}}
:=
\tilde c_{\Omega_c}(8R+4)^{-\tilde\alpha_{\Omega_c}}
\tilde\varepsilon_{\Omega_c}^{2\tilde\alpha_{\Omega_c}}.
\end{aligned}
\end{equation}
With $\Omega_c,\mathcal D,R,\rho_{\Omega_c},\tilde c_{\Omega_c}$,
$\tilde\alpha_{\Omega_c},\tilde s_{\Omega_c},p_\emptyset$, and $w$
fixed,
\(
\tilde r_{\Omega_c}
=
\bigO\!\left(
(1-\tilde t_{\mathrm{loc}})
\sqrt{\log\tfrac{1}{1-\tilde t_{\mathrm{loc}}}}
\right)
\ \text{as }\tilde t_{\mathrm{loc}}\uparrow1.
\)
A proof of this order is given in
Lemma~\ref{lem:cfg-polynomial-local-mass-radius}.

\begin{theorem}[CFG final-stage local-cluster attraction and absorption]\label{thm:cfg-final}
Under Assumption~\ref{assump:cfg-local-cluster}, fix $w>1$, let
$\tilde t_{\mathrm{loc}}$ and $\tilde r_{\Omega_c}$ be as above, and assume further that 
\begin{equation}\label{eq:cfg-final-capture-condition}
\dist\bigl(
\bx_{\tilde t_{\mathrm{loc}}},
\tilde t_{\mathrm{loc}}B_{\tilde r_{\Omega_c}}(\Omega_c)
\bigr)+\tilde r_{\Omega_c}
\le
\tilde t_{\mathrm{loc}}R.
\end{equation}
Let $\bx_t$ solve the CFG ODE (\ref{eq:cfg flow}) on
$[\tilde{t}_{\mathrm{loc}},1)$, 
then the following hold.
\begin{enumerate}
\item[(i)] \emph{(Attraction)} For every
$\tilde t_{\mathrm{loc}}\le s\le t<1$,
\[
\dist\bigl(\bx_t,tB_{\tilde r_{\Omega_c}}(\Omega_c)\bigr)
\le
\tfrac{1-t}{1-s}
\dist\bigl(\bx_s,sB_{\tilde r_{\Omega_c}}(\Omega_c)\bigr).
\]
Hence the function
$t\mapsto\dist\bigl(\bx_t,tB_{\tilde r_{\Omega_c}}(\Omega_c)\bigr)$
is non-increasing on $[\tilde{t}_{\mathrm{loc}},1)$.
\item[(ii)] Taking $s=\tilde t_{\mathrm{loc}}$ in part~(i) gives
\[
\dist\bigl(\bx_t,tB_{\tilde r_{\Omega_c}}(\Omega_c)\bigr)
\le
\tfrac{1-t}{1-\tilde{t}_{\mathrm{loc}}}
\dist\bigl(\bx_{\tilde{t}_{\mathrm{loc}}},\tilde{t}_{\mathrm{loc}}B_{\tilde r_{\Omega_c}}(\Omega_c)\bigr)
=\bigO(1-t),
\qquad
t\in[\tilde{t}_{\mathrm{loc}},1).
\]
\item[(iii)] \emph{(Absorption)} If there exists
\(\tilde{t}_{\Omega_c}\in[\tilde t_{\mathrm{loc}},1)\) such that
\(\bx_{\tilde{t}_{\Omega_c}}\in
\tilde{t}_{\Omega_c}B_{\tilde r_{\Omega_c}}(\Omega_c)\), then
\(\bx_t\in tB_{\tilde r_{\Omega_c}}(\Omega_c)\) for every
\(t\in[\tilde{t}_{\Omega_c},1)\).
\end{enumerate}
\end{theorem}
For the proof, see Appendix~\ref{proof:thm:cfg-final}.
Consequently, the same final-stage local-cluster mechanism governs the
unconditional flow ODE~(\ref{eq:uncond flow}), the conditional flow
ODE~(\ref{eq:cond flow}), and the CFG ODE~(\ref{eq:cfg flow}):
each trajectory is attracted at
rate \(\bigO(1-t)\) to a scaled inflated neighborhood of the corresponding
cluster, the associated inflation radii satisfy the same asymptotic upper
bound, and once a trajectory enters this neighborhood, it remains there.
See illustration in Figure~\ref{fig:cfg-stagewise-geometry}(d).

\begin{remark}[Discrete CFG statements]\label{rem:cfg-discrete-summary}
The CFG Euler counterparts are collected in Appendix~\ref{thm:cfg-euler-stagewise}. They use the same exact mechanism as the unconditional discrete arguments.
\end{remark}

\begin{remark} [Terminal-stage simplification for CFG sampling]
    Motivated by Theorem~\ref{thm:prediction-gap}, in the late stage one may replace the CFG vector field (\ref{eq:cfg field}) by the unconditional network field $\vtheta(t,\bx_t,\emptyset)$ or the conditional $\vtheta(t,\bx_t,c)$ with little change in the local dynamics, saving one network evaluation per step. Empirically, applying this switch sufficiently late has little effect on sample quality \citep{DBLP:conf/ecai/MalarzKZTS25,DBLP:conf/nips/KynkaanniemiAKL24}; see Figure~\ref{fig:cfg-switch-off-exp} and Table~\ref{tab:cfg-switch-off-fid} for experimental results.
\end{remark}

\section{Flow Matching with a General Time Schedule}\label{sec:gen-schedule}

Our previous sections suggest that the final-stage estimate is governed by \(1-t\). This motivates replacing the
standard time interpolation \((t, 1-t)\) by a more general time schedule \((a(t), 1-a(t))\), where
\begin{equation}\label{eq:general-schedule-assumption}
a\in C^1([0,1]),
\qquad
a(0)=0,
\qquad
a(1)=1,
\qquad
\dot a(t)>0 \quad \text{for }t\in[0,1).
\end{equation}
Thus $a(t)$ is strictly increasing on $[0,1)$. The identity choice $a(t)=t$ recovers standard training. 
For $\star\in\{\emptyset,c\}$, let $\rY_\emptyset\sim p_\emptyset$ and
$\rY_c\sim p_c$, and define
\begin{equation*}
\rX_{t,\star}^a
=
a(t)\rY_\star+(1-a(t))\bxi,
\qquad
\bxi\sim\mathcal N(\bzero,\Id).
\end{equation*}
The variables $t\sim\mathrm{Unif}[0,1]$, $\rY_\star$, and $\bxi$ are sampled
independently, and the training process now uses
\begin{equation}\label{eq:generalized-loss}
\mathcal L_{\star,a}(\theta)
=
\E_{t,\rY_\star,\bxi}
\left[
\left\|
\vtheta(t,\rX_{t,\star}^a,\star)
-\dot a(t)(\rY_\star-\bxi)
\right\|^2
\right].
\end{equation}
For CFG, set
$\mathcal L_{\mathrm{cfg},a} 
=q\mathcal L_{c,a}+(1-q)\mathcal L_{\emptyset,a}$.
The time change is therefore incorporated into both the interpolation states
and the regression targets, rather than being applied in sampling. As in Section~\ref{sec:cfm}, we use
$\mathcal L_{\star,a}(\bu)$ for the function-space counterpart of
\eqref{eq:generalized-loss}, obtained by replacing
$\vtheta(t,\cdot,\star)$ with $\bu(t,\cdot)$ for $\bu\in\mathcal V$.
The ideal field below refers to the minimizer of this function-space functional. 

\begin{proposition}[General-schedule dynamics]\label{prop:general-schedule}
Under \eqref{eq:general-schedule-assumption}, the following claims hold.
\begin{enumerate}
\item[(i)] For each $\star\in\{\emptyset,c\}$, the function-space loss
$\mathcal L_{\star,a}$ has a unique global minimizer
$\vv_\star^a\in\mathcal V$. Its canonical continuous extension to
$[0,1)\times\R^d$ satisfies
\begin{equation}\label{eq:general-field-identity}
\vv_\star^a(t,\bx)
=\dot a(t)\vv_\star(a(t),\bx),
\qquad t\in[0,1),\quad \bx\in\R^d.
\end{equation}
Here \(\vv_\star(\cdot,\cdot)\) are defined in \eqref{eq:uncond ideal field} and \eqref{eq:cond ideal field}. When the unconditional and conditional branches use the same schedule, their
ideal CFG extrapolation satisfies
\begin{equation*}
\vcfg^a(t,\bx)
:=(1-w)\vv_\emptyset^a(t,\bx)+w\vv_c^a(t,\bx)
=\dot a(t)\vcfg(a(t),\bx).
\end{equation*}

\item[(ii)] For every $T\in(0,1)$, the general-schedule unconditional flow,
conditional flow, and CFG ODEs admit unique $C^1$ solutions on $[0,T]$.
For any of these three dynamics, 
let $\bx_t$ and $\bx_t^a$ denote the
solutions of the standard \((a(t) = t)\) and general-schedule ODEs.
If they start from the same
initial condition, then their solutions satisfy
\begin{equation}\label{eq:general-trajectory-identity}
\bx_t^a=\bx_{a(t)},
\qquad t\in[0,1).
\end{equation}

\item[(iii)] Fix $t_0\in(0,1)$. If the corresponding
standard unconditional trajectory satisfies the final-stage setup and
assumptions of Theorem~\ref{thm:final-local} with
$t_{\mathrm{loc}}=a(t_0)$, and $r'_\Omega$ denotes the corresponding
inflation radius, then for
$t_0\le s\le t<1$, $\dist\bigl(\bx_t^a,a(t)B_{r'_\Omega}(\Omega)\bigr)$ is non-increasing in $t$ and satisfies
\begin{equation}\label{eq:general-final-ratio}
\dist\bigl(\bx_t^a,a(t)B_{r'_\Omega}(\Omega)\bigr)
\le
\tfrac{1-a(t)}{1-a(s)}
\dist\bigl(\bx_s^a,a(s)B_{r'_\Omega}(\Omega)\bigr).
\end{equation}
For each fixed $s\in[t_0,1)$, this gives
$\dist\bigl(\bx_t^a,a(t)B_{r'_\Omega}(\Omega)\bigr)=\bigO(1-a(t))$ as \(t\uparrow1\).
If the trajectory enters
$a(t_\Omega^a)B_{r'_\Omega}(\Omega)$ at some $t_\Omega^a\in[t_0,1)$, it remains in
$a(t)B_{r'_\Omega}(\Omega)$ for every later \(t\in[t_\Omega^a,1)\). A similar result holds for the conditional and CFG ODEs under the corresponding final-stage assumptions.
\end{enumerate}
\end{proposition}
For the proof, see Appendix~\ref{app:gen-schedule-proofs}.
Hence the general schedule does not create a new ideal spatial curve; it
changes the clock used to travel the curve. This can be viewed as a
time reparameterization incorporated at the training side of the generative
model.
Next, we show a specific example of a nonlinear schedule that gives a
quadratic terminal rate with respect to $t$.

\paragraph{Example.}
For $a(t)=t+t^2-t^3$, one has
\begin{equation*}
\dot a(t)=(1-t)(1+3t)>0,
\qquad
1-a(t)=(1-t)^2(1+t)\sim2(1-t)^2
\quad\text{as }t\uparrow1.
\end{equation*}
Proposition~\ref{prop:general-schedule} therefore yields
\begin{equation*}
\dist\bigl(\bx_t^a,a(t)B_{r'_\Omega}(\Omega)\bigr)
=
\bigO((1-t)^2)
\qquad\text{as }t\uparrow1.
\end{equation*}
Thus the final-stage moving distance decays quadratically in the remaining
time $1-t$ at the ideal continuous level. The consequences for
finite-step and adaptive numerical solvers are examined in
Section~\ref{exp:generalized schedule}.

\section{Experiments}\label{sec:experiments}

We test the particle-level mechanisms predicted by the theory in both
synthetic and image-generation settings. Similar synthetic trajectory-based experiments with exact empirical posterior means
 have been considered in \citet{wan2025elucidating} and \citet{liu2025pdeperspectivegenerativediffusion}.
 In the two-dimensional synthetic
experiments, the data laws are uniform empirical measures. The unconditional
law is $p_\emptyset=N^{-1}\sum_{n=1}^N\delta_{\by^{(n)}}$; when conditioning is
present, the conditional law is
$p_c=N_c^{-1}\sum_{n=1}^{N_c}\delta_{\by_c^{(n)}}$. Here, $N$ and $N_c$ are the
respective numbers of data points, and $\delta_{\by}$ denotes the Dirac measure at $\by$. For either law, we compute posterior means by exact finite sums
and obtain the ideal fields from the posterior-mean
representations~\eqref{eq:uncond ideal field} and~\eqref{eq:cond ideal field}.
Synthetic trajectories then follow explicit Euler,
$\bz_{i+1}=\bz_i+\Delta t\,\vv_\star(t_i,\bz_i)$, as in
\eqref{eq:ideal-euler}, directly probing the ideal discrete dynamics analyzed
in our theory. In the image-generation experiments, the ideal fields are
replaced by the learned networks in the parametrized ODEs~\eqref{eq:ufm-ode},
\eqref{eq:cfm-ode}, and~\eqref{eq:cfg-ode}; sampling uses explicit Euler unless
another solver is specified. Finally, we test a general time schedule with
both fixed-step and adaptive ODE solvers.

\subsection{Convex-hull Attraction and Absorption}

In a two-dimensional synthetic example, we take $p_\emptyset$ to be the
uniform empirical distribution on the nonconvex point cloud $\mathcal D$
shown in Figure~\ref{fig:convex-hull-exp}. Starting from the displayed point
$\bz_0$ at time $t_0=0.3$, we simulate the explicit-Euler unconditional dynamics on the
shifted grid $t_i=t_0+i\Delta t$ with explicit Euler and stepsize
$\Delta t=0.05$, and track
$\dist(\bz_i,t_i\Conv(\mathcal D))$. Figure~\ref{fig:convex-hull-exp}
illustrates the attraction and absorption in
Theorem~\ref{cor:convex-hull-euler}: the distance decreases to zero when the
trajectory first enters $t_i\Conv(\mathcal D)$ and remains zero afterward.

\begin{figure}[h]
\centering
\begin{minipage}{0.50\textwidth}
\centering
\includegraphics[width=\linewidth]{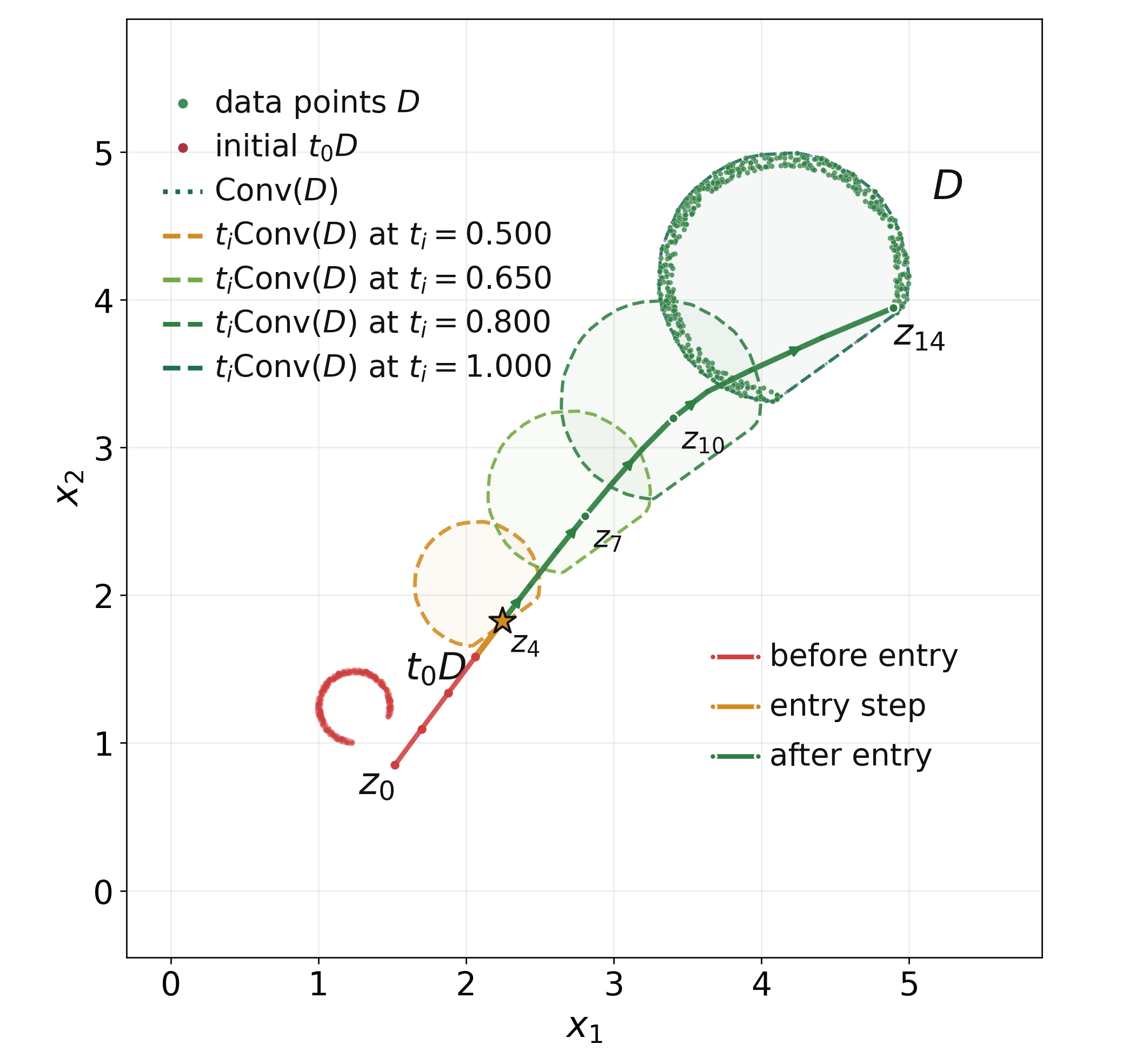}
\end{minipage}
\hfill
\begin{minipage}{0.45\textwidth}
\centering
\includegraphics[width=\linewidth]{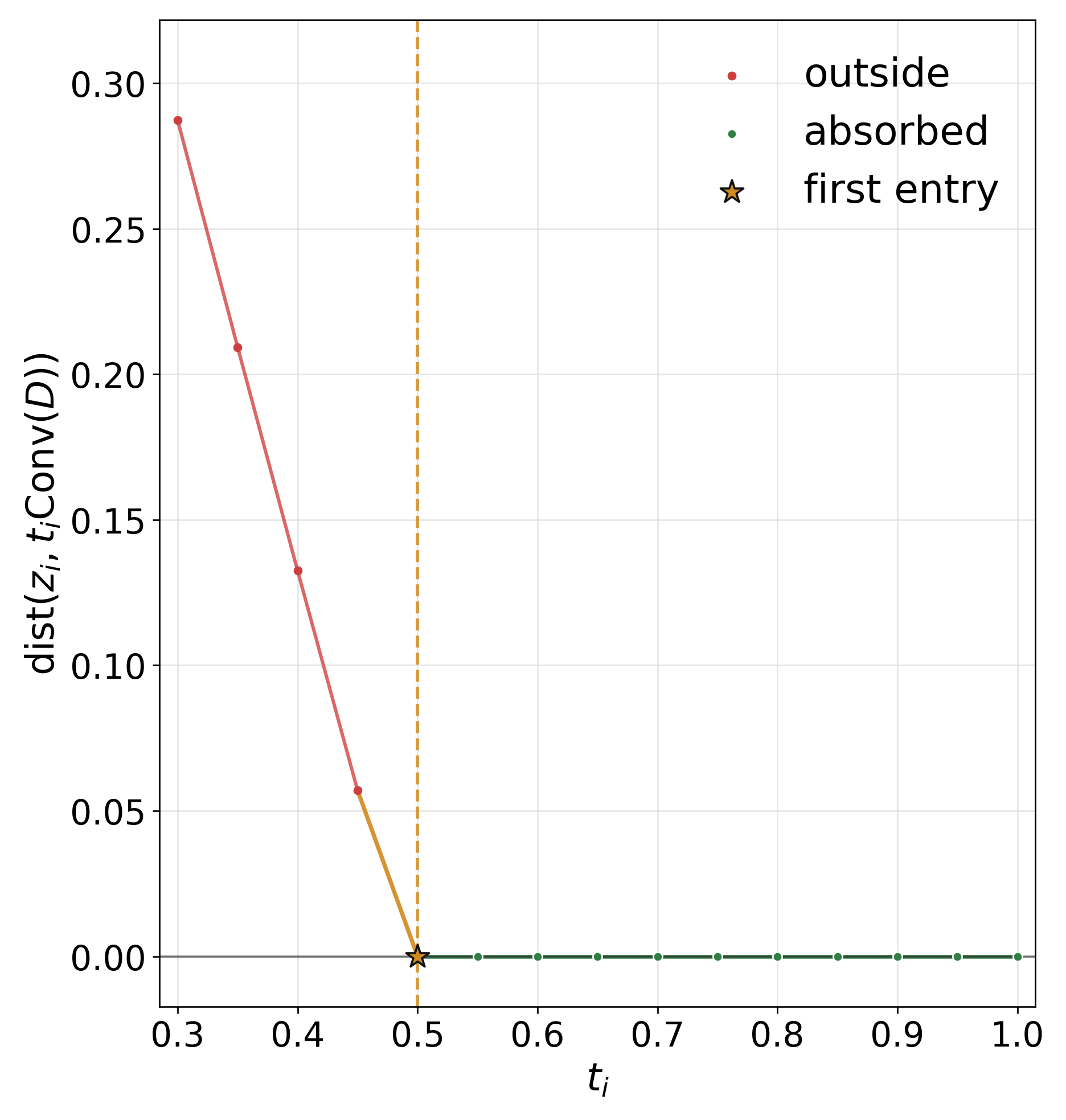}
\end{minipage}
\caption{Convex-hull attraction and absorption. Left: the Euler trajectory
and selected scaled convex-hull snapshots; the star marks the first entry into
$t_i\Conv(\mathcal D)$, and the dashed curves show the moving convex hull at
selected grid times. Right: $\dist(\bz_i,t_i\Conv(\mathcal D))$ decreases to
zero and remains zero after the first entry.}
\label{fig:convex-hull-exp}
\end{figure}

\subsection{Final-stage Local-cluster Attraction and Absorption}

We next take $\mathcal D=\bigcup_{j=1}^5\Omega_j$, where $\Omega_j$ is the $j$-th nonconvex point clusters shown in
Figure~\ref{fig:local-cluster-exp}, and let $p_\emptyset$ be the uniform
empirical distribution on their union. Starting at time $t_0=0.4$, we
simulate five explicit-Euler trajectories $\{\bz_i^j\}_{i=0}^{12}, j=1,\ldots,5$,  one initialized near
each scaled target cluster, with $\Delta t=0.05$. Using the empirically chosen display
radius $r=0.025$, we track
$\dist(\bz_i^j,t_iB_r(\Omega_j))$. Figure~\ref{fig:local-cluster-exp}
illustrates the discrete final-stage attraction and absorption in
Theorem~\ref{cor:final-local-euler}: each distance decreases to zero at the first entry and remains zero afterward.
Importantly, the relevant geometry is the scaled cluster $t_i\Omega_j$, rather than the unscaled cluster $\Omega_j$. For example, for $j=4$, $\dist(z_0^4,\Omega_1)<\dist(z_0^4,\Omega_4)$, whereas $\dist(z_0^4,t_0\Omega_4)<\dist(z_0^4,t_0\Omega_1)$ with $t_0=0.4$; accordingly, $\{z_i^4\}_{i=0}^{12}$ is attracted to and absorbed by $t_iB_r(\Omega_4)$.
The displayed radius is used only to
visualize this mechanism and is not the analytically constructed radius in the
theorem.

\begin{figure}[h]
\centering
\begin{minipage}{0.5\textwidth}
\centering
\includegraphics[width=\linewidth]{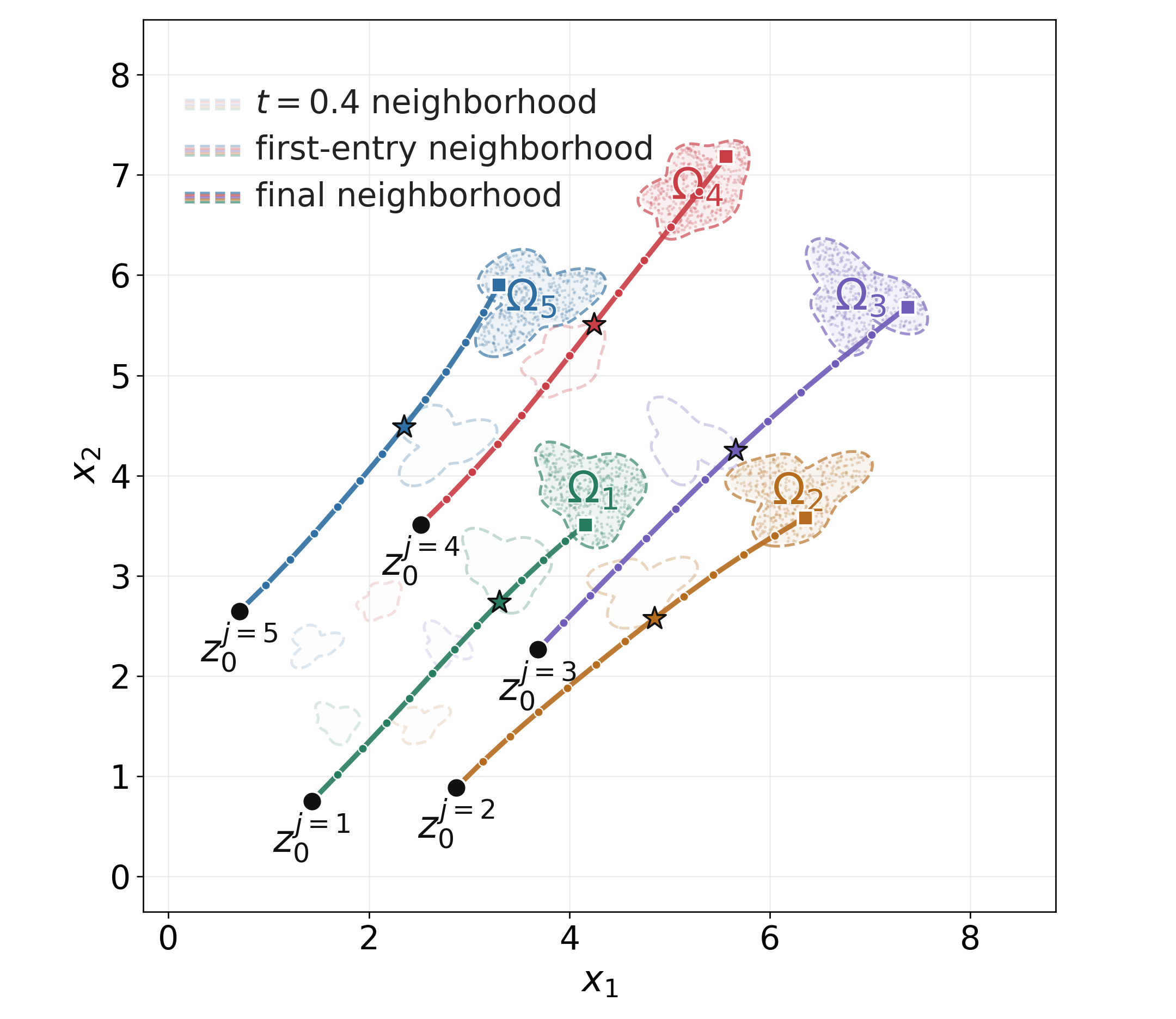}
\end{minipage}
\hfill
\begin{minipage}{0.45\textwidth}
\centering
\includegraphics[width=\linewidth]{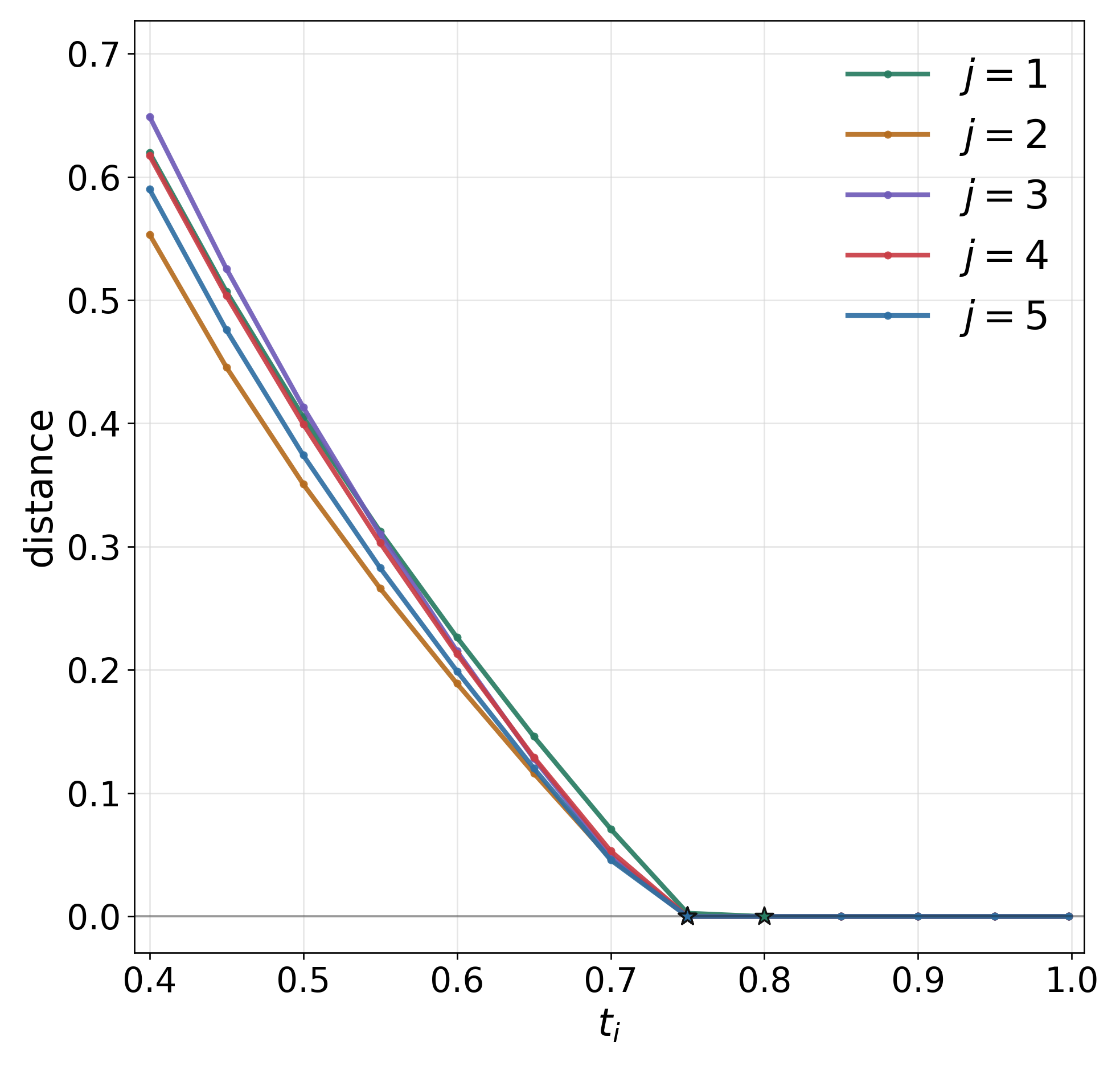}
\end{minipage}
\caption{Final-stage local-cluster attraction and absorption. Left: five
Euler trajectories enter their corresponding scaled cluster neighborhoods;
stars mark the first entry, squares mark the final positions, and dashed contours show
the initial, first-entry, and final scaled-neighborhood snapshots. Right:
$\dist(\bz_i^j,t_iB_r(\Omega_j))$ decreases to zero and remains zero after
first entry. 
}
\label{fig:local-cluster-exp}

\vspace{-1cm}
\end{figure}

\subsection{Classifier-free Guidance Dynamics}

We next test the two CFG-specific dynamical effects from
Section~\ref{sec:cfg} in the same synthetic geometry. We retain the uniform
empirical unconditional distribution on
$\mathcal D=\bigcup_{j=1}^5\Omega_j$, set $\mathcal D_c=\Omega_5$, and take
$p_c$ to be the uniform empirical distribution on $\Omega_5$. All displayed
CFG Euler trajectories start from the common point $\bz_0=(-1,0)$. The left
panel of Figure~\ref{fig:cfg-trajectory-final-exp} plots the full trajectories
for $w\in\{1,1.5,2,2.5,3\}$. Increasing $w$ moves the extrapolated mean
$\bar{\by}_{\mathrm{cfg}}=\bar{\by}+w(\bar{\by}_c-\bar{\by})$ from the
unconditional mean toward and beyond the conditional mean, while all curves subsequently turn toward $\Omega_5$ as
the terminal dynamics recover conditional local-cluster geometry, in accordance
with the discrete early- and final-stage results in
Theorems~\ref{thm:cfg-early-euler} and~\ref{thm:cfg-final-euler}, respectively. The right panel
quantifies entry into and absorption by $t_iB_{0.05}(\Omega_5)$.

Figure~\ref{fig:cfg-trajectory-final-exp} therefore previews both CFG regimes
examined below. Figures~\ref{fig:cfg-early-exp} and
\ref{fig:cfg-imagenet-mean-exp} develop the early mechanism, while
Figures~\ref{fig:cfg-gap-exp} and~\ref{fig:cfg-switch-off-exp} examine the
terminal prediction-gap decay and its image-generation consequence.

\begin{figure}[!htbp]
\centering
\begin{minipage}{0.50\textwidth}
\centering
\includegraphics[width=\linewidth]{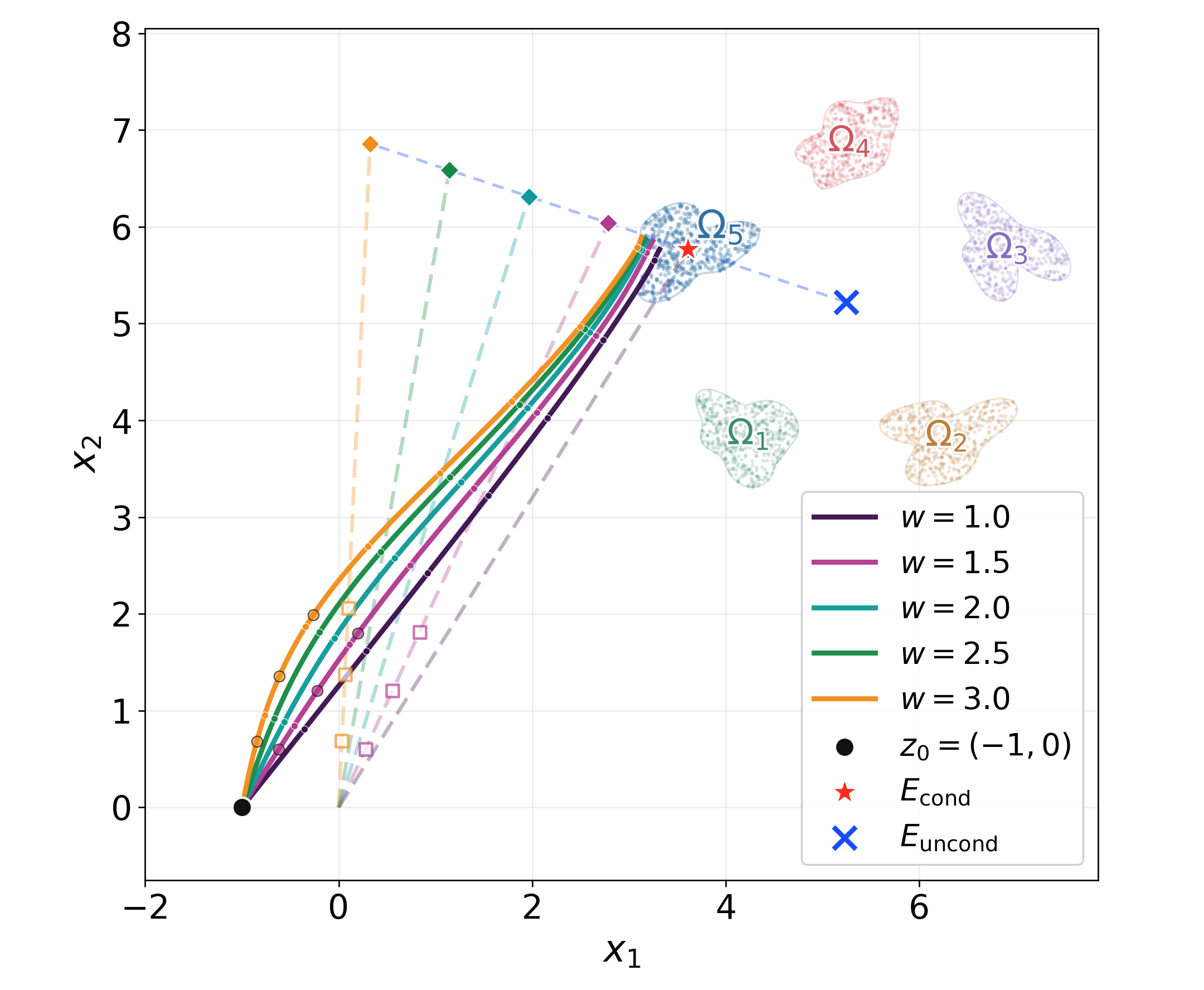}
\end{minipage}
\hfill
\begin{minipage}{0.45\textwidth}
\centering
\includegraphics[width=\linewidth]{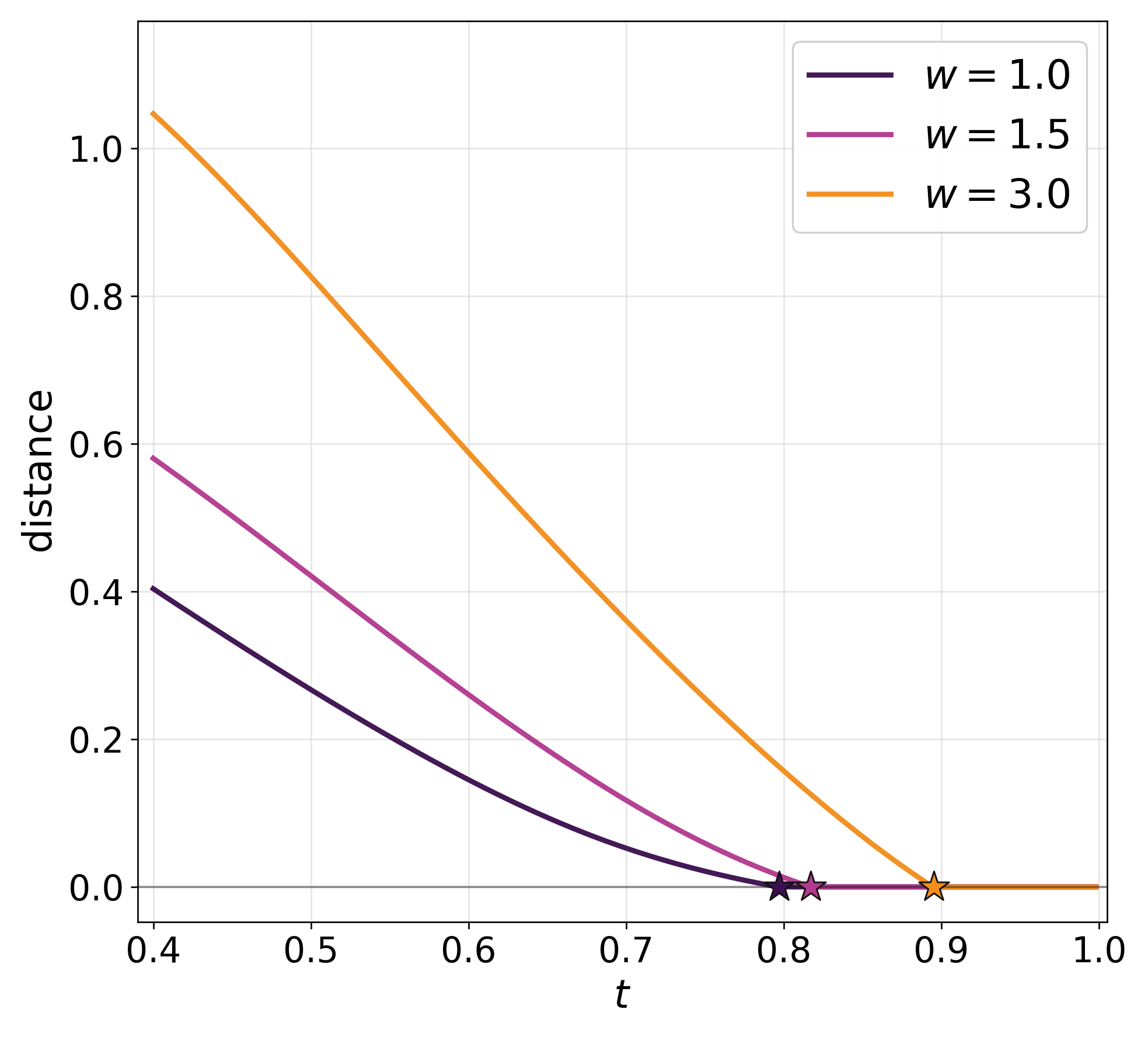}
\end{minipage}
\caption{CFG trajectory geometry. Left: larger $w$ shifts the extrapolated mean
$\bar{\by}_{\mathrm{cfg}}$ and redirects the early trajectory, while the full
trajectories subsequently turn toward $\Omega_5$; dashed rays indicate scaled
extrapolated means, circles mark trajectory snapshots, and squares mark
scaled-mean snapshots. Right: In the final stage, $\dist(\bz_i,t_iB_{0.05}(\Omega_5))$ 
decreases to zero and remains zero after
first entry. 
}
\label{fig:cfg-trajectory-final-exp}

\vspace{-1cm}
\end{figure}

\subsubsection{Early Attraction to the Extrapolated Mean}
Here, we track $\dist\bigl(\bz_i,t_iB_r(\bar{\by}_{\mathrm{cfg}})\bigr)$ for $w=1.5$ and $w=3.0$ with $r=0.05$. In Figure~\ref{fig:cfg-early-exp}, both curves decrease and stay below the linear reference $(1-t_i)\|\bx_0\|$. 
These observations are consistent with the discrete early-stage result in Theorem~\ref{thm:cfg-early-euler}.

\begin{figure}[!htbp]
\centering
\begin{minipage}{0.48\textwidth}
\centering
\includegraphics[width=\linewidth]{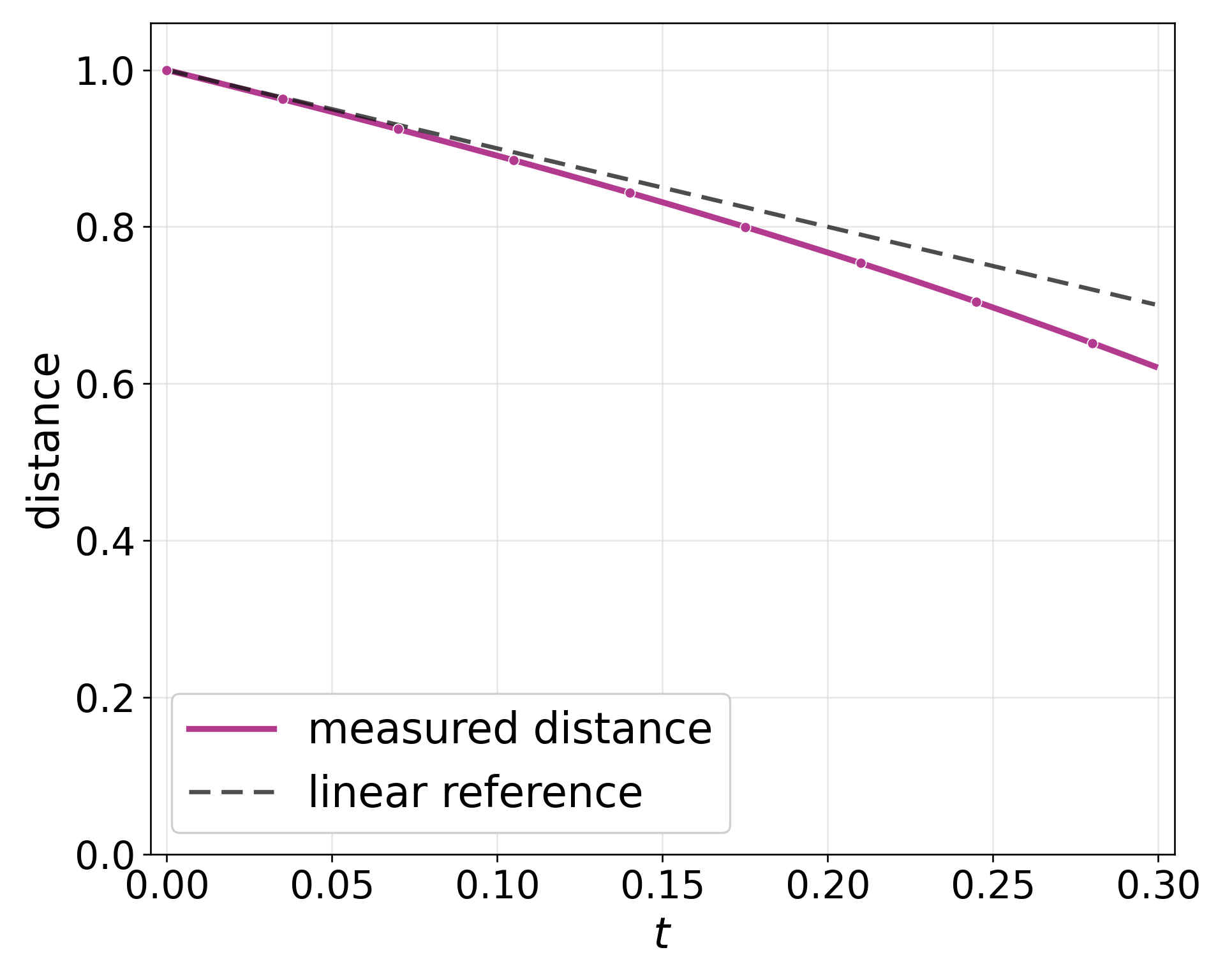}
\end{minipage}
\hfill
\begin{minipage}{0.48\textwidth}
\centering
\includegraphics[width=\linewidth]{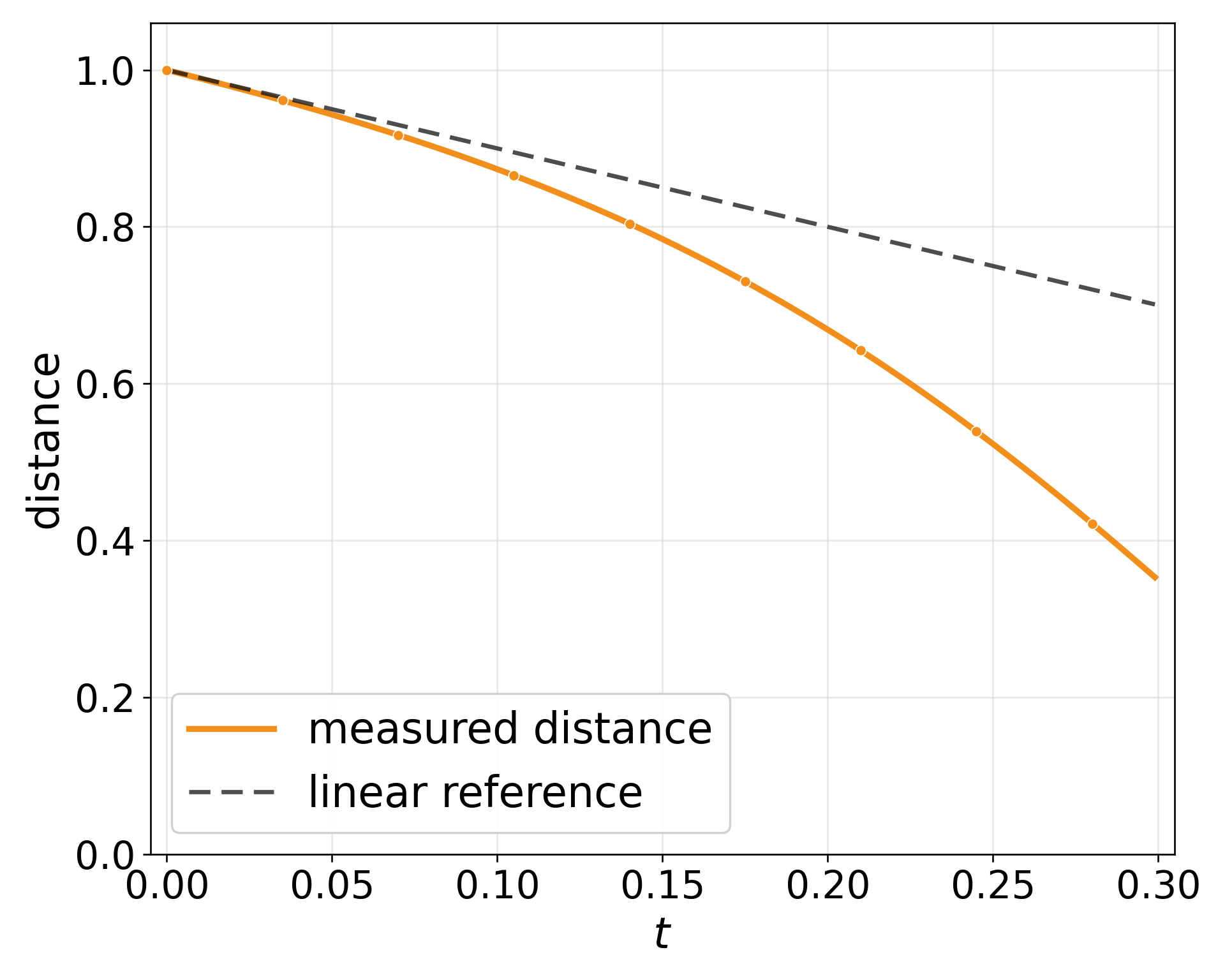}
\end{minipage}
\caption{Early-stage CFG attraction to the extrapolated mean with radius $r=0.05$. 
Left: $w=1.5$. Right: $w=3.0$. Here, $\dist(\bz_i,t_iB_r(\bar{\by}_{\mathrm{cfg}}))$ decreases and stays below the linear reference $(1-t_i)\|\bx_0\|$. }
\label{fig:cfg-early-exp}

\vspace{-1cm}
\end{figure}

\subsubsection{Prediction-gap decay and final cluster attraction and absorption}

We next measure the prediction gap $\|\vu(t_i,\bz_i)-\vc(t_i,\bz_i)\|$ along CFG Euler trajectories and fit the terminal values by the form $\tfrac{C_1}{1-t_i}\exp\!\left(-\frac{C_2}{(1-t_i)^2}\right)$ suggested by the upper bound in Theorem~\ref{thm:prediction-gap}.
Figure~\ref{fig:cfg-gap-exp} shows that the fitted asymptotic form captures the measured prediction gap for both guidance scales. Near the target cluster, the CFG extrapolation term becomes negligible, so the late-stage dynamics reduce to conditional local-cluster behavior.
The right panel of Figure~\ref{fig:cfg-trajectory-final-exp} shows the corresponding attraction and absorption: for each $w$,  $\dist(\bz_i,t_iB_{0.05}(\Omega_5))$ decays to zero and remains zero after first entry, consistent with the discrete CFG final-stage result in Theorem~\ref{thm:cfg-final-euler}.

\begin{figure}[h]
\centering
\begin{minipage}{0.48\textwidth}
\centering
\includegraphics[width=\linewidth]{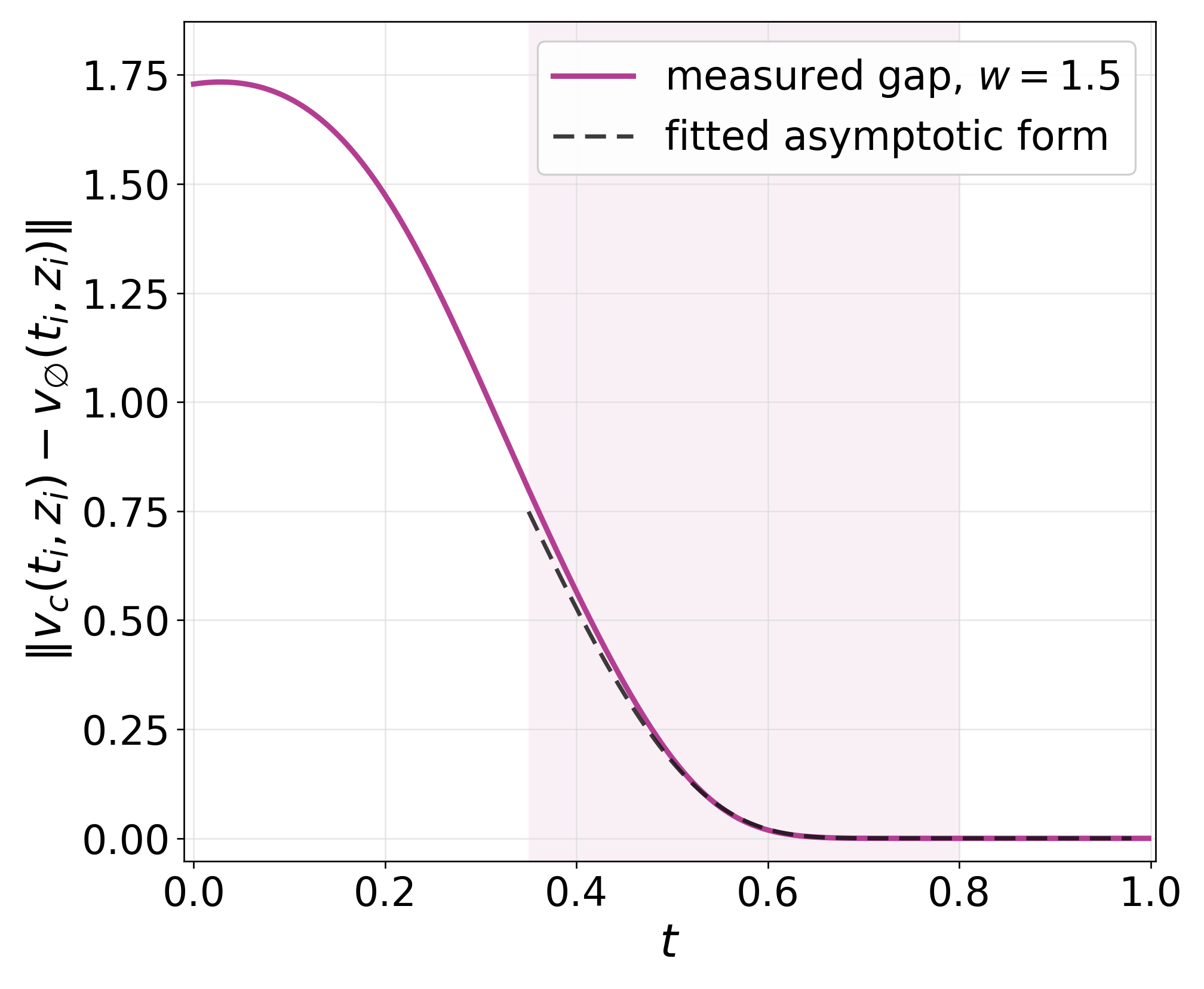}
\end{minipage}
\hfill
\begin{minipage}{0.48\textwidth}
\centering
\includegraphics[width=\linewidth]{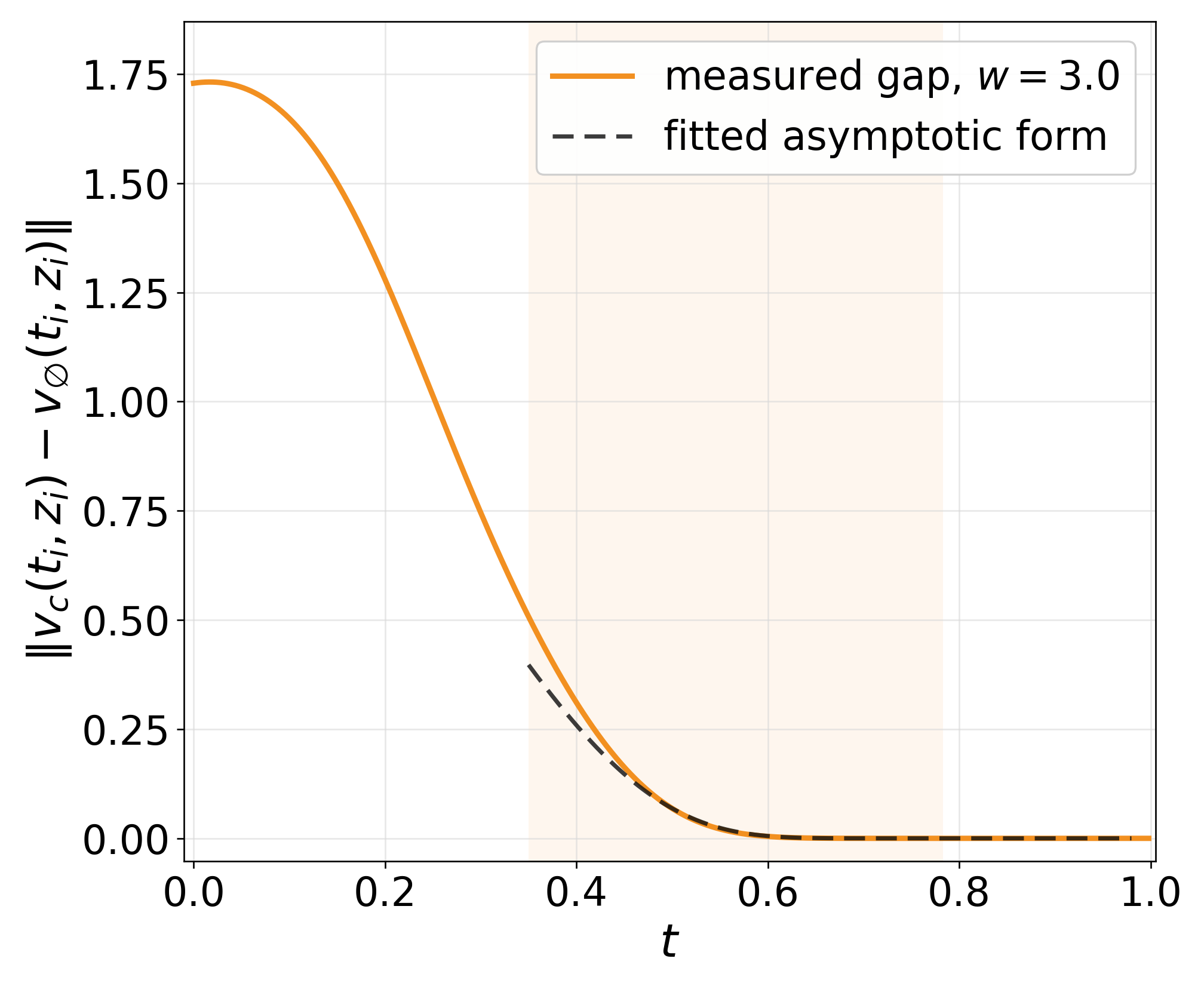}
\end{minipage}
\caption{Final-stage Prediction-gap decay. The measured gaps are well fit by the asymptotic form $\frac{C_1}{1-t_i}\exp(-\frac{C_2}{(1-t_i)^2})$ for $w=1.5$ and $w=3.0$; the shaded interval indicates the grid-time window used for the fit.}
\label{fig:cfg-gap-exp}

\vspace{-1cm}
\end{figure}

\subsubsection{Image-generation setting}

\paragraph{Early-stage extrapolated mean.} 
We visualize the same effect on ImageNet-256 goldfinch generation. We compute the empirical unconditional mean, the goldfinch class mean, and the extrapolated means $\bar{\by}_{\mathrm{cfg}}^w=\bar{\by}_{\emptyset}+w(\bar{\by}_c-\bar{\by}_{\emptyset})$. Figure~\ref{fig:cfg-imagenet-mean-exp} shows that increasing $w$ amplifies class-specific structure in the attractor and in samples generated from the same Gaussian initialization. Moderate guidance improves class alignment, while too small guidance gives poor class signal and very large guidance produces over-saturation.

\begin{figure}[h]
\centering
\setlength{\tabcolsep}{1pt}
\begin{tabular}{@{}cccccc@{}}
\includegraphics[width=0.155\textwidth]{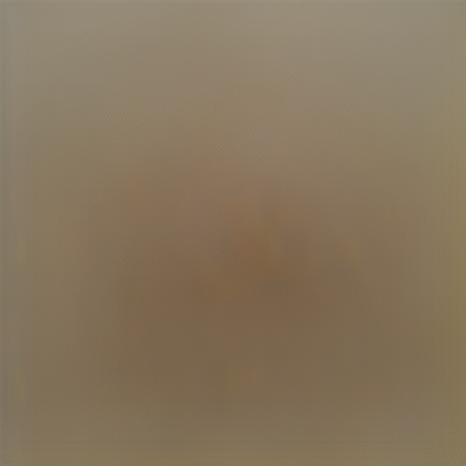} &
\includegraphics[width=0.155\textwidth]{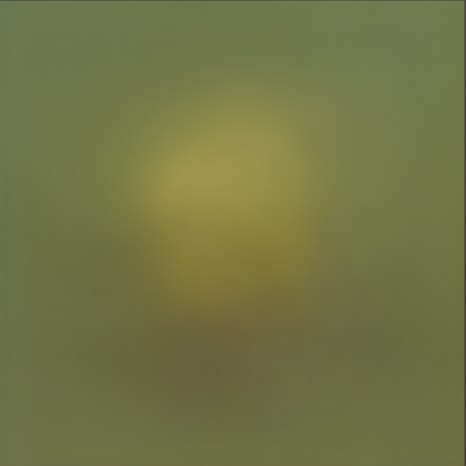} &
\includegraphics[width=0.155\textwidth]{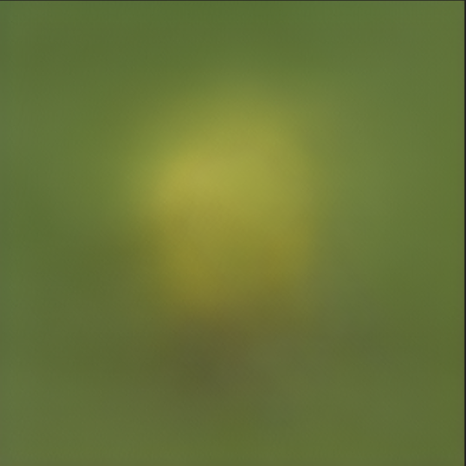} &
\includegraphics[width=0.155\textwidth]{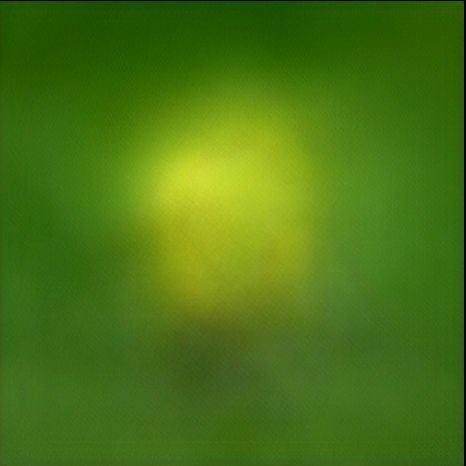} &
\includegraphics[width=0.155\textwidth]{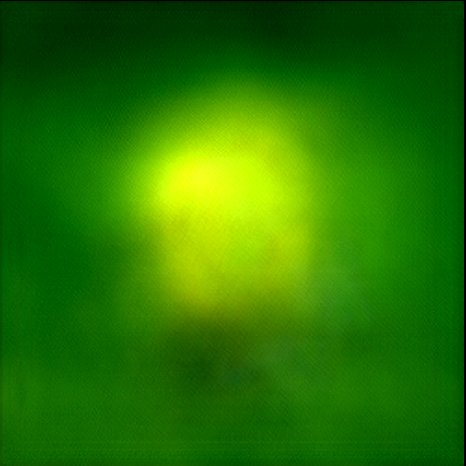} &
\includegraphics[width=0.155\textwidth]{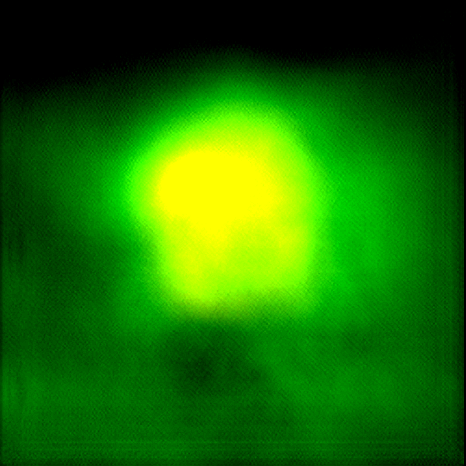} \\
\small unconditional mean &
\small class mean &
\small $E_{\mathrm{cfg}}^{1.5}$ &
\small $E_{\mathrm{cfg}}^{3}$ &
\small $E_{\mathrm{cfg}}^{5}$ &
\small $E_{\mathrm{cfg}}^{10}$
\end{tabular}

\vspace{0.4em}

\begin{tabular}{@{}ccccc@{}}
\includegraphics[width=0.18\textwidth]{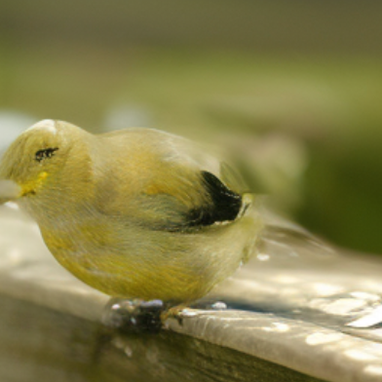} &
\includegraphics[width=0.18\textwidth]{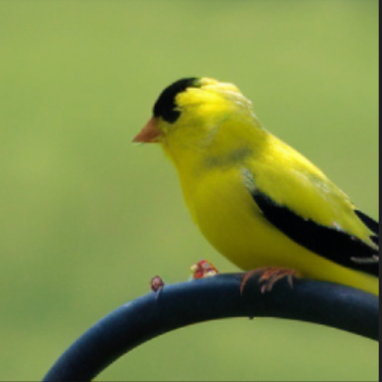} &
\includegraphics[width=0.18\textwidth]{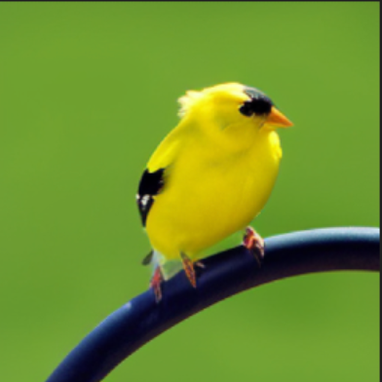} &
\includegraphics[width=0.18\textwidth]{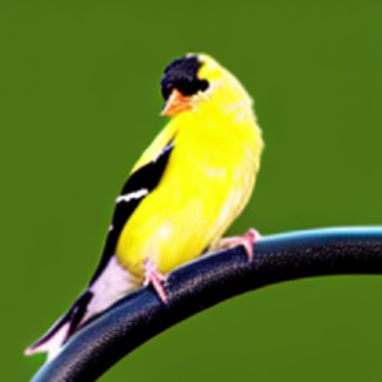} &
\includegraphics[width=0.18\textwidth]{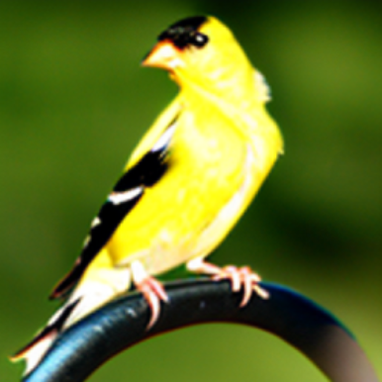} \\
\small $w=1.0$ &
\small $w=1.5$ &
\small $w=3$ &
\small $w=5$ &
\small $w=10$
\end{tabular}
\caption{ImageNet-256 goldfinch early-attractor visualization. Top: unconditional, class, and extrapolated means. Bottom: samples from the same initialization. Larger $w$ amplifies class-specific structure, while very large $w$ over-saturates the output.}
\label{fig:cfg-imagenet-mean-exp}

\vspace{-1cm}
\end{figure}

\paragraph{Late-stage CFG switch-off.} On ImageNet-256 with DiT-XL/2-256 \citep{peebles2023scalable}, $w=2$, and 20 sampling steps, we remove the extrapolation term $w(\vtheta(\cdot,\cdot,c)-\vtheta(\cdot,\cdot,\emptyset))$ after step $k\in\{6,12,16,20\}$ and finish with $\vu$ alone.
 Figure~\ref{fig:cfg-switch-off-exp} shows that removal after 6 steps weakens performance, while removal after 12 or 16 steps remains fairly close to full CFG. Table~\ref{tab:cfg-switch-off-fid} gives the corresponding FID values on 50k generated images: removal after 12 or 16 steps nearly matches full CFG, whereas removal after 6 steps is noticeably worse.
 This behavior is consistent with Theorem~\ref{thm:prediction-gap}: once the prediction gap has decayed in the late stage, this removal has little effect, whereas removing it too early has a significant impact.

\begin{figure}[h]
\centering
\setlength{\tabcolsep}{6pt}
\begin{tabular}{@{}cccc@{}}
\includegraphics[width=0.235\textwidth]{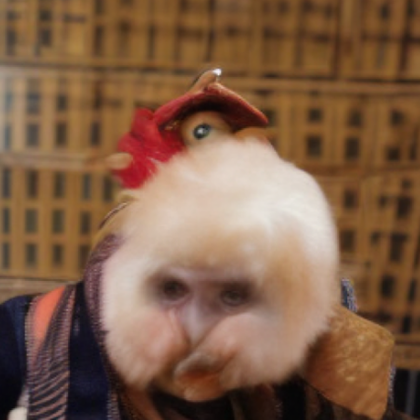} &
\includegraphics[width=0.235\textwidth]{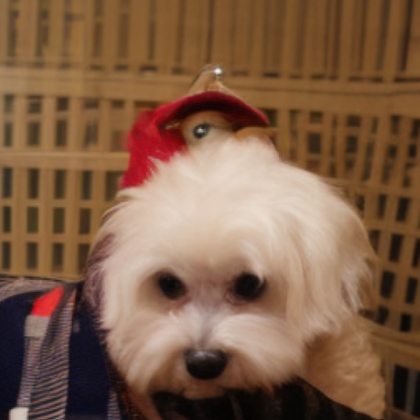} &
\includegraphics[width=0.235\textwidth]{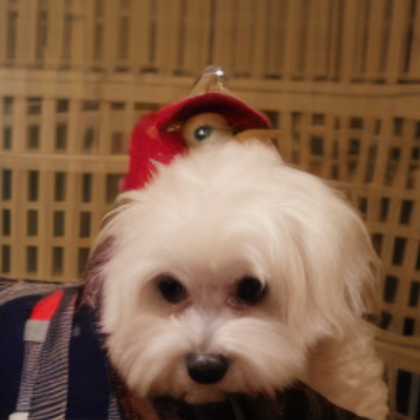} &
\includegraphics[width=0.235\textwidth]{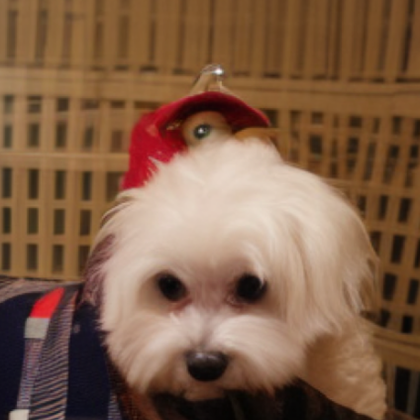} \\
\includegraphics[width=0.235\textwidth]{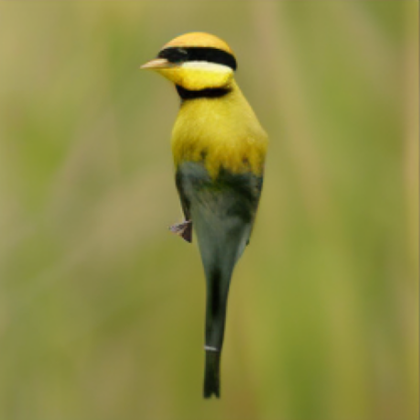} &
\includegraphics[width=0.235\textwidth]{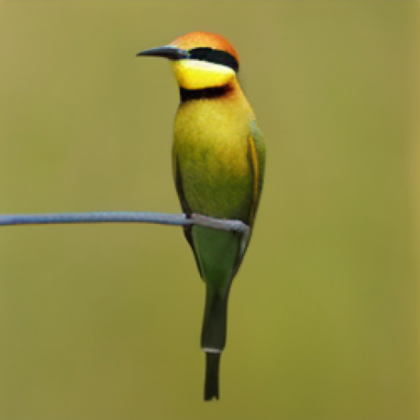} &
\includegraphics[width=0.235\textwidth]{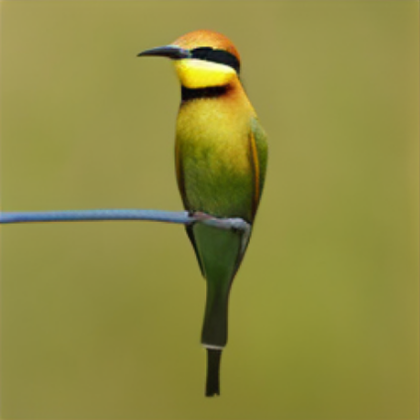} &
\includegraphics[width=0.235\textwidth]{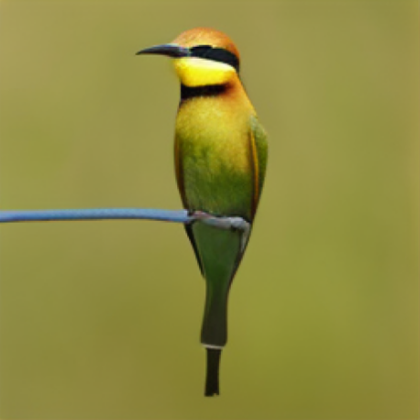} \\
\small after 6 &
\small after 12 &
\small after 16 &
\small full CFG
\end{tabular}
\caption[Late-stage switch-off of CFG on ImageNet-256.]{Late-stage removal of CFG extrapolation on ImageNet-256 with $w=2$ and 20 steps. After the indicated step, we remove $w(\vc-\vu)$ and finish with $\vu$; removing too early weakens performance, while removing after 12 or 16 steps stays close to the full CFG.}
\label{fig:cfg-switch-off-exp}

\vspace{-1cm}
\end{figure}

\begin{table}[h]
\centering
\label{tab:cfg-switch-off-fid}
\begingroup
\footnotesize
\setlength{\tabcolsep}{2pt}
\begin{tabular*}{0.62\textwidth}{@{\extracolsep{\fill}}lcccc@{}}
\toprule
\textbf{Switch-off} & \textbf{after 6} & \textbf{after 12} & \textbf{after 16} & \textbf{full CFG} \\
\textbf{NFE} & 26 & 32 & 36 & 40 \\
\midrule
\textbf{FID $\downarrow$} & 8.45 & 6.24 & 6.23 & 6.24 \\
\bottomrule
\end{tabular*}
\endgroup
\caption[Late-stage CFG switch-off FID.]{FID $\downarrow$ for late-stage CFG switch-off on ImageNet-256 using DiT-XL/2-256, $w=2$, 20 steps, step size \(1/20\), and 50k generated images. }
\end{table}

\subsection{Flow Matching with a General Time Schedule}\label{exp:generalized schedule}

\paragraph{Compared constructions.}
 Let $\theta_A$ denote the parameters obtained by
standard training with $a(t)=t$, and let $\theta_C$ denote those obtained by
schedule $a(t)=t+t^2-t^3$ training. The network architecture and training budget are
the same. On the uniform grid $t_i=i\Delta t\in[0,1]$, we compare three
fixed-step CFG samplers:
\begin{itemize}
\item \textbf{Scheme A} uses the standard objective~\eqref{eq:cfg-loss} and
applies explicit Euler to the parametrized CFG field in~\eqref{eq:cfg-ode}.
Its update is
$\bz_{i+1}^{A}=\bz_i^{A}+\Delta t\,\boldsymbol
v_{\mathrm{cfg},\theta_A}(t_i,\bz_i^{A},c)$.
\item \textbf{Scheme B} reuses the standard-trained network $\theta_A$ but
applies sampling-only reparameterization on the mapped nonuniform grid
$a(t_i)$. Its update is
$\bz_{i+1}^{B}=\bz_i^{B}+\bigl(a(t_{i+1})-a(t_i)\bigr)\boldsymbol
v_{\mathrm{cfg},\theta_A}(a(t_i),\bz_i^{B},c)$.
\item \textbf{Scheme C} uses direct general-schedule training with
\eqref{eq:generalized-loss} and samples on the uniform $t$-grid. Its update is
$\bz_{i+1}^{C}=\bz_i^{C}+\Delta t\,\boldsymbol
v_{\mathrm{cfg},\theta_C}(t_i,\bz_i^{C},c)$.
\end{itemize}
At the function-space ideal level,
Proposition~\ref{prop:general-schedule}(i)--(ii) gives
$\vcfg^a(t,\bx)=\dot a(t)\vcfg(a(t),\bx)$ and
$\bx_t^a=\bx_{a(t)}$. Thus, the standard and general-schedule ideal ODEs trace
the same spatial curve, parameterized by $t$ and $a(t)$, respectively. Meanwhile, the
learned implementations and their Euler discretizations need not coincide.

\paragraph{Euler interpretation.}
At the ideal-field level, Scheme B has the Euler increment
$\bigl(a(t_{i+1})-a(t_i)\bigr)\vcfg(a(t_i),\bz_i)$.
For Scheme C, Proposition~\ref{prop:general-schedule}(i) yields the calculation
\begin{equation*}
\Delta t\,\vcfg^a(t_i,\bz_i)
=\Delta t\,\dot a(t_i)\vcfg(a(t_i),\bz_i).
\end{equation*}
If $a\in C^2([0,1])$, then
\(a(t_{i+1})-a(t_i)
=
\dot a(t_i)\Delta t+\bigO(\Delta t^2).\)
Thus, the ideal Scheme B and Scheme C Euler increments are first-order
equivalent, with a per-step $\bigO(\Delta t^2)$ discrepancy. For the learned
fields, the schemes may differ additionally because $\theta_A$ and $\theta_C$
are obtained from different training objectives. We next examine the
fixed-step and adaptive-solver settings.

\paragraph{Fixed-step results.}
Figure~\ref{fig:gen-schedule-frames} compares Schemes A, B, and C under the same
base uniform $t$-grid with $\Delta t=1/105$. Comparing to Scheme A, Scheme C reaches recognizable
low-smoothing states at smaller $t$. This is consistent with
Proposition~\ref{prop:general-schedule}(ii): at the ideal continuous level,
$\bx_t^a=\bx_{a(t)}$. For $a(t)=t+t^2-t^3$,
$1-a(t)=(1+t)(1-t)^2=\bigO((1-t)^2)$ as $t\uparrow1$.
Thus, its residual smoothing scale is asymptotically smaller than the standard
schedule's $1-t$. Figure~\ref{fig:gen-schedule-frames} is a finite-network, finite-step illustration of
this ideal scaling.
\begin{figure}[h]
\centering
\setlength{\tabcolsep}{2pt}
\begin{tabular}{@{}r@{\hspace{0.35em}}c@{}}
\raisebox{1.35ex}{\small Scheme A} &
\includegraphics[width=0.78\textwidth]{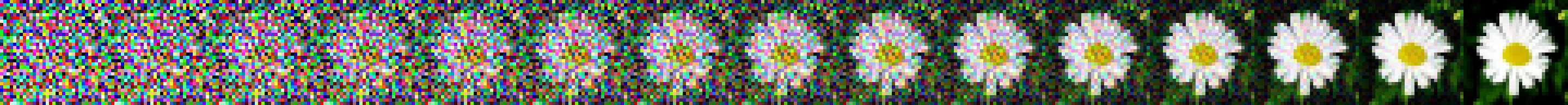}\\[-0.25em]
\raisebox{1.35ex}{\small Scheme B} &
\includegraphics[width=0.78\textwidth]{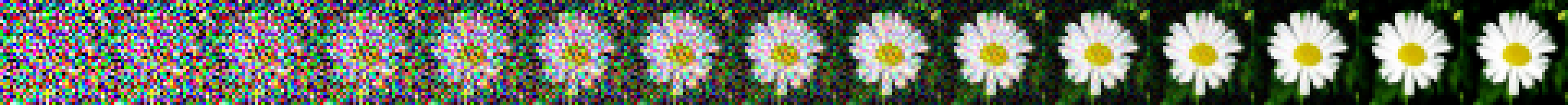}\\[-0.25em]
\raisebox{1.35ex}{\small Scheme C} &
\includegraphics[width=0.78\textwidth]{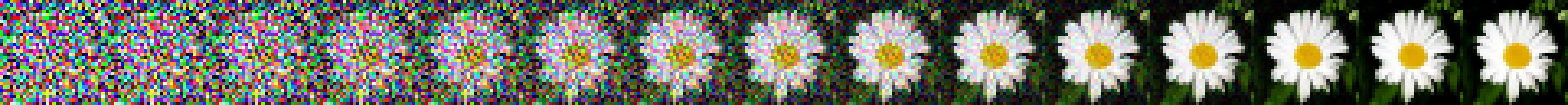}\\[0.25em]
\raisebox{1.35ex}{\small Scheme A} &
\includegraphics[width=0.78\textwidth]{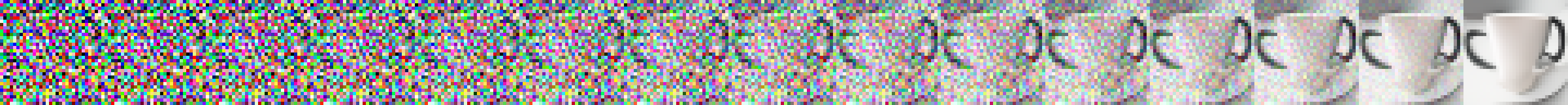}\\[-0.25em]
\raisebox{1.35ex}{\small Scheme B} &
\includegraphics[width=0.78\textwidth]{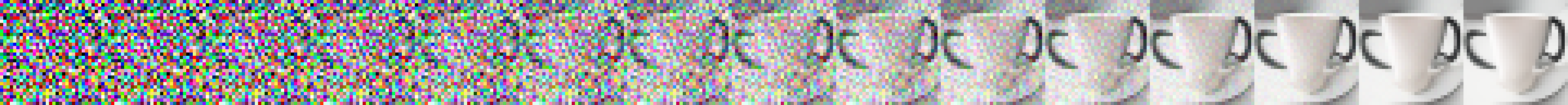}\\[-0.25em]
\raisebox{1.35ex}{\small Scheme C} &
\includegraphics[width=0.78\textwidth]{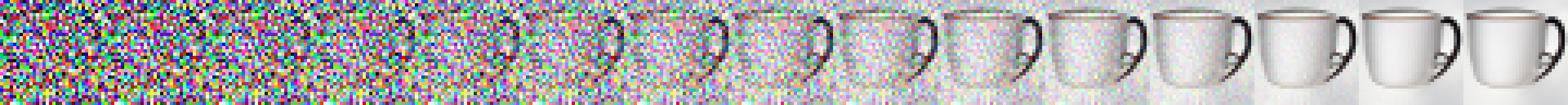}\\[0.25em]
\raisebox{1.35ex}{\small Scheme A} &
\includegraphics[width=0.78\textwidth]{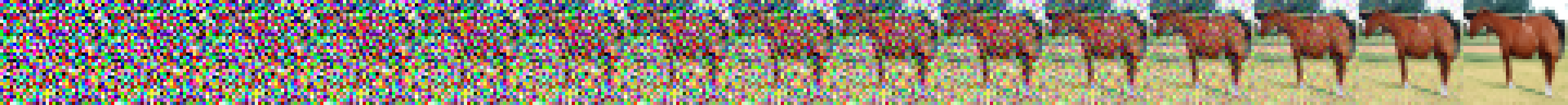}\\[-0.25em]
\raisebox{1.35ex}{\small Scheme B} &
\includegraphics[width=0.78\textwidth]{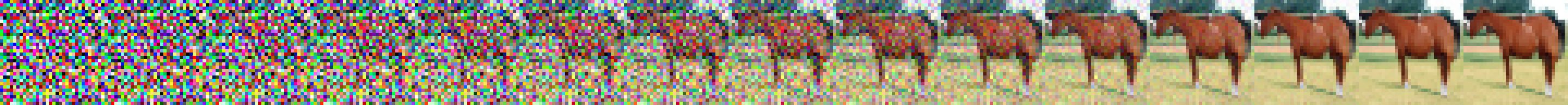}\\[-0.25em]
\raisebox{1.35ex}{\small Scheme C} &
\includegraphics[width=0.78\textwidth]{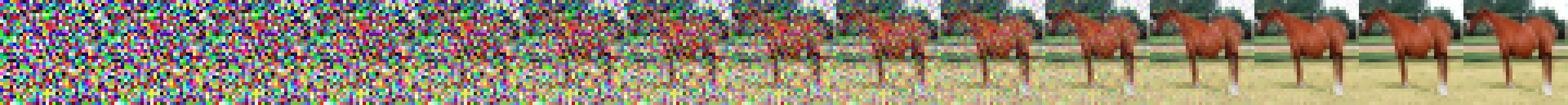}
\end{tabular}
\caption{Intermediate denoising states for Schemes A, B, and C under Euler
sampling with fixed base step size $\Delta t=1/105$. The three blocks show
daisy, cup, and horse trajectories, respectively; each row contains 15 frames.}
\label{fig:gen-schedule-frames}

\vspace{-1cm}
\end{figure}

Table~\ref{tab:gen-schedule-fid} compares FID on MNIST \citep{lecun1998gradient},
CIFAR-10 \citep{krizhevsky2009learning}, and ImageNet $32\times32$
\citep{deng2009imagenet,chrabaszcz2017downsampled} using 50k generated images in each setting. Under the same NFE and guidance scale,
Scheme B gives higher FID than Scheme A in all three reported settings, whereas
Scheme C gives the lowest FID in every case. Thus, the observed FID improvement
appears only when the nonlinear schedule is incorporated into training;
applying the same reparameterization only at sampling does not produce the
improvement in these experiments.
\begin{table}[h]
\centering
\footnotesize
\setlength{\tabcolsep}{4pt}
\label{tab:gen-schedule-fid}
\begin{tabular}{@{}lccccc@{}}
\toprule
\textbf{Dataset} & \textbf{NFE} & \textbf{$w$}
& A: standard & B: sampling reparam. & \textbf{C: training reparam.} \\
\midrule
MNIST & 60 & 1.5 & 10.24 & 10.67 & \textbf{9.78} \\
CIFAR-10 & 100 & 1.1 & 5.20 & 5.38 & \textbf{4.85} \\
ImageNet $32\times32$ & 100 & 1.2 & 5.05 & 5.45 & \textbf{4.93} \\
\bottomrule
\end{tabular}
\caption{Ablation of training and sampling reparameterization. FID $\downarrow$ for (A) standard training and Euler sampling, (B) standard training with reparameterization in sampling, and (C) direct general-schedule training with Euler sampling.}
\vspace{-0.5cm}
\end{table}

\paragraph{Adaptive solvers.}
DOPRI5 is an adaptive Runge--Kutta ODE solver that selects step sizes to meet an
error tolerance. Adaptive solvers can produce different NFE and wall-clock
costs under different time parameterizations even when the corresponding ideal
ODEs trace the same spatial curve. Figure~\ref{fig:gen-schedule-dopri5} compares these two directly trained models corresponding to Schemes A and C, using
$a(t)=t$ and $a(t)=t+t^2-t^3$, respectively. At the same solver tolerances, the
nonlinear-schedule model reduces both mean NFE/image and mean wall-clock
time/image across all six reported classes, with reductions of approximately
$19\%$--$27\%$.
\begin{figure}[h]
\centering
\includegraphics[width=0.98\textwidth]{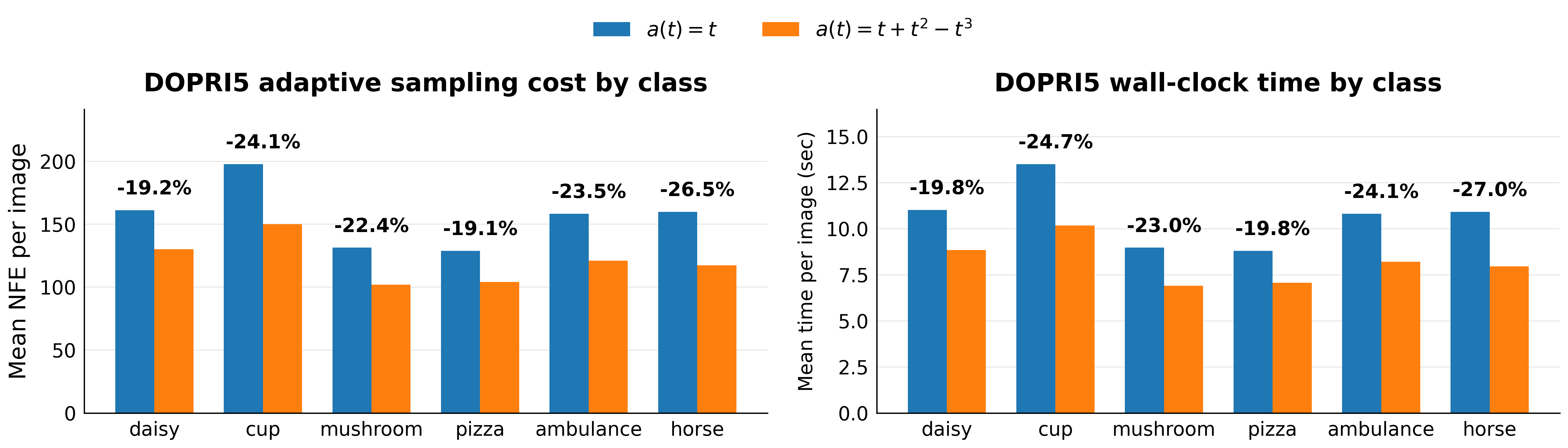}
\caption[DOPRI5 adaptive-solver cost.]{DOPRI5 adaptive-solver cost by class for
the directly trained models corresponding to Schemes A and C. Relative to
$a(t)=t$, the nonlinear schedule $a(t)=t+t^2-t^3$ reduces both mean NFE/image
and mean wall-clock time/image across all reported classes.
}
\label{fig:gen-schedule-dopri5}

\vspace{-1cm}
\end{figure}

Table~\ref{tab:gen-schedule-dopri5-quality} gives the corresponding aggregate comparison across the six classes. For the standard and nonlinear schedules,
respectively, FID is $5.74$ versus $5.67$, mean NFE/image is $156.21$ versus
$120.84$, and mean time/image is $10.67$ s versus $8.19$ s. Therefore in the
reported setting, the nonlinear schedule has a lower measured adaptive solver
cost together with a slightly lower reported FID.
\begin{table}[h]
\centering
\footnotesize
\label{tab:gen-schedule-dopri5-quality}
\setlength{\tabcolsep}{6pt}
\begin{tabular}{@{}lccc@{}}
\toprule
\textbf{Schedule}
& \textbf{FID $\downarrow$}
& \textbf{Mean NFE/image $\downarrow$}
& \textbf{Mean time/image (s) $\downarrow$} \\
\midrule
$a(t)=t$
& 5.74
& 156.21
& 10.67 \\
$a(t)=t+t^2-t^3$
& \textbf{5.67}
& \textbf{120.84}
& \textbf{8.19} \\
\bottomrule
\end{tabular}
\caption{Quality in FID($\downarrow$) and sampling cost under DOPRI5. Both rows use the same DOPRI5 solver tolerances
($\mathrm{atol}=\mathrm{rtol}=10^{-5}$) and guidance scale $w=3.5$.}
\end{table}

\section{Conclusion}\label{sec:conclusion}
This paper developed a particle-level geometric theory for flow matching and classifier-free guidance. For unconditional flow matching, we established a stagewise description in which trajectories were successively attracted toward scale-dependent geometries associated with the global mean, the data convex hull, and, under suitable local conditions, a possibly nonconvex local cluster. For CFG, we showed that guidance modified the early and intermediate attractors, while its correction became exponentially small near a well-isolated target cluster and the terminal dynamics recovered the local geometry of conditional flow matching. We further established exact discrete counterparts for the explicit Euler scheme. As a further consequence, we showed that a general time schedule reparameterized the ideal continuous trajectory without changing its spatial path.

Several directions remain for future work. A natural next step is an end-to-end theory that explains the transitions between geometric stages and the selection of the eventual local cluster without assuming capture in advance. Another important direction is to extend the analysis from ideal to learned velocity fields and quantify how approximation errors perturb the attraction and absorption mechanisms. The stage-dependent CFG behavior also motivates a principled characterization of when guidance can be weakened or removed during sampling.

 \acks{}
This work was partially supported by the National Natural Science Foundation of China under Grant 12571564, Guangdong Basic and Applied Research Foundation 2024A1515012347, Hong Kong Research Grant Council (HKRGC) GRF grants 16307325, 16306124, and 16307023.

\appendix

\section{Guide to the Appendix}\label{app:guide}
The appendices follow the proof dependencies shown in
Figure~\ref{fig:appendix-roadmap}. Appendix~\ref{app:cfm-project} establishes
regularity of the ideal vector fields and existence and uniqueness of the
corresponding ODE solutions, and Appendix~\ref{app:shared-lemmas}
collects the shared continuous and discrete tools.
Appendix~\ref{app:uncond-proofs} proves the unconditional stagewise results and
their Euler counterparts, Appendix~\ref{app:cfg-proofs} proves the CFG results,
and Appendix~\ref{app:gen-schedule-proofs} treats the general time schedule.
The continuous proofs use moving-distance differential estimates, while the
discrete proofs rely on the exact geometry of the Euler update.

\begin{figure}[h]
\centering
\begin{tikzpicture}[
>=Latex,
font=\footnotesize,
box/.style={
draw=black,
rounded corners=4pt,
thick,
align=center,
fill=blue!6,
text width=0.78\textwidth,
inner sep=5pt
},
auxbox/.style={
box,
fill=blue!6
},
arrow/.style={
-{Latex[length=2.5mm]},
thick
}
]
\node[auxbox] (wellposed) {Appendix~\ref{app:cfm-project}: ideal flow ODEs\\
vector-field regularity and existence--uniqueness};
\node[box, below=0.75cm of wellposed] (aux) {Appendix~\ref{app:shared-lemmas}: auxiliary tools\\
continuous scaled-distance calculus and discrete Euler attraction};
\node[box, below=0.75cm of aux] (uncond) {Appendix~\ref{app:uncond-proofs}: unconditional flow ODE analysis\\
early-stage mean attraction $\rightarrow$ convex-hull attraction / absorption $\rightarrow$ final local-cluster results};
\node[box, below=0.75cm of uncond] (cfg) {Appendix~\ref{app:cfg-proofs}: CFG ODE analysis\\
extrapolated mean attraction $\rightarrow$ conditional convex-hull neighborhoods $\rightarrow$ prediction-gap decay $\rightarrow$ CFG final local-cluster results };
\node[box, below=0.75cm of cfg] (genschedule) {Appendix~\ref{app:gen-schedule-proofs}: general time schedule\\
 };
\draw[arrow] (wellposed) -- (aux);
\draw[arrow] (aux) -- (uncond);
\draw[arrow] (uncond) -- (cfg);
\draw[arrow] (cfg) -- (genschedule);
\end{tikzpicture}
\caption{Appendix-proof dependency map. The appendix is read from top to bottom.}
\label{fig:appendix-roadmap}
\end{figure}
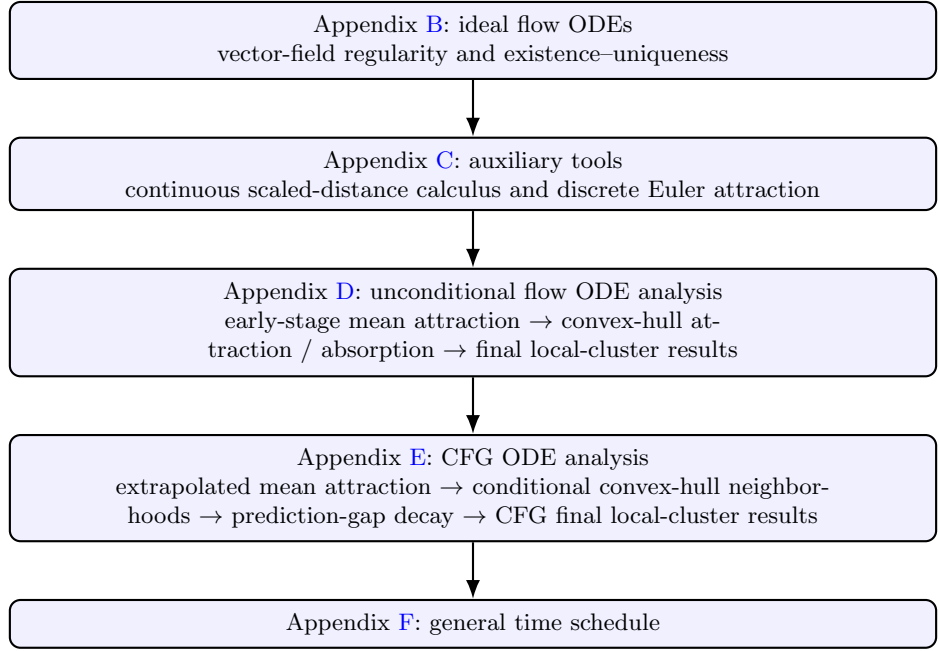

\section{Regularity and Existence--Uniqueness of the Ideal Flow ODEs}\label{app:cfm-project}

\subsection{Proof of Theorem~\ref{thm:wellposed-flow}}\label{proof:thm:wellposed-flow}
\begin{proof}
(i)
Fix $T \in (0,1)$ and define, for $(t,\bx)\in[0,T]\times\R^d$,
$w_t(\bx,\by):=\exp\left(-\frac{\|\bx-t\by\|^2}{2(1-t)^2}\right).$
For the unconditional field, write
$Z_t(\bx):=\int_{\mathcal{D}} w_t(\bx,\by)p_\emptyset(\dd \by),$ $A_t(\bx):=\int_{\mathcal{D}} \by\,w_t(\bx,\by)p_\emptyset(\dd \by),$
then $\hat{\by}_t(\bx)=A_t(\bx)/Z_t(\bx)$. Since $\mathcal{D}$ is bounded and $T<1$, the functions $w_t(\bx,\by)$ and their first partial derivatives $\partial_t w_t(\bx,\by)$ and $\nabla_{\bx}w_t(\bx,\by)$ are uniformly bounded in $\by\in \mathcal{D}$ on every compact subset of $[0,T]\times\R^d$. Therefore differentiation under the integral sign \citep{folland1999real} shows that $A_t(\bx)$ and $Z_t(\bx)$ are continuously differentiable in $(t,\bx)$. Because $Z_t(\bx)>0$, the quotient formula implies that $\hat{\by}_t(\bx)$, and hence $\vu(t,\bx)=(\hat{\by}_t(\bx)-\bx)/(1-t)$, is continuously differentiable in $(t,\bx)$ on $[0,T]\times\R^d$.

We now compute the Jacobian. Differentiating $Z_t$ and $A_t$ gives
\begin{align*}
\nabla_{\bx} Z_t(\bx)
&=
-\tfrac{1}{(1-t)^2}
\int_{\mathcal{D}}
(\bx-t\by)\,w_t(\bx,\by)p_\emptyset(\dd \by), \\
\nabla_{\bx} A_t(\bx)
&=
-\tfrac{1}{(1-t)^2}
\int_{\mathcal{D}}
\by(\bx-t\by)^\top w_t(\bx,\by)p_\emptyset(\dd \by).
\end{align*}
Using the posterior notation induced by $\eta_t^{\bx}$, these identities become
\begin{align*}
\frac{\nabla_{\bx} Z_t(\bx)}{Z_t(\bx)}
=
-\frac{\bx-t\hat{\by}_t(\bx)}{(1-t)^2}, \quad
\frac{\nabla_{\bx} A_t(\bx)}{Z_t(\bx)}
=
-\frac{\hat{\by}_t(\bx)\bx^\top - t\int_{\mathcal{D}} \by\by^\top\eta_t^{\bx}(\dd \by)}{(1-t)^2}.
\end{align*}
Applying the quotient rule to $\hat{\by}_t(\bx)=A_t(\bx)/Z_t(\bx)$ yields
\begin{align*}
\nabla_{\bx}\hat{\by}_t(\bx)
&=
\frac{\nabla_{\bx}A_t(\bx)}{Z_t(\bx)}
- \hat{\by}_t(\bx)\left(\frac{\nabla_{\bx}Z_t(\bx)}{Z_t(\bx)}\right)^\top \notag\\
&=
\tfrac{t}{(1-t)^2}
\left(
\int_{\mathcal{D}} \by\by^\top\eta_t^{\bx}(\dd \by)
- \hat{\by}_t(\bx)\hat{\by}_t(\bx)^\top
\right) =
\tfrac{t}{(1-t)^2}\Cov_{\eta_t^{\bx}}(\by).
\end{align*}
Therefore
\begin{equation*}
\mathsf{D}_{\bx}\vu(t,\bx)
=
\frac{\nabla_{\bx}\hat{\by}_t(\bx)-\Id}{1-t}
=
-\frac{1}{1-t}\Id
+ \frac{t}{(1-t)^3}\Cov_{\eta_t^{\bx}}(\by).
\end{equation*}

The conditional field is handled in exactly the same way, with $p_\emptyset$ replaced by $p_c$ and $\eta_t^{\bx}$ replaced by $\eta_{t,c}^{\bx}$, giving
\begin{equation*}
\mathsf{D}_{\bx}\vc(t,\bx)
=
-\frac{1}{1-t}\Id
+ \frac{t}{(1-t)^3}\Cov_{\eta_{t,c}^{\bx}}(\by).
\end{equation*}
Since $\eta_t^{\bx}$ and $\eta_{t,c}^{\bx}$ are supported in $\mathcal{D}$, for every unit vector $\boldsymbol{u}\in\R^d$ we have
\begin{equation*}
\boldsymbol{u}^\top \Cov_{\eta_t^{\bx}}(\by)\boldsymbol{u}
=
\int_{\mathcal{D}}
\bigl\langle \boldsymbol{u},\by-\hat{\by}_t(\bx)\bigr\rangle^2
\eta_t^{\bx}(\dd \by)
\le
\diam(\mathcal{D})^2,
\end{equation*}
and the same bound holds for $\Cov_{\eta_{t,c}^{\bx}}(\by)$. Hence
\begin{equation*}
\|\Cov_{\eta_t^{\bx}}(\by)\|_{\rm op},
\ \|\Cov_{\eta_{t,c}^{\bx}}(\by)\|_{\rm op}
\le
\diam(\mathcal{D})^2,
\end{equation*}
which implies
\begin{equation*}
\|\mathsf{D}_{\bx}\vu(t,\bx)\|_{\rm op},
\ \|\mathsf{D}_{\bx}\vc(t,\bx)\|_{\rm op}
\le
\frac{1}{1-T}
+ \frac{T\,\diam(\mathcal{D})^2}{(1-T)^3},
\qquad
(t,\bx)\in[0,T]\times\R^d.
\end{equation*}

(ii)
Thus $\vu$ and $\vc$ are globally Lipschitz in $\bx$, uniformly for
$t\in[0,T]$. Moreover, continuity of the vector fields implies that
$\sup_{t\in[0,T]}\|\vv(t,\boldsymbol{0})\|<\infty$, and hence each
field satisfies a uniform linear-growth bound. Therefore the
Picard--Lindel\"of theorem, together with the above linear-growth bound (\citealp{teschl2012ordinary}), gives a
unique classical solution on all of $[0,T]$ for every initial condition
$\bx(0)=\bx_0\in\R^d$.
Since the corresponding vector field is continuous, the derivative
of each solution is continuous; hence
$\bx_\cdot\in C^1([0,T])\subset AC([0,T])$.

(iii)
For fixed $w>1$, $\vcfg(t,\bx)=(1-w)\vu(t,\bx)+w\vc(t,\bx),$
so $\vcfg$ is continuously differentiable in $(t,\bx)$ and globally Lipschitz in $\bx$ on $[0,T]\times\R^d$ as well. The same existence-and-uniqueness and regularity argument therefore gives, for every initial condition $\bx(0)=\bx_0\in\R^d$, a unique $C^1$ solution of the CFG ODE on $[0,T]$.
\end{proof}

\subsection{Proofs for Section~\ref{sec:homotopy}}

\begin{lemma}[Uniform positive mass on compact support]\label{lem:compact-support-uniform-mass}
If $\mu$ is a Borel probability measure on $\R^d$ and
$\mathcal D:=\supp(\mu)$ is compact, then for every $\delta>0$,
$m_\delta:=\inf_{\by\in\mathcal D}\mu(\mathcal D\cap B_\delta(\by))>0$.
\end{lemma}

\begin{proof}
Fix $\delta>0$. Since $\mathcal D$ is compact, there exist finitely many points $\by_1,\ldots,\by_N\in\mathcal D$ such that $\mathcal D\subset\bigcup_{j=1}^N B_{\delta/2}(\by_j)$. Since each $\by_j\in\supp(\mu)$, every open neighborhood of $\by_j$ has positive $\mu$-mass. Also $\mu(\mathcal D^c)=0$, so $\mu(\mathcal D\cap B_{\delta/2}(\by_j))>0$ for every $j$. Let
$m_\delta':=\min_{1\le j\le N}\mu(\mathcal D\cap B_{\delta/2}(\by_j))>0$.
For any $\by\in\mathcal D$, choose $j$ with $\by\in B_{\delta/2}(\by_j)$, then $B_{\delta/2}(\by_j)\subset B_\delta(\by)$, and hence $\mu(\mathcal D\cap B_\delta(\by))\ge m_\delta'$. Taking the infimum over $\by\in\mathcal D$ proves the claim.

\end{proof}

\begin{lemma}[Distance under scaling of a compact set]\label{lem:distance-under-scaling}
Let $\mathcal D\subset\R^d$ be nonempty and compact, and define
\(
M_{\mathcal D}:=\max_{\by\in\mathcal D}\|\by\|.
\)
Then, for every $\bx\in\R^d$ and $t\in[0,1]$,
\[
\left|
\dist(\bx,t\mathcal D)-\dist(\bx,\mathcal D)
\right|
\le
(1-t)M_{\mathcal D}.
\]
Consequently, for every fixed $\bx\in\R^d$,
\[
\dist^2(\bx,t\mathcal D)
\to
\dist^2(\bx,\mathcal D),
\qquad t\uparrow1.
\]
\end{lemma}

\begin{proof}
Fix $\bx\in\R^d$ and $t\in[0,1]$. Since $\mathcal D$ is compact,
we can choose $\by^\ast\in\mathcal D$ such that
$
\|\bx-\by^\ast\|=\dist(\bx,\mathcal D).
$
Since $t\by^\ast\in t\mathcal D$,
$\dist(\bx,t\mathcal D)\le\|\bx-t\by^\ast\|\le
\dist(\bx,\mathcal D)+(1-t)M_{\mathcal D}$. Conversely, choose
$\by_t\in\mathcal D$ such that
$\|\bx-t\by_t\|=\dist(\bx,t\mathcal D)$, then
$\dist(\bx,\mathcal D)\le\|\bx-\by_t\|\le
\dist(\bx,t\mathcal D)+(1-t)M_{\mathcal D}$.
Combining the two inequalities proves the distance estimate. The convergence
of the squared distances follows immediately.
\end{proof}

\subsubsection{Proof of Proposition~\ref{prop:soft-distance-representation}(ii)--(iii)}\label{proof:prop:soft-distance-representation}
\begin{proof}
Fix \(t\in[0,1)\) and \(\boldsymbol x\in\mathbb R^d\), and set \(\sigma:=1-t>0\).
Because \(\mathcal D\) is compact, the function
$
\boldsymbol y\longmapsto\|\boldsymbol x-t\boldsymbol y\|^2
$ attains its minimum and maximum on \(\mathcal D\). Therefore, for every \(\boldsymbol y\in\mathcal D\),
$
\exp\left(
-\tfrac{\max_{\boldsymbol z\in\mathcal D}
\|\boldsymbol x-t\boldsymbol z\|^2}{2\sigma^2}
\right)
\le
\exp\left(
-\tfrac{\|\boldsymbol x-t\boldsymbol y\|^2}{2\sigma^2}
\right)
\le
\exp\left(
-\tfrac{\operatorname{dist}^2(\boldsymbol x,t\mathcal D)}
{2\sigma^2}
\right).
$ Integrating with respect to the probability measure \(p_{\emptyset}\) preserves these inequalities. Applying the decreasing function
$
r\longmapsto-2\sigma^2\log r$,  then reverses their order. By the definition of the smoothed squared distance, this gives
\[
\operatorname{dist}^2(\boldsymbol x,t\mathcal D)
\le
\operatorname{dist}_t^2(\boldsymbol x,t\mathcal D)
\le
\max_{\boldsymbol y\in\mathcal D}
\|\boldsymbol x-t\boldsymbol y\|^2.
\]This proves part (ii).

For part (iii), let $
C:=\operatorname{Conv}(t\mathcal D)
=t\operatorname{Conv}(\mathcal D).$ The set \(C\) is compact and convex. We first observe that
$
\max_{\boldsymbol z\in C}
\|\boldsymbol x-\boldsymbol z\|
=
\max_{\boldsymbol z\in t\mathcal D}
\|\boldsymbol x-\boldsymbol z\|.$ Indeed, one inequality follows from \(t\mathcal D\subseteq C\). Conversely, if
$
\boldsymbol z
=
\sum_{i=1}^N\lambda_i\boldsymbol z_i
\in C,
\boldsymbol z_i\in t\mathcal D,
$ then 
$\|\bx-\bz\|=\left\|\sum_{i=1}^N\lambda_i(\bx-\bz_i)\right\| \le \sum_{i=1}^N\lambda_i\|\bx-\bz_i\| \le \max_{\boldsymbol{\zeta}\in t\mathcal{D}}\|\bx-\boldsymbol{\zeta}\|.$
Taking the maximum over $\bz\in\Conv(t\mathcal{D})$ gives the reverse inequality.
Combining this equality with part (ii) and the inclusion \(t\mathcal D\subseteq C\), we obtain
$
\min_{\boldsymbol z\in C}
\|\boldsymbol x-\boldsymbol z\|^2
\le
\operatorname{dist}_t^2(\boldsymbol x,t\mathcal D)
\le
\max_{\boldsymbol z\in C}
\|\boldsymbol x-\boldsymbol z\|^2.
$
By compactness, choose \(\boldsymbol z_-,\boldsymbol z_+\in C\) at which the minimum and maximum are attained. Since \(C\) is convex, the segment
$
(1-s)\boldsymbol z_-+s\boldsymbol z_+,
0\le s\le1, $ lies entirely in \(C\). The function
$
s\longmapsto
\left\|
\boldsymbol x-
\bigl((1-s)\boldsymbol z_-+s\boldsymbol z_+\bigr)
\right\|^2$ is continuous, and its values at \(s=0\) and \(s=1\) are respectively the minimum and maximum above. The intermediate value theorem therefore gives some \(s_\ast\in[0,1]\) such that, with
\[
\boldsymbol z_t(\boldsymbol x)
:=
(1-s_\ast)\boldsymbol z_-+s_\ast\boldsymbol z_+,
\]we have
$
\operatorname{dist}_t^2(\boldsymbol x,t\mathcal D)
=
\|\boldsymbol x-\boldsymbol z_t(\boldsymbol x)\|^2.
$ Moreover,
$
\boldsymbol z_t(\boldsymbol x)
\in C,$ which proves part (iii). 


Fix $\boldsymbol x\in\mathbb R^d$,  $t\in[0,1)$, choose $\boldsymbol y_t^\ast\in\mathcal D$ such that $\|\boldsymbol x-t\boldsymbol y_t^\ast\|=\operatorname{dist}(\boldsymbol x,t\mathcal D)$.
Given $\delta>0$, by Lemma~\ref{lem:compact-support-uniform-mass}, there exists $m_\delta>0$, independent of $t$, such that $p_{\emptyset}(\mathcal D\cap B_\delta(\boldsymbol y_t^\ast))\ge m_\delta$. If $\boldsymbol y\in\mathcal D\cap B_\delta(\boldsymbol y_t^\ast)$, then $\|t\boldsymbol y-t\boldsymbol y_t^\ast\|\le t\delta\le\delta$, so
\(\begin{aligned}
\|\boldsymbol x-t\boldsymbol y\|
\le
\|\boldsymbol x-t\boldsymbol y_t^\ast\|
+\|t\boldsymbol y_t^\ast-t\boldsymbol y\|
\le
\operatorname{dist}(\boldsymbol x,t\mathcal D)+\delta.
\end{aligned}\) Restricting the defining integral of the smoothed distance to this neighborhood gives
$$
\int_{\mathcal D}
\exp\left(-\tfrac{\|\boldsymbol x-t\boldsymbol y\|^2}{2(1-t)^2}\right)
p_{\emptyset}(\mathrm d\boldsymbol y)
\ge
m_\delta
\exp\left(
-\tfrac{(\operatorname{dist}(\boldsymbol x,t\mathcal D)+\delta)^2}
{2(1-t)^2}
\right).
$$  Applying the decreasing function $r\mapsto-2(1-t)^2\log r$ gives
\[
\operatorname{dist}_t^2(\boldsymbol x,t\mathcal D)
\le
\bigl(\operatorname{dist}(\boldsymbol x,t\mathcal D)+\delta\bigr)^2
+
2(1-t)^2\log\tfrac1{m_\delta}.
\]Together with the lower bound in part (ii), we obtain
\[
\begin{aligned}
0
&\le
\operatorname{dist}_t^2(\boldsymbol x,t\mathcal D)
-\operatorname{dist}^2(\boldsymbol x,t\mathcal D)
\le
2\operatorname{dist}(\boldsymbol x,t\mathcal D)\delta
+\delta^2
+2(1-t)^2\log\tfrac1{m_\delta}.
\end{aligned}
\]Let $M_{\mathcal D}:=\max_{\boldsymbol y\in\mathcal D}\|\boldsymbol y\|$. Since $\operatorname{dist}(\boldsymbol x,t\mathcal D)
=
\|\boldsymbol x-t\boldsymbol y_t^\ast\|
\le
\|\boldsymbol x\|+M_{\mathcal D},$ letting $t\uparrow1$ with $\delta$ fixed yields
\[
\limsup_{t\uparrow1}
\left[
\operatorname{dist}_t^2(\boldsymbol x,t\mathcal D)
-\operatorname{dist}^2(\boldsymbol x,t\mathcal D)
\right]
\le
2(\|\boldsymbol x\|+M_{\mathcal D})\delta+\delta^2.
\]Because $\delta>0$ is arbitrary, letting $\delta\downarrow0$ shows that the difference on the left converges to zero. Lemma~\ref{lem:distance-under-scaling} also gives $\operatorname{dist}^2(\boldsymbol x,t\mathcal D)\to\operatorname{dist}^2(\boldsymbol x,\mathcal D)$. Therefore,
$
\operatorname{dist}_t^2(\boldsymbol x,t\mathcal D)
\rightarrow
\operatorname{dist}^2(\boldsymbol x,\mathcal D).
$ Finally, $f_t(\boldsymbol x)=\tfrac12\operatorname{dist}_t^2(\boldsymbol x,t\mathcal D)-\tfrac{1-t}{2}|\boldsymbol x|^2$, and the second term tends to zero. Hence $f_t(\boldsymbol x)\to\tfrac12\operatorname{dist}^2(\boldsymbol x,\mathcal D)$, proving part (iv).
\end{proof}

\section{Auxiliary Tools for Continuous and Discrete Dynamics}\label{app:shared-lemmas}

We collect the auxiliary tools used in the stagewise analysis. The
continuous-time argument differentiates the squared distance from the
normalized trajectory to a fixed closed set and then scales the resulting
formula back to the moving target. The discrete-time tool gives the exact
one-step Euler counterpart.

\subsection{Continuous-time distance calculus for scaled sets}


\begin{lemma}[\citep{kolasinski2023regularity}Derivative of the squared distance]
\label{lem:squared-distance-differentiability}
Let \(\bz \in \mathbb R^d\) and \(S\subset\mathbb R^d\) be nonempty and closed. Suppose that
\(\Proj_S(\boldsymbol z)\) is well-defined, and write
\(\boldsymbol p:=\Proj_S(\boldsymbol z)\). Then
\(\tfrac12\dist^2(\cdot,S)\) is differentiable at \(\boldsymbol z\), with
\begin{equation}\label{eq:squared-distance-gradient}
\nabla_{\boldsymbol z}\tfrac12\dist^2(\boldsymbol z,S)
=
\boldsymbol z-\boldsymbol p.
\end{equation}
\end{lemma}

\begin{lemma}[Derivative of scaled squared distance]\label{lem:dist-derivative}
Let \(\Omega\subset\mathbb R^d\) be nonempty and closed, let \(I\) be
an interval, let \(f\in C^1(I)\) satisfy \(f(t)>0\), and let
\(\boldsymbol x_\cdot\in C^1(I)\). Suppose that, for every \(t\in I\),
\(
\boldsymbol p_t
:=
\operatorname{Proj}_{\Omega}
\left(\boldsymbol x_t/f(t)
\right)
\)
is well-defined. Then
\(t\mapsto \tfrac12\operatorname{dist}^2(\boldsymbol x_t,f(t)\Omega)\) is
differentiable on \(I\), and
\[
\frac{\mathrm d}{\mathrm dt}\frac12
\operatorname{dist}^2\bigl(\boldsymbol x_t,f(t)\Omega\bigr)
=
-\left\langle
\dot{\boldsymbol x}_t-\dot f(t)\boldsymbol p_t,
f(t)\boldsymbol p_t-\boldsymbol x_t
\right\rangle.
\]
\end{lemma}

\begin{proof}
Positive scaling gives
\[
\operatorname{dist}^2\bigl(\boldsymbol x_t,f(t)\Omega\bigr)
=
f(t)^2
\operatorname{dist}^2\left(
\boldsymbol x_t/f(t),
\Omega\right).
\]
By Lemma~\ref{lem:squared-distance-differentiability},
\[
\nabla\!\left[
\tfrac12\operatorname{dist}^2(\,\cdot\,,\Omega)
\right]\!\left(\boldsymbol x_t/f(t)\right)
=\boldsymbol x_t/f(t)
-\boldsymbol p_t.
\]
The ordinary chain rule, therefore, gives
\[
\begin{aligned}
\frac{\mathrm d}{\mathrm dt}
\frac12
\operatorname{dist}^2\bigl(\boldsymbol x_t,f(t)\Omega\bigr)
={}&
f(t)\dot f(t)
\left\|\boldsymbol x_t/f(t)
-\boldsymbol p_t
\right\|^2
+
f(t)^2
\left\langle \boldsymbol x_t/f(t)
-\boldsymbol p_t,
\dot{\boldsymbol x}_t/f(t)
- \dot f(t)\boldsymbol x_t/f(t)^2
\right\rangle
\\
={}&
\left\langle
\boldsymbol x_t-f(t)\boldsymbol p_t,
\dot{\boldsymbol x}_t-\dot f(t)\boldsymbol p_t
\right\rangle.
\end{aligned}
\]
\end{proof}

Now we present the  continuous-time lemma for the attraction and absorption of scaled targets, which is the key tool for the continuous-time analysis of the stagewise dynamics in Section~\ref{app:uncond-proofs} and Section~\ref{app:cfg-proofs}.

\begin{lemma}[Continuous attraction and absorption for scaled targets]
\label{lem:continuous-scaled-convex-contraction}
Let \(S\subset\R^d\) be nonempty and closed, let
\(I\subset(0,1)\) be an interval, and 
\(\bx_\cdot\in C^1(I)\) satisfy
\begin{equation}\label{eq:continuous-scaled-target-dynamics}
\dot\bx_t
=
\frac{\boldsymbol m_t-\bx_t}{1-t},
\qquad t\in I,
\end{equation}
for a map \(t\mapsto\boldsymbol m_t\in\R^d\).
Assume that, for every \(t\in I\), the projection
\(
\bq_t:=\Proj_{tS}(\bx_t)
\)
is well-defined, and that there exists a nonempty closed convex set
\(C_t\subset S\) such that
\begin{equation}\label{eq:continuous-common-convex-subset}
\bq_t/t\in C_t,
\qquad
\boldsymbol m_t\in C_t.
\end{equation}
Then the following hold.
\begin{enumerate}
\item[(i)] \emph{(Attraction)} For every \(s,t\in I\) with \(s\le t\),
\begin{equation}\label{eq:continuous-scaled-convex-contraction}
\dist(\bx_t,tS)
\le
\tfrac{1-t}{1-s}\dist(\bx_s,sS).
\end{equation}
Consequently, \(t\mapsto\dist(\bx_t,tS)\) is non-increasing on \(I\).
\item[(ii)] \emph{(Absorption)} For every \(s\in I\), if
\(\bx_s\in sS\), then
\(
\bx_t\in tS
\)
for every \(t\in I\) with \(t\ge s\).
\end{enumerate}
\end{lemma}

\begin{proof}
By positive scaling and the uniqueness of the projection,
\[
{\bq_t}/{t}
=
\Proj_S\left({\bx_t}/{t}\right).
\]
Lemma~\ref{lem:dist-derivative}, applied with \(f(t)=t\), gives
\begin{equation}\label{eq:continuous-scaled-target-distance-derivative}
\tfrac{\dd}{\dd t}\dist^2(\bx_t,tS)
=
-2\left\langle
\dot\bx_t-\tfrac{\bq_t}{t},
\bq_t-\bx_t
\right\rangle.
\end{equation}
This identity holds at every \(t\in I\). For every
\(\lambda\in(0,1)\), convexity gives
\(
\bq_t+\lambda(t\boldsymbol m_t-\bq_t)
\in
tC_t
\subset
tS.
\)
Since \(\bq_t\) is a nearest point to \(\bx_t\) in \(tS\),
\[
0
\le
\left\|
\bx_t-\bq_t-\lambda(t\boldsymbol m_t-\bq_t)
\right\|^2
-\|\bx_t-\bq_t\|^2.
\]
Expanding, dividing by \(2\lambda>0\), and letting
\(\lambda\downarrow0\) yields
\begin{equation}\label{eq:continuous-local-variational-inequality}
\left\langle
t\boldsymbol m_t-\bq_t,
\bq_t-\bx_t
\right\rangle
\ge0.
\end{equation}
Using \eqref{eq:continuous-scaled-target-dynamics} and the identity
\[
\dot\bx_t-\frac{\bq_t}{t}
=
\frac{\bq_t-\bx_t}{1-t}
+
\frac{t\boldsymbol m_t-\bq_t}{t(1-t)}
\]
in \eqref{eq:continuous-scaled-target-distance-derivative}, and then
applying \eqref{eq:continuous-local-variational-inequality} yields, for every \(t\in I\),
\begin{align}\label{eq:ineq-master}
\frac{\dd}{\dd t}\dist^2(\bx_t,tS)
&=
-\tfrac{2}{1-t}\dist^2(\bx_t,tS)
-
\tfrac{2}{t(1-t)}\left\langle
t\boldsymbol m_t-\bq_t,
\bq_t-\bx_t
\right\rangle \notag \\
&\le
-\tfrac{2}{1-t}\dist^2(\bx_t,tS).
\end{align}
When \(\dist(\bx_t,tS)=0\), one has \(\bq_t=\bx_t\).
By \eqref{eq:continuous-scaled-target-distance-derivative},
 both sides of \eqref{eq:ineq-master} are zero. Equivalently,
\begin{equation}\label{eq:continuous-scaled-target-weighted-distance}
\frac{\dd}{\dd t}
\left[
\tfrac{\dist^2(\bx_t,tS)}{(1-t)^2}
\right]
\le0
\qquad\text{for every }t\in I.
\end{equation}
Integrating this inequality and taking square roots we obtain
\[
\dist(\bx_t,tS)
\le
\tfrac{1-t}{1-s}\dist(\bx_s,sS),
\]
which gives
\eqref{eq:continuous-scaled-convex-contraction} and proves part~(i). Since
\((1-t)/(1-s)\le1\),
\(t\mapsto\dist(\bx_t,tS)\) is non-increasing. If
\(\bx_s\in sS\) for some \(s\in I\), then the right-hand side of
\eqref{eq:continuous-scaled-convex-contraction} is zero for every
\(t\in I\) with \(t\ge s\), which proves part~(ii).
\end{proof}

\subsection{Discrete-time attraction and absorption for scaled targets}


\begin{lemma}[Discrete attraction and absorption for scaled targets]\label{lem:euler-scaled-convex-contraction}
Let $S\subset\R^d$ be nonempty and closed. Let $0\le s<t\le1$, define
\[
a:=\frac{1-t}{1-s},
\qquad
b:=\frac{t-s}{1-s}.
\]
For $\bx,\by\in\R^d$, define $\bx^+:=a\bx+b\by$,
then the following hold.
\begin{enumerate}
\item[(i)] \emph{(Attraction)} Suppose that, when \(s>0\), there
exist a nearest point \(\boldsymbol q\in sS\) and a convex set
\(C\subset S\) such that
\[
\|\boldsymbol x-\boldsymbol q\|
=
\dist(\boldsymbol x,sS),
\qquad
\tfrac{\boldsymbol q}{s}\in C,
\qquad
\boldsymbol y\in C.
\]
When \(s=0\), suppose instead that \(\boldsymbol y\in S\), then
\begin{equation}\label{eq:euler-scaled-convex-contraction}
\dist(\bx^+,tS)
\le
a\,\dist(\bx,sS)= \tfrac{1-t}{1-s}\dist(\bx,sS). 
\end{equation}
\item[(ii)] \emph{(Absorption)} Suppose first that \(s>0, \boldsymbol x\in sS\),
 and there exists a convex set \(C\subset S\)
such that
\(
{\boldsymbol x}/{s}\in C,
\
\boldsymbol y\in C,
\)
then \(\boldsymbol x^+\in tS\).
If \(s=0\), suppose instead that
\(\boldsymbol x=\boldsymbol0\) and \(\boldsymbol y\in S\), then again \(\boldsymbol x^+\in tS\).
\end{enumerate}
\end{lemma}

\begin{proof}
For part~(i), first suppose that \(s>0\). Let
\(\boldsymbol q\in sS\) be the nearest point, and define
\(
\boldsymbol q^+:=a\boldsymbol q+b\boldsymbol y.
\) Since \(\boldsymbol q/s,\boldsymbol y\in C\), we have
\[
\frac{\boldsymbol q^+}{t}
=
\frac{as}{t}\frac{\boldsymbol q}{s}
+
\frac bt\boldsymbol y,
\] where \(as/t\) and \(b/t\) are nonnegative and sum to one because \(as+b=t\). Thus, \(\boldsymbol q^+/t\) is a convex combination of two points in \(C\). By convexity,
\(
\boldsymbol q^+/t\in C\subset S,
\) and therefore \(\boldsymbol q^+\in tS\).
Using \(\boldsymbol q^+\) as a comparison point in \(tS\), we obtain
\[
\begin{aligned}
\dist(\boldsymbol x^+,tS)
&\le
\|\boldsymbol x^+-\boldsymbol q^+\|
=
\|a\boldsymbol x+b\boldsymbol y
-a\boldsymbol q-b\boldsymbol y\|
\\
&=
a\|\boldsymbol x-\boldsymbol q\|
=
a\,\dist(\boldsymbol x,sS).
\end{aligned}
\]

On the other hand, if \(s=0\),  then \(a=1-t\), \(b=t\), and
\(
\boldsymbol x^+=(1-t)\boldsymbol x+t\boldsymbol y.
\) Since \(\boldsymbol y\in S\), the point \(t\boldsymbol y\) belongs to \(tS\). Hence,
\[
\begin{aligned}
\dist(\boldsymbol x^+,tS)
\le
\|\boldsymbol x^+-t\boldsymbol y\|
=
(1-t)\|\boldsymbol x\|
=
a\,\dist(\boldsymbol x,0S).
\end{aligned}
\]
This proves part~(i).

For part~(ii), first suppose that \(s>0\). Since
\(\boldsymbol x\in sS\), we may take
\(\boldsymbol q=\boldsymbol x\) in part~(i). The one-step estimate gives
\(
\dist(\boldsymbol x^+,tS)
\le
a\,\dist(\boldsymbol x,sS)
=0.
\)
Since \(tS\) is closed, it follows that \(\boldsymbol x^+\in tS\).
If \(s=0\), then \(\boldsymbol x=\boldsymbol0\), so
\(
\boldsymbol x^+=t\boldsymbol y\in tS.
\)
This proves part~(ii).
\end{proof}

\section{Proofs for Section \ref{sec:uncond}}\label{app:uncond-proofs}

The unconditional proofs below are organized by stage. Each stage has a
continuous ODE statement and an Euler-discrete counterpart. In both cases,
the normalized projection point and the posterior mean are placed in a
convex subset of the static target. Lemma~\ref{lem:continuous-scaled-convex-contraction}
turns this geometry into continuous attraction and conditional absorption,
while parts~(i)--(ii) of
Lemma~\ref{lem:euler-scaled-convex-contraction} give its exact
grid-level counterparts.

\subsection{Common posterior-mean facts}

The next fact is the standard barycenter property of convex hulls; see \citet{bogachev2007measure} for the compact-support characterization and \citet{rockafellar1970convex} for the separation and finite-dimensional convex-hull results used in the argument. Since this observation is repeatedly used for posterior means below, for completeness, we offer the proof here.

\begin{lemma}[Posterior means stay in the closed convex hull]\label{lem:mean-in-closed-conv}
Let $\mu \in \Prob(\R^d)$ be a probability measure with finite first moment, and suppose $\supp(\mu) \subset \mathcal{D} \subset \R^d$. Then
$\int_{\R^d} \by\,\mu(\dd \by)\in\overline{\Conv(\mathcal{D})}.$
If \(\mathcal D\) is compact, then \(\operatorname{Conv}(\mathcal D)\) is compact and therefore closed. In this case,
$\int_{\R^d} \by\,\mu(\dd \by)\in\Conv(\mathcal{D}).$
\end{lemma}

\begin{proof}
Write $\bar{\by}_\mu := \int_{\R^d} \by\,\mu(\dd \by).$
Suppose for contradiction that $\bar{\by}_\mu \notin \overline{\Conv(\mathcal{D})}$. Since $\overline{\Conv(\mathcal{D})}$ is closed and convex, the finite-dimensional separating hyperplane theorem yields some $\ba \in \R^d$ and $\alpha \in \R$ such that
\begin{equation*}
\langle \ba,\bar{\by}_\mu\rangle > \alpha \ge \langle \ba,\bz\rangle
\qquad
\text{for all } \bz \in \overline{\Conv(\mathcal{D})}.
\end{equation*}
Because $\supp(\mu) \subset \mathcal{D} \subset \overline{\Conv(\mathcal{D})}$, we have $\langle \ba,\by\rangle \le \alpha$ for $\mu$-almost every $\by$. Integrating gives
$\langle \ba,\bar{\by}_\mu\rangle=\int_{\R^d} \langle \ba,\by\rangle\,\mu(\dd \by)\le\alpha,$
which contradicts the strict separation. Hence $\bar{\by}_\mu \in \overline{\Conv(\mathcal{D})}$. If $\mathcal{D}$ is compact. In $\R^d$, the convex hull of a compact set is compact, hence closed. Therefore
$\overline{\Conv(\mathcal{D})} = \Conv(\mathcal{D}),$ 
and the second claim follows.
\end{proof}

\subsection{Continuous early-stage mean attraction}

\begin{lemma}[A priori norm bound]\label{lem:early-norm-bound}
Let $\bx_t$ solve the unconditional flow ODE with initial condition $\bx_0$, and let
$M_{\mathcal D}:=\max_{\by\in\mathcal D}\|\by\|$.
Then, for every $t\in[0,1)$ for which the solution is defined,
\begin{equation*}
\|\bx_t\|
\le
(1-t)\|\bx_0\|+tM_{\mathcal D}.
\end{equation*}
\end{lemma}

\begin{proof}
Since $\hat{\by}_t(\bx_t)$ is the mean of a probability measure supported on $\mathcal D$,
\begin{equation*}
\|\hat{\by}_t(\bx_t)\|
\le
\int_{\mathcal D}\|\by\|\,\eta_t^{\bx_t}(\dd\by)
\le
M_{\mathcal D}.
\end{equation*}
The unconditional flow ODE \eqref{eq:uncond flow} gives
$
\frac{\dd}{\dd t}\left(\frac{\bx_t}{1-t}\right)
=
\frac{\hat{\by}_t(\bx_t)}{(1-t)^2}. $
Integrating from $0$ to $t$ and multiplying by $1-t$ yields
\begin{equation*}
\bx_t
=
(1-t)\bx_0
+
(1-t)\int_0^t
\frac{\hat{\by}_s(\bx_s)}{(1-s)^2}\,\dd s.
\end{equation*}
Therefore,
\begin{align*}
\|\bx_t\|
\le
(1-t)\|\bx_0\|
+
(1-t)M_{\mathcal D}\int_0^t\frac{\dd s}{(1-s)^2} =
(1-t)\|\bx_0\|+tM_{\mathcal D}.
\end{align*}
\end{proof}

\subsubsection{Proof of Theorem~\ref{thm:early-mean}}

\begin{proof}\label{proof:thm:early-mean}
The expression defining $r_{\mathrm{mean}}$ satisfies
\[
M_{\mathcal D}
\left[
\exp\left(
\tfrac{
tD_{\emptyset}
\bigl(
(1-t)\|\bx_0\|+2tM_{\mathcal D}
\bigr)
}{
(1-t)^2
}
\right)-1
\right]
=
\bigO(t)
\qquad
\text{as }t\downarrow0,
\]
while $(1-t)\|\bx_0\|\to\|\bx_0\|>0$. Hence, the admissible choices
asserted in the main-text setup exist. Fix $t_{\mathrm{mean}}$ and
$r_{\mathrm{mean}}$ as there. By the choice of $t_{\mathrm{mean}}$,
\[
(1-t_{\mathrm{mean}})\|\bx_0\|
>
t_{\mathrm{mean}}
\bigl(
D_\emptyset+r_{\mathrm{mean}}
\bigr).
\]

We first prove part~(i) of the theorem, which is the non-entry property.
As in the proof of Lemma~\ref{lem:early-norm-bound}, the flow ODE gives
\(
\bx_t
=
(1-t)\bx_0
+
(1-t)\int_0^t
\frac{\hat{\by}_s(\bx_s)}{(1-s)^2}\,\dd s.
\)
Since
\(
(1-t)\int_0^t\tfrac{\dd s}{(1-s)^2}=t,
\)
subtracting \(t\bar{\boldsymbol y}\)  gives
\begin{equation}
    \label{eq:yhat_ybar}
\bx_t-t\bar{\by}
=
(1-t)\bx_0
+
(1-t)\int_0^t
\tfrac{
\hat{\by}_s(\bx_s)-\bar{\by}
}{
(1-s)^2
}\,\dd s.
\end{equation}
By Lemma~\ref{lem:mean-in-closed-conv},
$\hat{\by}_s(\bx_s),\bar{\by}\in\Conv(\mathcal D)$,
hence,
$
\|\hat{\by}_s(\bx_s)-\bar{\by}\|
\le
\diam(\Conv(\mathcal D))
=
D_\emptyset.
$
Applying the reverse triangle inequality to \eqref{eq:yhat_ybar}, 
we obtain
\[\begin{aligned}
\|\boldsymbol x_t-t\bar{\boldsymbol y}\|
&\ge
(1-t)\|\boldsymbol x_0\|
-
(1-t)\int_0^t
\tfrac{
\left\|
\hat{\boldsymbol y}_s(\boldsymbol x_s)
-\bar{\boldsymbol y}
\right\|
}{
(1-s)^2
}\,\mathrm ds\\
&\ge
(1-t)\|\boldsymbol x_0\|
-
(1-t)D_\emptyset
\int_0^t\tfrac{\mathrm ds}{(1-s)^2}\\
&=
(1-t)\|\boldsymbol x_0\|-tD_\emptyset.
\end{aligned}
\]
Using
$tB_{r_{\mathrm{mean}}}(\bar{\by})
=B_{tr_{\mathrm{mean}}}(t\bar{\by})$,
we obtain, for every $t\in[0,t_{\mathrm{mean}}]$,
\begin{align*}
\dist\!\left(
\bx_t,
tB_{r_{\mathrm{mean}}}(\bar{\by})
\right)
&=
\max\left\{
\|\boldsymbol x_t-t\bar{\boldsymbol y}\|
-tr_{\mathrm{mean}},
0
\right\}\ge
(1-t)\|\bx_0\|-t(D_\emptyset+r_{\mathrm{mean}})\\
&\ge
(1-t_{\mathrm{mean}})\|\bx_0\|
-
t_{\mathrm{mean}}(D_\emptyset+r_{\mathrm{mean}})
>0,
\end{align*}
where the last inequality follows from the choice of
$t_{\mathrm{mean}}$ above.
This proves part~(i).

We next show that the posterior mean stays inside the fixed mean ball
$B_{r_{\mathrm{mean}}}(\bar{\by})$.
Fix $t\in(0,t_{\mathrm{mean}}]$ and let
\begin{equation*}
\begin{aligned}
w_t(\by)
:=
\exp\!\left(-\tfrac{\|\bx_t-t\by\|^2}{2(1-t)^2}\right).
\end{aligned}
\end{equation*}
For any $\by_1,\by_2 \in \mathcal{D}$,
\[
\begin{aligned}
\left|
\log w_t(\boldsymbol y_1)
-
\log w_t(\boldsymbol y_2)
\right|
&=
\frac{
\left|
\|\boldsymbol x_t-t\boldsymbol y_1\|^2
-
\|\boldsymbol x_t-t\boldsymbol y_2\|^2
\right|
}{
2(1-t)^2
}\\ &=
\frac{
\left|
2t\langle
\boldsymbol x_t,
\boldsymbol y_2-\boldsymbol y_1
\rangle
+
t^2
\bigl(
\|\boldsymbol y_1\|^2-\|\boldsymbol y_2\|^2
\bigr)
\right|
}{
2(1-t)^2
}.
\end{aligned}
\]
Since
\(
\|\boldsymbol y_1-\boldsymbol y_2\|
\le D_\emptyset
\) 
and
\(
\begin{aligned}
\left|
\|\boldsymbol y_1\|^2-\|\boldsymbol y_2\|^2
\right|
=
\left|
\left\langle
\boldsymbol y_1-\boldsymbol y_2,
\boldsymbol y_1+\boldsymbol y_2
\right\rangle
\right|\le
2D_\emptyset M_{\mathcal D},
\end{aligned}
\)
we obtain
\begin{equation}
\label{eq:log_wt}
\left|
\log w_t(\boldsymbol y_1)
-
\log w_t(\boldsymbol y_2)
\right|
\le
\frac{
tD_\emptyset\|\boldsymbol x_t\|
+
t^2D_\emptyset M_{\mathcal D}
}{
(1-t)^2
}
\le
\frac{
tD_\emptyset
\bigl(
(1-t)\|\boldsymbol x_0\|
+
2tM_{\mathcal D}
\bigr)
}{
(1-t)^2
},
\end{equation}
where the last inequality follows from Lemma~\ref{lem:early-norm-bound}.
Exponentiating \eqref{eq:log_wt}, we obtain
\[
\begin{aligned}
&
\exp\left(
-\tfrac{
tD_\emptyset
\left(
(1-t)\|\boldsymbol x_0\|
+
2tM_{\mathcal D}
\right)
}{
(1-t)^2
}
\right)\le
\tfrac{
w_t(\boldsymbol y_1)
}{
w_t(\boldsymbol y_2)
}\le
\exp\left(
\tfrac{
tD_\emptyset
\left(
(1-t)\|\boldsymbol x_0\|
+
2tM_{\mathcal D}
\right)
}{
(1-t)^2
}
\right).
\end{aligned}
\]Fixing \(\boldsymbol y_1=\boldsymbol y\), multiplying through by
\(w_t(\boldsymbol y_2)\), and integrating with respect to
\(p_\emptyset(\mathrm d\boldsymbol y_2)\), we obtain
\[
\begin{aligned}
&
\exp\left(
-\tfrac{
tD_\emptyset
\left(
(1-t)\|\boldsymbol x_0\|
+
2tM_{\mathcal D}
\right)
}{
(1-t)^2
}
\right)\le
\tfrac{
w_t(\boldsymbol y)
}{
\int_{\mathcal D}
w_t(\boldsymbol z)\,
p_\emptyset(\mathrm d\boldsymbol z)
}\le
\exp\left(
\tfrac{
tD_\emptyset
\left(
(1-t)\|\boldsymbol x_0\|
+
2tM_{\mathcal D}
\right)
}{
(1-t)^2
}
\right).
\end{aligned}
\]Because the normalized weight is positive, taking logarithms gives
\begin{equation}
\label{eq:log_wty}
\begin{aligned}
\left|
\log\left(
\tfrac{
w_t(\boldsymbol y)
}{
\int_{\mathcal D}
w_t(\boldsymbol z)\,
p_\emptyset(\mathrm d\boldsymbol z)
}
\right)
\right|\le
\tfrac{
tD_\emptyset
\bigl(
(1-t)\|\boldsymbol x_0\|
+
2tM_{\mathcal D}
\bigr)
}{
(1-t)^2
}.
\end{aligned}
\end{equation}
Since 
\(
|u|\le L
\) implies \(
|e^u-1|\le e^L-1,
\) applying it to the logarithm in \eqref{eq:log_wty}, we obtain
\[
\begin{aligned}
&
\left|
\tfrac{
w_t(\boldsymbol y)
}{
\int_{\mathcal D}
w_t(\boldsymbol z)\,
p_\emptyset(\mathrm d\boldsymbol z)
}
-1
\right|\le
\exp\left(
\tfrac{
tD_\emptyset
\left(
(1-t)\|\boldsymbol x_0\|
+
2tM_{\mathcal D}
\right)
}{
(1-t)^2
}
\right)-1.
\end{aligned}
\]
Since \(\hat{\boldsymbol y}_t(\boldsymbol x_t)
=
\tfrac{
\int_{\mathcal D}
\boldsymbol y\,w_t(\boldsymbol y)\,
p_\emptyset(\mathrm d\boldsymbol y)
}{
\int_{\mathcal D}
w_t(\boldsymbol y)\,
p_\emptyset(\mathrm d\boldsymbol y)
},\)
we obtain
\begin{equation}
\label{eq:haty-bary}
    \begin{aligned}
\left\|
\widehat{\boldsymbol y}_t(\boldsymbol x_t)
-\bar{\boldsymbol y}
\right\|&=\left\|
\int_{\mathcal D}
\boldsymbol y
\left[
\tfrac{
w_t(\boldsymbol y)
}{
\int_{\mathcal D}
w_t(\boldsymbol z)\,
p_\emptyset(\mathrm d\boldsymbol z)
}
-1
\right]
p_\emptyset(\mathrm d\boldsymbol y)\right\|\\ &\le
M_{\mathcal D}
\left[
\exp\left(
\tfrac{
tD_\emptyset
\left(
(1-t)\|\boldsymbol x_0\|
+
2tM_{\mathcal D}
\right)
}{
(1-t)^2
}
\right)-1
\right].
\end{aligned}
\end{equation}
where $\exp\left(
\tfrac{
tD_\emptyset
\left(
(1-t)\|\boldsymbol x_0\|
+
2tM_{\mathcal D}
\right)
}{
(1-t)^2
}
\right)$ is nondecreasing in  $t\in[0,1)$.
It follows from the definition of $r_{\mathrm{mean}}$ that
\[
\|\hat{\by}_t(\bx_t)-\bar{\by}\|
\le
r_{\mathrm{mean}},
\qquad
t\in(0,t_{\mathrm{mean}}],
\]
and the same conclusion is immediate at $t=0$.
Thus
\begin{equation}
\label{eq:thaty}
    \hat{\by}_t(\bx_t)
\in
 B_{r_{\mathrm{mean}}}(\bar{\by}),
\qquad
t\in[0,t_{\mathrm{mean}}].
\end{equation}

\noindent
\textbf{Attraction.}
The set \(B_{r_{\mathrm{mean}}}(\bar{\by})\) is nonempty, closed, and convex, so projection onto its positive scaling is well-defined. Since \eqref{eq:thaty} holds, 
applying part~(i) of
Lemma~\ref{lem:continuous-scaled-convex-contraction} on
\(I=(0,t_{\mathrm{mean}}]\) with
\(
S=C_t=B_{r_{\mathrm{mean}}}(\bar{\by}),
\
\boldsymbol m_t=\hat{\by}_t(\bx_t)
\) gives, for
\(0<s\le t\le t_{\mathrm{mean}}\),
\begin{equation}\label{eq:early-positive-time-attraction}
\dist\bigl(\bx_t,tB_{r_{\mathrm{mean}}}(\bar{\by})\bigr)
\le
\tfrac{1-t}{1-s}
\dist\bigl(\bx_s,sB_{r_{\mathrm{mean}}}(\bar{\by})\bigr).
\end{equation}
Lastly, to include $s=0$, observe that
$sB_{r_{\mathrm{mean}}}(\bar{\by})
=B_{sr_{\mathrm{mean}}}(s\bar{\by})$ and $\bx_s\to\bx_0$ as
$s\downarrow0$. Hence
\[
\dist\bigl(\bx_s,sB_{r_{\mathrm{mean}}}(\bar{\by})\bigr)
=
\max\{\|\bx_s-s\bar{\by}\|-s r_{\mathrm{mean}},0\}
\rightarrow
\|\bx_0\|.
\]
Letting $s\downarrow0$ completes part~(ii) of the theorem and gives
\[
\dist\bigl(\bx_t,tB_{r_{\mathrm{mean}}}(\bar{\by})\bigr)
\le
(1-t)\|\bx_0\|,
\qquad
t\in[0,t_{\mathrm{mean}}],
\]
which is part~(iii) of the theorem.

Finally, part~(i) makes the distance positive at every time in
\([0,t_{\mathrm{mean}}]\). Therefore, whenever
\(0\le s<t\le t_{\mathrm{mean}}\), the estimate in part~(ii) and
\((1-t)/(1-s)<1\) show that the distance is strictly decreasing.

\end{proof}

\subsection{Discrete  early-stage mean attraction}\label{proof:cor:early-mean-euler}


\begin{lemma}[Euler affine representation and bounds]\label{lem:early-euler-apriori}
Consider the unconditional Euler iterates on the uniform grid
$t_i=i\Delta t$, $i=0,\ldots,K$, with stepsize $\Delta t=1/K$.
Then, for every $i=0,\ldots,K$, there exists
$\ba_i\in\Conv(\mathcal D)$ such that
\[
\bz_i=(1-t_i)\bx_0+t_i\ba_i.
\]
Consequently,
\[
\|\bz_i\|
\le
(1-t_i)\|\bx_0\|+t_iM_{\mathcal D} \quad
\text{ and } \quad
\|\bz_i-\bx_0\|
\le
t_i\bigl(\|\bx_0\|+M_{\mathcal D}\bigr).
\]
\end{lemma}

\begin{proof}
Since \(\mathcal D\) is nonempty and compact,
by Lemma~\ref{lem:mean-in-closed-conv}, we have 
\(
\hat{\by}_{t_i}(\bz_i)\in\Conv(\mathcal D),
 i=0,\ldots,K-1.
\)
Using $\vv_\star=\vu$, the Euler update can be written as
\[
\begin{aligned}
\boldsymbol{z}_{i+1}
&=
\boldsymbol{z}_i
+
\Delta t\,
\tfrac{
\widehat{\boldsymbol{y}}_{t_i}(\boldsymbol{z}_i)-\boldsymbol{z}_i
}{
1-t_i
}=
a_i\boldsymbol{z}_i
+
b_i\widehat{\boldsymbol{y}}_{t_i}(\boldsymbol{z}_i),
\end{aligned}
\]
where
\(
a_i:=\tfrac{1-t_{i+1}}{1-t_i},
b_i:=\tfrac{t_{i+1}-t_i}{1-t_i}.
\)
For $i=0,\ldots,K-1$, we have
\[
a_i,b_i\ge0,
\qquad
a_i+b_i=1,
\qquad
a_it_i+b_i=t_{i+1}.
\]

Here, we prove the affine representation by induction. Since $\mathcal D$ is
nonempty, choose any $\ba_0\in\Conv(\mathcal D)$, then
\(
\bz_0=(1-t_0)\bx_0+t_0\ba_0,
\) since  $t_0=0$. 
Now suppose that
\[
\bz_i=(1-t_i)\bx_0+t_i\ba_i, \text{ for some  } \ba_i\in\Conv(\mathcal D). 
\]
 Since $t_{i+1}>0$, define
\(
\ba_{i+1}
:=
\tfrac{a_it_i}{t_{i+1}}\ba_i
+
\tfrac{b_i}{t_{i+1}}
\hat{\by}_{t_i}(\bz_i). 
\)
The two coefficients are nonnegative and satisfy
\(
\tfrac{a_it_i}{t_{i+1}}
+
\tfrac{b_i}{t_{i+1}}
=1.
\)
Hence $\ba_{i+1}\in\Conv(\mathcal D)$. Moreover, we have 
\(
\bz_{i+1}
=
(1-t_{i+1})\bx_0+t_{i+1}\ba_{i+1}.
\)
This proves the affine representation.

Finally, since $\|\ba_i\|\le M_{\mathcal D}$, we obtain 
\[
\begin{aligned}
\|\boldsymbol{z}_i\|
&=
\left\|
(1-t_i)\boldsymbol{x}_0+t_i\boldsymbol{a}_i
\right\|
\le
(1-t_i)\|\boldsymbol{x}_0\|
+t_i\|\boldsymbol{a}_i\|
\le
(1-t_i)\|\boldsymbol{x}_0\|
+t_iM_{\mathcal D}.
\end{aligned}
\]
Also,
\(
\bz_i-\bx_0=t_i(\ba_i-\bx_0),
\)
and therefore
\[
\|\bz_i-\bx_0\|
\le
t_i\bigl(\|\bx_0\|+M_{\mathcal D}\bigr).
\]

\end{proof}

\begin{lemma}[Early Euler posterior localization]\label{lem:early-euler-posterior-localization}
In the setting of Theorem~\ref{cor:early-mean-euler}, consider the
unconditional Euler iterates.
Then, for every grid point satisfying $t_i\le t_{\mathrm{mean}}$,
\begin{equation}\label{eq:early-euler-posterior-localization}
\|\hat{\by}_{t_i}(\bz_i)-\bar{\by}\|
\le
r_{\mathrm{mean}}.
\end{equation}
\end{lemma}

\begin{proof}
Fix an index $i$ such that $t_i\le t_{\mathrm{mean}}$.
The calculation of \eqref{eq:haty-bary} in the proof of
Theorem~\ref{thm:early-mean}, applied at $t=t_i$ and $\bx=\bz_i$,
gives
\[
\left\|
\hat{\by}_{t_i}(\bz_i)-\bar{\by}
\right\|
\le
M_{\mathcal D}
\left[
\exp\left(
\tfrac{
t_iD_\emptyset\|\bz_i\|
+
t_i^2D_\emptyset M_{\mathcal D}
}{
(1-t_i)^2
}
\right)-1
\right].
\]
This estimate also covers $t_i=0$, for which
$\hat{\by}_0(\bx_0)=\bar{\by}$.
By Lemma~\ref{lem:early-euler-apriori},
\(
\|\bz_i\|
\le
(1-t_i)\|\bx_0\|+t_iM_{\mathcal D}.
\)
Therefore,
\[
\left\|
\hat{\by}_{t_i}(\bz_i)-\bar{\by}
\right\|
\le
M_{\mathcal D}
\left[
\exp\left(
\tfrac{
t_iD_\emptyset
\bigl(
(1-t_i)\|\bx_0\|
+
2t_iM_{\mathcal D}
\bigr)
}{
(1-t_i)^2
}
\right)-1
\right].
\]
The exponent equals
\[
D_\emptyset
\left[
\tfrac{t_i}{1-t_i}\|\bx_0\|
+
2M_{\mathcal D}
\left(
\tfrac{t_i}{1-t_i}
\right)^2
\right],
\]
which is nondecreasing in $t_i\in[0,1)$. Since
$t_i\le t_{\mathrm{mean}}$, the definition of
$r_{\mathrm{mean}}$ yields
\[
\left\|
\hat{\by}_{t_i}(\bz_i)-\bar{\by}
\right\|
\le
r_{\mathrm{mean}}.
\]
\end{proof}

\subsubsection{Proof of Theorem~\ref{cor:early-mean-euler}}

\begin{proof}
By Lemma~\ref{lem:early-euler-apriori}, for every grid point there
exists $\ba_i\in\Conv(\mathcal D)$ such that
\[
\bz_i=(1-t_i)\bx_0+t_i\ba_i.
\]
Since $\bar{\by}\in\Conv(\mathcal D)$,
\(
\|\ba_i-\bar{\by}\|
\le
D_\emptyset.
\)
Therefore, for every $i$ satisfying $t_i\le t_{\mathrm{mean}}$,
\begin{align*}
\dist\!\left(
\bz_i,
t_iB_{r_{\mathrm{mean}}}(\bar{\by})
\right)
&=
\max\left\{
\|\boldsymbol z_i-t_i\bar{\boldsymbol y}\|
-t_ir_{\mathrm{mean}},
0
\right\}
\\
&\ge
\|((1-t_i)\bx_0+t_i\ba_i)-t_i\bar{\by}\|-t_ir_{\mathrm{mean}}\\
&\ge
(1-t_i)\|\bx_0\|
-
t_i\bigl(D_\emptyset+r_{\mathrm{mean}}\bigr)\\
&\ge
(1-t_{\mathrm{mean}})\|\bx_0\|
-
t_{\mathrm{mean}}
\bigl(D_\emptyset+r_{\mathrm{mean}}\bigr)
>0.
\end{align*}
The last inequality follows from the early-stage admissibility condition
in Section~\ref{sec:uncond flow early stage}.
This proves part~(i) of Theorem~\ref{cor:early-mean-euler}.

Lemma~\ref{lem:early-euler-posterior-localization} gives
\(
\hat{\by}_{t_i}(\bz_i)
\in
B_{r_{\mathrm{mean}}}(\bar{\by})
\)
whenever $t_i\le t_{\mathrm{mean}}$.
For every $i$ such that $t_{i+1}\le t_{\mathrm{mean}}$, the Euler
update can be written as
\[
\bz_{i+1}
=
\tfrac{1-t_{i+1}}{1-t_i}\bz_i
+
\tfrac{t_{i+1}-t_i}{1-t_i}
\hat{\by}_{t_i}(\bz_i).
\]
 The set \(B_{r_{\mathrm{mean}}}(\bar{\by})\) is nonempty, closed and convex. If \(i>0\), the projection of \(\bz_i\) onto \(t_iB_{r_{\mathrm{mean}}}(\bar{\by})\) is well-defined.
Part~(i) of Lemma~\ref{lem:euler-scaled-convex-contraction}, applied with
\(
s=t_i,
\
t=t_{i+1},\)
\(\bx=\bz_i\),
\(\by=\hat{\by}_{t_i}(\bz_i)\)
\(
S=C=B_{r_{\mathrm{mean}}}(\bar{\by}),
\)
therefore gives
\begin{align}\label{eq:disc-early-one-step}
\dist\!\left(
\bz_{i+1},
t_{i+1}B_{r_{\mathrm{mean}}}(\bar{\by})
\right)
\le
\tfrac{1-t_{i+1}}{1-t_i}
\dist\!\left(
\bz_i,
t_iB_{r_{\mathrm{mean}}}(\bar{\by})
\right).
\end{align}
If \(i = 0\), then \(\hat{\by}_i(\bz_i) = \bar{\by} \in B_{r_{\mathrm{mean}}}(\bar{\by})\). By part~(i) of Lemma~\ref{lem:euler-scaled-convex-contraction}, \eqref{eq:disc-early-one-step} also holds.
Iterating from $j$ to $i$ gives
\[
\dist\!\left(
\bz_i,
t_iB_{r_{\mathrm{mean}}}(\bar{\by})
\right)
\le
\tfrac{1-t_i}{1-t_j}
\dist\!\left(
\bz_j,
t_jB_{r_{\mathrm{mean}}}(\bar{\by})
\right)
\]
for every $0\le j\le i$ with $t_i\le t_{\mathrm{mean}}$.

Part~(i) shows that
\(
\dist\!\left(
\bz_j,
t_jB_{r_{\mathrm{mean}}}(\bar{\by})
\right)>0.
\)
Hence the displayed estimate is strict whenever $j<i$, because
\(
\tfrac{1-t_i}{1-t_j}<1.
\)
This proves part~(ii) of the theorem.

Finally, taking $j=0$ and using
\(
0B_{r_{\mathrm{mean}}}(\bar{\by})=\{\bzero\}
\)
gives
\(
\dist\!\left(
\bz_i,
t_iB_{r_{\mathrm{mean}}}(\bar{\by})
\right)
\le
(1-t_i)\|\bx_0\|,
\)
which proves part~(iii) of the theorem.
The condition $\Delta t\le t_{\mathrm{mean}}/2$ ensures that the
early grid segment contains at least the two nonzero grid points
$t_1$ and $t_2$.
\end{proof}

\subsection{Convex-hull attraction and absorption}

\subsubsection{Proof of Theorem~\ref{thm:convex-hull}}\label{proof:thm:convex-hull}

\begin{proof}
\textbf{Attraction.}
Since \(\mathcal D\) is compact, \(\Conv(\mathcal D)\)
is compact and convex. Lemma~\ref{lem:mean-in-closed-conv} gives
\[
\hat{\by}_t(\bx_t)\in \Conv(\mathcal D),
\qquad t\in[0,1).
\]
The set \(\Conv(\mathcal D)\) is nonempty, closed, and convex, so projection onto
\(t\Conv(\mathcal D)\) is well-defined for every \(t>0\). Applying Lemma~\ref{lem:continuous-scaled-convex-contraction} on
\(I=(0,1)\) with
\(
S = C_t = \Conv(\mathcal D),
\
\boldsymbol m_t=\hat{\by}_t(\bx_t)
\) gives
\begin{equation}\label{eq:positive-time-convex-hull}
\dist(\bx_t,t\Conv(\mathcal D))
\le
\tfrac{1-t}{1-s}\dist(\bx_s,s\Conv(\mathcal D)),
\qquad
0<s\le t<1.
\end{equation}

We next extend the estimate to \(s=0\). 
Recall that
\(
M_{\mathcal D}:=\max_{\by\in\mathcal D}\|\by\|.
\)
If
\(
\by=\sum_{j=1}^m\lambda_j\by_j\in\Conv(\mathcal D),
\)
where \(\lambda_j\ge0\), \(\sum_{j=1}^m\lambda_j=1\), and
\(\by_j\in\mathcal D\), then
\(
\|\by\|
\le
\sum_{j=1}^m\lambda_j\|\by_j\|
\le
M_{\mathcal D}.
\)
Since \(\mathcal D\subset\Conv(\mathcal D)\), it follows that
\(
\sup_{\by\in\Conv(\mathcal D)}\|\by\|
=
M_{\mathcal D}.
\)
For every \(s>0\) and every \(\by\in\Conv(\mathcal D)\), the triangle
inequality gives
\[
\begin{aligned}
\|\bx_s-s\by\|
\ge
\|\bx_0\|-\|\bx_s-\bx_0\|-s\|\by\|
\ge
\|\bx_0\|-\|\bx_s-\bx_0\|-sM_{\mathcal D}.
\end{aligned}
\]
Taking the infimum over \(\by\in\Conv(\mathcal D)\) yields
\(
\dist\!\left(
\bx_s,
s\Conv(\mathcal D)
\right)
\ge
\|\bx_0\|-\|\bx_s-\bx_0\|-sM_{\mathcal D}.
\)

On the other hand, for every \(\by\in\Conv(\mathcal D)\),
\[
\begin{aligned}
\|\bx_s-s\by\|
\le
\|\bx_0\|+\|\bx_s-\bx_0\|+s\|\by\|
\le
\|\bx_0\|+\|\bx_s-\bx_0\|+sM_{\mathcal D}.
\end{aligned}
\]
Taking the infimum again gives
\(
\dist\!\left(
\bx_s,
s\Conv(\mathcal D)
\right)
\le
\|\bx_0\|+\|\bx_s-\bx_0\|+sM_{\mathcal D}.
\)
Combining the two bounds, we obtain
\begin{equation}\label{eq:uncond-zero-time-distance-bound}
\left|
\dist\!\left(
\bx_s,
s\Conv(\mathcal D)
\right)
-\|\bx_0\|
\right|
\le
\|\bx_s-\bx_0\|+sM_{\mathcal D}.
\end{equation}
Since \(\bx_s\to\bx_0\) as \(s\downarrow0\), we have
\(
\dist(\bx_s,s\Conv(\mathcal D))
\rightarrow
\|\bx_0\|
=
\dist(\bx_0,0\Conv(\mathcal D)).
\)
Letting \(s\downarrow0\) in
\eqref{eq:positive-time-convex-hull} proves 
\[
\dist\!\left(
\bx_t,
t\Conv(\mathcal D)
\right)
\le
\tfrac{1-t}{1-s}
\dist\!\left(
\bx_s,
s\Conv(\mathcal D)
\right),
\qquad
0\le s\le t<1.
\]
Since \((1-t)/(1-s)\le1\), the distance is
non-increasing, and taking \(s=0\) gives
\begin{align}\label{eq:conv-0}
\dist(\bx_t,t\Conv(\mathcal D))
\le
(1-t)\|\bx_0\|.
\end{align}
This proves parts~(i)--(ii) of the theorem.

\noindent
\textbf{Absorption.}
Suppose first that
\(t_{\mathrm{conv}}\in(0,1)\) and
\(\bx_{t_{\mathrm{conv}}}\in
t_{\mathrm{conv}}\Conv(\mathcal D)\).
Part~(ii) of Lemma~\ref{lem:continuous-scaled-convex-contraction}, applied on \(I=(0,1)\) with
\(\boldsymbol m_t=\hat{\by}_t(\bx_t)\) and
\(S = C_t=\Conv(\mathcal D)\), gives
\(
\bx_t\in t\Conv(\mathcal D),
\) for \(
t\in[t_{\mathrm{conv}},1).
\)
If \(t_{\mathrm{conv}}=0\), then we get 
\(\bx_0=\boldsymbol0\). Therefore, \eqref{eq:conv-0} has zero
right-hand side and yields the same conclusion.
This proves part~(iii) of the theorem.
\end{proof}

\subsubsection{Proof of Theorem \ref{cor:convex-hull-euler}}\label{proof:cor:convex-hull-euler}

\begin{proof}
Since \(\mathcal D\) is compact, \(\Conv(\mathcal D)\) is compact and
convex. By Lemma~\ref{lem:mean-in-closed-conv},
\(
\hat{\by}_{t_i}(\bz_i)\in\Conv(\mathcal D)
\)
for \(i=0,\ldots,N-1\). Moreover, the Euler update can be written as
\[
\bz_{i+1}
=
\tfrac{1-t_{i+1}}{1-t_i}\bz_i
+
\tfrac{t_{i+1}-t_i}{1-t_i}
\hat{\by}_{t_i}(\bz_i).
\]

We first consider \(i=0\). Since \(t_0=0\) and
\(\hat{\by}_0(\bz_0)\in\Conv(\mathcal D)\), the zero-time case of
part~(i) of Lemma~\ref{lem:euler-scaled-convex-contraction} applies
with \(s=t_0\), \(t=t_1\), \(\bx=\bz_0\),
\(\by=\hat{\by}_0(\bz_0)\), and \(S=\Conv(\mathcal D)\), and gives
\eqref{eq:convex-euler-one-step} below for \(i=0\).
Now fix \(i\in\{1,\ldots,N-1\}\), and let
\(\bq_i:=\Proj_{t_i\Conv(\mathcal D)}(\bz_i)\). This projection
is well-defined. 

\noindent
\textbf{Attraction.}
Here, we have  \(t_i>0\), \({\bq_i}/{t_i}\in\Conv(\mathcal D)\) and \(\hat{\by}_{t_i}(\bz_i)\in\Conv(\mathcal D).\)
Therefore, part~(i) of
Lemma~\ref{lem:euler-scaled-convex-contraction} applies with
\(s=t_i\), \(t=t_{i+1}\), \(\bx=\bz_i\),
\(\by=\hat{\by}_{t_i}(\bz_i)\), and
\(S=C=\Conv(\mathcal D)\), and gives
\eqref{eq:convex-euler-one-step} below. Thus, for every
\(i=0,\ldots,N-1\),
\begin{equation}\label{eq:convex-euler-one-step}
\dist\bigl(\bz_{i+1},t_{i+1}\Conv(\mathcal D)\bigr)
\le
\tfrac{1-t_{i+1}}{1-t_i}
\dist\bigl(\bz_i,t_i\Conv(\mathcal D)\bigr).
\end{equation}
Fix \(0\le j<i\le N\). Iterating
\eqref{eq:convex-euler-one-step} from \(j\) to \(i-1\) and
telescoping gives
\[
\dist\bigl(\bz_i,t_i\Conv(\mathcal D)\bigr)
\le
\left(
\prod_{\ell=j}^{i-1}
\tfrac{1-t_{\ell+1}}{1-t_\ell}
\right)
\dist\bigl(\bz_j,t_j\Conv(\mathcal D)\bigr)
=
\tfrac{1-t_i}{1-t_j}
\dist\bigl(\bz_j,t_j\Conv(\mathcal D)\bigr).
\]
For \(j=i\) with \(t_j<1\), the same estimate holds trivially.
Hence, the claimed pairwise estimate holds for every
\(0\le j\le i\le N\) with \(t_j<1\). Since
\(0\le(1-t_i)/(1-t_j)\le1\), the distance sequence is
nonincreasing. This proves part~(i) of the theorem.

Taking \(j=0\) and using \(t_0=0\), \(\bz_0=\bx_0\), and
\(0\Conv(\mathcal D)=\{\boldsymbol0\}\) gives
\[
\dist\bigl(\bz_i,t_i\Conv(\mathcal D)\bigr)
\le
(1-t_i)\|\bx_0\|,
\qquad
i=0,\ldots,N.
\]
This proves part~(ii) of the theorem.

\noindent
\textbf{Absorption.}
Suppose that
\(\bz_k\in t_k\Conv(\mathcal D)\)
for some \(k\in\{0,\ldots,N\}\).
We prove by induction that
\[
\bz_i\in t_i\Conv(\mathcal D),
\qquad
i=k,\ldots,N.
\]

The assertion at \(i=k\) is given.
Now fix \(i\in\{k,\ldots,N-1\}\) and suppose that
\(\bz_i\in t_i\Conv(\mathcal D)\).
If \(t_i>0\), then
\(\bz_i/t_i\in\Conv(\mathcal D)\) and
\(\hat{\by}_{t_i}(\bz_i)\in\Conv(\mathcal D)\).
Thus part~(ii) of
Lemma~\ref{lem:euler-scaled-convex-contraction}
applies with
\(s=t_i\), \(t=t_{i+1}\),
\(\bx=\bz_i\),
\(\by=\hat{\by}_{t_i}(\bz_i)\), and
\(S = C=\Conv(\mathcal D)\), and gives
\(
\bz_{i+1}\in t_{i+1}\Conv(\mathcal D).
\)
If \(t_i=0\), then
\(\bz_i\in0\Conv(\mathcal D)=\{\boldsymbol0\}\), while
\(\hat{\by}_0(\bz_i)\in\Conv(\mathcal D)\), so part~(ii) of
Lemma~\ref{lem:euler-scaled-convex-contraction} again gives the same conclusion.
Therefore the absorption follows by induction. This proves part~(iii) of the theorem.
\end{proof}

\subsection{Final-stage local-cluster attraction and absorption}

\subsubsection{Local geometric and posterior-localization lemmas}
For the final-stage results, we need to utilize several auxiliary lemmas.
Lemma~\ref{lem:prox-regular-projection-estimate} gives the basic
projection estimate, while
Lemmas~\ref{lem:polynomial-local-mass-radius} and
\ref{lem:inflated-neighborhood-projection-regularity} provide the local
mass and inflated-neighborhood geometry.
Lemma~\ref{lem:final-local-posterior} combines these ingredients with
cluster isolation to localize the posterior mean. Finally, Lemma~\ref{lem:continuous-scaled-convex-contraction}
then gives attraction and absorption.

A standard consequence of prox-regularity is the proximal-normal
estimate; see
\citet{poliquin2000local} and
\citet{rockafellar1998variational}. For \(t>0\), the scaled set
\(t\Omega\) is \(t\rho_\Omega\)-prox-regular. Hence, whenever
\(0<\dist(\bx,t\Omega)<t\rho_\Omega\), the projection
\(\Proj_{t\Omega}(\bx)\) is well-defined and, for every
\(\bz\in t\Omega\),
\begin{equation}\label{eq:prox-regular-basic-estimate}
\left\langle
\bx-\Proj_{t\Omega}(\bx),
\bz-\Proj_{t\Omega}(\bx)
\right\rangle
\le
\tfrac{\dist(\bx,t\Omega)}{2t\rho_\Omega}
\left\|\bz-\Proj_{t\Omega}(\bx)\right\|^2.
\end{equation}
This estimate is used in
Lemma~\ref{lem:prox-regular-projection-estimate}.

\begin{lemma}[Projection estimate for prox-regular clusters]\label{lem:prox-regular-projection-estimate}
Let \(\Omega\subset\R^d\) be nonempty, closed, and  prox-regular with radius
\(\rho_\Omega>0\). Fix
\(R\in(0,\rho_\Omega)\) and \(t>0\). If \(\bx\in\R^d\) satisfies
\(\dist(\bx,t\Omega)<tR\), then the projection
\(\Proj_{t\Omega}(\bx)\) is well-defined. Moreover, for every
\(\bz\in t\Omega\),
\begin{equation}\label{eq:prox-regular-projection-estimate}
\left\|\Proj_{t\Omega}(\bx)-\bz\right\|^2
\le
\frac{1}{1-R/\rho_\Omega}
\left(
\|\bx-\bz\|^2
-
\left\|\bx-\Proj_{t\Omega}(\bx)\right\|^2
\right).
\end{equation}
\end{lemma}

\begin{proof}
Suppose first that \(\dist(\bx,t\Omega)=0\). Since \(t\Omega\) is
closed, \(\bx\in t\Omega\), so the projection is well-defined and
\(\Proj_{t\Omega}(\bx)=\bx\). Thus, for every \(\bz\in t\Omega\),
\[
\left\|\Proj_{t\Omega}(\bx)-\bz\right\|^2
=
\|\bx-\bz\|^2
\le
\frac{1}{1-R/\rho_\Omega}
\left(
\|\bx-\bz\|^2
-
\left\|\bx-\Proj_{t\Omega}(\bx)\right\|^2
\right),
\]
where the inequality follows from \(0<R<\rho_\Omega\).

Now suppose that \(0<\dist(\bx,t\Omega)<tR\). The scaled set
\(t\Omega\) is \(t\rho_\Omega\)-prox-regular, and
\(\dist(\bx,t\Omega)<tR<t\rho_\Omega\); hence
\(\Proj_{t\Omega}(\bx)\) is well-defined. Write
\(\bq:=\Proj_{t\Omega}(\bx)\). By
\eqref{eq:prox-regular-basic-estimate}, for every \(\bz\in t\Omega\),
\[
\left\langle
\bx-\bq,
\bz-\bq
\right\rangle
\le
\tfrac{\dist(\bx,t\Omega)}{2t\rho_\Omega}
\|\bz-\bq\|^2.
\]
Consequently,
\[
\begin{aligned}
\|\bx-\bz\|^2-\|\bx-\bq\|^2
&=
\|\bq-\bz\|^2
-
2\left\langle
\bx-\bq,
\bz-\bq
\right\rangle\\
&\ge
\left(
1-\tfrac{\dist(\bx,t\Omega)}{t\rho_\Omega}
\right)
\|\bq-\bz\|^2
\ge
\left(
1-\tfrac{R}{\rho_\Omega}
\right)
\|\bq-\bz\|^2.
\end{aligned}
\]
Since \(1-R/\rho_\Omega>0\), dividing by this quantity and substituting
\(\bq=\Proj_{t\Omega}(\bx)\) proves
\eqref{eq:prox-regular-projection-estimate}.
\end{proof}

\begin{lemma}[Polynomial local-mass radius]\label{lem:polynomial-local-mass-radius}
Under Assumption~\ref{assump:local-cluster}, there exists
$t_{\mathrm{loc}}\in(0,1)$ satisfying
\eqref{eq:final-local-time-choice}. For any such $t_{\mathrm{loc}}$,
define $r_\Omega$, $\varepsilon_\Omega$, and $c_{\Omega,\mathrm{loc}}$
by \eqref{eq:final-local-inflation-radius}--\eqref{eq:final-local-quantities}.
Suppose
\(t\in[t_{\mathrm{loc}},1)\) and
\(\dist(\bx,t\Omega)<tR\). Let
\(
\boldsymbol q
:=
\operatorname{Proj}_{t\Omega}(\boldsymbol x)
\) and define
\begin{equation}
    \label{eq:poly-local-mass-set}
\Omega_t^0(\boldsymbol x)
:=
\left\{
\boldsymbol y\in\Omega:
\|\boldsymbol x-t\boldsymbol y\|^2
-
\|\boldsymbol x-\boldsymbol q\|^2
\le
\tfrac{\varepsilon_\Omega^2}{2}
\right\},
\end{equation}
then we have \(p_\emptyset\!\left(\Omega_t^0(\boldsymbol x)\right)
\ge c_{\Omega,\mathrm{loc}}.\)
Moreover, the inflation radius satisfies
\begin{equation}\label{eq:uncond-final-infla-radius}
r_\Omega
=
\bigO\!\left(
(1-t_{\mathrm{loc}})
\sqrt{
\log\tfrac{1}{1-t_{\mathrm{loc}}}
}
\right),
\qquad
t_{\mathrm{loc}}\uparrow1,
\end{equation}
where the implicit constant depends only on
\(R,\rho_\Omega,D_\Omega,c_\Omega,\alpha_\Omega,
M_{\mathcal D},p_\emptyset(\Omega)\), and \(\gamma_\Omega\).
\end{lemma}

\begin{proof}
Admissible final-stage times exist because
$(1-t)^2\log(1/(1-t))\to0$ as $t\uparrow1$, while both entries in the
minimum on the right-hand side of \eqref{eq:final-local-time-choice} are
strictly positive. 
The exact restriction \eqref{eq:final-local-time-choice} implies
$0<\varepsilon_\Omega\le1$ and
$\varepsilon_\Omega^2/(8R+4)\le s_\Omega$.
Moreover,
\(
\boldsymbol q/t=\by_t^\Omega(\bx)\in\Omega.
\)
Consider any
\(
\boldsymbol y
\in
\Omega\cap
B_{\varepsilon_\Omega^2/(8R+4)}
\!\left(\frac{\boldsymbol q}{t}\right).
\) Since \(t<1\), we have 
\(
\|t\boldsymbol y-\boldsymbol q\|
=
t\left\|
\boldsymbol y-\tfrac{\boldsymbol q}{t}
\right\|
\le
t\tfrac{\varepsilon_\Omega^2}{8R+4}
\le
\tfrac{\varepsilon_\Omega^2}{8R+4}.
\)
Also, because \(\boldsymbol q\) is the projection of \(\boldsymbol x\) onto \(t\Omega\), we obtain 
\(
\|\boldsymbol x-\boldsymbol q\|
=
\operatorname{dist}(\boldsymbol x,t\Omega)
<tR
\le R.
\) Expanding the difference of squared distances and applying the Cauchy–Schwarz inequality gives
\[
\begin{aligned}
\|\boldsymbol x-t\boldsymbol y\|^2
-
\|\boldsymbol x-\boldsymbol q\|^2
&=
2\left\langle
\boldsymbol x-\boldsymbol q,
\boldsymbol q-t\boldsymbol y
\right\rangle
+
\|\boldsymbol q-t\boldsymbol y\|^2\le
\tfrac{2R\varepsilon_\Omega^2}{8R+4}
+
\left(
\tfrac{\varepsilon_\Omega^2}{8R+4}
\right)^2.
\end{aligned}
\]Since \(R>0\), then 
\(
\frac{2R}{8R+4}\le\frac14.
\) Moreover, since \(0<\varepsilon_\Omega\le1\) and \(8R+4\ge4\),
\(
\left(
\tfrac{\varepsilon_\Omega^2}{8R+4}
\right)^2
=
\varepsilon_\Omega^2
\tfrac{\varepsilon_\Omega^2}{(8R+4)^2}
\le
\tfrac{\varepsilon_\Omega^2}{16}.
\) It follows that
\[
\begin{aligned}
\|\boldsymbol x-t\boldsymbol y\|^2
-
\|\boldsymbol x-\boldsymbol q\|^2
&\le
\left(\tfrac14+\tfrac1{16}\right)\varepsilon_\Omega^2=
\tfrac5{16}\varepsilon_\Omega^2<
\tfrac{\varepsilon_\Omega^2}{2}.
\end{aligned}
\]
Therefore,
\(
\Omega\cap
B_{\varepsilon_\Omega^2/(8R+4)}
\!\left(\tfrac{\boldsymbol q}{t}\right)
\subseteq
\Omega_t^0(\boldsymbol x).
\)
Since
\(
\frac{\varepsilon_\Omega^2}{8R+4}\le s_\Omega,
\) the uniform mass condition in
Assumption~\ref{assump:local-cluster}(ii), applied at
\(\boldsymbol q/t\in\Omega\), implies
\[
\begin{aligned}
p_\emptyset\!\left(\Omega_t^0(\boldsymbol x)\right)
&\ge
p_\emptyset\!\left(
\Omega\cap
B_{\varepsilon_\Omega^2/(8R+4)}
\!\left(\tfrac{\boldsymbol q}{t}\right)
\right)\ge
c_\Omega
\left(
\tfrac{\varepsilon_\Omega^2}{8R+4}
\right)^{\alpha_\Omega}=
c_{\Omega,\mathrm{loc}}.
\end{aligned}
\]

It remains to verify \eqref{eq:uncond-final-infla-radius}. 
By the definition of \(\varepsilon_\Omega\), the first term satisfies
\[
\tfrac{
\varepsilon_\Omega
}{
t_{\mathrm{loc}}\sqrt{1-R/\rho_\Omega}
}
=
\bigO\!\left(
(1-t_{\mathrm{loc}})
\sqrt{
\log\tfrac{1}{1-t_{\mathrm{loc}}}
}
\right),
\]
because \(1/t_{\mathrm{loc}}\) remains uniformly bounded as \(t_{\mathrm{loc}}\uparrow1\).
For the second term, the definition of \(\varepsilon_\Omega\)
gives the exact identity
\[
\exp\!\left(
-\tfrac{\varepsilon_\Omega^2}{4(1-t_{\mathrm{loc}})^2}
\right)
=
(1-t_{\mathrm{loc}})^{2\alpha_\Omega+1}.
\]
Moreover,
\[
c_{\Omega,\mathrm{loc}}
=
c_\Omega
\left(
\tfrac{2\alpha_\Omega+1}{2R+1}
\right)^{\alpha_\Omega}
(1-t_{\mathrm{loc}})^{2\alpha_\Omega}
\left(
\log\tfrac{1}{1-t_{\mathrm{loc}}}
\right)^{\alpha_\Omega}.
\]
Therefore,
\[
\begin{aligned}
&
\tfrac{D_\Omega}{c_{\Omega,\mathrm{loc}}}
\exp\!\left(
-\tfrac{\varepsilon_\Omega^2}
{4(1-t_{\mathrm{loc}})^2}
\right)
=
\tfrac{D_\Omega}{c_\Omega}
\left(
\tfrac{2R+1}{2\alpha_\Omega+1}
\right)^{\alpha_\Omega}
\tfrac{
1-t_{\mathrm{loc}}
}{
\left(
\log\tfrac{1}{1-t_{\mathrm{loc}}}
\right)^{\alpha_\Omega}
}.
\end{aligned}
\]
Since
\(
\tfrac{
1-t_{\mathrm{loc}}
}{
\left(
\log\tfrac{1}{1-t_{\mathrm{loc}}}
\right)^{\alpha_\Omega}
}
=
\bigO\!\left(
(1-t_{\mathrm{loc}})
\sqrt{
\log\frac{1}{1-t_{\mathrm{loc}}}
}
\right),
\)
the second term has the required order.
Finally, Assumption~\ref{assump:local-cluster}(iii) gives
\(\gamma_\Omega>0\), and hence
\(
\exp\!\left(
-\tfrac{
t_{\mathrm{loc}}^2\gamma_\Omega
}{
2(1-t_{\mathrm{loc}})^2
}
\right)
=
\littleo\!\left(
(1-t_{\mathrm{loc}})^m
\right)
\)
for every $m>0$. 
Taking \(m=1\), the third contribution is
\(\littleo(1-t_{\mathrm{loc}})\).
Applying these three estimates to the exact expression for \(r_\Omega\)
proves \eqref{eq:uncond-final-infla-radius}.
\end{proof}

\begin{lemma}[Inflated-neighborhood projection]\label{lem:inflated-neighborhood-projection-regularity}
Let \(\Omega\) be \(\rho_\Omega\)-prox-regular, let
\(R\in(0,\rho_\Omega)\), and let \(r\in[0,R)\). Suppose that
\(t\in(0,1)\) and \(\bx\in\R^d\) satisfy
\(
\dist(\bx,t\Omega)<tR.
\)
Define
\(
\bq:=\Proj_{t\Omega}(\bx).
\)
Then \(\bq\) is well-defined, and
\(\Proj_{tB_r(\Omega)}(\bx)\) is also well-defined, with
\begin{equation}\label{eq:inflated-neighborhood-proj-formula}
\Proj_{tB_r(\Omega)}(\bx)
=
\begin{cases}
\bx,
& \|\bx-\bq\|\le rt,
\\[1.2ex]
\displaystyle
\bq+rt\,\tfrac{\bx-\bq}{\|\bx-\bq\|},
& \|\bx-\bq\|>rt.
\end{cases}
\end{equation}
In both cases,
\begin{equation}\label{eq:inflated-neighborhood-normalized-proj}
\Proj_{tB_r(\Omega)}(\bx)/t
\in
B_r\left({\bq}/{t}\right).
\end{equation}
Moreover, the same point is the projection of \(\bx\) onto the closed
 ball \(B_{rt}(\bq)\).
\end{lemma}

\begin{proof}
Since \(\dist(\bx,t\Omega)<tR<t\rho_\Omega\), prox-regularity makes
\(\bq=\Proj_{t\Omega}(\bx)\) well-defined.

Suppose first that \(\|\bx-\bq\|\le rt\). Since \(\bq\in t\Omega\),
this implies \(\bx\in tB_r(\Omega)\). Hence the projection of \(\bx\)
onto \(tB_r(\Omega)\) is uniquely \(\bx\). Moreover,
\(
\left\|\frac{\bx}{t}-\frac{\bq}{t}\right\|\le r,
\)
which gives \eqref{eq:inflated-neighborhood-normalized-proj}. Since
\(\bx\in B_{rt}(\bq)\), it is also its own projection onto that closed
 ball.

It remains to consider \(\|\bx-\bq\|>rt\).
For any \(\boldsymbol{\zeta}\in tB_r(\Omega)\), choose
\(\boldsymbol{\omega}\in t\Omega\) with
\(\|\boldsymbol{\zeta}-\boldsymbol{\omega}\|\le rt\), then
\[
\|\bx-\boldsymbol{\zeta}\|
\ge
\|\bx-\boldsymbol{\omega}\|
-
\|\boldsymbol{\zeta}-\boldsymbol{\omega}\|
\ge
\|\bx-\bq\|-rt.
\]
Define
\(
\boldsymbol p
:=
\bq+rt\,\frac{\bx-\bq}{\|\bx-\bq\|}.
\)
Since \(\|\boldsymbol p-\bq\|=rt\) and \(\bq\in t\Omega\), we have
\(\boldsymbol p\in tB_r(\Omega)\). Moreover,
\[
\|\bx-\boldsymbol p\|= \left(
1-\tfrac{rt}{\|\boldsymbol{x}-\boldsymbol{q}\|}
\right)\|\bx-\bq\|=\|\bx-\bq\|-rt.
\]
Thus \(\boldsymbol p\) attains the preceding lower bound and is a
projection of \(\bx\) onto \(tB_r(\Omega)\). If
\(\boldsymbol{\zeta}\) also attains this lower bound, then
\[
\begin{aligned}
\|\bx-\bq\|-rt
&=
\|\bx-\boldsymbol{\zeta}\| \ge
\|\bx-\boldsymbol{\omega}\|
-
\|\boldsymbol{\zeta}-\boldsymbol{\omega}\| \ge
\|\bx-\bq\|-rt.
\end{aligned}
\]
Hence,  both inequalities are equalities. Moreover,
\[
0
=
\underbrace{
\|\bx-\boldsymbol{\omega}\|-\|\bx-\bq\|
}_{\ge 0}
+
\underbrace{
rt-\|\boldsymbol{\zeta}-\boldsymbol{\omega}\|
}_{\ge 0}.
\]
Therefore,  both nonnegative terms vanish, so
\[
\|\bx-\boldsymbol{\omega}\|=\|\bx-\bq\|,
\qquad
\|\boldsymbol{\zeta}-\boldsymbol{\omega}\|=rt.
\]
The first equality shows that \(\boldsymbol{\omega}\) is also a
projection of \(\bx\) onto \(t\Omega\). Since this projection is
unique, \(\boldsymbol{\omega}=\bq\). Thus
\[
\|\boldsymbol{\zeta}-\bq\|=rt.
\]
Equality in the first inequality above, together with
\(\boldsymbol{\omega}=\bq\), gives
\[
\|\bx-\bq\|
=
\|\bx-\boldsymbol{\zeta}\|
+
\|\boldsymbol{\zeta}-\bq\|.
\]
Equality in the Euclidean triangle inequality implies that
\(\boldsymbol{\zeta}-\bq\) is a nonnegative multiple of
\(\bx-\bq\). Since \(\|\boldsymbol{\zeta}-\bq\|=rt\), it follows that
\[
\boldsymbol{\zeta}
=
\bq+rt\,\tfrac{\bx-\bq}{\|\bx-\bq\|}
=
\boldsymbol p.
\]
Thus \eqref{eq:inflated-neighborhood-proj-formula} is the unique
projection onto \(tB_r(\Omega)\). Dividing by \(t\) proves
\eqref{eq:inflated-neighborhood-normalized-proj}. The same radial
formula is the projection onto
\(B_{rt}(\bq)\).
\end{proof}

The goal of Lemma~\ref{lem:final-local-posterior} below is to control
\(
\left\|
t\by_t^\Omega(\bx)-t\hat{\by}_t(\bx)
\right\|,
\)
where
\(t\by_t^\Omega(\bx)=\Proj_{t\Omega}(\bx)\)
is the projection point in the target cluster. To organize the proof,
we introduce the posterior mean restricted to \(\Omega\) as an auxiliary quantity:
\begin{equation}\label{eq:local-posterior-mean}
\hat{\by}_t^\Omega(\bx)
:=
\frac{
\int_{\Omega} \by
\exp\!\left(
-\frac{\|\bx-t\by\|^2}{2(1-t)^2}
\right)
p_\emptyset(\dd \by)
}{
\int_{\Omega}
\exp\!\left(
-\frac{\|\bx-t\by\|^2}{2(1-t)^2}
\right)
p_\emptyset(\dd \by)
}.
\end{equation}
The triangle inequality gives the decomposition
\[
\left\|
t\by_t^\Omega(\bx)-t\hat{\by}_t(\bx)
\right\|
\le
\left\|
t\hat{\by}_t^\Omega(\bx)-t\by_t^\Omega(\bx)
\right\|
+
\left\|
t\hat{\by}_t^\Omega(\bx)-t\hat{\by}_t(\bx)
\right\|.
\]
The first term is the localization error within \(\Omega\),
it produces the first two
contributions to \(r_\Omega\):
\(
\frac{
\varepsilon_\Omega
}{
t_{\mathrm{loc}}\sqrt{1-R/\rho_\Omega}
}
\) and 
\(
\frac{D_\Omega}{c_{\Omega,\mathrm{loc}}}
\exp\!\left(
-\frac{\varepsilon_\Omega^2}
{4(1-t_{\mathrm{loc}})^2}
\right).
\)
The second term measures the difference between the posterior mean
restricted to \(\Omega\) and the full posterior mean on \(\mathcal{D}\), it 
produces the third contribution:
\(
2M_{\mathcal D}
\frac{1-p_\emptyset(\Omega)}
{p_\emptyset(\Omega)}
\exp\!\left(
-\frac{
t_{\mathrm{loc}}^2\gamma_\Omega
}{
2(1-t_{\mathrm{loc}})^2
}
\right).
\)
These three uniform bounds for \(t\ge t_{\mathrm{loc}}\) are exactly
the three terms in the definition of \(r_\Omega\).

\begin{lemma}[Posterior-mean localization]\label{lem:final-local-posterior}
Under Assumption~\ref{assump:local-cluster}, let
$t_{\mathrm{loc}}$, $\varepsilon_\Omega$,
$c_{\Omega,\mathrm{loc}}$, and $r_\Omega$ satisfy
\eqref{eq:final-local-time-choice}--\eqref{eq:final-local-quantities}.
Suppose that
$t\in [t_{\mathrm{loc}},1)$ and $\bx\in\R^d$ satisfy
$\dist(\bx,t\Omega)<tR,$
then
\begin{equation}\label{eq:cluster-ball-bound}
\|t\by_t^\Omega(\bx) - t\hat{\by}_t(\bx)\| \le r_\Omega t,
\end{equation}
and consequently 
$t\hat{\by}_t(\bx) \in tB_{r_\Omega}(\by_t^\Omega(\bx)).$
\end{lemma}

\begin{proof}
Since $\dist(\bx,t\Omega)<tR$, the projection
$t\by_t^\Omega(\bx)=\Proj_{t\Omega}(\bx)$ is well defined.
With the restricted posterior mean defined in
\eqref{eq:local-posterior-mean}, we first compare it with the
projection point.

For $\by\in \Omega$, define the following auxiliary quantity
\[\Delta_t(\by) := \|\bx-t\by\|^2 - \|\bx-t\by_t^\Omega(\bx)\|^2.\] 
Next define auxiliary sets
\[\Omega_t^{\mathrm{near}}:=\{\by \in \Omega : \Delta_t(\by)\le \varepsilon_\Omega^2\}, \quad \Omega_t^{\mathrm{far}}:=\Omega \setminus \Omega_t^{\mathrm{near}}.\]
Writing
$w_t(\by)
:=
\exp\!\left(-\tfrac{\|\bx-t\by\|^2}{2(1-t)^2}\right),$
Assumption~\ref{assump:local-cluster}(ii) implies
\(p_\emptyset(\Omega)>0\). Since \(w_t(\by)>0\), we have
\(
\int_\Omega w_t(\by)p_\emptyset(\dd\by)>0.
\)
The definition of \(\hat{\by}_t^\Omega(\bx)\) gives
\[
t\hat{\by}_t^\Omega(\bx)-t\by_t^\Omega(\bx)
=
\tfrac{
\int_\Omega
\bigl(t\by-t\by_t^\Omega(\bx)\bigr)
w_t(\by)p_\emptyset(\dd\by)
}{
\int_\Omega w_t(\by)p_\emptyset(\dd\by)
}.
\]
Taking norms, and splitting
\(\Omega=\Omega_t^{\mathrm{near}}\cup\Omega_t^{\mathrm{far}}\), we obtain
\begin{align*}
\|t\hat{\by}_t^\Omega(\bx)-t\by_t^\Omega(\bx)\|
\le
\underbrace{\tfrac{
\int_{\Omega_t^{\mathrm{near}}}
\|t\by_t^\Omega(\bx)-t\by\|\,w_t(\by)
p_\emptyset(\dd \by)
}{
\int_{\Omega}w_t(\by)p_\emptyset(\dd \by)
}}_{I_{\mathrm{near}}}
+ \underbrace{
\tfrac{
\int_{\Omega_t^{\mathrm{far}}}
\|t\by_t^\Omega(\bx)-t\by\|\,w_t(\by)
p_\emptyset(\dd \by)
}{
\int_{\Omega} w_t(\by) p_\emptyset(\dd \by)
}}_{I_{\mathrm{far}}}.
\end{align*}
We now bound $I_{\mathrm{near}}$ and  $I_{\mathrm{far}}$ separately.

On $\Omega_t^{\mathrm{near}}$,
Lemma~\ref{lem:prox-regular-projection-estimate} gives
\(
\|t\by_t^\Omega(\bx)-t\by\|^2
\le
\frac{1}{1-R/\rho_\Omega}\Delta_t(\by)
\le
\frac{1}{1-R/\rho_\Omega}\varepsilon_\Omega^2,
\)
so we have
\begin{align}\label{eq:I_near}
    I_{\mathrm{near}} \le  \tfrac{1}{\sqrt{1-R/\rho_\Omega}}\varepsilon_\Omega \cdot 
\tfrac{
\int_{\Omega_t^{\mathrm{near}}}
\,w_t(\by)
p_\emptyset(\dd \by)
}{
\int_{\Omega}w_t(\by)p_\emptyset(\dd \by) 
} \le \tfrac{1}{\sqrt{1-R/\rho_\Omega}}\varepsilon_\Omega .
\end{align}

On $\Omega_t^{\mathrm{far}}$, both $t\by_t^\Omega(\bx)$ and $t\by$
lie in $t\Omega$, so their distance is at most \(tD_\Omega\).
Moreover, for $\by \in \Omega_t^{\mathrm{far}}$ and $\by_0 \in \Omega_t^0(\bx)$, where 
\(\Omega_t^0(\bx)\) is defined in \eqref{eq:poly-local-mass-set}, we have 
\[\|\bx-t\by\|^2 - \|\bx-t\by_0\|^2= \Delta_t(\by) - \Delta_t(\by_0) \ge\tfrac{\varepsilon_\Omega^2}{2},
\]
so we have a pointwise comparison
\begin{equation}\label{eq:near-far-pointwise-compare}
    w_t(\by) = w_t(\by_0)\exp\!\left(-\tfrac{\|\bx-t\by\|^2 - \|\bx-t\by_0\|^2}{2(1-t)^2}\right)
\le\exp\!\left(-\tfrac{\varepsilon_\Omega^2}{4(1-t)^2}\right)w_t(\by_0).
\end{equation}
Now we have the following identity
\begin{align}\label{eq:uncond-final-double-int}
p_\emptyset\bigl(\Omega_t^0(\bx)\bigr)
\int_{\Omega_t^{\mathrm{far}}}
w_t(\by)\,p_\emptyset(\dd\by)
=
\int_{\Omega_t^{\mathrm{far}}}
\int_{\Omega_t^0(\bx)}
w_t(\by)\,
p_\emptyset(\dd\by_0)\,
p_\emptyset(\dd\by).
\end{align}
Integrating \eqref{eq:near-far-pointwise-compare} over
\(\boldsymbol{y}\in\Omega_t^{\mathrm{far}}\) and
\(\boldsymbol{y}_0\in\Omega_t^0(\boldsymbol{x})\) gives
\begin{align*}
p_\emptyset\bigl(\Omega_t^0(\bx)\bigr)
\int_{\Omega_t^{\mathrm{far}}}
w_t(\by)\,p_\emptyset(\dd\by) & =  \int_{\Omega_t^{\mathrm{far}}}
\int_{\Omega_t^0(\bx)}
w_t(\by)\,
p_\emptyset(\dd\by_0)\,
p_\emptyset(\dd\by) \\
&\le
\exp\left(
-\tfrac{\varepsilon_\Omega^2}{4(1-t)^2}
\right)
\int_{\Omega_t^{\mathrm{far}}}
\int_{\Omega_t^0(\bx)}
w_t(\by_0)\,
p_\emptyset(\dd\by_0)\,
p_\emptyset(\dd\by)\\
&= \exp\left(
-\tfrac{\varepsilon_\Omega^2}{4(1-t)^2}
\right)p_\emptyset\bigl(\Omega_t^{\mathrm{far}}\bigr)
\int_{\Omega_t^0(\bx)}
w_t(\by_0)\,p_\emptyset(\dd\by_0)\\
&\le  \exp\left(
-\tfrac{\varepsilon_\Omega^2}{4(1-t)^2}
\right)
\int_{\Omega}
w_t(\by_0)\,p_\emptyset(\dd\by_0).
\end{align*}
where the last inequality is from  \(p_\emptyset(\Omega_t^{\mathrm{far}})\le 1\), and
\(\Omega_t^0(\bx)\subset\Omega\). 
By Lemma~\ref{lem:polynomial-local-mass-radius}, 
\(
p_\emptyset\bigl(\Omega_t^0(\bx)\bigr)
\ge
c_{\Omega,\mathrm{loc}}>0,
\) so we have
\[
\tfrac{
\int_{\Omega_t^{\mathrm{far}}}w_t(\by)p_\emptyset(\dd\by)
}{
\int_{\Omega}w_t(\by)p_\emptyset(\dd\by)
}
\le
\tfrac{1}{c_{\Omega,\mathrm{loc}}}
\exp\!\left(-\tfrac{\varepsilon_\Omega^2}{4(1-t)^2}\right).
\]
Since $\|t\by_t^\Omega(\bx)-t\by\| \le t\diam(\Omega) = tD_\Omega$ on $\Omega_t^{\mathrm{far}}$,  we conclude that
\begin{align}\label{eq:I_far}
    I_\mathrm{far} \le tD_\Omega  
\tfrac{
\int_{\Omega_t^{\mathrm{far}}}w_t(\by)p_\emptyset(\dd\by)
}{
\int_{\Omega}w_t(\by)p_\emptyset(\dd\by)
}
\le 
\tfrac{tD_\Omega}{c_{\Omega,\mathrm{loc}}}
\exp\!\left(-\tfrac{\varepsilon_\Omega^2}{4(1-t)^2}\right).
\end{align}
Combining the estimates \eqref{eq:I_near} and \eqref{eq:I_far} of $I_{\mathrm{near}}$ and  $I_{\mathrm{far}}$ gives
\begin{equation}\label{eq:local-posterior-bound}
\|t\hat{\by}_t^\Omega(\bx)-t\by_t^\Omega(\bx)\|
\le
\tfrac{1}{\sqrt{1-R/\rho_\Omega}}\varepsilon_\Omega
+
\tfrac{t D_\Omega}{c_{\Omega,\mathrm{loc}}}
\exp\!\left(-\tfrac{\varepsilon_\Omega^2}{4(1-t)^2}\right).
\end{equation}
We next compare the full posterior mean with the local posterior mean.
Write
\[
Z_{\Omega,t}
:=
\int_{\Omega} w_t(\by)p_\emptyset(\dd \by),
\qquad
Z_{\mathrm{out},t}
:=
\int_{\mathcal{D}\setminus\Omega}
w_t(\by)p_\emptyset(\dd \by).
\]
Splitting the numerator and denominator of the full posterior over
\(\Omega\) and \(\mathcal D\setminus\Omega\) gives
\begin{align*}
\hat{\by}_t(\bx)
=
\tfrac{
\int_\Omega \by w_t(\by)p_\emptyset(\dd\by)
+
\int_{\mathcal D\setminus\Omega}
\by w_t(\by)p_\emptyset(\dd\by)
}{
Z_{\Omega,t}+Z_{\mathrm{out},t}
}=
\tfrac{
Z_{\Omega,t}\hat{\by}_t^\Omega(\bx)
+
\int_{\mathcal D\setminus\Omega}
\by w_t(\by)p_\emptyset(\dd\by)
}{
Z_{\Omega,t}+Z_{\mathrm{out},t}
}.
\end{align*}
Consequently,
\begin{align}\label{eq:subtract-uncond-posterior}
\hat{\by}_t^\Omega(\bx)-\hat{\by}_t(\bx)
=
\tfrac{
Z_{\mathrm{out},t}\hat{\by}_t^\Omega(\bx)
-
\int_{\mathcal D\setminus\Omega}
\by w_t(\by)p_\emptyset(\dd\by)
}{
Z_{\Omega,t}+Z_{\mathrm{out},t}
}=
\tfrac{
\int_{\mathcal D\setminus\Omega}
\bigl(\hat{\by}_t^\Omega(\bx)-\by\bigr)
w_t(\by)p_\emptyset(\dd\by)
}{
Z_{\Omega,t}+Z_{\mathrm{out},t}
},
\end{align}
where the second equality uses
\(
Z_{\mathrm{out},t}\hat{\by}_t^\Omega(\bx)
=
\int_{\mathcal D\setminus\Omega}
\hat{\by}_t^\Omega(\bx)\,w_t(\by)p_\emptyset(\dd\by).
\)
Moreover,
\[
\|\hat{\by}_t^\Omega(\bx)\|
\le
\tfrac{
\int_\Omega \|\by\|w_t(\by)p_\emptyset(\dd\by)
}{
Z_{\Omega,t}
}
\le
M_{\mathcal D}.
\]
Thus, for \(\by\in\mathcal D\setminus\Omega\),
\(\|\hat{\by}_t^\Omega(\bx)-\by\|\le2M_{\mathcal D}\), and hence
\begin{align*}
\|t\hat{\by}_t^\Omega(\bx)-t\hat{\by}_t(\bx)\|
&\le
\tfrac{
t\int_{\mathcal D\setminus\Omega}
\|\hat{\by}_t^\Omega(\bx)-\by\|w_t(\by)p_\emptyset(\dd\by)
}{
Z_{\Omega,t}+Z_{\mathrm{out},t}
}\le
2tM_{\mathcal D}
\tfrac{Z_{\mathrm{out},t}}{Z_{\Omega,t}+Z_{\mathrm{out},t}}
\le
2tM_{\mathcal D}
\tfrac{Z_{\mathrm{out},t}}{Z_{\Omega,t}}.
\end{align*}
The assumed distance condition gives
$\|\bx-t\by_t^\Omega(\bx)\| \le tR.$
Since $t\by_t^\Omega(\bx) \in t\Omega$ and the diameter of $\Omega$
is $D_\Omega$, we obtain
$\|\bx-t\by\| \le t(R+D_\Omega),$
$\by \in \Omega,$
which implies
\begin{equation}\label{eq:Z-inside}
Z_{\Omega,t}
\ge
p_\emptyset(\Omega)
\exp\!\left(-\tfrac{t^2(R+D_\Omega)^2}{2(1-t)^2}\right).
\end{equation}
On the other hand, since $\bx\in B_{tR}(t\Omega)$ and
$\delta_\Omega=\dist(B_R(\Omega),\mathcal{D}\setminus\Omega)$,
scaling gives $\|\bx-t\by\| \ge t\delta_\Omega,$
$\by \in \mathcal{D}\setminus\Omega,$
so
\begin{equation}\label{eq:Z-outside}
Z_{\mathrm{out},t}
\le
(1-p_\emptyset(\Omega))
\exp\!\left(-\tfrac{t^2\delta_\Omega^2}{2(1-t)^2}\right).
\end{equation}
Hence
\begin{equation}\label{eq:global-local-gap}
\|t\hat{\by}_t^\Omega(\bx)-t\hat{\by}_t(\bx)\|
\le
2tM_{\mathcal{D}} \tfrac{1-p_\emptyset(\Omega)}{p_\emptyset(\Omega)}
\exp\!\left(
-\tfrac{t^2\gamma_\Omega}{2(1-t)^2}
\right).
\end{equation}
Combining \eqref{eq:local-posterior-bound} and \eqref{eq:global-local-gap} yields the desired bound
\begin{align*}
\|t\by_t^\Omega(\bx) - t\hat{\by}_t(\bx)\|\le
\tfrac{\varepsilon_\Omega}
{\sqrt{1-R/\rho_\Omega}}
+
\tfrac{t D_\Omega}{c_{\Omega,\mathrm{loc}}}
\exp\!\left(-\tfrac{\varepsilon_\Omega^2}{4(1-t)^2}\right)
+
2tM_{\mathcal{D}} \tfrac{1-p_\emptyset(\Omega)}{p_\emptyset(\Omega)}
\exp\!\left(
-\tfrac{t^2\gamma_\Omega}{2(1-t)^2}
\right).
\end{align*}
Assumption~\ref{assump:local-cluster}(iii) gives $\gamma_\Omega>0$.
Since \(t\ge t_{\mathrm{loc}}\),
\(
\frac{\varepsilon_\Omega}{\sqrt{1-R/\rho_\Omega}}
\le
t\,
\frac{\varepsilon_\Omega}
{t_{\mathrm{loc}}\sqrt{1-R/\rho_\Omega}}.
\)
Moreover, the two exponential factors are nonincreasing in \(t\), so
\begin{align*}
\exp\!\left(
-\tfrac{\varepsilon_\Omega^2}{4(1-t)^2}
\right)
\le
\exp\!\left(
-\tfrac{\varepsilon_\Omega^2}{4(1-t_{\mathrm{loc}})^2}
\right), \quad
\exp\!\left(
-\tfrac{t^2\gamma_\Omega}{2(1-t)^2}
\right)
\le
\exp\!\left(
-\tfrac{t_{\mathrm{loc}}^2\gamma_\Omega}
{2(1-t_{\mathrm{loc}})^2}
\right).
\end{align*}
Applying these three inequalities term by term to the preceding bound
gives
\[
\|t\by_t^\Omega(\bx)-t\hat{\by}_t(\bx)\|
\le
 r_\Omega t,
\]
which is \eqref{eq:cluster-ball-bound}.
Thus $t\hat{\by}_t(\bx) \in tB_{r_\Omega}(\by_t^\Omega(\bx)).$
\end{proof}

\subsubsection{Proof of Theorem~\ref{thm:final-local}}\label{proof:thm:final-local}

\begin{proof}
Condition~\eqref{eq:final-local-capture-condition} and
\(t_{\mathrm{loc}}<1\) give
\[
r_\Omega
\le
t_{\mathrm{loc}}R
-
\dist\bigl(
\bx_{t_{\mathrm{loc}}},
t_{\mathrm{loc}}B_{r_\Omega}(\Omega)
\bigr)
<R.
\]
Moreover, \(r_\Omega>0\) and \(t_{\mathrm{loc}}<1\), so
\[
\dist(\bx_{t_{\mathrm{loc}}},t_{\mathrm{loc}}\Omega)
\le
\dist\bigl(\bx_{t_{\mathrm{loc}}},
t_{\mathrm{loc}}B_{r_\Omega}(\Omega)\bigr)
+t_{\mathrm{loc}}r_\Omega
<
\dist\bigl(\bx_{t_{\mathrm{loc}}},
t_{\mathrm{loc}}B_{r_\Omega}(\Omega)\bigr)
+r_\Omega
\le
t_{\mathrm{loc}}R.
\]
Define the first-exit time
\[
\tau^*
:=
\inf
\left\{
t\in(t_{\mathrm{loc}},1):
\dist(\bx_t,t\Omega)\ge tR
\right\},
\]
with the convention \(\tau^*=1\) if the set is empty. Continuity and
the initial strict inequality give \(\tau^*>t_{\mathrm{loc}}\) and
\(
\dist(\bx_t,t\Omega)<tR,
\
t\in[t_{\mathrm{loc}},\tau^*).
\)
For  every \(t\in[t_{\mathrm{loc}},\tau^*)\), 
 we get that \(\Proj_{tB_{r_\Omega}(\Omega)}(\bx_t)\) is well-defined and that
\(
\Proj_{tB_{r_\Omega}(\Omega)}(\bx_t)/t\in B_{r_\Omega}\bigl(\by_t^\Omega(\bx_t)\bigr)
\) from Lemma~\ref{lem:inflated-neighborhood-projection-regularity}. In addition, 
Lemma~\ref{lem:final-local-posterior} gives
\(\hat{\by}_t(\bx_t)\in B_{r_\Omega}\bigl(\by_t^\Omega(\bx_t)\bigr).\) 

Now, the set \(B_{r_\Omega}\bigl(\by_t^\Omega(\bx_t)\bigr)\) is  nonempty, closed and
convex, and
\(
B_{r_\Omega}\bigl(\by_t^\Omega(\bx_t)\bigr)\subset B_{r_\Omega}(\Omega)
\)
because \(\by_t^\Omega(\bx_t)\in\Omega\). Hence, applying part~(i) of
Lemma~\ref{lem:continuous-scaled-convex-contraction}  on
\(I=[t_{\mathrm{loc}},\tau^*)\), with
\(S=B_{r_\Omega}(\Omega),C_t = B_{r_\Omega}\bigl(\by_t^\Omega(\bx_t)\bigr)\) and \(\boldsymbol m_t=\hat{\by}_t(\bx_t)\), gives, for every
\(t<\tau^*\),
\begin{equation}\label{eq:final-local-global-bound}
\dist\bigl(\bx_t,tB_{r_\Omega}(\Omega)\bigr)
\le
\tfrac{1-t}{1-t_{\mathrm{loc}}}
\dist\bigl(\bx_{t_{\mathrm{loc}}},
t_{\mathrm{loc}}B_{r_\Omega}(\Omega)\bigr).
\end{equation}
Combining \eqref{eq:final-local-global-bound} with the capture
condition gives, for \(t<\tau^*\),
\[
\begin{aligned}
\dist\bigl(\bx_t,tB_{r_\Omega}(\Omega)\bigr)+r_\Omega
\le
\tfrac{1-t}{1-t_{\mathrm{loc}}}
\bigl(t_{\mathrm{loc}}R-r_\Omega\bigr)+r_\Omega
=
tR-
\tfrac{t-t_{\mathrm{loc}}}{1-t_{\mathrm{loc}}}
(R-r_\Omega) \le tR.
\end{aligned}
\]
Suppose that \(\tau^*<1\). Letting \(t\uparrow\tau^*\) in the
preceding estimate and using \(\tau^*>t_{\mathrm{loc}}\) and
\(r_\Omega<R\) gives
\[
\dist\bigl(\bx_{\tau^*},\tau^*B_{r_\Omega}(\Omega)\bigr)+r_\Omega
<
\tau^*R.
\]
Consequently,
\[
\dist(\bx_{\tau^*},\tau^*\Omega)
\le
\dist\bigl(\bx_{\tau^*},\tau^*B_{r_\Omega}(\Omega)\bigr)
+\tau^*r_\Omega
\le
\dist\bigl(\bx_{\tau^*},\tau^*B_{r_\Omega}(\Omega)\bigr)+r_\Omega
<
\tau^*R.
\]
On the other hand, the first-exit definition and continuity imply
\(
\dist(\bx_{\tau^*},\tau^*\Omega)\ge\tau^*R,
\)
which is a contradiction. Hence \(\tau^*=1\).

\noindent
\textbf{Attraction.}
For any \(t_{\mathrm{loc}}\le s\le t<1\), the same argument as in \eqref{eq:final-local-global-bound} with
 part~(i) of
Lemma~\ref{lem:continuous-scaled-convex-contraction} gives
\[
\dist\bigl(\bx_t,tB_{r_\Omega}(\Omega)\bigr)
\le
\tfrac{1-t}{1-s}
\dist\bigl(\bx_s,sB_{r_\Omega}(\Omega)\bigr).
\]
Since \((1-t)/(1-s)\le1\), the distance is non-increasing on
\([t_{\mathrm{loc}},1)\). This proves part~(i) of the theorem.

Taking \(s=t_{\mathrm{loc}}\) gives
\eqref{eq:final-local-global-bound} for every
\(t\in[t_{\mathrm{loc}},1)\). Moreover, the factor multiplying
\(1-t\) on the right-hand side is independent of \(t\), so
\eqref{eq:final-local-global-bound} is \(\bigO(1-t)\). This proves
part~(ii) of the theorem.

\noindent
\textbf{Absorption.}
Suppose that 
\(
\bx_{t_\Omega}\in
t_\Omega B_{r_\Omega}(\Omega)
\)
for some \(t_\Omega\in[t_{\mathrm{loc}},1)\), then part~(ii) of
Lemma~\ref{lem:continuous-scaled-convex-contraction}
applied on \([t_{\mathrm{loc}},1)\)  with 
\(
S=B_{r_\Omega}(\Omega),
C_t=
B_{r_\Omega}\!\left(
\by_t^\Omega(\bx_t)
\right)
\) and \(\boldsymbol m_t=\hat{\by}_t(\bx_t)\) gives 
\[
\bx_t\in tB_{r_\Omega}(\Omega)
\] for every \(t\in[t_\Omega,1)\).
This proves part~(iii) of the theorem.
\end{proof}

\subsubsection{Proof of Theorem~\ref{cor:final-local-euler}}\label{proof:cor:final-local-euler}

\begin{proof}
The capture condition and \(t_{\mathrm{loc}}<1\) give
\[
r_\Omega
\le
t_{\mathrm{loc}}R
-
\dist\bigl(
\bz_0,t_{\mathrm{loc}}B_{r_\Omega}(\Omega)
\bigr)
<R.
\]
Moreover, \(r_\Omega>0\), so
\[
\begin{aligned}
\dist(\bz_0,t_{\mathrm{loc}}\Omega)
\le
\dist\bigl(
\bz_0,t_{\mathrm{loc}}B_{r_\Omega}(\Omega)
\bigr)
+t_{\mathrm{loc}}r_\Omega
\le 
t_{\mathrm{loc}}R - (1-t_{\mathrm{loc}})r_\Omega
< t_{\mathrm{loc}}R.
\end{aligned}
\]
For every \(i=0,\ldots,N-1\), the Euler update has the affine form
\[
\bz_{i+1}
=
\tfrac{1-t_{i+1}}{1-t_i}\bz_i
+
\tfrac{t_{i+1}-t_i}{1-t_i}
\hat{\by}_{t_i}(\bz_i).
\]
We prove by induction that, for every \(i=0,\ldots,N\),
\begin{equation}\label{eq:final-euler-neighborhood-bound}
\dist\bigl(\bz_i,t_iB_{r_\Omega}(\Omega)\bigr)
+r_\Omega
\le
t_iR.
\end{equation}
The case \(i=0\) is exactly the capture condition.
Suppose that \eqref{eq:final-euler-neighborhood-bound} holds for some
\(i<N\). Since \(t_i<1\) and \(r_\Omega>0\), it implies
\[
\begin{aligned}
\dist(\bz_i,t_i\Omega)
\le
\dist\bigl(\bz_i,t_iB_{r_\Omega}(\Omega)\bigr)
+t_ir_\Omega
\le
t_iR-(1-t_i)r_\Omega
<t_iR.
\end{aligned}
\]
Hence, 
\(
t_i\by_{t_i}^\Omega(\bz_i)
=
\Proj_{t_i\Omega}(\bz_i)
\)
is well defined whenever \eqref{eq:final-euler-neighborhood-bound} holds, and
Lemma~\ref{lem:final-local-posterior} gives
\begin{equation}\label{eq:final-euler-posterior-localization}
\hat{\by}_{t_i}(\bz_i)
\in
B_{r_\Omega}\bigl(\by_{t_i}^\Omega(\bz_i)\bigr).
\end{equation}
Since \(\by_{t_i}^\Omega(\bz_i)\in\Omega\),
\(
B_{r_\Omega}\bigl(\by_{t_i}^\Omega(\bz_i)\bigr)
\subset
B_{r_\Omega}(\Omega)\)
is nonempty, closed, and convex.
Let
\(
\bq_i:=\Proj_{t_iB_{r_\Omega}(\Omega)}(\bz_i), 
\)
then 
Lemma~\ref{lem:inflated-neighborhood-projection-regularity} and
\eqref{eq:final-euler-posterior-localization} give
\(\left\{
\bq_i/t_i,
\hat{\by}_{t_i}(\bz_i)
\right\}
\subset
B_{r_\Omega}\bigl(\by_{t_i}^\Omega(\bz_i)\bigr)
\subset B_{r_\Omega}(\Omega).\)
If
\(
\dist(\bz_i,t_iB_{r_\Omega}(\Omega))>0,
\)
then  part~(i) of
Lemma~\ref{lem:euler-scaled-convex-contraction}, applied with
\(
s=t_i,\ t=t_{i+1}, \ 
\bx = \bz_i, \by = \hat{\by}_{t_i}(\bz_i),
S=B_{r_\Omega}(\Omega),\
C=B_{r_\Omega}\bigl(\by_{t_i}^\Omega(\bz_i)\bigr),
\)
gives 
\begin{align}\label{eq:final-euler-one-step-attraction}
\dist\bigl(\bz_{i+1},t_{i+1}B_{r_\Omega}(\Omega)\bigr)
\le
\tfrac{1-t_{i+1}}{1-t_i}
\dist\bigl(\bz_i,t_iB_{r_\Omega}(\Omega)\bigr).
\end{align}
If
\(
\dist(\bz_i,t_iB_{r_\Omega}(\Omega))=0,
\)
then \(\bq_i=\bz_i\), so
\(
\left\{
\tfrac{\bz_i}{t_i},
\hat{\by}_{t_i}(\bz_i)
\right\}
\subset
B_{r_\Omega}\bigl(\by_{t_i}^\Omega(\bz_i)\bigr)
\subset B_{r_\Omega}(\Omega).
\)
Part~(ii) of Lemma~\ref{lem:euler-scaled-convex-contraction}, applied
with
\(
s=t_i,\ t=t_{i+1}, \ 
S=B_{r_\Omega}(\Omega),
\
C=B_{r_\Omega}\bigl(\by_{t_i}^\Omega(\bz_i)\bigr),
\)
gives \(\bz_{i+1}\in t_{i+1}B_{r_\Omega}(\Omega)\). Hence the next
distance is zero, so \eqref{eq:final-euler-one-step-attraction} below holds trivially.
Thus, in both cases, \eqref{eq:final-euler-one-step-attraction} holds. 
Combining this estimate with  the induction assumption gives
\begin{align*}
\dist\bigl(
\bz_{i+1},
t_{i+1}B_{r_\Omega}(\Omega)
\bigr)+r_\Omega
&\le
\tfrac{1-t_{i+1}}{1-t_i}
\bigl(t_iR-r_\Omega\bigr)+r_\Omega
=
t_{i+1}R
-
\tfrac{t_{i+1}-t_i}{1-t_i}
(R-r_\Omega)
<t_{i+1}R.
\end{align*}
The strict inequality follows from \(t_{i+1}>t_i\) and
\(r_\Omega<R\), and is valid when \(t_{i+1}=1\).

By induction, \eqref{eq:final-euler-neighborhood-bound} holds for all
\(i=0,\ldots,N\). The  calculation above also
shows that \(\dist(\bz_i,t_i\Omega)<t_iR\) for every update index
\(i=0,\ldots,N-1\). 

\noindent
\textbf{Attraction.}
The estimate \eqref{eq:final-euler-one-step-attraction} established above holds for every \(i<N\).
For every \(0\le j\le i\le N\) with \(t_j<1\), iterating it from
index \(j\) to index \(i-1\) gives
\[
\dist\bigl(\bz_i,t_iB_{r_\Omega}(\Omega)\bigr)
\le
\tfrac{1-t_i}{1-t_j}
\dist\bigl(\bz_j,t_jB_{r_\Omega}(\Omega)\bigr),
\]
where the case \(i=j\) is immediate. Since
\((1-t_i)/(1-t_j)\le1\), the sequence
\(
i\mapsto
\dist\bigl(\bz_i,t_iB_{r_\Omega}(\Omega)\bigr)
\)
is non-increasing. This proves part~(i) of the theorem.

Taking \(j=0\) gives
\[
\dist\bigl(\bz_i,t_iB_{r_\Omega}(\Omega)\bigr)
\le
\tfrac{1-t_i}{1-t_{\mathrm{loc}}}
\dist\bigl(
\bz_0,
t_{\mathrm{loc}}B_{r_\Omega}(\Omega)
\bigr),
\qquad
i=0,\ldots,N.
\]
The factor multiplying \(1-t_i\) on the right-hand side is independent
of \(i\), so the bound is \(\bigO(1-t_i)\). This proves part~(ii) of
the theorem.

\noindent
\textbf{Absorption.}
Suppose that
\(\bz_k\in t_kB_{r_\Omega}(\Omega)\)
for some \(k\in\{0,\ldots,N\}\).
We prove by induction that
\(
\bz_i\in t_iB_{r_\Omega}(\Omega),
i=k,\ldots,N.
\)
The assertion at \(i=k\) is the entrance assumption; if
\(k=N\), there is nothing further to prove. Suppose that it holds at
some \(i\in\{k,\ldots,N-1\}\). The already-established bound
\eqref{eq:final-euler-neighborhood-bound}, together with \(t_i<1\) and
\(r_\Omega>0\), gives
\(\dist(\bz_i,t_i\Omega)<t_iR\), so the projection and posterior-localization
lemmas apply at this update index. Since
\(\bz_i\in t_iB_{r_\Omega}(\Omega)\), its projection onto that set is
\(\bz_i\). Lemma~\ref{lem:inflated-neighborhood-projection-regularity}
therefore gives
\[
\bz_i/t_i\in
B_{r_\Omega}(\by_{t_i}^\Omega(\bz_i)).
\]
Here, \eqref{eq:final-euler-posterior-localization} gives that
\(\hat{\by}_{t_i}(\bz_i)\) is in the same convex ball. This ball is
contained in \(B_{r_\Omega}(\Omega)\), so part~(ii) of
Lemma~\ref{lem:euler-scaled-convex-contraction} gives
\(\bz_{i+1}\in t_{i+1}B_{r_\Omega}(\Omega)\), completing the induction.
This proves part~(iii) of the theorem.
\end{proof}

\section{Proofs for Section~\ref{sec:cfg}}\label{app:cfg-proofs}

The CFG proofs use the same scaled-target mechanism as the unconditional
arguments. The stage-specific work is to verify pointwise that the
extrapolated posterior mean and the normalized projection point lie in a
nonempty closed convex subset of the static target. Parts~(i)--(ii) of
Lemma~\ref{lem:continuous-scaled-convex-contraction} then give continuous
attraction and absorption, while parts~(i)--(ii) of
Lemma~\ref{lem:euler-scaled-convex-contraction} give the exact Euler
counterparts.

\begin{lemma}[A priori CFG norm bound]\label{lem:cfg-norm-bound}
Let $\bx_t$ solve the CFG ODE with initial condition $\bx_0$.
Then, for every $t\in[0,1)$ for which the solution is defined,
\begin{equation*}
\|\bx_t\|
\le
(1-t)\|\bx_0\|
+
t\bigl((w-1)M_{\mathcal D}+wM_{\mathcal D_c}\bigr).
\end{equation*}
\end{lemma}

\begin{proof}
Let $\hat{\by}_{t,\mathrm{cfg}}(\bx_t):=\hat{\by}_t(\bx_t) + w\bigl(\hat{\by}_{t,c}(\bx_t)-\hat{\by}_t(\bx_t)\bigr).$ 
Since $\|\hat{\by}_t(\bx_t)\| \le M_{\mathcal{D}}$ and $\|\hat{\by}_{t,c}(\bx_t)\| \le M_{\mathcal{D}_c}$, we have
\begin{equation*}
\|\hat{\by}_{t,\mathrm{cfg}}(\bx_t)\|
\le
(w-1)\|\hat{\by}_t(\bx_t)\| + w\|\hat{\by}_{t,c}(\bx_t)\|
\le
(w-1)M_{\mathcal D}+wM_{\mathcal D_c}.
\end{equation*}
The CFG ODE has the form
$\dot{\bx}_t=\tfrac{\hat{\by}_{t,\mathrm{cfg}}(\bx_t)-\bx_t}{1-t},$
so the calculation in
Lemma~\ref{lem:early-norm-bound}, with
$\hat{\by}_t(\bx_t)$ and $M_{\mathcal D}$ replaced by
$\hat{\by}_{t,\mathrm{cfg}}(\bx_t)$ and
\((w-1)M_{\mathcal D}+wM_{\mathcal D_c}\), respectively, gives the
claimed estimate.
\end{proof}

\subsection{Proof of Theorem~\ref{thm:cfg-early}}\label{proof:thm:cfg-early}

\begin{proof}
The expression defining \(\tilde r_{\mathrm{mean}}\) satisfies
\[
\begin{aligned}
&
(w-1)M_{\mathcal D}
\left[
\exp\left(
\tfrac{
tD_\emptyset
}{
(1-t)^2
}
\left[
\begin{aligned}
(1-t)\|\bx_0\|+
tw(M_{\mathcal D}+M_{\mathcal D_c})
\end{aligned}
\right]
\right)-1
\right]
\\
&+
wM_{\mathcal D_c}
\left[
\exp\left(
\tfrac{
tD_c
}{
(1-t)^2
}
\left[
\begin{aligned}
(1-t)\|\bx_0\|
+
t\bigl(
(w-1)M_{\mathcal D}
+
(w+1)M_{\mathcal D_c}
\bigr)
\end{aligned}
\right]
\right)-1
\right]
=
\bigO(t)
\quad
\text{as }t\downarrow0,
\end{aligned}
\]
while
\((1-t)\|\bx_0\|\to\|\bx_0\|>0\). Hence the admissible choices
asserted in the main-text setup exist. Fix
\(\tilde t_{\mathrm{mean}}\) and \(\tilde r_{\mathrm{mean}}\) as there.
By the choice of \(\tilde t_{\mathrm{mean}}\),
\[
(1-\tilde t_{\mathrm{mean}})\|\bx_0\|
>
\tilde t_{\mathrm{mean}}
\left[
(w-1)D_\emptyset+wD_c+\tilde r_{\mathrm{mean}}
\right].
\]

We first prove the non-entry property, which is part~(i) of the theorem. The definitions of \(\hat{\by}_{s,\mathrm{cfg}}(\bx_s), \bar{\by}_{\mathrm{cfg}}\) together with
Lemma~\ref{lem:mean-in-closed-conv} give
\[
\begin{aligned}
\left\|
\hat{\by}_{s,\mathrm{cfg}}(\bx_s)
-
\bar{\by}_{\mathrm{cfg}}
\right\|
&\le
(w-1)
\left\|
\hat{\by}_{s}(\bx_s)-\bar{\by}
\right\|
+
w
\left\|
\hat{\by}_{s,c}(\bx_s)-\bar{\by}_c
\right\|
\le
(w-1)D_\emptyset+wD_c.
\end{aligned}
\]
The calculation in the proof of
Theorem~\ref{thm:early-mean}, with  the substitutions
\(
\hat{\by}_s(\bx_s)
\) replaced by  \(
\hat{\by}_{s,\mathrm{cfg}}(\bx_s),\
\bar{\by}
\) replaced by  \(
\bar{\by}_{\mathrm{cfg}},
\)
and \(
D_\emptyset
\) replaced by  \(
(w-1)D_\emptyset+wD_c,
\)
therefore gives, for every \(t\in[0,\tilde t_{\mathrm{mean}}]\),
\[
\begin{aligned}
\dist\!\left(
\bx_t,
tB_{\tilde r_{\mathrm{mean}}}
(\bar{\by}_{\mathrm{cfg}})
\right)
&\ge
(1-t)\|\bx_0\|
-
t\left[
(w-1)D_\emptyset+wD_c+\tilde r_{\mathrm{mean}}
\right]
\\
&\ge
(1-\tilde t_{\mathrm{mean}})\|\bx_0\|
-
\tilde t_{\mathrm{mean}}
\left[
(w-1)D_\emptyset+wD_c+\tilde r_{\mathrm{mean}}
\right]
>0,
\end{aligned}
\]
where the last inequality follows from the choice of
\(\tilde t_{\mathrm{mean}}\) above. This proves part~(i).

\noindent
\textbf{Attraction.}
For part~(ii) of the theorem, we start with
\begin{equation}\label{eq:cfg-early-decompose}
\left\|
\hat{\by}_{t,\mathrm{cfg}}(\bx_t)
-
\bar{\by}_{\mathrm{cfg}}
\right\|
\le
(w-1)
\left\|
\hat{\by}_t(\bx_t)-\bar{\by}
\right\|
+
w
\left\|
\hat{\by}_{t,c}(\bx_t)-\bar{\by}_c
\right\|.
\end{equation}
We now localize \(\hat{\by}_{t,c}(\bx_t)\) and \(\hat{\by}_t(\bx_t)\). By
Lemma~\ref{lem:cfg-norm-bound},
\[
\|\bx_t\|
\le
(1-t)\|\bx_0\|
+
t\bigl(
(w-1)M_{\mathcal D}
+
wM_{\mathcal D_c}
\bigr).
\]
For the unconditional posterior, the 
calculation in the proof of Theorem~\ref{thm:early-mean} gives
\[
\left\|
\hat{\by}_t(\bx_t)-\bar{\by}
\right\|
\le
M_{\mathcal D}
\left[
\exp\left(
\tfrac{
tD_\emptyset\|\bx_t\|
+
t^2D_\emptyset M_{\mathcal D}
}{
(1-t)^2
}
\right)-1
\right].
\]
Plugging in the norm bound of \(\bx_t\) above gives
\[
\begin{aligned}
tD_\emptyset\|\bx_t\|
+t^2D_\emptyset M_{\mathcal D}
&\le
tD_\emptyset
\left[
(1-t)\|\bx_0\|
+
t\bigl(
(w-1)M_{\mathcal D}
+
wM_{\mathcal D_c}
\bigr)
+
tM_{\mathcal D}
\right]
\\
&=
tD_\emptyset
\left[
(1-t)\|\bx_0\|
+
tw(M_{\mathcal D}+M_{\mathcal D_c})
\right].
\end{aligned}
\]
Consequently,
\begin{equation}
\label{eq:cfg-early-unconditional-posterior-bound}
\left\|
\hat{\by}_t(\bx_t)-\bar{\by}
\right\|
\le
M_{\mathcal D}
\left[
\exp\left(
\tfrac{
tD_\emptyset
\left[
(1-t)\|\bx_0\|
+
tw(M_{\mathcal D}+M_{\mathcal D_c})
\right]
}{
(1-t)^2
}
\right)-1
\right]=M_{\mathcal D} (\exp (g_{\emptyset}(t))-1).
\end{equation}
Applying the same calculation to
\((p_c,\mathcal D_c)\) gives
\[
\left\|
\hat{\by}_{t,c}(\bx_t)-\bar{\by}_c
\right\|
\le
M_{\mathcal D_c}
\left[
\exp\left(
\tfrac{
tD_c\|\bx_t\|
+
t^2D_cM_{\mathcal D_c}
}{
(1-t)^2
}
\right)-1
\right].
\]
Again, Plugging in the norm bound of \(\bx_t\) gives
\[
\begin{aligned}
tD_c\|\bx_t\|
+t^2D_cM_{\mathcal D_c}
&\le
tD_c
\left[
(1-t)\|\bx_0\|
+
t\bigl(
(w-1)M_{\mathcal D}
+
wM_{\mathcal D_c}
\bigr)
+
tM_{\mathcal D_c}
\right]
\\
&=
tD_c
\left[
(1-t)\|\bx_0\|
+
t\bigl(
(w-1)M_{\mathcal D}
+
(w+1)M_{\mathcal D_c}
\bigr)
\right].
\end{aligned}
\]
Consequently,
\begin{equation}
\label{eq:cfg-early-conditional-posterior-bound}
\left\|
\hat{\by}_{t,c}(\bx_t)-\bar{\by}_c
\right\|
\le
M_{\mathcal D_c} (\exp (g_{c}(t))-1)
\end{equation}
The two exponents in
\eqref{eq:cfg-early-unconditional-posterior-bound} and
\eqref{eq:cfg-early-conditional-posterior-bound} 
are nondecreasing for  \(t \in [0,1)\).
 Substituting \eqref{eq:cfg-early-unconditional-posterior-bound} and \eqref{eq:cfg-early-conditional-posterior-bound} into \eqref{eq:cfg-early-decompose} therefore
gives, for every \(t\in(0,\tilde t_{\mathrm{mean}}]\),
\(
\left\|
\hat{\by}_{t,\mathrm{cfg}}(\bx_t)
-
\bar{\by}_{\mathrm{cfg}}
\right\|
\le
\tilde r_{\mathrm{mean}}
\)
by the definition of \(\tilde r_{\mathrm{mean}}\). At the boundary \(t=0\), the
posterior weights are constant in the data variable, so
\(
\hat{\by}_0(\bx_0)=\bar{\by},
\
\hat{\by}_{0,c}(\bx_0)=\bar{\by}_c,
\
\hat{\by}_{0,\mathrm{cfg}}(\bx_0)
=
\bar{\by}_{\mathrm{cfg}}.
\)
Thus
\begin{equation}
\label{eq:cfg-early-posterior-inclusion}
\hat{\by}_{t,\mathrm{cfg}}(\bx_t)
\in
B_{\tilde r_{\mathrm{mean}}}
(\bar{\by}_{\mathrm{cfg}}),
\qquad
t\in[0,\tilde t_{\mathrm{mean}}].
\end{equation}
Here, \(B_{\tilde r_{\mathrm{mean}}}(\bar{\by}_{\mathrm{cfg}})\) is nonempty, closed and convex, so projection onto its
positive scaling is well-defined.
Applying part~(i) of
Lemma~\ref{lem:continuous-scaled-convex-contraction} on
\(I=(0,\tilde t_{\mathrm{mean}}]\) with
\(
S=C_t
=
B_{\tilde r_{\mathrm{mean}}}(\bar{\by}_{\mathrm{cfg}}),
\
\boldsymbol m_t=\hat{\by}_{t,\mathrm{cfg}}(\bx_t)
\)
gives, for
\(0<s\le t\le\tilde t_{\mathrm{mean}}\),
\begin{align}\label{eq:cfg-early-estimate}
\dist\!\left(
\bx_t,
tB_{\tilde r_{\mathrm{mean}}}
(\bar{\by}_{\mathrm{cfg}})
\right)
\le
\tfrac{1-t}{1-s}
\dist\!\left(
\bx_s,
sB_{\tilde r_{\mathrm{mean}}}
(\bar{\by}_{\mathrm{cfg}})
\right).
\end{align}
Moreover,
\(
sB_{\tilde r_{\mathrm{mean}}}
(\bar{\by}_{\mathrm{cfg}})
=
B_{s\tilde r_{\mathrm{mean}}}
(s\bar{\by}_{\mathrm{cfg}})
\rightarrow
\{\bzero\},
\)
and \(\bx_s\to\bx_0\) as \(s\downarrow0\). Hence
\[
\dist\!\left(
\bx_s,
sB_{\tilde r_{\mathrm{mean}}}
(\bar{\by}_{\mathrm{cfg}})
\right)
\rightarrow
\|\bx_0\|
\qquad
\text{as }s\downarrow0.
\]
Letting \(s\downarrow0\) proves \eqref{eq:cfg-early-estimate} for every
\(0\le s\le t\le\tilde t_{\mathrm{mean}}\). 
By part~(i), the
distance at every earlier time \(s\) is positive. Therefore, when
\(s<t\), the strict inequality \((1-t)/(1-s)<1\) implies that the
distance decreases strictly.
This proves part~(ii) of the theorem. Taking \(s=0\) gives
\[
\dist\!\left(
\bx_t,
tB_{\tilde r_{\mathrm{mean}}}
(\bar{\by}_{\mathrm{cfg}})
\right)
\le
(1-t)\|\bx_0\|,
\qquad
t\in[0,\tilde t_{\mathrm{mean}}],
\]
which proves part~(iii) of the theorem.
\end{proof}

\subsection{Proof of Theorem~\ref{thm:cfg-convex}}\label{proof:thm:cfg-convex}

\begin{proof}
By Lemma~\ref{lem:mean-in-closed-conv},
\(
\hat{\by}_{t,c}(\bx_t)
\in
\Conv(\mathcal D_c)
\subset
\Conv(\mathcal D),
\
\hat{\by}_t(\bx_t)
\in
\Conv(\mathcal D).
\)
Since
\(
\diam(\Conv(\mathcal D))
=
\diam(\mathcal D)
=
D_\emptyset,
\)
we have
\(
\left\|
\hat{\by}_{t,c}(\bx_t)-\hat{\by}_t(\bx_t)
\right\|
\le
D_\emptyset.
\)
Therefore,
\[
0\le\epsilon_{\mathrm{infl}}(t_{\mathrm{infl}})\le(w-1)D_\emptyset<\infty.
\]
Recall that the CFG extrapolation is
\(
\hat{\by}_{t,\mathrm{cfg}}(\bx_t)
=
\hat{\by}_{t,c}(\bx_t)
+
(w-1)
\left(
\hat{\by}_{t,c}(\bx_t)-\hat{\by}_t(\bx_t)
\right).
\)
Since
\(\hat{\by}_{t,c}(\bx_t)\in\Conv(\mathcal D_c)\), for every
\(t\in[t_{\mathrm{infl}},1)\), we have 
\[
\dist\!\left(
\hat{\by}_{t,\mathrm{cfg}}(\bx_t),
\Conv(\mathcal D_c)
\right)
\le
\left\|
\hat{\by}_{t,\mathrm{cfg}}(\bx_t)
-
\hat{\by}_{t,c}(\bx_t)
\right\|
=
(w-1)
\left\|
\hat{\by}_{t,c}(\bx_t)-\hat{\by}_t(\bx_t)
\right\|
\le
\epsilon_{\mathrm{infl}}(t_{\mathrm{infl}}).
\]
Hence,
\begin{equation}
\label{eq:cfg-large-convex-posterior-inclusion}
\hat{\by}_{t,\mathrm{cfg}}(\bx_t)
\in
B_{\epsilon_{\mathrm{infl}}(t_{\mathrm{infl}})}(\Conv(\mathcal D_c)),
\qquad
t\in[t_{\mathrm{infl}},1).
\end{equation}

\noindent 
\textbf{Attraction.}
Since \(B_{\epsilon_{\mathrm{infl}}(t_{\mathrm{infl}})}(\Conv(\mathcal D_c))\) is nonempty, closed and convex,
projection onto every positive scaling of \(B_{\epsilon_{\mathrm{infl}}(t_{\mathrm{infl}})}(\Conv(\mathcal D_c))\) is
well-defined.
Applying part~(i) of
Lemma~\ref{lem:continuous-scaled-convex-contraction} on
\(I=[t_{\mathrm{infl}},1)\) if \(t_{\mathrm{infl}}>0\), and on
\(I=(0,1)\) if \(t_{\mathrm{infl}}=0\), with
\(
S=C_t=B_{\epsilon_{\mathrm{infl}}(t_{\mathrm{infl}})}(\Conv(\mathcal D_c))
\) and \(
\boldsymbol m_t=\hat{\by}_{t,\mathrm{cfg}}(\bx_t),
\)
gives, for every
\(t_{\mathrm{infl}}\le s\le t<1\) with \(s>0\),
\begin{align}\label{eq:cfg-large-hull}
\dist\!\left(
\bx_t,
tB_{\epsilon_{\mathrm{infl}}(t_{\mathrm{infl}})}(\Conv(\mathcal D_c))
\right)
\le
\tfrac{1-t}{1-s}
\dist\!\left(
\bx_s,
sB_{\epsilon_{\mathrm{infl}}(t_{\mathrm{infl}})}(\Conv(\mathcal D_c))
\right).
\end{align}
It remains only to include \(s=0\) when \(t_{\mathrm{infl}}=0\).
The same triangle-inequality argument used to obtain
\eqref{eq:uncond-zero-time-distance-bound}, with
\(\Conv(\mathcal D)\) and \(M_{\mathcal D}\) replaced by
\(B_{\epsilon_{\mathrm{infl}}(t_{\mathrm{infl}})}(\Conv(\mathcal D_c))\) and \(M_{\mathrm{infl}}\) respectively, gives,
for every \(s>0\),
\[
\left|
\dist(\bx_s,sB_{\epsilon_{\mathrm{infl}}(t_{\mathrm{infl}})}(\Conv(\mathcal D_c)))-\|\bx_0\|
\right|
\le
\|\bx_s-\bx_0\|+s\max_{\by\in B_{\epsilon_{\mathrm{infl}}(t_{\mathrm{infl}})}(\Conv(\mathcal D_c))}\|\by\|.
\]
Since \(\bx_s\to\bx_0\) as \(s\downarrow0\),
\[
\dist(\bx_s,sB_{\epsilon_{\mathrm{infl}}(t_{\mathrm{infl}})}(\Conv(\mathcal D_c)))
\to\|\bx_0\|
=\dist(\bx_0,0B_{\epsilon_{\mathrm{infl}}(t_{\mathrm{infl}})}(\Conv(\mathcal D_c))).
\]
Letting \(s\downarrow0\) in \eqref{eq:cfg-large-hull} proves the same
estimate for \(s=0\). Thus \eqref{eq:cfg-large-hull} holds for every
\(t_{\mathrm{infl}}\le s\le t<1\), including when
\(t_{\mathrm{infl}}=0\), and the distance is non-increasing on the
whole interval \([t_{\mathrm{infl}},1)\).
Taking
\(s=t_{\mathrm{infl}}\) in \eqref{eq:cfg-large-hull} gives
\[
\dist\!\left(
\bx_t,
tB_{\epsilon_{\mathrm{infl}}(t_{\mathrm{infl}})}(\Conv(\mathcal D_c))
\right)
\le
\tfrac{1-t}{1-t_{\mathrm{infl}}}
\dist\!\left(
\bx_{t_{\mathrm{infl}}},
t_{\mathrm{infl}}
B_{\epsilon_{\mathrm{infl}}(t_{\mathrm{infl}})}(\Conv(\mathcal D_c))
\right),
\qquad
t\in[t_{\mathrm{infl}},1).
\]

\noindent 
\textbf{Absorption.}
Suppose first that
\(\tilde t_{\mathrm{conv}}\in(0,1)\) and
\(
\bx_{\tilde t_{\mathrm{conv}}}
\in
\tilde t_{\mathrm{conv}}
B_{\epsilon_{\mathrm{infl}}(t_{\mathrm{infl}})}
(\Conv(\mathcal D_c)).
\)
Applying part~(ii) of
Lemma~\ref{lem:continuous-scaled-convex-contraction}
on
\(I=[\tilde t_{\mathrm{conv}},1)\), with \(S = C_t=B_{\epsilon_{\mathrm{infl}}(t_{\mathrm{infl}})}
(\Conv(\mathcal D_c))\) and 
\(\boldsymbol m_t=\hat{\by}_{t, \mathrm{cfg}}(\bx_t)\),
 gives, for every \(t\in[\tilde t_{\mathrm{conv}},1),\)
\[
\bx_t
\in
tB_{\epsilon_{\mathrm{infl}}(t_{\mathrm{infl}})}
(\Conv(\mathcal D_c)).
\]

If \(\tilde t_{\mathrm{conv}}=0\), then necessarily
\(t_{\mathrm{infl}}=0\), and of course
\(\bx_0=\boldsymbol0\). Taking \(s=0\) in
\eqref{eq:cfg-large-hull} therefore gives the same conclusion.
This proves part~(iii) of the theorem.
\end{proof}

\subsection{Proof of Corollary~\ref{cor:cfg-convex-small}}\label{proof:cor:cfg-convex-small}

\begin{proof}
Theorem~\ref{thm:cfg-convex}(iii) gives
\(
\bx_t
\in
tB_{\epsilon_{\mathrm{infl}}(t_{\mathrm{infl}})}(\Conv(\mathcal D_c)),
\
t\in[\tilde{t}_{\mathrm{conv}},1).
\)
Fix \(t\in[\tilde t_{\mathrm{conv}},1)\),
suppose that \(t>0\). Dividing the preceding set inclusion by \(t\) gives
\(
\bx_t /{t}
\in
B_{\epsilon_{\mathrm{infl}}(t_{\mathrm{infl}})}
(\Conv(\mathcal D_c)).
\)
For every \(\by\in\mathcal D_c\), the triangle inequality and
\(\diam(\Conv(\mathcal D_c))=\diam(\mathcal D_c)=D_c\) therefore give
\(
\left\|\bx_t /{t} -\by\right\|
\le
\epsilon_{\mathrm{infl}}(t_{\mathrm{infl}})
+\diam(\Conv(\mathcal D_c))
=
\epsilon_{\mathrm{infl}}(t_{\mathrm{infl}})+D_c.
\)
Multiplying by \(t\) yields
\begin{equation}
\label{eq:xt_ty}
    \|\bx_t-t\by\|
\le
t\bigl(\epsilon_{\mathrm{infl}}(t_{\mathrm{infl}})+D_c\bigr),
\qquad \by\in\mathcal D_c.
\end{equation}
On the other hand, the definition
\(
\rho_{\mathrm{infl}}
\) 
implies that, for every \(\by\in\mathcal D\setminus\mathcal D_c\),
\(
\left\|\bx_t /{t}-\by\right\|
\ge \rho_{\mathrm{infl}}.
\)
Multiplying by \(t\) gives
\begin{equation}
\label{eq:xt_ty2}
    \|\bx_t-t\by\|
\ge
t\rho_{\mathrm{infl}},
\qquad
\by\in\mathcal D\setminus\mathcal D_c.
\end{equation}
If \(t=0\), then
\(\bx_t\in tB_{\epsilon_{\mathrm{infl}}(t_{\mathrm{infl}})}
(\Conv(\mathcal D_c))=\{\bzero\}\), so \eqref{eq:xt_ty} and \eqref{eq:xt_ty2} are immediate.


Now again we write
\(
w_t(\by)
:=
\exp\left(
-\frac{
\|\bx_t-t\by\|^2
}{
2(1-t)^2
}
\right),
\)
and define
\[
Z_{c,t}
:=
\int_{\mathcal D_c}
w_t(\by)p_\emptyset(\dd\by),
\qquad
Z_{\mathrm{out},t}
:=
\int_{\mathcal D\setminus\mathcal D_c}
w_t(\by)p_\emptyset(\dd\by).
\]
Together with \eqref{eq:xt_ty} and \eqref{eq:xt_ty2}, we have 
\begin{align}
Z_{c,t}
\ge
p_\emptyset(\mathcal D_c)
\exp\left(
-\tfrac{
t^2(\epsilon_{\mathrm{infl}}(t_{\mathrm{infl}})+D_c)^2
}{
2(1-t)^2
}
\right), \quad 
Z_{\mathrm{out},t}
\le
\bigl(1-p_\emptyset(\mathcal D_c)\bigr)
\exp\left(
-\tfrac{
t^2\rho_{\mathrm{infl}}^2
}{
2(1-t)^2
}
\right).
\end{align}
Dividing these estimates yields
\begin{align}\label{eq:cfg-convex-ratio-bound}
\tfrac{Z_{\mathrm{out},t}}{Z_{c,t}}
\le
\tfrac{1-p_\emptyset(\mathcal D_c)}
{p_\emptyset(\mathcal D_c)}
\exp\left(
-\tfrac{
t^2\left[
\rho_{\mathrm{infl}}^2-(\epsilon_{\mathrm{infl}}(t_{\mathrm{infl}})+D_c)^2
\right]
}{
2(1-t)^2
}
\right)
=
\tfrac{1-p_\emptyset(\mathcal D_c)}
{p_\emptyset(\mathcal D_c)}
\exp\left(
-\tfrac{
t^2\Delta_{\mathrm{infl}}
}{
2(1-t)^2
}
\right).
\end{align}
Now note that the conditional and unconditional posterior means can be written as
\begin{equation*}
    \begin{split}
        \hat{\by}_{t,c}(\bx_t)
&=
\tfrac{
1
}{
Z_{c,t}
}
\int_{\mathcal D_c}
\by\,w_t(\by)p_\emptyset(\dd\by),\\
\hat{\by}_t(\bx_t)
& =
\tfrac{
1}{
Z_{c,t}+Z_{\mathrm{out},t}
}\left(
\int_{\mathcal D_c}
\by\,w_t(\by)p_\emptyset(\dd\by)
+
\int_{\mathcal D\setminus\mathcal D_c}
\by\,w_t(\by)p_\emptyset(\dd\by)
\right).
    \end{split}
\end{equation*}
Using
\(
\int_{\mathcal D_c}
\by\,w_t(\by)p_\emptyset(\dd\by)
=
Z_{c,t}\hat{\by}_{t,c}(\bx_t)
\) and \(
Z_{\mathrm{out},t}
=
\int_{\mathcal D\setminus\mathcal D_c}
w_t(\by)p_\emptyset(\dd\by),
\)
we obtain
\begin{equation}\label{eq:posterior-mean-difference-algebra}
\begin{aligned}
\hat{\by}_{t,c}(\bx_t)-\hat{\by}_t(\bx_t)
&=
\hat{\by}_{t,c}(\bx_t)
-
\tfrac{1}{
Z_{c,t}+Z_{\mathrm{out},t}
}\left(
Z_{c,t}\hat{\by}_{t,c}(\bx_t)
+
\int_{\mathcal D\setminus\mathcal D_c}
\by\,w_t(\by)p_\emptyset(\dd\by)
\right)
\\
&=
\tfrac{1}{
Z_{c,t}+Z_{\mathrm{out},t}
}\left(
Z_{\mathrm{out},t}\hat{\by}_{t,c}(\bx_t)
-
\int_{\mathcal D\setminus\mathcal D_c}
\by\,w_t(\by)p_\emptyset(\dd\by)
\right)\\ 
&=
\tfrac{1}{
Z_{c,t}+Z_{\mathrm{out},t}
}
\displaystyle
\int_{\mathcal D\setminus\mathcal D_c}
\bigl(
\hat{\by}_{t,c}(\bx_t)-\by
\bigr)
w_t(\by)p_\emptyset(\dd\by).
\end{aligned}
\end{equation}
Because
\(\hat{\by}_{t,c}(\bx_t)\in\Conv(\mathcal D_c)\subset\Conv(\mathcal D)\)
and \(\by\in\mathcal D\),
\(
\left\|
\hat{\by}_{t,c}(\bx_t)-\by
\right\|
\le
\diam(\Conv(\mathcal D))
=
D_\emptyset.
\)
Consequently,
\begin{align}\label{eq:cfg-convex-hat-bound}
\left\|\hat{\by}_{t,c}(\bx_t)-\hat{\by}_t(\bx_t)\right\|
\le
D_\emptyset
\tfrac{Z_{\mathrm{out},t}}
{Z_{c,t}+Z_{\mathrm{out},t}}
\le
D_\emptyset
\tfrac{Z_{\mathrm{out},t}}{Z_{c,t}}.
\end{align}
Combining \eqref{eq:cfg-convex-ratio-bound} and \eqref{eq:cfg-convex-hat-bound} gives
\[
\left\|
\hat{\by}_{t,c}(\bx_t)-\hat{\by}_t(\bx_t)
\right\|
\le
D_\emptyset
\tfrac{1-p_\emptyset(\mathcal D_c)}
{p_\emptyset(\mathcal D_c)}
\exp\left(
-\tfrac{t^2\Delta_{\mathrm{infl}}}{2(1-t)^2}
\right).
\]
Since \(s\mapsto s^2/(1-s)^2\) is increasing on \([0,1)\), for each
\(t\in[\tilde t_{\mathrm{conv}},1)\), the definition
\eqref{eq:cfg-tail-inflation-function} and the preceding estimate
give
\[
\begin{aligned}
\epsilon_{\mathrm{infl}}(t)
=
(w-1)
\sup_{s\in[t,1)}
\left\|
\hat{\by}_{s,c}(\bx_s)-\hat{\by}_s(\bx_s)
\right\|
\le
(w-1)D_\emptyset
\tfrac{1-p_\emptyset(\mathcal D_c)}
{p_\emptyset(\mathcal D_c)}
\exp\left(
-\tfrac{t^2\Delta_{\mathrm{infl}}}{2(1-t)^2}
\right).
\end{aligned}
\]
This proves
\eqref{eq:cfg-convex-exponential-inflation}. In particular,
\(
\epsilon_{\mathrm{infl}}(t)
\rightarrow0
\
\text{as }t\uparrow1.
\)
Since \(\epsilon_{\mathrm{infl}}(t_{\mathrm{infl}})>0\), choose
\(t_{\mathrm{ref}}\in[\tilde{t}_{\mathrm{conv}},1)\) sufficiently close to
\(1\) so that, with
\(
\epsilon_{\mathrm{ref}}
:=
\epsilon_{\mathrm{infl}}(t_{\mathrm{ref}}),
\)
\(
0\le\epsilon_{\mathrm{ref}}<\epsilon_{\mathrm{infl}}(t_{\mathrm{infl}}).
\)

For every \(t\in[t_{\mathrm{ref}},1)\), the definition of $\epsilon_{\mathrm{ref}}$
gives
\[
(w-1)
\left\|
\hat{\by}_{t,c}(\bx_t)-\hat{\by}_t(\bx_t)
\right\|
\le
\epsilon_{\mathrm{ref}}.
\]
The CFG extrapolation identity now gives, for every \(t\in[t_{\mathrm{ref}},1)\),
\[
\begin{aligned}
&
\dist\!\left(
\hat{\by}_{t,\mathrm{cfg}}(\bx_t),
\Conv(\mathcal D_c)
\right)
\le
\left\|
\hat{\by}_{t,\mathrm{cfg}}(\bx_t)
-
\hat{\by}_{t,c}(\bx_t)
\right\|
=
(w-1)
\left\|
\hat{\by}_{t,c}(\bx_t)-\hat{\by}_t(\bx_t)
\right\|
\le
\epsilon_{\mathrm{ref}}.
\end{aligned}
\]
Therefore
\begin{equation}
\label{eq:cfg-small-convex-posterior-inclusion}
\hat{\by}_{t,\mathrm{cfg}}(\bx_t)
\in
B_{\epsilon_{\mathrm{ref}}}(\Conv(\mathcal D_c)),
\qquad
t\in[t_{\mathrm{ref}},1).
\end{equation}

\noindent
\textbf{Attraction.}
The set \(B_{\epsilon_{\mathrm{ref}}}(\Conv(\mathcal D_c))\) is nonempty, closed, and convex, so projection onto its positive
scaling is well-defined. Applying part~(i) of Lemma~\ref{lem:continuous-scaled-convex-contraction} on 
\(I=[t_{\mathrm{ref}},1)\) with
\(
S=C_t=B_{\epsilon_{\mathrm{ref}}}(\Conv(\mathcal D_c))
\) and
\(\boldsymbol m_t=\hat{\by}_{t,\mathrm{cfg}}(\bx_t),
\) gives,
for every
\(t_{\mathrm{ref}}\le s\le t<1\),
\begin{align}\label{eq:cfg-small-hull}
\dist\!\left(
\bx_t,
tB_{\epsilon_{\mathrm{ref}}}(\Conv(\mathcal D_c))
\right)
\le
\tfrac{1-t}{1-s}
\dist\!\left(
\bx_s,
sB_{\epsilon_{\mathrm{ref}}}(\Conv(\mathcal D_c))
\right).
\end{align}
Part~(i) of Lemma~\ref{lem:continuous-scaled-convex-contraction} also
shows that the refined distance is non-increasing.
Taking \(s=t_{\mathrm{ref}}\) in \eqref{eq:cfg-small-hull} gives
\[
\dist\!\left(
\bx_t,
tB_{\epsilon_{\mathrm{ref}}}(\Conv(\mathcal D_c))
\right)
\le
\tfrac{1-t}{1-t_{\mathrm{ref}}}
\dist\!\left(
\bx_{t_{\mathrm{ref}}},
t_{\mathrm{ref}}
B_{\epsilon_{\mathrm{ref}}}(\Conv(\mathcal D_c))
\right),
\qquad
t\in[t_{\mathrm{ref}},1).
\]
This proves parts~(i)--(ii) of Corollary~\ref{cor:cfg-convex-small}.

\noindent
\textbf{Absorption.}
Suppose that 
\(
\bx_{\tilde t'_{\mathrm{conv}}}
\in
\tilde t'_{\mathrm{conv}}
B_{\epsilon_{\mathrm{ref}}}(\Conv(\mathcal D_c)).
\)
Since
\(
\epsilon_{\mathrm{ref}}
=\epsilon_{\mathrm{infl}}(t_{\mathrm{ref}})
<\epsilon_{\mathrm{infl}}(t_{\mathrm{infl}})
\)
and \(\epsilon_{\mathrm{infl}}\) is non-increasing,
\(t_{\mathrm{ref}}>t_{\mathrm{infl}}\ge 0\).
Thus \(t_{\mathrm{ref}}>0\).
Applying part~(ii) of
Lemma~\ref{lem:continuous-scaled-convex-contraction}
on \(I=[t_{\mathrm{ref}},1)\), with
\(S=C_t=B_{\epsilon_{\mathrm{ref}}}(\Conv(\mathcal D_c))\) and
\(\boldsymbol m_t=\hat{\by}_{t,\mathrm{cfg}}(\bx_t)\) gives
\[
\bx_t
\in
tB_{\epsilon_{\mathrm{ref}}}(\Conv(\mathcal D_c)),
\qquad
t\in[\tilde t'_{\mathrm{conv}},1).
\]
This proves part~(iii).
\end{proof}

\subsection{Proof of Theorem~\ref{thm:prediction-gap}}\label{proof:thm:prediction-gap}

\begin{proof}
Fix \(t\in[t_{\mathrm{gap}},1)\) and
\(\bx\in B_{tR}(t\Omega_c)\), and write
\[
w_t(\by)
:=
\exp\left(-\tfrac{\|\bx-t\by\|^2}{2(1-t)^2}\right),
\quad
Z_{c,t}
:=
\int_{\mathcal D_c}w_t(\by)p_\emptyset(\dd\by),
\quad
Z_{\mathrm{out},t}
:=
\int_{\mathcal D\setminus\mathcal D_c}
w_t(\by)p_\emptyset(\dd\by).
\]
Applying the calculation in
\eqref{eq:posterior-mean-difference-algebra}, with \(\bx_t\) replaced by
\(\bx\), gives
\begin{equation}\label{eq:prediction-gap-posterior-identity}
\hat{\by}_{t,c}(\bx)-\hat{\by}_t(\bx)
=
\tfrac{1}{Z_{c,t}+Z_{\mathrm{out},t}}
\int_{\mathcal D\setminus\mathcal D_c}
\bigl(\hat{\by}_{t,c}(\bx)-\by\bigr)
w_t(\by)p_\emptyset(\dd\by)
.
\end{equation}
Likewise, the estimate \eqref{eq:cfg-convex-hat-bound}, with
\(\bx_t\) replaced by \(\bx\), gives
\begin{equation}\label{eq:prediction-gap-posterior-bound}
\left\|\hat{\by}_{t,c}(\bx)-\hat{\by}_t(\bx)\right\|
\le
D_\emptyset
\tfrac{Z_{\mathrm{out},t}}{Z_{c,t}}
\le
2M_{\mathcal D}
\tfrac{Z_{\mathrm{out},t}}{Z_{c,t}}.
\end{equation}
Since \(\bx/t\in B_R(\Omega_c)\), choose \(\bz\in\Omega_c\) such that
\(\|\bx/t-\bz\|\le R\). For every \(\by\in\Omega_c\),
\[
\left\|\bx/{t}-\by\right\|
\le
\left\|\bx/{t}-\bz\right\|+\|\bz-\by\|
\le R+D_{\Omega_c}.
\]
Consequently,
\[
Z_{c,t}
\ge
p_\emptyset(\Omega_c)
\exp\left(-\tfrac{t^2(R+D_{\Omega_c})^2}{2(1-t)^2}\right).
\]
On the other hand,
\(\mathcal D\setminus\mathcal D_c
\subset\mathcal D\setminus\Omega_c\). Thus, for every
\(\by\in\mathcal D\setminus\mathcal D_c\),
\[
\left\|\bx/{t}-\by\right\|
\ge
\dist\bigl(B_R(\Omega_c),\mathcal D\setminus\Omega_c\bigr),
\]
and hence
\[
Z_{\mathrm{out},t}
\le
\bigl(1-p_\emptyset(\Omega_c)\bigr)
\exp\left(
-\tfrac{
t^2\dist(B_R(\Omega_c),\mathcal D\setminus\Omega_c)^2
}{2(1-t)^2}
\right).
\]
Dividing the preceding two estimates of $Z_{\mathrm{out},t}$ and $Z_{c,t}$ gives
\begin{equation}\label{eq:prediction-gap-posterior-mass-ratio}
\tfrac{Z_{\mathrm{out},t}}{Z_{c,t}}
\le
\tfrac{1-p_\emptyset(\Omega_c)}{p_\emptyset(\Omega_c)}
\exp\left(
-\tfrac{t^2}{2(1-t)^2}
\left[
\dist\bigl(B_R(\Omega_c),\mathcal D\setminus\Omega_c\bigr)^2
-(R+D_{\Omega_c})^2
\right]
\right).
\end{equation}
Combining \eqref{eq:prediction-gap-posterior-bound} and
\eqref{eq:prediction-gap-posterior-mass-ratio}, and using the
definition of \(\gamma_{\Omega_c}\), yields
\[
\left\|\hat{\by}_{t,c}(\bx)-\hat{\by}_t(\bx)\right\|
\le
2M_{\mathcal D}
\tfrac{1-p_\emptyset(\Omega_c)}{p_\emptyset(\Omega_c)}
\exp\left(
-\tfrac{t^2\gamma_{\Omega_c}}{2(1-t)^2}
\right).
\]
Since \(t\ge t_{\mathrm{gap}}\) and \(\gamma_{\Omega_c}>0\),
\(
\tfrac{t^2\gamma_{\Omega_c}}{2}
\ge
\tfrac{t_{\mathrm{gap}}^2\gamma_{\Omega_c}}{2}
>0,
\)
which proves the stated posterior-mean-gap estimate.
Finally, the posterior-mean representations of the two fields give
\[
\|\vc(t,\bx)-\vu(t,\bx)\|
=
\tfrac{1}{1-t}
\left\|\hat{\by}_{t,c}(\bx)-\hat{\by}_t(\bx)\right\|
\le
\tfrac{2M_{\mathcal D}}{1-t}
\tfrac{1-p_\emptyset(\Omega_c)}{p_\emptyset(\Omega_c)}
\exp\left(
-\tfrac{t_{\mathrm{gap}}^2\gamma_{\Omega_c}}{2(1-t)^2}
\right).
\]
 Both $\left\|\hat{\by}_{t,c}(\bx)-\hat{\by}_t(\bx)\right\|$ and  $\|\vc(t,\bx)-\vu(t,\bx)\|$ obviously converge to zero as \(t\uparrow1\).
\end{proof}

\begin{assumption}[CFG local-cluster regularity]\label{assump:cfg-local-cluster}
Let $\Omega_c\subset\mathcal D_c$ be a nonempty closed
local cluster satisfying $p_\emptyset(\Omega_c)>0$, and let $R>0$. Write
\(
D_{\Omega_c}:=\diam(\Omega_c),
\
\delta_{\Omega_c}:=\dist(B_R(\Omega_c),\mathcal D\setminus\Omega_c),
\
\gamma_{\Omega_c}:=\delta_{\Omega_c}^2-(R+D_{\Omega_c})^2.
\)
Assume the following conditions.
\begin{enumerate}
\item[(i)] \emph{Local prox-regular geometry:} the target cluster $\Omega_c$ is $\rho_{\Omega_c}$-prox-regular for some $\rho_{\Omega_c}>R$.
\item[(ii)] \emph{Uniform local polynomial mass:} there exist constants $\tilde c_{\Omega_c}>0$, $\tilde\alpha_{\Omega_c}>0$, and $\tilde s_{\Omega_c}>0$ such that
\[
p_c\bigl(\Omega_c\cap B_s(\by)\bigr)
\ge \tilde c_{\Omega_c} s^{\tilde\alpha_{\Omega_c}},
\qquad
\by\in\Omega_c,\quad 0<s\le\tilde s_{\Omega_c}.
\]
\item[(iii)] \emph{Cluster isolation:} $\gamma_{\Omega_c}=\delta_{\Omega_c}^2-(R+D_{\Omega_c})^2>0$.
\end{enumerate}
\end{assumption}
Under Assumption~\ref{assump:cfg-local-cluster}(i), whenever
$\dist(\bx,t\Omega_c)<tR$, the projection is well-defined and we write
\(
\by_t^{\Omega_c}(\bx)
:=
\tfrac{1}{t}\Proj_{t\Omega_c}(\bx)
=
\Proj_{\Omega_c}(\bx/t).
\)

\subsection{Final-stage attraction and absorption in CFG}

\begin{lemma}[Conditional local-mass radius]
\label{lem:cfg-polynomial-local-mass-radius}
Under Assumption~\ref{assump:cfg-local-cluster}, let $w>1$, and let
$\tilde t_{\mathrm{loc}}$, $\tilde r_{\Omega_c}$,
$\tilde\varepsilon_{\Omega_c}$, and $\tilde c_{\Omega_c,\mathrm{loc}}$
satisfy
\eqref{eq:cfg-final-time-choice}--\eqref{eq:cfg-final-quantities}.
Suppose
$t\in[\tilde t_{\mathrm{loc}},1)$ and $\dist(\bx,t\Omega_c)<tR$. Define
\[
\widetilde{\Omega}_{c,t}^0(\bx)
:=
\left\{
\by\in\Omega_c:
\|\bx-t\by\|^2
-
\|\bx-t\by_t^{\Omega_c}(\bx)\|^2
\le
\tfrac{\tilde\varepsilon_{\Omega_c}^2}{2}
\right\},
\]
then
\begin{equation}\label{eq:cfg-polynomial-local-mass-set}
p_c\bigl(\widetilde{\Omega}_{c,t}^0(\bx)\bigr)
\ge
\tilde c_{\Omega_c,\mathrm{loc}}.
\end{equation}
Moreover, with
\(\Omega_c,\mathcal D,R,\rho_{\Omega_c},\tilde c_{\Omega_c},
\tilde\alpha_{\Omega_c},\tilde s_{\Omega_c},p_\emptyset\), and \(w\) fixed,
\begin{equation}\label{eq:cfg-final-inflation-radius-order}
\tilde r_{\Omega_c}
=
\bigO\!\left(
(1-\tilde t_{\mathrm{loc}})
\sqrt{\log\tfrac{1}{1-\tilde t_{\mathrm{loc}}}}
\right)
\qquad
\text{as }\tilde t_{\mathrm{loc}}\uparrow1.
\end{equation}
\end{lemma}

\begin{proof}
Admissible CFG final-stage times exist because
$(1-t)^2\log(1/(1-t))\to0$ as $t\uparrow1$, while both entries in the
minimum on the right-hand side of \eqref{eq:cfg-final-time-choice} are
strictly positive.
The pointwise distance condition and prox-regularity give
$\by_t^{\Omega_c}(\bx)\in\Omega_c$. The exact restriction
\eqref{eq:cfg-final-time-choice} implies
$0<\tilde\varepsilon_{\Omega_c}\le1$ and
$\tilde\varepsilon_{\Omega_c}^2/(8R+4)\le\tilde s_{\Omega_c}$. Apply the
calculation in the first part of
Lemma~\ref{lem:polynomial-local-mass-radius} with the substitutions
\[
\bigl(
p_\emptyset,c_\Omega,\alpha_\Omega,s_\Omega,
\varepsilon_\Omega,c_{\Omega,\mathrm{loc}},\Omega_t^0
\bigr)
\longmapsto
\bigl(
p_c,\tilde c_{\Omega_c},\tilde\alpha_{\Omega_c},\tilde s_{\Omega_c},
\tilde\varepsilon_{\Omega_c},\tilde c_{\Omega_c,\mathrm{loc}},
\widetilde{\Omega}_{c,t}^0
\bigr).
\]
Its mass input is exactly the uniform condition in
Assumption~\ref{assump:cfg-local-cluster}(ii), applied at
$\by_t^{\Omega_c}(\bx)\in\Omega_c$ with
$s=\tilde\varepsilon_{\Omega_c}^2/(8R+4)$. It gives
\eqref{eq:cfg-polynomial-local-mass-set}.

The calculation for the first two quantities in the definition of
\(r_\Omega\) in the second part of
Lemma~\ref{lem:polynomial-local-mass-radius}, with the same
substitutions, gives
\[
\tfrac{\tilde\varepsilon_{\Omega_c}}
{\tilde t_{\mathrm{loc}}\sqrt{1-R/\rho_{\Omega_c}}}
+
\tfrac{D_{\Omega_c}}{\tilde c_{\Omega_c,\mathrm{loc}}}
\exp\left(
-\tfrac{\tilde\varepsilon_{\Omega_c}^2}
{4(1-\tilde t_{\mathrm{loc}})^2}
\right)
=
\bigO\!\left(
(1-\tilde t_{\mathrm{loc}})
\sqrt{\log\tfrac{1}{1-\tilde t_{\mathrm{loc}}}}
\right).
\]
For \(\tilde t_{\mathrm{loc}}\ge1/2\),
Assumption~\ref{assump:cfg-local-cluster}(iii) gives
\(
\frac{\tilde t_{\mathrm{loc}}^2\gamma_{\Omega_c}}{2}
\ge
\frac{\gamma_{\Omega_c}}{8}
>0.
\)
Consequently, for every fixed \(m>0\),
\(
\exp\left(
-\frac{\tilde t_{\mathrm{loc}}^2\gamma_{\Omega_c}}
{2(1-\tilde t_{\mathrm{loc}})^2}
\right)
=
\littleo\bigl((1-\tilde t_{\mathrm{loc}})^m\bigr)
\
\text{as }\tilde t_{\mathrm{loc}}\uparrow1.
\)
Since
\(2wM_{\mathcal D}(1-p_\emptyset(\Omega_c))/p_\emptyset(\Omega_c)\)
is fixed, the third term in the definition of \(\tilde r_{\Omega_c}\) is
\(\littleo(1-\tilde t_{\mathrm{loc}})\), and hence is
\(\bigO((1-\tilde t_{\mathrm{loc}})
\sqrt{\log(1/(1-\tilde t_{\mathrm{loc}}))})\).
Combining these estimates with the definition of \(\tilde r_{\Omega_c}\)
proves \eqref{eq:cfg-final-inflation-radius-order}.
\end{proof}

\begin{lemma}[Conditional localization]
\label{lem:cfg-final-conditional-posterior}
Under Assumption~\ref{assump:cfg-local-cluster}, let $w>1$, and let
$\tilde t_{\mathrm{loc}}$, $\tilde r_{\Omega_c}$,
$\tilde\varepsilon_{\Omega_c}$, and $\tilde c_{\Omega_c,\mathrm{loc}}$
satisfy
\eqref{eq:cfg-final-time-choice}--\eqref{eq:cfg-final-quantities}.
Suppose that
\[
t\in[\tilde t_{\mathrm{loc}},1),\quad
\dist(\bx,t\Omega_c)<tR,
\]
then we have
\begin{align}
\left\|
t\by_t^{\Omega_c}(\bx)-t\hat{\by}_{t,c}(\bx)
\right\|
&\le
\tfrac{1}{\sqrt{1-R/\rho_{\Omega_c}}}\tilde\varepsilon_{\Omega_c}
+
\tfrac{tD_{\Omega_c}}{\tilde c_{\Omega_c,\mathrm{loc}}}
\exp\left(
-\tfrac{\tilde\varepsilon_{\Omega_c}^2}{4(1-t)^2}
\right)\notag\\
&+
2tM_{\mathcal D}
\tfrac{1-p_\emptyset(\Omega_c)}{p_\emptyset(\Omega_c)}
\exp\left(
-\tfrac{\tilde t_{\mathrm{loc}}^2\gamma_{\Omega_c}}{2(1-t)^2}
\right).
\label{eq:cfg-conditional-localization}
\end{align}
\end{lemma}

\begin{proof}
We first define the conditional posterior mean restricted to the local cluster \(\Omega_c\):
\(
\hat{\by}_{t,c}^{\Omega_c}(\bx)
:=
\tfrac{
\int_{\Omega_c}
\by
\exp\left(
-\frac{\|\bx-t\by\|^2}{2(1-t)^2}
\right)
p_c(\dd\by)
}{
\int_{\Omega_c}
\exp\left(
-\frac{\|\bx-t\by\|^2}{2(1-t)^2}
\right)
p_c(\dd\by)
}.
\)
Lemma~\ref{lem:cfg-polynomial-local-mass-radius} gives
\(p_c(\widetilde{\Omega}_{c,t}^0(\bx))
\ge\tilde c_{\Omega_c,\mathrm{loc}}\). Applying the calculation leading to
\eqref{eq:local-posterior-bound} in Lemma~\ref{lem:final-local-posterior} with the substitutions
\(
\bigl(
p_\emptyset,\varepsilon_\Omega,c_{\Omega,\mathrm{loc}},
\Omega_t^0
\bigr)
\longmapsto
\bigl(
p_c,\tilde\varepsilon_{\Omega_c},\tilde c_{\Omega_c,\mathrm{loc}},
\widetilde{\Omega}_{c,t}^0
\bigr)
\) gives
\begin{equation}\label{eq:cfg-conditional-within-cluster}
\left\|
t\by_t^{\Omega_c}(\bx)-t\hat{\by}_{t,c}^{\Omega_c}(\bx)
\right\|
\le
\tfrac{1}{\sqrt{1-R/\rho_{\Omega_c}}}\tilde\varepsilon_{\Omega_c}
+
\tfrac{tD_{\Omega_c}}{\tilde c_{\Omega_c,\mathrm{loc}}}
\exp\left(
-\tfrac{\tilde\varepsilon_{\Omega_c}^2}{4(1-t)^2}
\right).
\end{equation}

We next bound
\(\|t\hat{\by}_{t,c}^{\Omega_c}(\bx)-t\hat{\by}_{t,c}(\bx)\|\).
Write
\[
w_t(\by)
:=
\exp\left(
-\tfrac{\|\bx-t\by\|^2}{2(1-t)^2}
\right),
\quad
\widetilde{Z}_{\Omega_c,t}
:=
\int_{\Omega_c} w_t(\by)p_c(\dd\by),
\quad
\widetilde{Z}_{\mathrm{out},t}
:=
\int_{\mathcal D_c\setminus\Omega_c}w_t(\by)p_c(\dd\by).
\]
Direct subtraction and the same calculation as in \eqref{eq:subtract-uncond-posterior} give
\[
\hat{\by}_{t,c}^{\Omega_c}(\bx)-\hat{\by}_{t,c}(\bx)
=
\tfrac{
\int_{\mathcal D_c\setminus\Omega_c}
\bigl(\hat{\by}_{t,c}^{\Omega_c}(\bx)-\by\bigr)
w_t(\by)p_c(\dd\by)
}{\widetilde{Z}_{\Omega_c,t}+\widetilde{Z}_{\mathrm{out},t}}.
\]
By Lemma~\ref{lem:mean-in-closed-conv},
\(\hat{\by}_{t,c}^{\Omega_c}(\bx)\in\Conv(\Omega_c)\). Hence
\(\|\hat{\by}_{t,c}^{\Omega_c}(\bx)-\by\|
\le2M_{\mathcal D}\) for every
\(\by\in\mathcal D_c\setminus\Omega_c\), and therefore
\begin{equation}\label{eq:cfg-conditional-posterior-ratio-step}
\left\|
t\hat{\by}_{t,c}^{\Omega_c}(\bx)-t\hat{\by}_{t,c}(\bx)
\right\|
\le
2tM_{\mathcal D}
\tfrac{\widetilde{Z}_{\mathrm{out},t}}
{\widetilde{Z}_{\Omega_c,t}+\widetilde{Z}_{\mathrm{out},t}}
\le
2tM_{\mathcal D}
\tfrac{\widetilde{Z}_{\mathrm{out},t}}{\widetilde{Z}_{\Omega_c,t}}.
\end{equation}
Since \(\dist(\bx,t\Omega_c)<tR\) and \(\diam(\Omega_c) = D_{\Omega_c}\), the same calculation leading to \eqref{eq:Z-inside} gives
\[
\widetilde{Z}_{\Omega_c,t}
\ge
p_c(\Omega_c)
\exp\left(-\tfrac{t^2(R+D_{\Omega_c})^2}{2(1-t)^2}\right)
\]
and, because
\(\mathcal D_c\setminus\Omega_c\subset\mathcal D\setminus\Omega_c\), the same calculation leading to \eqref{eq:Z-outside} gives
\[
\widetilde{Z}_{\mathrm{out},t}
\le
\bigl(1-p_c(\Omega_c)\bigr)
\exp\left(
-\tfrac{
t^2\dist(B_R(\Omega_c),\mathcal D\setminus\Omega_c)^2
}{2(1-t)^2}
\right).
\]
Moreover,
\(
\tfrac{1-p_c(\Omega_c)}{p_c(\Omega_c)}
=
\tfrac{p_\emptyset(\mathcal D_c)-p_\emptyset(\Omega_c)}
{p_\emptyset(\Omega_c)}
\le
\tfrac{1-p_\emptyset(\Omega_c)}{p_\emptyset(\Omega_c)},
\)
\(
\tfrac{t^2}{2}
\left[
\dist(B_R(\Omega_c),\mathcal D\setminus\Omega_c)^2
-(R+D_{\Omega_c})^2
\right]
\ge
\tfrac{\tilde t_{\mathrm{loc}}^2\gamma_{\Omega_c}}{2}
\) for \(t\ge\tilde t_{\mathrm{loc}}\).
Dividing the two displayed estimates for
\(\widetilde Z_{\mathrm{out},t}\) and
\(\widetilde Z_{\Omega_c,t}\), and then applying the preceding two
inequalities, gives
\begin{equation}\label{eq:cfg-conditional-posterior-mass-ratio}
\tfrac{\widetilde Z_{\mathrm{out},t}}
{\widetilde Z_{\Omega_c,t}}
\le
\tfrac{1-p_\emptyset(\Omega_c)}{p_\emptyset(\Omega_c)}
\exp\left(
-\tfrac{\tilde t_{\mathrm{loc}}^2\gamma_{\Omega_c}}{2(1-t)^2}
\right).
\end{equation}
Combining \eqref{eq:cfg-conditional-posterior-ratio-step} and
\eqref{eq:cfg-conditional-posterior-mass-ratio} gives
\begin{equation}\label{eq:cfg-conditional-posterior-difference}
\left\|
t\hat{\by}_{t,c}^{\Omega_c}(\bx)-t\hat{\by}_{t,c}(\bx)
\right\|
\le
2tM_{\mathcal D}
\tfrac{1-p_\emptyset(\Omega_c)}{p_\emptyset(\Omega_c)}
\exp\left(
-\tfrac{\tilde t_{\mathrm{loc}}^2\gamma_{\Omega_c}}{2(1-t)^2}
\right).
\end{equation}
Combining
\eqref{eq:cfg-conditional-within-cluster} and
\eqref{eq:cfg-conditional-posterior-difference} proves
\eqref{eq:cfg-conditional-localization}.
\end{proof}

\begin{lemma}[CFG localization]\label{lem:cfg-final-local-posterior}
Under Assumption~\ref{assump:cfg-local-cluster}, let $w>1$, and let
$\tilde t_{\mathrm{loc}}$, $\tilde r_{\Omega_c}$,
$\tilde\varepsilon_{\Omega_c}$, and $\tilde c_{\Omega_c,\mathrm{loc}}$
satisfy
\eqref{eq:cfg-final-time-choice}--\eqref{eq:cfg-final-quantities}.
Suppose that
$t\in[\tilde{t}_{\mathrm{loc}},1)$ and
$\dist(\bx,t\Omega_c)<tR.$
Define the extrapolated posterior mean by
\(
\hat{\by}_{t,\mathrm{cfg}}(\bx)
:=
\hat{\by}_{t,c}(\bx)
+(w-1)\bigl(\hat{\by}_{t,c}(\bx)-\hat{\by}_{t}(\bx)\bigr),
\)
then
\begin{equation}\label{eq:cfg-cluster-ball-bound}
\left\|t\by_t^{\Omega_c}(\bx)-t\hat{\by}_{t,\mathrm{cfg}}(\bx)\right\|
\le
t\tilde r_{\Omega_c},
\end{equation}
and consequently $\hat{\by}_{t,\mathrm{cfg}}(\bx) \in B_{\tilde r_{\Omega_c}}(\by_t^{\Omega_c}(\bx)).$
\end{lemma}

\begin{proof}
The triangle inequality gives
\[
\left\|
t\by_t^{\Omega_c}(\bx)-t\hat{\by}_{t,\mathrm{cfg}}(\bx)
\right\|
\le
\left\|
t\by_t^{\Omega_c}(\bx)-t\hat{\by}_{t,c}(\bx)
\right\|
+
t(w-1)
\left\|
\hat{\by}_{t,c}(\bx)-\hat{\by}_t(\bx)
\right\|.
\]
Lemma~\ref{lem:cfg-final-conditional-posterior} bounds the first term
on the right-hand side by
\[
\begin{aligned}
\left\|
t\by_t^{\Omega_c}(\bx)-t\hat{\by}_{t,c}(\bx)
\right\|
\le
\tfrac{1}{\sqrt{1-R/\rho_{\Omega_c}}}\tilde\varepsilon_{\Omega_c}
+
\tfrac{tD_{\Omega_c}}{\tilde c_{\Omega_c,\mathrm{loc}}}
\exp\left(
-\tfrac{\tilde\varepsilon_{\Omega_c}^2}{4(1-t)^2}
\right)
+
2tM_{\mathcal D}
\tfrac{1-p_\emptyset(\Omega_c)}{p_\emptyset(\Omega_c)}
\exp\left(
-\tfrac{\tilde t_{\mathrm{loc}}^2\gamma_{\Omega_c}}{2(1-t)^2}
\right).
\end{aligned}
\]
For the second term, Theorem~\ref{thm:prediction-gap}, with
\(t_{\mathrm{gap}}=\tilde t_{\mathrm{loc}}\), gives
\[
\begin{aligned}
t(w-1)
\left\|
\hat{\by}_{t,c}(\bx)-\hat{\by}_t(\bx)
\right\|
\le
2t(w-1)M_{\mathcal D}
\tfrac{1-p_\emptyset(\Omega_c)}{p_\emptyset(\Omega_c)}
\exp\left(
-\tfrac{\tilde t_{\mathrm{loc}}^2\gamma_{\Omega_c}}{2(1-t)^2}
\right).
\end{aligned}
\]
Adding the preceding two estimates yields
\begin{align}
\left\|
t\by_t^{\Omega_c}(\bx)-t\hat{\by}_{t,\mathrm{cfg}}(\bx)
\right\|\notag&
\le
\tfrac{1}{\sqrt{1-R/\rho_{\Omega_c}}}\tilde\varepsilon_{\Omega_c}
+
\tfrac{tD_{\Omega_c}}{\tilde c_{\Omega_c,\mathrm{loc}}}
\exp\left(
-\tfrac{\tilde\varepsilon_{\Omega_c}^2}{4(1-t)^2}
\right)\\
&+
2twM_{\mathcal D}
\tfrac{1-p_\emptyset(\Omega_c)}{p_\emptyset(\Omega_c)}
\exp\left(
-\tfrac{\tilde t_{\mathrm{loc}}^2\gamma_{\Omega_c}}{2(1-t)^2}
\right).
\label{eq:cfg-localization-combination}
\end{align}
Since \(t\ge\tilde t_{\mathrm{loc}}\), we have 
\(\tfrac{1}{\sqrt{1-R/\rho_{\Omega_c}}}\tilde\varepsilon_{\Omega_c}
\le
t\tfrac{1}{\sqrt{1-R/\rho_{\Omega_c}}}
\frac{\tilde\varepsilon_{\Omega_c}}{\tilde t_{\mathrm{loc}}}\) and 
\begin{align*}
\exp\left(
-\tfrac{\tilde\varepsilon_{\Omega_c}^2}{4(1-t)^2}
\right)
\le
\exp\left(
-\tfrac{\tilde\varepsilon_{\Omega_c}^2}
{4(1-\tilde t_{\mathrm{loc}})^2}
\right), \quad 
\exp\left(
-\tfrac{\tilde t_{\mathrm{loc}}^2\gamma_{\Omega_c}}{2(1-t)^2}
\right)
\le
\exp\left(
-\tfrac{\tilde t_{\mathrm{loc}}^2\gamma_{\Omega_c}}
{2(1-\tilde t_{\mathrm{loc}})^2}
\right).
\end{align*}
Together with the definition of \(\tilde r_{\Omega_c}\), these comparisons
turn \eqref{eq:cfg-localization-combination} into
\(\|t\by_t^{\Omega_c}(\bx)-t\hat{\by}_{t,\mathrm{cfg}}(\bx)\|
\le t\tilde r_{\Omega_c}\).
This proves \eqref{eq:cfg-cluster-ball-bound}, and equivalently
\(
t\hat{\by}_{t,\mathrm{cfg}}(\bx)
\in
B_{t\tilde r_{\Omega_c}}\bigl(t\by_t^{\Omega_c}(\bx)\bigr).
\)
\end{proof}

\subsubsection{Proof of Theorem~\ref{thm:cfg-final}}\label{proof:thm:cfg-final}

\begin{proof}
The inequality~\eqref{eq:cfg-final-capture-condition} and
\(\tilde t_{\mathrm{loc}}<1\) give
\(
\tilde r_{\Omega_c}
\le
\tilde t_{\mathrm{loc}}R
-
\dist\bigl(
\bx_{\tilde t_{\mathrm{loc}}},
\tilde t_{\mathrm{loc}}B_{\tilde r_{\Omega_c}}(\Omega_c)
\bigr)
<R.
\)
Moreover, Since \(\tilde\varepsilon_{\Omega_c}>0\), the definition of
\(\tilde r_{\Omega_c}\) gives \(\tilde r_{\Omega_c}>0\). Together with 
\(\tilde t_{\mathrm{loc}}<1\), we have 
\[
\begin{aligned}
\dist\bigl(
\bx_{\tilde t_{\mathrm{loc}}},
\tilde t_{\mathrm{loc}}\Omega_c
\bigr)
&\le
\dist\bigl(
\bx_{\tilde t_{\mathrm{loc}}},
\tilde t_{\mathrm{loc}}B_{\tilde r_{\Omega_c}}(\Omega_c)
\bigr)
+\tilde t_{\mathrm{loc}}\tilde r_{\Omega_c}
\\
&<
\dist\bigl(
\bx_{\tilde t_{\mathrm{loc}}},
\tilde t_{\mathrm{loc}}B_{\tilde r_{\Omega_c}}(\Omega_c)
\bigr)
+\tilde r_{\Omega_c}
\le
\tilde t_{\mathrm{loc}}R.
\end{aligned}
\]
Define the first-exit time
\[
\tilde\tau^*
:=
\inf
\left\{
t\in(\tilde t_{\mathrm{loc}},1):
\dist(\bx_t,t\Omega_c)\ge tR
\right\},
\]
with the convention \(\tilde\tau^*=1\) if the set is empty. Continuity
and the preceding strict inequality give
\(\tilde\tau^*>\tilde t_{\mathrm{loc}}\) and
\(
\dist(\bx_t,t\Omega_c)<tR,
\
t\in[\tilde t_{\mathrm{loc}},\tilde\tau^*).
\)
For every \(t\in[\tilde t_{\mathrm{loc}},\tilde\tau^*)\),
Lemma~\ref{lem:inflated-neighborhood-projection-regularity} gives that
\(\Proj_{tB_{\tilde r_{\Omega_c}}(\Omega_c)}(\bx_t)\) is
well-defined and
\(
\Proj_{tB_{\tilde r_{\Omega_c}}(\Omega_c)}(\bx_t)/t
\in
B_{\tilde r_{\Omega_c}}\bigl(\by_t^{\Omega_c}(\bx_t)\bigr).
\)
In addition, Lemma~\ref{lem:cfg-final-local-posterior} gives
\(
\hat{\by}_{t,\mathrm{cfg}}(\bx_t)
\in
B_{\tilde r_{\Omega_c}}\bigl(\by_t^{\Omega_c}(\bx_t)\bigr).
\)

The set
\(B_{\tilde r_{\Omega_c}}(\by_t^{\Omega_c}(\bx_t))\) is nonempty,
closed, and convex, and
\(
B_{\tilde r_{\Omega_c}}\bigl(\by_t^{\Omega_c}(\bx_t)\bigr)
\subset
B_{\tilde r_{\Omega_c}}(\Omega_c)
\)
because \(\by_t^{\Omega_c}(\bx_t)\in\Omega_c\). Thus part~(i) of
Lemma~\ref{lem:continuous-scaled-convex-contraction} applies on
\(I=[\tilde t_{\mathrm{loc}},\tilde\tau^*)\), with
\(
S=B_{\tilde r_{\Omega_c}}(\Omega_c),
\
C_t=B_{\tilde r_{\Omega_c}}\bigl(\by_t^{\Omega_c}(\bx_t)\bigr),
\
\boldsymbol m_t=\hat{\by}_{t,\mathrm{cfg}}(\bx_t),
\)
and gives, for every \(t<\tilde\tau^*\),
\begin{equation}\label{eq:cfg-final-global-bound}
\dist\bigl(\bx_t,tB_{\tilde r_{\Omega_c}}(\Omega_c)\bigr)
\le
\tfrac{1-t}{1-\tilde t_{\mathrm{loc}}}
\dist\bigl(
\bx_{\tilde t_{\mathrm{loc}}},
\tilde t_{\mathrm{loc}}B_{\tilde r_{\Omega_c}}(\Omega_c)
\bigr).
\end{equation}
Combining \eqref{eq:cfg-final-global-bound} with
\eqref{eq:cfg-final-capture-condition} gives, for every
\(t<\tilde\tau^*\),
\[
\begin{aligned}
\dist\bigl(\bx_t,tB_{\tilde r_{\Omega_c}}(\Omega_c)\bigr)
+\tilde r_{\Omega_c}
&\le
\tfrac{1-t}{1-\tilde t_{\mathrm{loc}}}
\bigl(\tilde t_{\mathrm{loc}}R-\tilde r_{\Omega_c}\bigr)
+\tilde r_{\Omega_c}
=
tR-
\tfrac{t-\tilde t_{\mathrm{loc}}}{1-\tilde t_{\mathrm{loc}}}
(R-\tilde r_{\Omega_c})
\le tR.
\end{aligned}
\]
Suppose that \(\tilde\tau^*<1\). Letting
\(t\uparrow\tilde\tau^*\) in the preceding estimate and using
\(\tilde\tau^*>\tilde t_{\mathrm{loc}}\) and
\(\tilde r_{\Omega_c}<R\) gives
\[
\dist\bigl(
\bx_{\tilde\tau^*},
\tilde\tau^*B_{\tilde r_{\Omega_c}}(\Omega_c)
\bigr)
+\tilde r_{\Omega_c}
<
\tilde\tau^*R.
\]
Consequently,
\[
\dist(\bx_{\tilde\tau^*},\tilde\tau^*\Omega_c)
\le
\dist\bigl(\bx_{\tilde\tau^*},
\tilde\tau^*B_{\tilde r_{\Omega_c}}(\Omega_c)\bigr)
+\tilde\tau^*\tilde r_{\Omega_c}
\le
\dist\bigl(\bx_{\tilde\tau^*},
\tilde\tau^*B_{\tilde r_{\Omega_c}}(\Omega_c)\bigr)
+\tilde r_{\Omega_c}
<
\tilde\tau^*R.
\]
On the other hand, the first-exit definition and continuity imply
\(
\dist(\bx_{\tilde\tau^*},\tilde\tau^*\Omega_c)
\ge
\tilde\tau^*R,
\)
which is a contradiction. Hence \(\tilde\tau^*=1\).

\noindent 
\textbf{Attraction.}
For any \(\tilde t_{\mathrm{loc}}\le s\le t<1\), the same argument as
in \eqref{eq:cfg-final-global-bound}, with part~(i) of
Lemma~\ref{lem:continuous-scaled-convex-contraction}, gives
\[
\dist\bigl(
\bx_t,tB_{\tilde r_{\Omega_c}}(\Omega_c)
\bigr)
\le
\tfrac{1-t}{1-s}
\dist\bigl(
\bx_s,sB_{\tilde r_{\Omega_c}}(\Omega_c)
\bigr).
\]
Since \((1-t)/(1-s)\le1\), the distance is non-increasing on
\([\tilde t_{\mathrm{loc}},1)\). This proves part~(i) of the theorem.

Taking \(s=\tilde t_{\mathrm{loc}}\) gives
\eqref{eq:cfg-final-global-bound} for every
\(t\in[\tilde t_{\mathrm{loc}},1)\). Moreover,
the factor multiplying \(1-t\) on the right-hand side is independent
of \(t\), so
\eqref{eq:cfg-final-global-bound} is \(\bigO(1-t)\). This proves
part~(ii) of the theorem.

\noindent
\textbf{Absorption.}
Suppose that
\(
\bx_{\tilde t_{\Omega_c}}
\in
\tilde t_{\Omega_c}
B_{\tilde r_{\Omega_c}}(\Omega_c)
\)
for some \(\tilde t_{\Omega_c}\in[\tilde t_{\mathrm{loc}},1)\), then
part~(ii) of
Lemma~\ref{lem:continuous-scaled-convex-contraction}
applied on \([\tilde t_{\mathrm{loc}},1)\) with
\(
S=B_{\tilde r_{\Omega_c}}(\Omega_c),
C_t=
B_{\tilde r_{\Omega_c}}\!\left(
\by_t^{\Omega_c}(\bx_t)
\right)
\)
and
\(\boldsymbol m_t=\hat{\by}_{t,\mathrm{cfg}}(\bx_t)\) gives
\[
\bx_t\in tB_{\tilde r_{\Omega_c}}(\Omega_c)
\]
for every \(t\in[\tilde t_{\Omega_c},1)\).
This proves part~(iii) of the theorem.
\end{proof}

\subsection{CFG Euler stagewise counterparts}\label{thm:cfg-euler-stagewise}

\begin{theorem}[Discrete version of Theorem~\ref{thm:cfg-early}]
\label{thm:cfg-early-euler}
Under the setting preceding Theorem~\ref{thm:cfg-early}, consider the CFG
Euler scheme on the complete grid $t_i=i\Delta t$, $i=0,\ldots,K$,
$\Delta t=1/K$, initialized at \(\bz_0=\bx_0\), with update
\(\bz_{i+1}=\bz_i+\Delta t\,\vcfg(t_i,\bz_i), i= 0, \dots, K-1\). If
$\Delta t\le\tilde t_{\mathrm{mean}}/2$, then the following hold.
\begin{enumerate}
\item[(i)] \emph{(Non-entry)} For every \(i\) with
$t_i\le\tilde t_{\mathrm{mean}}$,
\(
\bz_i
\notin
t_iB_{\tilde r_{\mathrm{mean}}}(\bar{\by}_{\mathrm{cfg}}).
\)

\item[(ii)] \emph{(Attraction)} For every $0\le j\le i$ with
$t_i\le\tilde t_{\mathrm{mean}}$,
\begin{equation}\label{eq:cfg-early-euler-ratio}
\dist\!\left(
\bz_i,
t_iB_{\tilde r_{\mathrm{mean}}}(\bar{\by}_{\mathrm{cfg}})
\right)
\le
\tfrac{1-t_i}{1-t_j}
\dist\!\left(
\bz_j,
t_jB_{\tilde r_{\mathrm{mean}}}(\bar{\by}_{\mathrm{cfg}})
\right).
\end{equation}
Hence the sequence
\(
i\mapsto
\dist\!\left(
\bz_i,
t_iB_{\tilde r_{\mathrm{mean}}}(\bar{\by}_{\mathrm{cfg}})
\right)
\)
is strictly decreasing with $t_i\le\tilde t_{\mathrm{mean}}$.

\item[(iii)] Taking $j=0$ in part~(ii) gives
\(\dist\!\left(
\bz_i,
t_iB_{\tilde r_{\mathrm{mean}}}(\bar{\by}_{\mathrm{cfg}})
\right)
\le
(1-t_i)\|\bx_0\|,
\ t_i\le\tilde t_{\mathrm{mean}}.\)
\end{enumerate}
\end{theorem}

\begin{proof}
The CFG Euler update has the exact
affine form
\begin{equation}\label{eq:cfg-euler-affine-step}
\bz_{i+1}
=
\tfrac{1-t_{i+1}}{1-t_i}\bz_i
+
\tfrac{t_{i+1}-t_i}{1-t_i}
\hat{\by}_{t_i,\mathrm{cfg}}(\bz_i),
\end{equation}
where
\begin{equation}\label{eq:cfg-euler-posterior-extrapolation}
\hat{\by}_{t_i,\mathrm{cfg}}(\bz)
=
(1-w)\hat{\by}_{t_i}(\bz)
+
w\hat{\by}_{t_i,c}(\bz).
\end{equation}

For every active index \(i<K\),
\eqref{eq:cfg-euler-posterior-extrapolation} and
Lemma~\ref{lem:mean-in-closed-conv} give
\begin{equation}\label{eq:cfg-early-euler-signed-posterior}
\hat{\by}_{t_i,\mathrm{cfg}}(\bz_i)
\in
(1-w)\Conv(\mathcal D)+w\Conv(\mathcal D_c),
\qquad
i=0,\ldots,K-1.
\end{equation}
Set
\(
A_{\mathrm{cfg}}
:=
(1-w)\Conv(\mathcal D)+w\Conv(\mathcal D_c).
\)
This set is compact and convex. We claim that, for every \(i\), there
exists \(\ba_i\in A_{\mathrm{cfg}}\) such that
\begin{equation}\label{eq:cfg-early-euler-affine}
\bz_i=(1-t_i)\bx_0+t_i\ba_i.
\end{equation}
Indeed, at \(i=0\), we have \(t_0=0\) and \(\bz_0=\bx_0\), so
\eqref{eq:cfg-early-euler-affine} holds after choosing any
\(\ba_0\in A_{\mathrm{cfg}}\). Suppose it holds at an active index
\(i<K\). Substituting it into \eqref{eq:cfg-euler-affine-step} gives
\begin{equation}\label{eq:cfg-early-euler-affine-induction}
\begin{aligned}
\bz_{i+1}
&=
(1-t_{i+1})\bx_0
+
\tfrac{t_i(1-t_{i+1})}{1-t_i}\ba_i
+
\tfrac{t_{i+1}-t_i}{1-t_i}
\hat{\by}_{t_i,\mathrm{cfg}}(\bz_i).
\end{aligned}
\end{equation}
The two coefficients in the last two terms satisfy
\(
\tfrac{t_i(1-t_{i+1})}{1-t_i}
+
\tfrac{t_{i+1}-t_i}{1-t_i}
=
t_{i+1}.
\)
Since \(\ba_i\in A_{\mathrm{cfg}}\),
\(\hat{\by}_{t_i,\mathrm{cfg}}(\bz_i)\in A_{\mathrm{cfg}}\) by
\eqref{eq:cfg-early-euler-signed-posterior}, and
\(A_{\mathrm{cfg}}\) is convex, the last two terms in
\eqref{eq:cfg-early-euler-affine-induction} belong to
\(t_{i+1}A_{\mathrm{cfg}}\). Hence there exists
\(\ba_{i+1}\in A_{\mathrm{cfg}}\) such that
\(\bz_{i+1}=(1-t_{i+1})\bx_0+t_{i+1}\ba_{i+1}\), which completes the
induction.
Thus, for each $i$, there exist
$\by_i\in\Conv(\mathcal D)$ and
$\by_{i,c}\in\Conv(\mathcal D_c)$ such that
\(\ba_i=(1-w)\by_i+w\by_{i,c}\), and hence
\(\bz_i=(1-t_i)\bx_0+
t_i[(1-w)\by_i+w\by_{i,c}]\).
Hence
\begin{equation}\label{eq:cfg-early-euler-norm-bound}
\|\bz_i\|
\le
(1-t_i)\|\bx_0\|
+
t_i\bigl((w-1)M_{\mathcal D}+wM_{\mathcal D_c}\bigr).
\end{equation}
Moreover, because
$\bar{\by}\in\Conv(\mathcal D)$ and
$\bar{\by}_c\in\Conv(\mathcal D_c)$, we have 
\(
\|
(1-w)\by_i+w\by_{i,c}-\bar{\by}_{\mathrm{cfg}}
\|
\le (w-1)D_\emptyset+wD_c.
\)
Consequently, whenever $t_i\le\tilde t_{\mathrm{mean}}$,
\begin{equation}\label{eq:cfg-early-euler-nonentry}
\begin{aligned}
\dist\!\left(
\bz_i,
t_iB_{\tilde r_{\mathrm{mean}}}(\bar{\by}_{\mathrm{cfg}})
\right)
&\ge
(1-t_i)\|\bx_0\|
-
t_i\left[
(w-1)D_\emptyset+wD_c+\tilde r_{\mathrm{mean}}
\right]
\\
&\ge
(1-\tilde t_{\mathrm{mean}})\|\bx_0\|
-
\tilde t_{\mathrm{mean}}
\left[
(w-1)D_\emptyset+wD_c+\tilde r_{\mathrm{mean}}
\right]
>0,
\end{aligned}
\end{equation}
where the last inequality follows from the setting in Theorem~\ref{thm:cfg-early}. This proves the non-entry property.
It remains to localize \(\hat{\by}_{t_i,\mathrm{cfg}}(\bz_i)
\) at the grid point \(\bz_i\).
The pointwise estimates
\eqref{eq:cfg-early-unconditional-posterior-bound} and
\eqref{eq:cfg-early-conditional-posterior-bound} applied at
\((t,\bx)=(t_i,\bz_i)\), with
\eqref{eq:cfg-early-euler-norm-bound}, give, for
\(0<t_i\le\tilde t_{\mathrm{mean}}\),
\[
\begin{aligned}
\left\|
\hat{\by}_{t_i}(\bz_i)-\bar{\by}
\right\|
&\le
M_{\mathcal D}
\left[
\exp\left(
\tfrac{
t_iD_\emptyset
\left[
(1-t_i)\|\bx_0\|
+
t_iw(M_{\mathcal D}+M_{\mathcal D_c})
\right]
}{
(1-t_i)^2
}
\right)-1
\right],
\\
\left\|
\hat{\by}_{t_i,c}(\bz_i)-\bar{\by}_c
\right\|
&\le
M_{\mathcal D_c}
\left[
\exp\left(
\tfrac{
t_iD_c
\left[
(1-t_i)\|\bx_0\|
+
t_i\bigl(
(w-1)M_{\mathcal D}
+
(w+1)M_{\mathcal D_c}
\bigr)
\right]
}{
(1-t_i)^2
}
\right)-1
\right].
\end{aligned}
\]
The above two right-hand sides are nondecreasing in \(t_i\).
Moreover, at \(t_i=0\),
\(
\hat{\by}_{0}(\bx_0)=\bar{\by},
\
\hat{\by}_{0,c}(\bx_0)=\bar{\by}_c,
\
\hat{\by}_{0,\mathrm{cfg}}(\bx_0)=\bar{\by}_{\mathrm{cfg}}.
\)
Thus, for every \(t_i\le\tilde t_{\mathrm{mean}}\), the definition of \(\tilde r_{\mathrm{mean}}\) at
\(t=\tilde t_{\mathrm{mean}}\) in the setup preceding
Theorem~\ref{thm:cfg-early}, and the triangle inequality give
\begin{equation}\label{eq:cfg-early-euler-posterior}
\begin{aligned}
\left\|
\hat{\by}_{t_i,\mathrm{cfg}}(\bz_i)
-
\bar{\by}_{\mathrm{cfg}}
\right\|
&=
\left\|
(1-w)
\bigl(\hat{\by}_{t_i}(\bz_i)-\bar{\by}\bigr)
+
w
\bigl(\hat{\by}_{t_i,c}(\bz_i)-\bar{\by}_c\bigr)
\right\|
\\
&\le
(w-1)
\left\|
\hat{\by}_{t_i}(\bz_i)-\bar{\by}
\right\|
+
w
\left\|
\hat{\by}_{t_i,c}(\bz_i)-\bar{\by}_c
\right\|
\le
\tilde r_{\mathrm{mean}}.
\end{aligned}
\end{equation}
Equivalently,
\(
\hat{\by}_{t_i,\mathrm{cfg}}(\bz_i)
\in
B_{\tilde r_{\mathrm{mean}}}(\bar{\by}_{\mathrm{cfg}}).
\)

\noindent 
\textbf{Attraction.}
The set \(
B_{\tilde r_{\mathrm{mean}}}(\bar{\by}_{\mathrm{cfg}})
\)
is nonempty, closed, and convex. Fix an index satisfying
\(t_{i+1}\le\tilde t_{\mathrm{mean}}\). If \(t_i=0\),
\(
\hat{\by}_{0,\mathrm{cfg}}(\bz_0)
=
\bar{\by}_{\mathrm{cfg}}
\in
B_{\tilde r_{\mathrm{mean}}}(\bar{\by}_{\mathrm{cfg}}).
\)
If \(t_i>0\), let
\(
\bq_i:=\Proj_{t_iB_{\tilde r_{\mathrm{mean}}}(\bar{\by}_{\mathrm{cfg}})}(\bz_i),
\)
which is well defined. Moreover,
\(\bq_i/t_i\in B_{\tilde r_{\mathrm{mean}}}(\bar{\by}_{\mathrm{cfg}})\), while
we also have
\(
\hat{\by}_{t_i,\mathrm{cfg}}(\bz_i)\in B_{\tilde r_{\mathrm{mean}}}(\bar{\by}_{\mathrm{cfg}}).
\)
Thus part~(i) of
Lemma~\ref{lem:euler-scaled-convex-contraction}, with
\(s=t_i\), \(t=t_{i+1}\),
\(\bx=\bz_i\),
\(\by=\hat{\by}_{t_i,\mathrm{cfg}}(\bz_i)\) and
\(S=C=B_{\tilde r_{\mathrm{mean}}}(\bar{\by}_{\mathrm{cfg}})\)
, gives 
\begin{equation}\label{eq:cfg-early-euler-one-step}
\dist(\bz_{i+1},t_{i+1}B_{\tilde r_{\mathrm{mean}}}(\bar{\by}_{\mathrm{cfg}}))
\le
\tfrac{1-t_{i+1}}{1-t_i}
\dist(\bz_i,t_iB_{\tilde r_{\mathrm{mean}}}(\bar{\by}_{\mathrm{cfg}})).
\end{equation}
Iterating \eqref{eq:cfg-early-euler-one-step} from \(j\) to \(i-1\)
and telescoping the factors proves
\eqref{eq:cfg-early-euler-ratio}. If \(j<i\), then
\eqref{eq:cfg-early-euler-nonentry} and
\((1-t_i)/(1-t_j)<1\) give
\[
\dist(\bz_i,t_iB_{\tilde r_{\mathrm{mean}}}(\bar{\by}_{\mathrm{cfg}}))
\le
\tfrac{1-t_i}{1-t_j}
\dist(\bz_j,t_jB_{\tilde r_{\mathrm{mean}}}(\bar{\by}_{\mathrm{cfg}}))
<
\dist(\bz_j,t_jB_{\tilde r_{\mathrm{mean}}}(\bar{\by}_{\mathrm{cfg}})).
\]
Thus the distance sequence is strictly decreasing between distinct
early grid points.
Taking \(j=0\) in \eqref{eq:cfg-early-euler-ratio} and using
\(0B_{\tilde r_{\mathrm{mean}}}(\bar{\by}_{\mathrm{cfg}})
=\{\bzero\}\) and \(\bz_0=\bx_0\) gives
\[
\dist(\bz_i,t_iB_{\tilde r_{\mathrm{mean}}}(\bar{\by}_{\mathrm{cfg}}))
\le
(1-t_i)\|\bx_0\|,
\]
which is part~(iii) of the theorem.
The inequality
\(\Delta t\le\tilde t_{\mathrm{mean}}/2\) ensures that this early
segment contains the two nonzero grid points \(t_1\) and \(t_2\).

\end{proof}

\begin{theorem}[Discrete version of Theorem~\ref{thm:cfg-convex}]
\label{thm:cfg-convex-euler}
Under the setting of Theorem~\ref{thm:cfg-convex}, consider the shifted CFG
Euler grid \(t_i=t_{\mathrm{infl}}+i\Delta t\), \(i=0,\ldots,N\), where
\(N\ge1\), \(\Delta t>0\), and \(t_N\le1\), initialized at
\(\bz_0=\bx_{t_{\mathrm{infl}}}\), with update
\(\bz_{i+1}=\bz_i+\Delta t\,\vcfg(t_i,\bz_i),i=0,\dots, N-1\). For every
\(j\in\{0,\ldots,N\}\) satisfying \(t_j<1\), define the grid-tail
inflation radius
\begin{equation}\label{eq:cfg-grid-tail-inflation-function}
\epsilon_{\mathrm{infl}}^{\mathrm{grid}}(j)
:=
(w-1)
\max_{\substack{j\le m\le N\\ t_m<1}}
\left\|
\hat{\by}_{t_m,c}(\bz_m)-\hat{\by}_{t_m}(\bz_m)
\right\|.
\end{equation}
Then \(\epsilon_{\mathrm{infl}}^{\mathrm{grid}}(j)\) is non-increasing
in \(j\), and the following hold.
\begin{enumerate}
\item[(i)] \emph{(Attraction)} For every \(0\le j\le i\le N\) with
\(t_j<1\),
\begin{equation}\label{eq:cfg-convex-euler-ratio}
\dist\bigl(
\bz_i,
t_iB_{\epsilon_{\mathrm{infl}}^{\mathrm{grid}}(j)}(\Conv(\mathcal D_c))
\bigr)
\le
\tfrac{1-t_i}{1-t_j}
\dist\bigl(
\bz_j,
t_jB_{\epsilon_{\mathrm{infl}}^{\mathrm{grid}}(j)}(\Conv(\mathcal D_c))
\bigr).
\end{equation}
For each fixed \(j\), the sequence
\(
i\mapsto
\dist\bigl(
\bz_i,
t_iB_{\epsilon_{\mathrm{infl}}^{\mathrm{grid}}(j)}
(\Conv(\mathcal D_c))
\bigr)
\)
is non-increasing on \(i=j,\ldots,N\).

\item[(ii)] Taking \(j=0\) in part~(i) gives, for every \(i=0,\ldots,N\),
\begin{equation}\label{eq:cfg-convex-euler-initial-bound}
\begin{aligned}
\dist\bigl(
\bz_i,
t_iB_{\epsilon_{\mathrm{infl}}^{\mathrm{grid}}(0)}(\Conv(\mathcal D_c))
\bigr)
\le
\tfrac{1-t_i}{1-t_{\mathrm{infl}}}
\dist\bigl(
\bz_0,
t_{\mathrm{infl}}
B_{\epsilon_{\mathrm{infl}}^{\mathrm{grid}}(0)}(\Conv(\mathcal D_c))
\bigr).
\end{aligned}
\end{equation}

\item[(iii)] \emph{(Absorption)} Fix \(j\in\{0,\ldots,N\}\) with
\(t_j<1\). If there exists \(k\in\{j,\ldots,N\}\) such that
\(
\bz_k
\in
t_kB_{\epsilon_{\mathrm{infl}}^{\mathrm{grid}}(j)}
(\Conv(\mathcal D_c)),
\)
then
\(
\bz_i
\in
t_iB_{\epsilon_{\mathrm{infl}}^{\mathrm{grid}}(j)}
(\Conv(\mathcal D_c))
\)
for every \(i=k,\ldots,N\).
\end{enumerate}
\end{theorem}

\begin{proof}
We use the exact CFG Euler identities
\eqref{eq:cfg-euler-affine-step} and
\eqref{eq:cfg-euler-posterior-extrapolation} established above.
The index sets in the maximum
\eqref{eq:cfg-grid-tail-inflation-function} are nested as \(j\)
increases. Hence
\(
\epsilon_{\mathrm{infl}}^{\mathrm{grid}}(j_2)
\le
\epsilon_{\mathrm{infl}}^{\mathrm{grid}}(j_1)
\)
whenever \(0\le j_1\le j_2\le N\) and \(t_{j_2}<1\).
Fix \(j\in\{0,\ldots,N\}\) with \(t_j<1\). 
The set \(B_{\epsilon_{\mathrm{infl}}^{\mathrm{grid}}(j)}
(\Conv(\mathcal D_c))\) is nonempty, closed, and convex.
For every active index \(m\in\{j,\ldots,N-1\}\),
Lemma~\ref{lem:mean-in-closed-conv} gives
\(\hat{\by}_{t_m,c}(\bz_m)\in\Conv(\mathcal D_c)\). Using
\eqref{eq:cfg-euler-posterior-extrapolation} and
\eqref{eq:cfg-grid-tail-inflation-function} gives
\begin{equation}\label{eq:cfg-convex-euler-posterior}
\begin{aligned}
\dist\!\left(
\hat{\by}_{t_m,\mathrm{cfg}}(\bz_m),
\Conv(\mathcal D_c)
\right)
&\le
\left\|
\hat{\by}_{t_m,\mathrm{cfg}}(\bz_m)
-
\hat{\by}_{t_m,c}(\bz_m)
\right\|
\\
&=
(w-1)
\left\|
\hat{\by}_{t_m,c}(\bz_m)
-
\hat{\by}_{t_m}(\bz_m)
\right\|
\le
\epsilon_{\mathrm{infl}}^{\mathrm{grid}}(j).
\end{aligned}
\end{equation}
Consequently,
\(
\hat{\by}_{t_m,\mathrm{cfg}}(\bz_m)\in B_{\epsilon_{\mathrm{infl}}^{\mathrm{grid}}(j)}(\Conv(\mathcal D_c)).
\)

\noindent
\textbf{Attraction.}
For every active index \(m\in\{j,\ldots,N-1\}\),
\eqref{eq:cfg-euler-affine-step} reads
\(
\bz_{m+1}
=
\tfrac{1-t_{m+1}}{1-t_m}\bz_m
+
\tfrac{t_{m+1}-t_m}{1-t_m}
\hat{\by}_{t_m,\mathrm{cfg}}(\bz_m).
\)
If \(t_j=0\), then \(j=0\). Since
\(\hat{\by}_{0,\mathrm{cfg}}(\bz_0)\in
B_{\epsilon_{\mathrm{infl}}^{\mathrm{grid}}(0)}
(\Conv(\mathcal D_c))\), the zero-time case of part~(i) of
Lemma~\ref{lem:euler-scaled-convex-contraction} applies with
\(s=t_0\), \(t=t_1\), \(\bx=\bz_0\),
\(\by=\hat{\by}_{0,\mathrm{cfg}}(\bz_0)\), and
\(S=B_{\epsilon_{\mathrm{infl}}^{\mathrm{grid}}(0)}
(\Conv(\mathcal D_c))\), and gives
\eqref{eq:cfg-convex-euler-one-step} below for \(m=0\).

Now fix \(m\in\{j,\ldots,N-1\}\) with \(t_m>0\), and let
\(
\bq_m:=
\Proj_{t_mB_{\epsilon_{\mathrm{infl}}^{\mathrm{grid}}(j)}
(\Conv(\mathcal D_c))}(\bz_m)
\).
We have 
\({\bq_m}/{t_m}
\in
B_{\epsilon_{\mathrm{infl}}^{\mathrm{grid}}(j)}
(\Conv(\mathcal D_c))\) and
\(
\hat{\by}_{t_m,\mathrm{cfg}}(\bz_m)
\in
B_{\epsilon_{\mathrm{infl}}^{\mathrm{grid}}(j)}
(\Conv(\mathcal D_c)).
\)
Therefore, part~(i) of
Lemma~\ref{lem:euler-scaled-convex-contraction} applies with
\(s=t_m\), \(t=t_{m+1}\), \(\bx=\bz_m\),
\(\by=\hat{\by}_{t_m,\mathrm{cfg}}(\bz_m)\),  and
\(
S=C=B_{\epsilon_{\mathrm{infl}}^{\mathrm{grid}}(j)}
(\Conv(\mathcal D_c))
\), and gives \eqref{eq:cfg-convex-euler-one-step} below for every
such \(m\).
Thus, for every \(m\in\{j,\ldots,N-1\}\),
\begin{equation}\label{eq:cfg-convex-euler-one-step}
\dist\bigl(
\bz_{m+1},
t_{m+1}B_{\epsilon_{\mathrm{infl}}^{\mathrm{grid}}(j)}
(\Conv(\mathcal D_c))
\bigr)
\le
\tfrac{1-t_{m+1}}{1-t_m}
\dist\bigl(
\bz_m,
t_mB_{\epsilon_{\mathrm{infl}}^{\mathrm{grid}}(j)}
(\Conv(\mathcal D_c))
\bigr).
\end{equation}
Fix \(i\in\{j+1,\ldots,N\}\). Iterating
\eqref{eq:cfg-convex-euler-one-step} from \(m=j\) to \(m=i-1\)
and telescoping gives
\[
\begin{aligned}
\dist\bigl(
\bz_i,
t_iB_{\epsilon_{\mathrm{infl}}^{\mathrm{grid}}(j)}
(\Conv(\mathcal D_c))
\bigr)
&\le
\left(
\prod_{m=j}^{i-1}
\tfrac{1-t_{m+1}}{1-t_m}
\right)
\dist\bigl(
\bz_j,
t_jB_{\epsilon_{\mathrm{infl}}^{\mathrm{grid}}(j)}
(\Conv(\mathcal D_c))
\bigr)
\\
&\quad=
\tfrac{1-t_i}{1-t_j}
\dist\bigl(
\bz_j,
t_jB_{\epsilon_{\mathrm{infl}}^{\mathrm{grid}}(j)}
(\Conv(\mathcal D_c))
\bigr).
\end{aligned}
\]
This is \eqref{eq:cfg-convex-euler-ratio}; the case \(i=j\) is
immediate. Since \(0\le(1-t_i)/(1-t_j)\le1\), the corresponding
distance sequence is non-increasing. Taking \(j=0\) in
\eqref{eq:cfg-convex-euler-ratio} gives
\eqref{eq:cfg-convex-euler-initial-bound}.

\noindent
\textbf{Absorption.}
Fix \(j\in\{0,\ldots,N\}\) with \(t_j<1\), and suppose that
\(
\bz_k
\in
t_kB_{\epsilon_{\mathrm{infl}}^{\mathrm{grid}}(j)}
(\Conv(\mathcal D_c))
\)
for some \(k\in\{j,\ldots,N\}\).
We prove by induction that
\[
\bz_i
\in
t_iB_{\epsilon_{\mathrm{infl}}^{\mathrm{grid}}(j)}
(\Conv(\mathcal D_c)),
\qquad
i=k,\ldots,N.
\]

The assertion at \(i=k\) is given.
Now fix \(i\in\{k,\ldots,N-1\}\) and suppose that it holds at \(i\).
If \(t_i>0\), then
\(
{\bz_i}/{t_i},
\
\hat{\by}_{t_i,\mathrm{cfg}}(\bz_i)
\in
B_{\epsilon_{\mathrm{infl}}^{\mathrm{grid}}(j)}
(\Conv(\mathcal D_c)),
\)
where the second inclusion follows from
\eqref{eq:cfg-convex-euler-posterior}.
Therefore, part~(ii) of
Lemma~\ref{lem:euler-scaled-convex-contraction},
applied with
\(s=t_i\), \(t=t_{i+1}\),
\(\bx=\bz_i\),
\(\by=\hat{\by}_{t_i,\mathrm{cfg}}(\bz_i)\), and
\(
S = C=
B_{\epsilon_{\mathrm{infl}}^{\mathrm{grid}}(j)}
(\Conv(\mathcal D_c)),
\)
gives
\[
\bz_{i+1}
\in
t_{i+1}
B_{\epsilon_{\mathrm{infl}}^{\mathrm{grid}}(j)}
(\Conv(\mathcal D_c)).
\]

If \(t_i=0\), then the induction hypothesis gives
\(\bz_i=\boldsymbol0\), while
\eqref{eq:cfg-convex-euler-posterior} gives
\[
\hat{\by}_{0,\mathrm{cfg}}(\bz_i)
\in
B_{\epsilon_{\mathrm{infl}}^{\mathrm{grid}}(j)}
(\Conv(\mathcal D_c)).
\]
Thus part~(ii) of
Lemma~\ref{lem:euler-scaled-convex-contraction} gives the same conclusion.
The result follows by induction, which proves part~(iii).

\end{proof}

\begin{corollary}[Discrete version of Corollary~\ref{cor:cfg-convex-small}]
\label{cor:cfg-convex-euler-small}
In the setting of Theorem~\ref{thm:cfg-convex-euler}, suppose
\(\epsilon_{\mathrm{infl}}^{\mathrm{grid}}(0)>0\). Define
\[
\begin{aligned}
\rho_{\mathrm{infl}}^{\mathrm{grid}}
&:=
\dist\bigl(
B_{\epsilon_{\mathrm{infl}}^{\mathrm{grid}}(0)}(\Conv(\mathcal D_c)),
\mathcal D\setminus\mathcal D_c
\bigr),
\\
\Delta_{\mathrm{infl}}^{\mathrm{grid}}
&:=
\bigl(\rho_{\mathrm{infl}}^{\mathrm{grid}}\bigr)^2
-
\bigl(\epsilon_{\mathrm{infl}}^{\mathrm{grid}}(0)+D_c\bigr)^2,
\end{aligned}
\]
and assume
\(
\Delta_{\mathrm{infl}}^{\mathrm{grid}}>0.
\)
If
\(
\bz_k\in
t_kB_{\epsilon_{\mathrm{infl}}^{\mathrm{grid}}(0)}
(\Conv(\mathcal D_c))
\)
for some \(k\in\{0,\ldots,N\}\) with \(0<t_k<1\), then for every
\(j\in\{k,\ldots,N\}\) with \(t_j<1\),
\begin{equation}\label{eq:cfg-convex-euler-exponential-inflation}
\epsilon_{\mathrm{infl}}^{\mathrm{grid}}(j)
\le
(w-1)D_\emptyset
\tfrac{1-p_\emptyset(\mathcal D_c)}
{p_\emptyset(\mathcal D_c)}
\exp\left(
-\tfrac{
t_j^2\Delta_{\mathrm{infl}}^{\mathrm{grid}}
}{
2(1-t_j)^2
}
\right).
\end{equation}
Suppose there exists \(\ell\in\{k,\ldots,N\}\) with \(t_\ell<1\)
satisfying
\[
(w-1)D_\emptyset
\tfrac{1-p_\emptyset(\mathcal D_c)}
{p_\emptyset(\mathcal D_c)}
\exp\left(
-\tfrac{
t_\ell^2\Delta_{\mathrm{infl}}^{\mathrm{grid}}
}{
2(1-t_\ell)^2
}
\right)
<
\epsilon_{\mathrm{infl}}^{\mathrm{grid}}(0),
\]
and fix one such \(\ell\). Define
\(
\epsilon_{\mathrm{ref}}^{\mathrm{grid}}
:=
\epsilon_{\mathrm{infl}}^{\mathrm{grid}}(\ell),
\)
then
\(
0\le
\epsilon_{\mathrm{ref}}^{\mathrm{grid}}
<
\epsilon_{\mathrm{infl}}^{\mathrm{grid}}(0).
\)
Moreover, the following hold.
\begin{enumerate}
\item[(i)] \emph{(Refined attraction)} For every
\(\ell\le j\le i\le N\) with \(t_j<1\),
\begin{equation}\label{eq:cfg-convex-euler-refined-ratio}
\dist\bigl(
\bz_i,
t_iB_{\epsilon_{\mathrm{ref}}^{\mathrm{grid}}}(\Conv(\mathcal D_c))
\bigr)
\le
\tfrac{1-t_i}{1-t_j}
\dist\bigl(
\bz_j,
t_jB_{\epsilon_{\mathrm{ref}}^{\mathrm{grid}}}(\Conv(\mathcal D_c))
\bigr).
\end{equation}
Hence the sequence
\(
i\mapsto
\dist\bigl(
\bz_i,
t_iB_{\epsilon_{\mathrm{ref}}^{\mathrm{grid}}}
(\Conv(\mathcal D_c))
\bigr)
\)
is non-increasing on \(i=\ell,\ldots,N\).

\item[(ii)] Taking \(j=\ell\) in part~(i) gives, for every
\(i=\ell,\ldots,N\),
\begin{equation}\label{eq:cfg-convex-euler-refined-initial-bound}
\begin{aligned}
\dist\bigl(
\bz_i,
t_iB_{\epsilon_{\mathrm{ref}}^{\mathrm{grid}}}(\Conv(\mathcal D_c))
\bigr)
\le
\tfrac{1-t_i}{1-t_\ell}
\dist\bigl(
\bz_\ell,
t_\ell B_{\epsilon_{\mathrm{ref}}^{\mathrm{grid}}}
(\Conv(\mathcal D_c))
\bigr).
\end{aligned}
\end{equation}

\item[(iii)] \emph{(Refined absorption)} If there exists
\(k'\in\{\ell,\ldots,N\}\) such that
\(
\bz_{k'}
\in
t_{k'}B_{\epsilon_{\mathrm{ref}}^{\mathrm{grid}}}
(\Conv(\mathcal D_c)),
\)
then
\(
\bz_i
\in
t_iB_{\epsilon_{\mathrm{ref}}^{\mathrm{grid}}}
(\Conv(\mathcal D_c))
\)
for every \(i=k',\ldots,N\).
\end{enumerate}
\end{corollary}

\begin{proof}
Theorem~\ref{thm:cfg-convex-euler}(iii) and
the assumption give
\begin{equation}\label{eq:cfg-convex-euler-large-capture}
\tfrac{\bz_m}{t_m}
\in
B_{\epsilon_{\mathrm{infl}}^{\mathrm{grid}}(0)}
(\Conv(\mathcal D_c)),
\qquad
m=k,\ldots,N.
\end{equation}
This is well defined because \(t_m\ge t_k>0\).

Fix \(j\in\{k,\ldots,N\}\) with \(t_j<1\), and consider any
\(m\in\{j,\ldots,N\}\) with \(t_m<1\). The capture inclusion
\eqref{eq:cfg-convex-euler-large-capture} and
\(\diam(\Conv(\mathcal D_c))=D_c\) give the inside estimate below,
while the definition of \(\rho_{\mathrm{infl}}^{\mathrm{grid}}\)
gives the outside estimate:
\begin{equation}\label{eq:cfg-convex-euler-inside-outside}
\begin{aligned}
\|\bz_m-t_m\by\|
&\le
t_m\left(
\epsilon_{\mathrm{infl}}^{\mathrm{grid}}(0)+D_c
\right),
&& \by\in\mathcal D_c,
\\
\|\bz_m-t_m\by\|
&\ge
t_m\rho_{\mathrm{infl}}^{\mathrm{grid}},
&& \by\in\mathcal D\setminus\mathcal D_c.
\end{aligned}
\end{equation}
Apply the calculation in 
\eqref{eq:cfg-convex-ratio-bound},
\eqref{eq:posterior-mean-difference-algebra}, and
\eqref{eq:cfg-convex-hat-bound} with the substitutions
\(t \mapsto t_m,\
\bx_t\mapsto\bz_m, \) \(
\epsilon_{\mathrm{infl}}(t_{\mathrm{infl}})
\mapsto\epsilon_{\mathrm{infl}}^{\mathrm{grid}}(0), \
\rho_{\mathrm{infl}}
\mapsto\rho_{\mathrm{infl}}^{\mathrm{grid}}, \
\Delta_{\mathrm{infl}}
\mapsto\Delta_{\mathrm{infl}}^{\mathrm{grid}}.\)
Using \eqref{eq:cfg-convex-euler-inside-outside}, we obtain
\begin{equation}\label{eq:cfg-convex-euler-refined-gap}
\left\|
\hat{\by}_{t_m,c}(\bz_m)-\hat{\by}_{t_m}(\bz_m)
\right\|
\le
D_\emptyset
\tfrac{1-p_\emptyset(\mathcal D_c)}
{p_\emptyset(\mathcal D_c)}
\exp\left(
-\tfrac{
t_m^2\Delta_{\mathrm{infl}}^{\mathrm{grid}}
}{
2(1-t_m)^2
}
\right).
\end{equation}
Since \(t/(1-t)\) is increasing on \([0,1)\) and \(m\ge j\),
\[
\left\|
\hat{\by}_{t_m,c}(\bz_m)-\hat{\by}_{t_m}(\bz_m)
\right\|
\le
D_\emptyset
\tfrac{1-p_\emptyset(\mathcal D_c)}
{p_\emptyset(\mathcal D_c)}
\exp\left(
-\tfrac{
t_j^2\Delta_{\mathrm{infl}}^{\mathrm{grid}}
}{
2(1-t_j)^2
}
\right).
\]
Taking the maximum over all such \(m\) and multiplying by \(w-1\)
proves \eqref{eq:cfg-convex-euler-exponential-inflation}.
At \(j=\ell\), that estimate and the choice of \(\ell\) give
\[
0\le
\epsilon_{\mathrm{ref}}^{\mathrm{grid}}
=
\epsilon_{\mathrm{infl}}^{\mathrm{grid}}(\ell)
<
\epsilon_{\mathrm{infl}}^{\mathrm{grid}}(0).
\]

For every active index \(m\in\{\ell,\ldots,N-1\}\), the definition of
\(\epsilon_{\mathrm{ref}}^{\mathrm{grid}}\) gives
\(
(w-1)
\|\hat{\by}_{t_m,c}(\bz_m)-\hat{\by}_{t_m}(\bz_m)\|
\le\epsilon_{\mathrm{ref}}^{\mathrm{grid}}.
\)
Moreover,
\(\hat{\by}_{t_m,c}(\bz_m)\in\Conv(\mathcal D_c)\). Hence
\eqref{eq:cfg-euler-posterior-extrapolation} gives
\begin{equation}\label{eq:cfg-convex-euler-refined-posterior}
\begin{aligned}
\dist\!\left(
\hat{\by}_{t_m,\mathrm{cfg}}(\bz_m),
\Conv(\mathcal D_c)
\right)
&\le
\left\|
\hat{\by}_{t_m,\mathrm{cfg}}(\bz_m)
-
\hat{\by}_{t_m,c}(\bz_m)
\right\|
\\
&=
(w-1)
\left\|
\hat{\by}_{t_m,c}(\bz_m)
-
\hat{\by}_{t_m}(\bz_m)
\right\|
\le
\epsilon_{\mathrm{ref}}^{\mathrm{grid}}.
\end{aligned}
\end{equation}
Consequently,
\(
\hat{\by}_{t_m,\mathrm{cfg}}(\bz_m)
\in
B_{\epsilon_{\mathrm{ref}}^{\mathrm{grid}}}(\Conv(\mathcal D_c)).
\)

\noindent 
\textbf{Attraction.}
The set
\(
B_{\epsilon_{\mathrm{ref}}^{\mathrm{grid}}}(\Conv(\mathcal D_c))
\)
is nonempty, closed, and convex. For every active index
\(m\in\{\ell,\ldots,N-1\}\), \eqref{eq:cfg-euler-affine-step} reads
\(
\bz_{m+1}
=
\tfrac{1-t_{m+1}}{1-t_m}\bz_m
+
\tfrac{t_{m+1}-t_m}{1-t_m}
\hat{\by}_{t_m,\mathrm{cfg}}(\bz_m).
\)
Since \(t_m\ge t_\ell\ge t_k>0\), let
\(
\bq_m:=
\Proj_{t_mB_{\epsilon_{\mathrm{ref}}^{\mathrm{grid}}}
(\Conv(\mathcal D_c))}(\bz_m)
\).
Then \({\bq_m}/{t_m}
\in
B_{\epsilon_{\mathrm{ref}}^{\mathrm{grid}}}
(\Conv(\mathcal D_c))\) and 
\[
\begin{aligned}
\|\bz_m-\bq_m\|
=
\dist\bigl(
\bz_m,
t_mB_{\epsilon_{\mathrm{ref}}^{\mathrm{grid}}}
(\Conv(\mathcal D_c))
\bigr),
\
\hat{\by}_{t_m,\mathrm{cfg}}(\bz_m)
\in
B_{\epsilon_{\mathrm{ref}}^{\mathrm{grid}}}
(\Conv(\mathcal D_c)).
\end{aligned}
\]
Therefore, part~(i) of
Lemma~\ref{lem:euler-scaled-convex-contraction} applies with
\(s=t_m\), \(t=t_{m+1}\), \(\bx=\bz_m\),
\(\by=\hat{\by}_{t_m,\mathrm{cfg}}(\bz_m)\), and
\(
S=C=B_{\epsilon_{\mathrm{ref}}^{\mathrm{grid}}}
(\Conv(\mathcal D_c))
\), and gives \eqref{eq:cfg-convex-euler-refined-one-step} below.
Thus, for every \(m\in\{\ell,\ldots,N-1\}\),
\begin{equation}\label{eq:cfg-convex-euler-refined-one-step}
\dist\!\left(
\bz_{m+1},
t_{m+1}B_{\epsilon_{\mathrm{ref}}^{\mathrm{grid}}}
(\Conv(\mathcal D_c))
\right)
\le
\tfrac{1-t_{m+1}}{1-t_m}
\dist\!\left(
\bz_m,
t_mB_{\epsilon_{\mathrm{ref}}^{\mathrm{grid}}}
(\Conv(\mathcal D_c))
\right).
\end{equation}
Fix \(\ell\le j<i\le N\). Iterating
\eqref{eq:cfg-convex-euler-refined-one-step} from \(m=j\) to
\(m=i-1\) and telescoping gives
\[
\begin{aligned}
\dist\bigl(
\bz_i,
t_iB_{\epsilon_{\mathrm{ref}}^{\mathrm{grid}}}
(\Conv(\mathcal D_c))
\bigr)
&\le
\left(
\prod_{m=j}^{i-1}
\tfrac{1-t_{m+1}}{1-t_m}
\right)
\dist\bigl(
\bz_j,
t_jB_{\epsilon_{\mathrm{ref}}^{\mathrm{grid}}}
(\Conv(\mathcal D_c))
\bigr)
\\
&=
\tfrac{1-t_i}{1-t_j}
\dist\bigl(
\bz_j,
t_jB_{\epsilon_{\mathrm{ref}}^{\mathrm{grid}}}
(\Conv(\mathcal D_c))
\bigr).
\end{aligned}
\]
This is \eqref{eq:cfg-convex-euler-refined-ratio}; the case \(i=j\)
is immediate. Since \(0\le(1-t_i)/(1-t_j)\le1\), the refined
distance sequence is non-increasing. Taking \(j=\ell\) in
\eqref{eq:cfg-convex-euler-refined-ratio} gives
\eqref{eq:cfg-convex-euler-refined-initial-bound}.

\noindent
\textbf{Absorption.}
Suppose that
\(
\bz_{k'}
\in
t_{k'}B_{\epsilon_{\mathrm{ref}}^{\mathrm{grid}}}
(\Conv(\mathcal D_c))
\)
for some \(k'\in\{\ell,\ldots,N\}\).
We prove by induction that
\[
\bz_i
\in
t_iB_{\epsilon_{\mathrm{ref}}^{\mathrm{grid}}}
(\Conv(\mathcal D_c)),
\qquad
i=k',\ldots,N.
\]

The assertion at \(i=k'\) is given.
Now fix \(i\in\{k',\ldots,N-1\}\) and suppose that it holds at \(i\).
Since \(t_i\ge t_\ell\ge t_k>0\),
\(
{\bz_i}/{t_i}
\in
B_{\epsilon_{\mathrm{ref}}^{\mathrm{grid}}}
(\Conv(\mathcal D_c)),
\)
while \eqref{eq:cfg-convex-euler-refined-posterior} gives
\(
\hat{\by}_{t_i,\mathrm{cfg}}(\bz_i)
\in
B_{\epsilon_{\mathrm{ref}}^{\mathrm{grid}}}
(\Conv(\mathcal D_c)).
\)
Therefore, part~(ii) of
Lemma~\ref{lem:euler-scaled-convex-contraction},
applied with
\(s=t_i\), \(t=t_{i+1}\),
\(\bx=\bz_i\),
\(\by=\hat{\by}_{t_i,\mathrm{cfg}}(\bz_i)\), and
\(
S = C =
B_{\epsilon_{\mathrm{ref}}^{\mathrm{grid}}}
(\Conv(\mathcal D_c)),
\)
gives
\[
\bz_{i+1}
\in
t_{i+1}
B_{\epsilon_{\mathrm{ref}}^{\mathrm{grid}}}
(\Conv(\mathcal D_c)).
\]
The result follows by induction, which proves part~(iii).

\end{proof}

\begin{theorem}[Discrete version of Theorem~\ref{thm:cfg-final}]
\label{thm:cfg-final-euler}
Under Assumption~\ref{assump:cfg-local-cluster}, fix $w>1$, and let
$\tilde t_{\mathrm{loc}}$ and $\tilde r_{\Omega_c}$ be as in the setup
preceding Theorem~\ref{thm:cfg-final}. Consider the shifted grid
$t_i=\tilde{t}_{\mathrm{loc}}+i\Delta t$, $i=0,\ldots,N$ where $N\geq 1,  \Delta t>0$, and
$t_N\le1$, and the shifted CFG Euler update
$\bz_{i+1}=\bz_i+\Delta t\,\vcfg(t_i,\bz_i), i= 0 , \dots, N-1$ initialized at
$\bz_0=\bx_{\tilde{t}_{\mathrm{loc}}}$. Assume
\begin{equation}\label{eq:cfg-final-euler-capture-condition}
\dist\bigl(
\bz_0,
\tilde t_{\mathrm{loc}}B_{\tilde r_{\Omega_c}}(\Omega_c)
\bigr)
+
\tilde r_{\Omega_c}
\le
\tilde t_{\mathrm{loc}}R,
\end{equation}
then the following hold.
\begin{enumerate}
\item[(i)] \emph{(Attraction)} For every
$0\le j\le i\le N$ with $t_j<1$,
\begin{equation}\label{eq:cfg-final-euler-ratio}
\dist(\bz_i,t_iB_{\tilde r_{\Omega_c}}(\Omega_c))
\le
\tfrac{1-t_i}{1-t_j}
\dist(\bz_j,t_jB_{\tilde r_{\Omega_c}}(\Omega_c)).
\end{equation}
Hence the sequence
$i\mapsto\dist(\bz_i,t_iB_{\tilde r_{\Omega_c}}(\Omega_c))$ is
non-increasing on \(i=0,\ldots,N\).

\item[(ii)] Taking $j=0$ in part~(i) gives
\[
\dist(\bz_i,t_iB_{\tilde r_{\Omega_c}}(\Omega_c))
\le
\tfrac{1-t_i}{1-\tilde{t}_{\mathrm{loc}}}
\dist(\bz_0,\tilde{t}_{\mathrm{loc}}B_{\tilde r_{\Omega_c}}(\Omega_c))
=\bigO(1-t_i),
\qquad i=0,\ldots,N.
\]

\item[(iii)] \emph{(Absorption)} If there exists
$k\in\{0,\ldots,N\}$ such that
$\bz_k\in t_kB_{\tilde r_{\Omega_c}}(\Omega_c)$, then
$\bz_i\in t_iB_{\tilde r_{\Omega_c}}(\Omega_c)$ for every
$i=k,\ldots,N$.
\end{enumerate}
\end{theorem}

\begin{proof}
By \eqref{eq:cfg-final-inflation-radius} and
\(\tilde\varepsilon_{\Omega_c}>0\), we have
\(\tilde r_{\Omega_c}>0\); moreover,
\eqref{eq:cfg-final-euler-capture-condition} and
\(t_0=\tilde t_{\mathrm{loc}}<1\) give
\begin{equation}\label{eq:cfg-final-euler-initial-validity}
0<\tilde r_{\Omega_c}
\le
t_0R-\dist(\bz_0,t_0B_{\tilde r_{\Omega_c}}(\Omega_c))
<R.
\end{equation}
We prove by induction that
\begin{equation}\label{eq:cfg-final-euler-step-invariant}
\dist\bigl(
\bz_i,
t_iB_{\tilde r_{\Omega_c}}(\Omega_c)
\bigr)
+
\tilde r_{\Omega_c}
\le
t_iR,
\qquad
i=0,\ldots,N.
\end{equation}
The case \(i=0\) is exactly
\eqref{eq:cfg-final-euler-capture-condition}.
Suppose that \eqref{eq:cfg-final-euler-step-invariant} holds at an
index \(i<N\). Since \(t_i<1\) and \(\tilde r_{\Omega_c}>0\),
\begin{align}\label{eq:cfg-final-euler-local-validity-step}
\dist(\bz_i,t_i\Omega_c)
&\le
\dist\bigl(
\bz_i,
t_iB_{\tilde r_{\Omega_c}}(\Omega_c)
\bigr)
+
t_i\tilde r_{\Omega_c}\notag\\
&=
\dist\bigl(
\bz_i,
t_iB_{\tilde r_{\Omega_c}}(\Omega_c)
\bigr)
+
\tilde r_{\Omega_c}-(1-t_i)\tilde r_{\Omega_c}\notag\\
&\le
t_iR-(1-t_i)\tilde r_{\Omega_c}
<
t_iR.
\end{align}
By \eqref{eq:cfg-final-euler-local-validity-step}, the prox-regularity of \(\Omega_c\) allows us to define
\(
t_i\by_i^{\Omega_c}:=\Proj_{t_i\Omega_c}(\bz_i).
\)
Lemma~\ref{lem:cfg-final-local-posterior} now gives
\begin{equation}\label{eq:cfg-final-euler-posterior}
\hat{\by}_{t_i,\mathrm{cfg}}(\bz_i)
\in
B_{\tilde r_{\Omega_c}}(\by_i^{\Omega_c}).
\end{equation}
Since \(\by_i^{\Omega_c}\in\Omega_c\),
\(
B_{\tilde r_{\Omega_c}}(\by_i^{\Omega_c})
\subset
B_{\tilde r_{\Omega_c}}(\Omega_c)
\)
is nonempty, closed and convex.
Let
\(
\bq_i:=\Proj_{t_iB_{\tilde r_{\Omega_c}}(\Omega_c)}(\bz_i).
\)
Lemma~\ref{lem:inflated-neighborhood-projection-regularity} and
\eqref{eq:cfg-final-euler-posterior} give
\(\left\{
\tfrac{\bq_i}{t_i},
\hat{\by}_{t_i,\mathrm{cfg}}(\bz_i)
\right\}
\subset
B_{\tilde r_{\Omega_c}}(\by_i^{\Omega_c})
\subset
B_{\tilde r_{\Omega_c}}(\Omega_c).
\)
If
\(
\dist(\bz_i,t_iB_{\tilde r_{\Omega_c}}(\Omega_c))>0,
\)
then  part~(i) of
Lemma~\ref{lem:euler-scaled-convex-contraction}, applied with
\(
s=t_i,\ t=t_{i+1},\ 
S=B_{\tilde r_{\Omega_c}}(\Omega_c),\
C=B_{\tilde r_{\Omega_c}}(\by_i^{\Omega_c}),
\)
gives \eqref{eq:cfg-final-euler-one-step} stated below.
If
\(
\dist(\bz_i,t_iB_{\tilde r_{\Omega_c}}(\Omega_c))=0,
\)
then \(\bq_i=\bz_i\), so
\(
\left\{
\tfrac{\bz_i}{t_i},
\hat{\by}_{t_i,\mathrm{cfg}}(\bz_i)
\right\}
\subset
B_{\tilde r_{\Omega_c}}(\by_i^{\Omega_c})
\subset
B_{\tilde r_{\Omega_c}}(\Omega_c).
\)
Part~(ii) of Lemma~\ref{lem:euler-scaled-convex-contraction}, applied
with
\(
s=t_i,\ t=t_{i+1},\ 
S=B_{\tilde r_{\Omega_c}}(\Omega_c),
\
C=B_{\tilde r_{\Omega_c}}(\by_i^{\Omega_c}),
\)
gives \(\bz_{i+1}\in
t_{i+1}B_{\tilde r_{\Omega_c}}(\Omega_c)\). Hence the next distance
is zero, so \eqref{eq:cfg-final-euler-one-step} below holds trivially.
Thus, in both cases,
\begin{align}\label{eq:cfg-final-euler-one-step}
\dist(\bz_{i+1},t_{i+1}B_{\tilde r_{\Omega_c}}(\Omega_c))
\le
\tfrac{1-t_{i+1}}{1-t_i}
\dist(\bz_i,t_iB_{\tilde r_{\Omega_c}}(\Omega_c)).
\end{align}
Combining this estimate with the induction assumption gives
\begin{equation}\label{eq:cfg-final-euler-invariant-step}
\begin{aligned}
\dist(\bz_{i+1},t_{i+1}B_{\tilde r_{\Omega_c}}(\Omega_c))+\tilde r_{\Omega_c}
&\le
\tfrac{1-t_{i+1}}{1-t_i}(t_iR-\tilde r_{\Omega_c})+\tilde r_{\Omega_c}
\\
&=
t_{i+1}R
-
\tfrac{t_{i+1}-t_i}{1-t_i}
(R-\tilde r_{\Omega_c})
<
t_{i+1}R.
\end{aligned}
\end{equation}
The strict inequality in
\eqref{eq:cfg-final-euler-invariant-step} follows from
\(t_{i+1}>t_i\) and
\(\tilde r_{\Omega_c}<R\), and remains valid when \(t_{i+1}=1\).
Thus \eqref{eq:cfg-final-euler-step-invariant} holds at \(i+1\).
By induction, \eqref{eq:cfg-final-euler-step-invariant} holds for all
\(i=0,\ldots,N\).

\noindent
\textbf{Attraction.}
The one-step estimate \eqref{eq:cfg-final-euler-one-step} holds for every
\(i<N\). For every \(0\le j\le i\le N\) with \(t_j<1\), iterating it
from index \(j\) to index \(i-1\) and telescoping the factors gives
\eqref{eq:cfg-final-euler-ratio}; the case \(i=j\) is immediate. Since
every one-step factor is at most one, the distance sequence is
non-increasing. Taking \(j=0\) gives the displayed
\(\bigO(1-t_i)\) bound in part~(ii).

\noindent
\textbf{Absorption.}
Suppose that
\(\bz_k\in t_kB_{\tilde r_{\Omega_c}}(\Omega_c)\)
for some \(k\in\{0,\ldots,N\}\).
We prove by induction that
\[
\bz_i\in t_iB_{\tilde r_{\Omega_c}}(\Omega_c),
\qquad
i=k,\ldots,N.
\]

The assertion at \(i=k\) is given.
Fix \(i\in\{k,\ldots,N-1\}\) and suppose that
\(\bz_i\in t_iB_{\tilde r_{\Omega_c}}(\Omega_c)\).
Here,
\eqref{eq:cfg-final-euler-local-validity-step} show that
\(\dist(\bz_i,t_i\Omega_c)<t_iR\). Hence
Lemma~\ref{lem:inflated-neighborhood-projection-regularity} applies.
Since \(\bz_i\) belongs to the inflated target, it is its own
projection, and therefore, with
\(t_i\by_i^{\Omega_c}=\Proj_{t_i\Omega_c}(\bz_i)\),
\[
\bz_i/t_i
\in
B_{\tilde r_{\Omega_c}}\!\left(\by_i^{\Omega_c}\right).
\]
Meanwhile, \eqref{eq:cfg-final-euler-posterior} gives
\[
\hat{\by}_{t_i,\mathrm{cfg}}(\bz_i)
\in
B_{\tilde r_{\Omega_c}}\!\left(
\by_i^{\Omega_c}
\right).
\] The set
\(B_{\tilde r_{\Omega_c}}\!\left(
\by_i^{\Omega_c}
\right)\subset B_{\tilde r_{\Omega_c}}(\Omega_c)\) is convex.
Therefore, part~(ii) of
Lemma~\ref{lem:euler-scaled-convex-contraction},
applied with  \(s=t_i\), \(t=t_{i+1}\),
\(\bx=\bz_i\),
\(\by=\hat{\by}_{t_i,\mathrm{cfg}}(\bz_i)\), and 
\(
S = B_{\tilde r_{\Omega_c}}(\Omega_c),
C=
B_{\tilde r_{\Omega_c}}\!\left(
\by_i^{\Omega_c}
\right)
\)
gives
\(
\bz_{i+1}
\in
t_{i+1}B_{\tilde r_{\Omega_c}}(\Omega_c).
\)
The conclusion follows by induction.
This proves part~(iii) of the theorem.

\end{proof}

\section{Proofs for Section~\ref{sec:gen-schedule}}\label{app:gen-schedule-proofs}

\subsection{Proof of Proposition~\ref{prop:general-schedule}}

\begin{proof}
\textbf{(i): ideal-field identity.}
The calculation follows the proof of Theorem~3.1 in
\citet{cai2026improvingclassifierfreeguidanceflow}. We give the full calculation for the unconditional branch.

Fix $t\in(0,1)$ and $\bx\in\R^d$. Conditioned on
$\rY_\emptyset=\by$, the random variable $\rX_{t,\emptyset}^a$ has density
\begin{equation*}
\varphi_t^a(\bx\mid\by)
:=
\tfrac{1}{\bigl(2\pi(1-a(t))^2\bigr)^{d/2}}
\exp\left(
-\tfrac{\|\bx-a(t)\by\|^2}{2(1-a(t))^2}
\right).
\end{equation*}
The marginal density of $\rX_{t,\emptyset}^a$ is
\(m_{\emptyset,a}(t,\bx)
:=
\int_{\mathcal D}
\varphi_t^a(\bx\mid\by)p_\emptyset(\dd\by)>0.\)
For every $A\in\mathcal B(\R^d)$, Bayes' formula gives
\begin{align*}
\mathbb P\!\left(
\rY_\emptyset\in A\mid \rX_{t,\emptyset}^a=\bx
\right)
&=
\tfrac{
\int_A\varphi_t^a(\bx\mid\by)p_\emptyset(\dd\by)
}{m_{\emptyset,a}(t,\bx)}=
\tfrac{
\int_A
\exp\left(-\frac{\|\bx-a(t)\by\|^2}{2(1-a(t))^2}\right)
p_\emptyset(\dd\by)
}{
\int_{\mathcal D}
\exp\left(-\frac{\|\bx-a(t)\by\|^2}{2(1-a(t))^2}\right)
p_\emptyset(\dd\by)
}
=\eta_{a(t)}^{\bx}(A).
\end{align*}
Let $\bu\in\mathcal V$. Writing the standard Gaussian density explicitly and
then making the change of variables
$\bx=a(t)\by+(1-a(t))\bxi$ gives
{\footnotesize \begin{align*}
\mathcal L_{\emptyset,a}(\bu)
&=\int_0^1\int_{\mathcal D}\int_{\R^d}
\left\|
\bu\bigl(t,a(t)\by+(1-a(t))\bxi\bigr)
-\dot a(t)(\by-\bxi)
\right\|^2
\tfrac{1}{(2\pi)^{d/2}}\exp\left(-\tfrac{\|\bxi\|^2}{2}\right)
\,\dd\bxi\,p_\emptyset(\dd\by)\,\dd t\\
&=
\int_0^1\int_{\mathcal D}\int_{\R^d}
\left\|
\bu(t,\bx)
-\dot a(t)
\left(
\by-\tfrac{\bx-a(t)\by}{1-a(t)}
\right)
\right\|^2\cdot
\tfrac{1}{\bigl(2\pi(1-a(t))^2\bigr)^{d/2}}
\exp\left(-\tfrac{\|\bx-a(t)\by\|^2}{2(1-a(t))^2}\right)
\,\dd\bx\,p_\emptyset(\dd\by)\,\dd t\\
&=
\int_0^1\int_{\R^d}\int_{\mathcal D}
\left\|
\bu(t,\bx)
-\tfrac{\dot a(t)}{1-a(t)}(\by-\bx)
\right\|^2
\varphi_t^a(\bx\mid\by)
p_\emptyset(\dd\by)\,\dd\bx\,\dd t.
\end{align*}}
The second equality uses
$\dd\bxi=(1-a(t))^{-d}\dd\bx$, and the last equality uses
$\by-(\bx-a(t)\by)/(1-a(t))=(\by-\bx)/(1-a(t))$. 
For fixed $(t,\bx)$, define
\begin{equation*}
K_{t,\bx}(\bz)
:=
\int_{\mathcal D}
\left\|
\bz-\tfrac{\dot a(t)}{1-a(t)}(\by-\bx)
\right\|^2
\varphi_t^a(\bx\mid\by)p_\emptyset(\dd\by),
\qquad \bz\in\R^d.
\end{equation*}
Its Hessian is
\(\mathsf D_{\bz}^2K_{t,\bx}(\bz)
=2m_{\emptyset,a}(t,\bx)\Id,\)
which is positive definite. Hence $K_{t,\bx}$ is strictly convex. Its unique
minimizer is characterized by
\begin{equation*}
\nabla_{\bz}K_{t,\bx}(\bz)
=2\int_{\mathcal D}
\left(
\bz-\tfrac{\dot a(t)}{1-a(t)}(\by-\bx)
\right)
\varphi_t^a(\bx\mid\by)p_\emptyset(\dd\by)
=\bzero.
\end{equation*}
Solving this equation and using the posterior identity above gives
\begin{align*}
\bz^\ast
&=
\tfrac{\dot a(t)}{1-a(t)}
\left(
\tfrac{\int_{\mathcal D}\by\,\varphi_t^a(\bx\mid\by)
p_\emptyset(\dd\by)}{m_{\emptyset,a}(t,\bx)}
-\bx
\right)=
\tfrac{\dot a(t)}{1-a(t)}
\left(
\int_{\mathcal D}\by\,\eta_{a(t)}^{\bx}(\dd\by)-\bx
\right)
=\dot a(t)\vu(a(t),\bx).
\end{align*}
Define $\vv_\emptyset^a(t,\bx):=\dot a(t)\vu(a(t),\bx)$ for
$t\in(0,1)$. Theorem~\ref{thm:wellposed-flow} and $a\in C^1([0,1])$ show that
$\vv_\emptyset^a\in\mathcal V$ and that the same formula gives its canonical
continuous extension to $t=0$. Moreover,
$\mathcal L_{\emptyset,a}(\vv_\emptyset^a)
\le\mathcal L_{\emptyset,a}(\bzero)<\infty$, since $\dot a$ is bounded,
$\mathcal D$ is bounded, and $\bxi$ has a finite second moment.

Since $\vv_\emptyset^a(t,\bx)$ minimizes $K_{t,\bx}$, for every
$\bu\in\mathcal V$ we have
\begin{align*}
\mathcal L_{\emptyset,a}(\bu)
=
\int_0^1\int_{\R^d}K_{t,\bx}(\bu(t,\bx))\,\dd\bx\,\dd t\ge
\int_0^1\int_{\R^d}
K_{t,\bx}(\vv_\emptyset^a(t,\bx))\,\dd\bx\,\dd t
=\mathcal L_{\emptyset,a}(\vv_\emptyset^a).
\end{align*}
Completing the square around the minimizer gives
\begin{equation*}
K_{t,\bx}(\bz)-K_{t,\bx}(\vv_\emptyset^a(t,\bx))
=m_{\emptyset,a}(t,\bx)
\|\bz-\vv_\emptyset^a(t,\bx)\|^2.
\end{equation*}
Because $m_{\emptyset,a}(t,\bx)>0$, equality of the two loss values forces
$\bu(t,\bx)=\vv_\emptyset^a(t,\bx)$ for almost every $(t,\bx)$, and
continuity gives equality everywhere on $(0,1)\times\R^d$. Therefore
$\vv_\emptyset^a$ is the unique global minimizer of
$\mathcal L_{\emptyset,a}$.

Replacing $p_\emptyset$, $\mathcal D$, and $\eta_{a(t)}^{\bx}$ throughout by
$p_c$, $\mathcal D_c$, and $\eta_{a(t),c}^{\bx}$ gives
\eqref{eq:general-field-identity} for $\star=c$; the same uniqueness and
extension argument applies. Thus, part~(i) holds for both branches.

When both branches use the same schedule, their identities give
\begin{align*}
\vcfg^a(t,\bx)
&=(1-w)\vv_\emptyset^a(t,\bx)+w\vv_c^a(t,\bx)\\
&=\dot a(t)\bigl[(1-w)\vv_\emptyset(a(t),\bx)
+w\vv_c(a(t),\bx)\bigr]
=\dot a(t)\vcfg(a(t),\bx).
\end{align*}

\textbf{(ii): existence--uniqueness and trajectory identity.}
Fix $T\in(0,1)$. Since $a$ is increasing and $a(T)<1$,
$a([0,T])\subset[0,a(T)]$. Part~(i) and
Theorem~\ref{thm:wellposed-flow} give
\begin{align*}
\sup_{\substack{t\in[0,T]\\\bx\in\R^d}}
\|\mathsf{D}_{\bx}\vv_\star^a(t,\bx)\|_{\rm op}
&\le
\|\dot a\|_{L^\infty(0,T)}
\left[
\tfrac{1}{1-a(T)}
+\tfrac{a(T)}{(1-a(T))^3}\diam(\mathcal D)^2
\right]
<\infty,
\end{align*}
for $\star\in\{\emptyset,c\}$; the same bound applies to the conditional
branch because $\mathcal D_c\subset\mathcal D$. Thus both branch fields are
continuous in $(t,\bx)$ and globally Lipschitz in $\bx$, uniformly for
$t\in[0,T]$. The CFG field is their linear combination, so the same result holds for $\vcfg^a$.
For each of these fields, continuity makes
$\sup_{t\in[0,T]}\|\vv^a(t,\bzero)\|$ finite; together with the global
Lipschitz bound, this gives a uniform linear-growth estimate. The
Picard--Lindel\"of theorem therefore gives a unique $C^1$ solution of each
general-schedule ODE on $[0,T]$.

Let $\vv$ be one of $\vu$, $\vc$, and $\vcfg$, let $\vv^a$ denote its
general-schedule counterpart, and let $\bx_s$ be the unique solution of the
corresponding standard ODE. The map $t\mapsto\bx_{a(t)}$ is $C^1$, and the
chain rule and part~(i) give
\begin{equation*}
\frac{\dd}{\dd t}\bx_{a(t)}
=\dot a(t)\vv(a(t),\bx_{a(t)})
=\vv^a(t,\bx_{a(t)}),
\qquad
\bx_{a(0)}=\bx_0.
\end{equation*}
Thus $t\mapsto\bx_{a(t)}$ solves the general-schedule ODE with the same
initial condition. Uniqueness of the general-schedule ODE yields
$\bx_t^a=\bx_{a(t)}$ on $[0,T]$. Since $T<1$ was arbitrary,
\eqref{eq:general-trajectory-identity} holds on $[0,1)$ for unconditional
flow, conditional flow, and CFG.

\textbf{(iii): transfer of the final-stage estimate.}\\ 

\noindent
\textbf{Attraction.}
For $t_0\le s\le t<1$, monotonicity of $a$ gives
$a(t_0)\le a(s)\le a(t)<1$. The two-time estimate established in
the proof of Theorem~\ref{thm:final-local}, applied to the
standard unconditional trajectory at times $a(s)$ and $a(t)$, gives
\begin{equation*}
\dist\bigl(\bx_t^a,a(t)B_{r'_\Omega}(\Omega)\bigr)
\le
\tfrac{1-a(t)}{1-a(s)}
\dist\bigl(\bx_s^a,a(s)B_{r'_\Omega}(\Omega)\bigr),
\end{equation*}
which proves \eqref{eq:general-final-ratio}. The monotonicity and
$\bigO(1-a(t))$ conclusions follow from the same standard-clock theorem.

\noindent
\textbf{Absorption.}
If
$\bx_{t_\Omega^a}^a=\bx_{a(t_\Omega^a)}$ belongs to
$a(t_\Omega^a)B_{r'_\Omega}(\Omega)$, its absorption statement gives
$\bx_t^a\in a(t)B_{r'_\Omega}(\Omega)$ for every $t\ge t_\Omega^a$.

Applying the same substitution to the conditional specialization of
Theorem~\ref{thm:final-local} and to Theorem~\ref{thm:cfg-final}, respectively,
proves the conditional and CFG conclusions.
\end{proof}

\bibliography{sample}

\end{document}